\documentclass[12pt, oneside]{book}

\usepackage[T1]{fontenc}
\usepackage{lmodern}

\usepackage[letterpaper, left=1in, right=1in, top=1in, bottom=1in]{geometry}

\usepackage{lipsum}
\usepackage{titlesec}

\titleformat{\chapter}[display]
  {\normalfont\Huge\sffamily} 
  {    \vspace{-4ex} 
  \raggedright \chaptername~\thechapter} 
  {0pt}
  {%
    \titlerule[0.5pt]\vspace{0.25ex}\titlerule[1pt]\vspace{2ex} 
    \raggedleft\Huge\bfseries\sffamily 
  }
\titleformat{\section}{\large\bfseries}{\thesection}{1em}{}
\titleformat{\subsection}{\normalsize\bfseries}{\thesubsection}{1em}{}
\titleformat{\subsubsection}{\normalsize\bfseries}{\thesubsubsection}{1em}{}

\newcounter{subsubsubsection}[subsubsection]
\renewcommand\thesubsubsubsection{\thesubsubsection.\arabic{subsubsubsection}}
\newcommand\subsubsubsection[1]{%
  \refstepcounter{subsubsubsection}%
  \vspace{1em}%
  \noindent\textbf{\thesubsubsubsection\quad #1}\par\nobreak\vspace{0.5em}%
}
\usepackage{fancyhdr}
\usepackage{parskip}

\usepackage{silence}
\usepackage[dvipsnames,table ]{xcolor}
\usepackage[colorlinks=true, linkcolor=blue!70!black, citecolor=blue!70!black, unicode=true]{hyperref}
\usepackage[final, nopatch=footnote]{microtype} 

\usepackage[
backend=biber,
backref=true,
natbib=true,
style=trad-alpha,
maxalphanames=6,
sorting=nyt
]{biblatex}

\usepackage{amsthm}
\usepackage{mathtools}
\usepackage{amsmath}
\usepackage{bm}
\usepackage{bbm}
\usepackage{amsfonts}
\usepackage{amssymb}
\usepackage{xpatch}
\usepackage{pifont}
\usepackage{array}
\usepackage{booktabs}
\usepackage{floatrow}
\newfloatcommand{capbtabbox}{table}[][\FBwidth]
\usepackage{blindtext}
\usepackage{caption}
\usepackage{subfigure}
\usepackage{algorithm}
\usepackage{algorithmicx}
\usepackage[noend]{algpseudocode}

\usepackage{bbm}

\newcommand\numberthis{\addtocounter{equation}{1}\tag{\theequation}} 
\allowdisplaybreaks

\usepackage{mathtools}

\DeclarePairedDelimiter{\abs}{\lvert}{\rvert} 
\DeclarePairedDelimiter{\brk}{[}{]}
\DeclarePairedDelimiter{\crl}{\{}{\}}
\DeclarePairedDelimiter{\prn}{(}{)}
\DeclarePairedDelimiter{\nrm}{\|}{\|}
\DeclarePairedDelimiter{\tri}{\langle}{\rangle}

\DeclarePairedDelimiter{\floor}{\lfloor}{\rfloor}

\let\Pr\undefined
 
\DeclareMathOperator{\En}{\mathbb{E}}

\DeclareMathOperator{\Pr}{\bbP}

\DeclareMathOperator*{\argmin}{argmin} 
\DeclareMathOperator*{\argmax}{argmax}

\newcommand{\ind}[1]{\mathbbm{1}\crl*{#1}}    
\newcommand{\indd}[1]{\mathbbm{1}\crl{#1}}    
\newcommand{\eps}{\varepsilon}

\newcommand{\ldef}{\vcentcolon=}

\newcommand{\wt}[1]{\widetilde{#1}}
\newcommand{\wh}[1]{\widehat{#1}}

\def\ddefloop#1{\ifx\ddefloop#1\else\ddef{#1}\expandafter\ddefloop\fi}
\def\ddef#1{\expandafter\def\csname bb#1\endcsname{\ensuremath{\mathbb{#1}}}}
\ddefloop ABCDEFGHIJKLMNOPQRSTUVWXYZ\ddefloop
\def\ddefloop#1{\ifx\ddefloop#1\else\ddef{#1}\expandafter\ddefloop\fi}
\def\ddef#1{\expandafter\def\csname b#1\endcsname{\ensuremath{\mathbf{#1}}}}
\ddefloop ABCDEFGHIJKLMNOPQRSTUVWXYZ\ddefloop
\def\ddef#1{\expandafter\def\csname c#1\endcsname{\ensuremath{\mathcal{#1}}}}
\ddefloop ABCDEFGHIJKLMNOPQRSTUVWXYZ\ddefloop
\def\ddef#1{\expandafter\def\csname h#1\endcsname{\ensuremath{\widehat{#1}}}}
\ddefloop ABCDEFGHIJKLMNOPQRSTUVWXYZabcdefghijklmnopqrsuvwxyz\ddefloop    
\def\ddef#1{\expandafter\def\csname hc#1\endcsname{\ensuremath{\widehat{\mathcal{#1}}}}}
\ddefloop ABCDEFGHIJKLMNOPQRSTUVWXYZ\ddefloop
\def\ddef#1{\expandafter\def\csname t#1\endcsname{\ensuremath{\widetilde{#1}}}}
\ddefloop ABCDEFGHIJKLMNOPQRSTUVWXYZ\ddefloop
\def\ddef#1{\expandafter\def\csname tc#1\endcsname{\ensuremath{\widetilde{\mathcal{#1}}}}}
\ddefloop ABCDEFGHIJKLMNOPQRSTUVWXYZ\ddefloop

\newcommand{\KL}[2]{\mathrm{KL}{\prn*{#1 \| #2}}}
\newcommand{\kl}[2]{\mathrm{kl}{\prn*{#1 \| #2}}}

\newcommand{\sign}{\mathrm{sign}}

\usepackage{tikz}

\newcommand{\proman}[1]{\prn*{\romannumeral #1}}

\newcommand{\overleq}[1]{\overset{ #1}{\leq{}}}
\newcommand{\overgeq}[1]{\overset{#1}{\geq{}}}
\newcommand{\overeq}[1]{\overset{#1}{=}}

\makeatletter
\newtheorem*{rep@theorem}{\rep@title}
\newcommand{\newreptheorem}[2]{%
\newenvironment{rep#1}[1]{%
 \def\rep@title{#2 \ref*{##1}}%
 \begin{rep@theorem}}%
 {\end{rep@theorem}}}
\makeatother

\newtheorem{claim}{Claim}
\newtheorem{assumption}{Assumption}
\newtheorem{conjecture}{Conjecture}
\newtheorem{corollary}{Corollary}
\newtheorem{proposition}{Proposition}

{\newtheorem{lemma}{Lemma}}

{\newtheorem{theorem}{Theorem}}
\newreptheorem{theorem}{Theorem}
\newtheorem{theorem*}{Theorem}
\newtheorem{definition}{Definition}
\newtheorem{question}{Question}

\numberwithin{theorem}{chapter}
\numberwithin{definition}{chapter}
\numberwithin{proposition}{chapter}
\numberwithin{lemma}{chapter}
\numberwithin{question}{chapter}
\numberwithin{conjecture}{chapter}

\usepackage{prettyref}
\newcommand{\pref}[1]{\prettyref{#1}}

\newcommand{\savehyperref}[2]{\texorpdfstring{\hyperref[#1]{#2}}{#2}}
\newrefformat{eq}{\savehyperref{#1}{\textup{(\ref*{#1})}}}
\newrefformat{eqn}{\savehyperref{#1}{Equation~\ref*{#1}}}
\newrefformat{fact}{\savehyperref{#1}{Fact~\ref*{#1}}}
\newrefformat{con}{\savehyperref{#1}{Conjecture~\ref*{#1}}}
\newrefformat{lem}{\savehyperref{#1}{Lemma~\ref*{#1}}}
\newrefformat{def}{\savehyperref{#1}{Definition~\ref*{#1}}}
\newrefformat{line}{\savehyperref{#1}{line~\ref*{#1}}}
\newrefformat{thm}{\savehyperref{#1}{Theorem~\ref*{#1}}}
\newrefformat{corr}{\savehyperref{#1}{Corollary~\ref*{#1}}}
\newrefformat{sec}{\savehyperref{#1}{Section~\ref*{#1}}}
\newrefformat{app}{\savehyperref{#1}{Appendix~\ref*{#1}}}
\newrefformat{ass}{\savehyperref{#1}{Assumption~\ref*{#1}}}
\newrefformat{ex}{\savehyperref{#1}{Example~\ref*{#1}}}
\newrefformat{fig}{\savehyperref{#1}{Figure~\ref*{#1}}}
\newrefformat{alg}{\savehyperref{#1}{Algorithm~\ref*{#1}}}
\newrefformat{rem}{\savehyperref{#1}{Remark~\ref*{#1}}}
\newrefformat{conj}{\savehyperref{#1}{Conjecture~\ref*{#1}}}
\newrefformat{prop}{\savehyperref{#1}{Proposition~\ref*{#1}}}
\newrefformat{proto}{\savehyperref{#1}{Protocol~\ref*{#1}}}
\newrefformat{prob}{\savehyperref{#1}{Problem~\ref*{#1}}}
\newrefformat{claim}{\savehyperref{#1}{Claim~\ref*{#1}}}
\newrefformat{tab}{\savehyperref{#1}{Table~\ref*{#1}}}
\newrefformat{chap}{\savehyperref{#1}{Chapter~\ref*{#1}}}
\newrefformat{q}{\savehyperref{#1}{Question~\ref*{#1}}}

\newcommand{\algcomment}[1]{\textcolor{blue!70!white}{\footnotesize{\texttt{\textbf{//
          #1}}}}}
\newcommand{\algcommentbig}[1]{\textcolor{blue!70!black}{\footnotesize{\texttt{\textbf{/*
          #1~*/}}}}}

\usepackage[shortlabels]{enumitem} 
\usepackage{makecell}
\usepackage{comment}
\usepackage{pifont}
\usepackage{caption}
\usepackage[dvipsnames]{xcolor}
\usepackage{nicefrac} 

\usepackage{tikz} 
\usetikzlibrary{arrows.meta, positioning}
\usepackage{wrapfig}
\usepackage{multirow}
\usepackage[export]{adjustbox}

\newcommand{\ccov}{C_\mathsf{cov}}
\newcommand{\cconc}{C_\mathsf{conc}}

\newcommand{\spancap}{\mathfrak{C}}
\newcommand{\poly}{\mathrm{poly}}

\newcommand{\unif}{\mathrm{Unif}}
\newcommand{\ber}{\mathrm{Ber}}

\newcommand{\oracle}{\cO_\mathrm{exp}}

\newcommand{\cpush}{C_\mathsf{push}}
\newcommand{\cpushcov}{C_\mathsf{push\_cov}}

\newcommand{\estpi}{{\wh{\pi}}}

\newcommand{\optpi}{\pi^\star}

\newcommand{\optlatp}{P_\mathsf{lat}}
\newcommand{\optlatr}{R_\mathsf{lat}}
\newcommand{\estlatp}{\wh{P}_\mathsf{lat}}
\newcommand{\estlatr}{\wh{R}_\mathsf{lat}}
\newcommand{\samplelatp}{\wt{P}_\mathsf{lat}}

\newcommand{\detalg}{{\textsf{PLHR.D}}}
\newcommand{\detdecoder}{{\textsf{Decoder.D}}}
\newcommand{\detrefit}{{\textsf{Refit.D}}}

\newcommand{\estlatentmdp}{\wh{M}_\mathsf{lat}}

\newcommand{\Pitest}{\Pi^{\mathsf{test}}}

\newcommand{\vestarg}[2]{V_\mathsf{mc}(#1 \mid #2)}

\newcommand{\qestarg}[2]{Q_\mathsf{mc}(#1 \mid #2)}

\newcommand{\taudec}{\epsilon_\mathsf{dec}}
\newcommand{\tauref}{\epsilon_\mathsf{tol}}

\newcommand{\stochalg}{{\textsf{PLHR}}}
\newcommand{\stochdecoder}{{\textsf{Decoder}}}
\newcommand{\stochrefit}{{\textsf{Refit}}}

\newcommand{\estmdp}{\wh{M}}
\newcommand{\estp}{\wh{P}}
\newcommand{\empobs}{x} 
\newcommand{\push}{_\sharp}

\newcommand{\estmdpobsspace}[1]{\statesp_{#1}[\estmdp]}
\newcommand{\pitest}[2]{\pi_{#1, #2}}

\newcommand{\nreset}{n_\mathsf{reset}}
\newcommand{\ndec}{n_\mathsf{dec}}

\newcommand{\nmc}{n_\mathsf{mc}}

\newcommand{\eventemulator}{\cE^\mathsf{init}}
\newcommand{\eventdec}{\cE^\mathsf{D}}
\newcommand{\eventref}{\cE^\mathsf{R}}

\newcommand{\epsPRS}[1]{\cS_{#1}^{\eps\textsf{-push}}}

\newcommand{\cXL}{\cX_\mathsf{L}}
\newcommand{\cXR}{\cX_\mathsf{R}}
\newcommand{\gobs}{\cG_\mathsf{obs}}
\newcommand{\cc}{\bbC}
\newcommand{\ccL}{\bbC_\mathsf{L}}
\newcommand{\ccR}{\bbC_\mathsf{R}}
\newcommand{\optdec}{\phi}
\newcommand{\projm}[1]{\mathsf{Proj}_{#1}}
\newcommand{\projR}{\projm{\cXR}}
\newcommand{\nrch}{n_\mathsf{reach}}
\newcommand{\nurch}{n_\mathsf{unreach}}

\newcommand{\resetmodel}{Hybrid Resets}

\newcommand{\psdp}{{\textsf{PSDP}}}
\newcommand{\cpi}{{\textsf{CPI}}} 

\newcommand{\statesp}{\cX}
\newcommand{\actionsp}{\cA}
\newcommand{\latentsp}{\cS}
\newcommand{\emission}{\psi}

\newcommand{\supp}{\mathrm{supp}}

\newcommand{\mcest}{\mathsf{MC}}
\newcommand{\violations}{\mathsf{Violations}}

\newcommand{\alg}{\mathsf{Alg}}

\newcommand{\sgood}{\mathsf{g}}
\newcommand{\sbad}{\mathsf{b}}
\newcommand{\sdis}{\mathsf{d}}

\newcommand{\dtv}{D_{\mathsf{TV}}}

\newcommand{\rootlayer}{\mathsf{RootLayer}}
\newcommand{\eventfresh}{\cE_F}
\newcommand{\eventnew}{\cE_N}

\newcommand{\cfnor}{\underline{\cF_{t,H}}}

\newcommand{\epspc}{\eps_{\mathsf{PC}}}
\newcommand{\epsstat}{\eps_\mathsf{stat}}

\newcommand{\pistar}{{\pi^\star}}
\newcommand{\pihat}{{\wh{\pi}}}
\newcommand{\Piell}{\Pi^{(\ell)}}

\newcommand{\dimRL}{\mathfrak{C}}

\newcommand{\non}{n_\mathsf{on}}
\newcommand{\ngen}{n_\mathsf{gen}}

\newcommand{\cMsto}{ \cM^{\mathrm{sto}} }
\newcommand{\cMdet}{ \cM^{\mathrm{det}} }

\newcommand{\bit}{\mathrm{bit}}

\newcommand{\idx}{\mathrm{idx}}
\newcommand{\jof}[1]{j[#1]}
\newcommand{\jpof}[1]{j'[#1]}
\newcommand{\jgof}[1]{j_g[#1]}
\newcommand{\jbof}[1]{j_b[#1]}

\newcommand{\Jrel}{\cJ_\mathrm{rel}}

\newcommand{\st}{\eta}
\newcommand{\Tmax}{{T_{\mathrm{max}}}}
\newcommand{\Compname}{spanning capacity}

\newcommand{\sunflower}{sunflower}
\newcommand{\coreset}{\sunflower}

\usepackage{mathtools, amsmath}
\newcommand{\cons}{\rightsquigarrow}
\newcommand{\coloneqq}{:=}
\newcommand{\plus}{+}

\newcommand{\datacollector}{\mathsf{DataCollector}}

\newcommand{\evaluate}{\mathsf{Evaluate}}
\newcommand{\estreach}{\mathsf{EstReachability}}

\renewcommand{\epsilon}{\varepsilon}
\renewcommand{\eps}{\varepsilon}

\newcommand{\Picb}{\Pi_{\mathrm{cb}}}

\newcommand{\Pitab}{\Pi_{\mathrm{tab}}}
\newcommand{\Pismall}{\Pi_{\mathrm{small}}}
\newcommand{\Pisingleton}{\Pi_{\mathrm{sing}}}
\newcommand{\PiLton}{\Pi_{\ell\mathrm{-ton}}}
\newcommand{\Pioneactive}{\Pi_{\mathrm{1-act}}}
\newcommand{\PiJactive}{\Pi_{j\mathrm{-act}}}
\newcommand{\Piactive}{\Pi_{\mathrm{act}}}

\newcommand{\mA}{\mathcal{A}}
\newcommand{\mD}{\mathcal{D}}

\newcommand{\EE}{\mathbb{E}}

\newcommand{\SHPalg}[1]{\statesp_{#1}^{\mathrm{rch}}}
\newcommand{\SRemalg}[1]{\statesp_{#1}^{\mathrm{rem}}}
\newcommand{\SHP}[1]{\statesp_{#1}^{\mathrm{rch}}}
\newcommand{\SRem}[1]{\statesp_{#1}^{\mathrm{rem}}}

\newcommand{\event}[3]{\mathfrak{T}\prn{#1 \rightarrow #2 ; \neg #3}}

\newcommand{\bard}[3]{\bar{d}^{#1}(#2; \neg#3)}

\newcommand{\Picore}{\Pi_{\mathrm{core}}}

\newcommand{\Stab}{\statesp^{\mathrm{tab}}}

\newcommand{\setall}[3]{ \mathfrak{T}_{#1}\prn{#2 \rightarrow #3}}

\usepackage{mathrsfs}

\newcommand{\MRPsign}{\mathfrak{M}}

\newcommand{\natarajan}{\mathrm{Ndim}}
\newcommand{\pseudo}{\mathrm{Pdim}}

\newcommand{\reachablestates}{\statesp^\mathrm{rch}}

\newcommand{\creach}{C^\mathsf{reach}}

\newcommand{\eEdim}{\underline{\mathsf{Edim}}}
\newcommand{\Edim}{\mathsf{Edim}}

\newcommand{\Sdim}{\mathsf{Sdim}}

\newcommand{\Tdim}{\mathsf{Tdim}}

\renewcommand{\dim}{\mathsf{dim}}
\renewcommand{\sign}{\mathsf{sign}}
\newcommand{\dc}{\mathsf{rk}}

\newcommand{\signrank}{\sign\text{-}\dc}
\newcommand{\Fparity}{\Pi^{\oplus}}

\newcommand{\Piopen}{\Pi_\mathsf{open}}

\usepackage[suppress]{color-edits} 
\addauthor{gl}{red}

\begin{document}

\begin{titlepage}
    \centering
    \vspace{1cm}
    {\sffamily \Huge\textbf{Agnostic Reinforcement Learning:\\Foundations and Algorithms}}\\[1cm]
    \large by \\[1ex]
    {\large Gene X.~Li}
    \vfill
    A thesis submitted\\
    in partial fulfillment of the requirements for\\
    the degree of\\[0.5cm]
    Doctor of Philosophy in Computer Science\\[0.5cm]
    at the\\[0.5cm]
    TOYOTA TECHNOLOGICAL INSTITUTE AT CHICAGO\\
    Chicago, Illinois\\
    August 2025\\[2cm]
    Thesis Committee:\\
    Nathan Srebro (Advisor)\\
    Avrim Blum\\
    Akshay Krishnamurthy\\
    Cong Ma
\end{titlepage}

\frontmatter
\setcounter{page}{2}
\thispagestyle{plain}
\begin{center}
    \rule{\linewidth}{2pt}\\
    {\sffamily \Huge \bfseries Agnostic Reinforcement Learning:\\Foundations and Algorithms} \\[2ex]
    
    \large by \\[1ex]
    {\large Gene X.~Li}
    \rule{\linewidth}{2pt} \\[3ex]
\end{center}

{\sffamily\textbf{Abstract}} \\[1ex]

Reinforcement Learning (RL) has demonstrated tremendous empirical success across numerous challenging domains. However, we lack a strong theoretical understanding of the statistical complexity of RL in environments with large state spaces, where function approximation is required for sample-efficient learning. This thesis addresses this gap by rigorously examining the statistical complexity of RL with function approximation from a learning theoretic perspective. Departing from a long history of prior work, we consider the weakest form of function approximation, called agnostic policy learning, in which the learner seeks to find the best policy in a given class $\Pi$, with no guarantee that $\Pi$ contains an optimal policy for the underlying task.

We systematically explore agnostic policy learning along three key axes: environment access---how a learner collects data from the environment; coverage conditions---intrinsic properties of the underlying MDP measuring the expansiveness of state-occupancy measures for policies in the class $\Pi$, and representational conditions--- structural assumptions on the class $\Pi$ itself. Within this comprehensive framework, we (1) design new learning algorithms with theoretical guarantees and (2) characterize fundamental performance bounds of any algorithm. Our results reveal significant statistical separations that highlight the power and limitations of agnostic policy learning.

\clearpage

\chapter*{Acknowledgements}

I'm grateful for the tremendous support of many people over the past six years.

First and foremost, I'd like to thank my advisor, Nati Srebro. I came into grad school with no idea of how to do theoretical research, and I'm grateful for the opportunity that Nati has given me. Nati's enthusiasm for research and commitment to research excellence are not to be taken for granted. As a first-year student, the tone was set when I attended the MLO reading group and discovered that it is possible to discuss a paper with Nati for 2 (sometimes even 3 or 4) hours. This level of attention to detail (the so-called rigor police) was not something that I had ever experienced before. In addition, Nati has given me intellectual freedom to work on what interests me, guidance in developing research taste, and the time and patience to develop technical skills so that I can tackle challenging, convoluted problems. In fact, the seeds of this thesis were planted in meetings with Nati during the very first quarter of my Ph.D. Thank you!

I also want to thank the rest of my committee: Avrim Blum, Cong Ma, and Akshay Krishnamurthy. I'm fortunate to have been able to work with each of you. Avrim has amazing intuition, even for research problems he's hearing about for the first time. One day I hope to acquire such wisdom. Cong has been an endless source of advice and mentorship ever since I was an undergrad at Princeton, and he was a grad student there. He was the first person who taught me how to actually read a research paper (not linearly as if it were a story, which I had been doing up until then). Akshay's expertise in RL and clear thinking have been indispensable. Furthermore, our meetings have always been on the Pareto frontier of fun \emph{and} productive.

I'd like to thank other mentors over the years. Pritish Kamath was the first example I had of what day-to-day theoretical research should be like, and his style of research has influenced me greatly. From Pritish I learned the art of carefully understanding simple things. Ayush Sekhari has been a great friend, collaborator, and mentor. Omar Montasser, as the senior student in our group, has been an exemplary role model. Thanks also to Jason Lee for hosting me in Summer 2022, Lin Chen and Adel Javanmard for my Student Researcher experience at Google, and Muye Wang and Sudhanshu Tungare for a fun summer at Two Sigma working with actual data. Going back to my Princeton days, I'd like to thank Yuxin Chen, Emmanuel Abbe, Ramon van Handel, Karthik Narasimhan, Peter Ramadge, and Rob Schapire for introducing me to research and inspiring me to do a Ph.D.

Thanks to my wonderful collaborators: Avrim Blum, Lin Chen, Dylan J.~Foster, Meghal Gupta, Adel Javanmard, Zeyu Jia, Anmol Kabra, Pritish Kamath, Akshay Krishnamurthy, Junbo Li, Cong Ma, Naren S.~Manoj, Vahab Mirrokni, Sasha Rakhlin, Aadirupa Saha, Ayush Sekhari, Nati Srebro, Zhaoran Wang, Chloe Yang, and Zhuoran Yang.

TTIC has been a one-of-a-kind place to do a Ph.D., and I am quite sad to leave. I especially want to thank the Chicago area students (in no particular order): Omar Montasser, Kavya Ravichandran, Lijia Zhou, Antares Chen, Max Ovsiankin, Nirmit Joshi, Marko Medvedev, Dimitar Chakarov, Han Shao, Luzhe Sun, Tushant Mittal, Kumar Kshitij Patel, Keziah Naggita, Pushkar Shukla, Vaidehi Srinivas, Kevin Stangl, David Yunis, Shuo Xie, Goutham Rajendran, Shashank Srivastava, Sudarshan Babu, and Shubham Toshniwal. Thank you for conspiring to spend student body funds together, serving up (and sometimes manufacturing) endless gossip, and being familiar faces at conferences/workshops. In addition, I have learned a great deal from the professors at TTIC: Nati Srebro, Madhur Tulsiani, Avrim Blum, Greg Shakhnarovich, Yury Makarychev, and David McAllester. I'd also like to thank the admin staff at TTIC for making it easy to focus on research, keeping us students well-fed, and being so responsive. Thanks to the IDEAL institute directors for fostering a broader sense of Chicago community: Aravindan Vijayaraghavan, Lev Reyzin, and Avrim Blum.

I'd also like to thank Mayee Chen (2 mod 3) and Charlie Hou (3 mod 3). During the pandemic, we started a group chat to work on exercises from Duchi's Information Theory lecture notes. Sadly, we did not do many of the problems, but we made up for it with countless hours of fruitful (and not so fruitful) discussions on research and Ph.D.~life.

Finally, I'd like to thank my parents, Dongmei Li and Zhaoliang Li. Thanks for instilling a sense of intellectual curiosity in me, giving me the confidence to pursue academics, and supporting me in all endeavors. Thanks to my sister, Greta Li, for teaching me Zoomer lingo. Lastly, I want to thank my fiancée, Emma Corless. You have been my best friend for the past 9 years, and I can't wait for us to be best friends for the rest of our lives. Thank you for your love and support.

\tableofcontents

\mainmatter

\chapter{Introduction}
Reinforcement Learning (RL) is a widely studied framework for sequential decision-making, in which an agent interacts with an environment and seeks to learn how to maximize its long-term or cumulative reward. 
RL has demonstrated impressive empirical success in a wide array of challenging tasks, from achieving superhuman performance in the game of Go  \citep{silver2017mastering}, to solving intricate robotic manipulation tasks \citep{lillicrap2015continuous, akkaya2019solving, ji2023dribblebot}, to unlocking reasoning capabilities in large language models \cite{jaech2024openai, guo2025deepseek}.
Many practical domains in RL involve rich observations such as images, text, or audio \citep{mnih2015human, li2016deep, ouyang2022training, team2025kimi}. Since these state spaces can be vast and complex, traditional tabular RL approaches \citep{kearns2002near, brafman2002r, azar2017minimax, jin2018q} cannot scale. This has led to a need to develop efficient and theoretically sound approaches for RL that utilize \emph{function approximation}---allowing one to generalize observational data across unknown states/actions.

The goal of this thesis is to study the statistical complexity of RL with function approximation from a learning theoretic perspective. We focus on the setting of \emph{policy learning}, which is arguably the most basic and fundamental setting in RL. Here, the learner is given an abstract function class $\Pi$ of \emph{policies} (mappings from states to actions). For example, $\Pi$ could be the set of all policies represented by a certain deep neural network architecture. The learner's objective is to interact with the unknown environment and return a policy $\estpi$ which performs nearly as well as the best policy in the class $\Pi$. As per the tradition of learning theory, we endeavor to make as few assumptions as possible on nature (i.e., the environment) itself, thus focusing on the so-called \emph{agnostic setting}.  

At a high level, this thesis contributes to the theoretical foundations of RL in two ways: we design new, statistically efficient learning algorithms with provable guarantees, and we characterize the fundamental limitations of any algorithm's performance. 

\section{Overarching Themes}
We highlight several central themes that unify the results presented in this thesis.

\paragraph{Motivating Vignette.} Recall the setup of supervised learning \cite{shalev2014understanding}: the learner is given an \emph{instance space} $\cX$, \emph{label space} $\cY = \crl{0, 1}$, and \emph{hypothesis class} $\Pi = \cY^\cX$, and they observe $n$ i.i.d.~samples $S \sim \cD^n$, where $\cD$ is an unknown distribution over $\cX \times \cY$. 
For parameters $\eps, \delta \in (0,1)$, the learner is required to, with probability at least $1-\delta$, return a predictor $\wh{\pi}: \cX \to \cY$ that competes with the best predictor in $\Pi$ in terms of misclassification error on a fresh sample drawn from $\cD$:
\begin{align*}
    L(\wh{\pi}) \le \inf_{\pi \in \Pi} L(\pi) + \eps, \quad \text{where} \quad 
    L(\pi) \coloneqq \En_{(x,y) \sim \cD} \brk*{\ind{\pi(x) \ne y} } \numberthis \label{eq:classification-error}
\end{align*}
We say a learning rule (or algorithm) $\alg: (\cX \times \cY)^\star \to \cY^\cX$ is an $(\eps, \delta)$-PAC learner for $\Pi$ if it satisfies \eqref{eq:classification-error} for any data distribution $\cD$. We say $\Pi$ is PAC-learnable if such a learning rule exists. This goal is \emph{agnostic}: we make no assumptions on the distribution $\cD$ at all!

For supervised learning, the story is more or less complete. A hypothesis class $\Pi$ is PAC-learnable if and only if its VC dimension is bounded. Furthermore, the Empirical Risk Minimization (ERM) learning rule---which returns the predictor in $\Pi$ with the smallest misclassification error $\wh{L}(\pi) \coloneqq \frac{1}{n}\sum_{i=1}^n \ind{\pi(x_i) \ne y_i}$ over the sample $S$---is statistically optimal, up to log factors \cite{vapnik1971uniform, vapnik1974theory, blumer1989learnability, ehrenfeucht1989general}. 

The central question of this thesis is: Can we develop such a theory for RL? We now discuss the overarching themes of this thesis, both drawing connections to and contrasting with the story in supervised learning.

\paragraph{\raisebox{0.25ex}{$\blacktriangleright$} Theme 1: Role of Interaction.} RL is a substantially richer paradigm than supervised learning. Due to the interactive and sequential nature of the problem, RL presents two significant challenges to learning agents: \emph{exploration}---the agent must deliberately explore the environment to gather information---as well as \emph{error amplification}---the agent must account for potential future errors when making decisions in the present. 

In this thesis, we carefully study the richness of the RL paradigm which comes from the \emph{interaction protocol}, the way that the agent is allowed to interact with the environment. We study the statistical complexity of RL under almost every widely considered interaction protocol, both in theory and in practice. Our inquiry reveals surprising tradeoffs: under some interaction protocols, a learning agent can effectively address the exploration and error amplification challenges, while under other interaction protocols, it is impossible to do so. These separations clarify the power/limitations of the interaction protocols.

\paragraph{\raisebox{0.25ex}{$\blacktriangleright$} Theme 2: Two Flavors of Complexity Measures.} Analogous to classical VC theory, we study the worst-case/minimax sample complexity of RL, introducing new complexity measures which solely depend on the policy class $\Pi$. 

However, while the classical VC theory is beautiful, those results can be rather pessimistic. Considerable effort has been invested in developing refined, \emph{instance-dependent} complexity measures which depend on the underlying distribution $\cD$. One example is the so-called \emph{Rademacher complexity}: letting $\sigma_1, \cdots, \sigma_n$ be i.i.d.~Rademacher random variables, we define
\begin{align*}
    \mathrm{Rad}_n(\Pi) \coloneqq \En_{\cD} \brk*{ \En_{\sigma \sim \crl{\pm 1}^n} \brk*{ \sup_{\pi \in \Pi} \frac{1}{n} \sum_{i=1}^n \sigma_i \pi(x_i) }}. 
\end{align*}
Note that the Rademacher complexity $\mathrm{Rad}_n(\Pi)$ is a complexity measure that depends both on the hypothesis class $\Pi$ as well as the distribution $\cD$ (through the expectation). 

Similarly, we investigate instance-dependent notions of complexity in RL, which depend on the interaction between the policy class and the environment. We study whether algorithms can adapt to simplicity in the environment, as measured by such complexity measures.

\paragraph{\raisebox{0.25ex}{$\blacktriangleright$} Theme 3: Challenges of Large State Spaces.} How can one efficiently explore an environment when the number of possible states is large or even infinite? We address this fundamental challenge in both upper and lower bounds: 
\begin{itemize}
    \item We develop new algorithmic techniques to perform \emph{policy evaluation}, uniformly approximating the values of all policies in the class $\Pi$. These techniques naturally extend the concept of \emph{uniform convergence} \cite{shalev2014understanding} to the interactive setting of policy learning.
    \item We prove information theoretic lower bounds which crucially rely on the large state space aspect of the problem. Specifically, it can be highly unlikely for a learning agent to ever encounter ``repeated'' states over the course of interacting with the environment. In some sense, our lower bounds distill the difficulty of policy learning into known hard statistical tasks such as uniformity testing or learning a complicated predictor. 
\end{itemize}

\section{Outline of the Thesis}

We now provide a high-level summary of the main contributions of this thesis.
\begin{center}
\rule{0.95\linewidth}{0.5pt}
\end{center}

\paragraph{\pref{chap:background}: Background and Problem Setup.} We provide background on the finite horizon Markov Decision Process (MDP) and commonly studied interaction protocols. We also formalize the problem of PAC Reinforcement Learning, highlight some basic results, and present an overview of theoretical approaches. Lastly, we introduce complexity measures based on coverage which will be studied throughout the thesis. In addition to several standard complexity measures, we present a new worst-case notion of complexity called the \emph{spanning capacity} (\pref{def:dimension}) which is solely a structural property of the policy class $\Pi$ itself. 

\paragraph{\pref{chap:eluder}: Policy Eluder Dimension.} We investigate a candidate complexity measure called the \emph{policy eluder dimension}, a combinatorial variant of the widely studied (scale-sensitive) eluder dimension \cite{russo2013eluder}. Roughly speaking, the policy eluder dimension is the longest sequence of adversarially chosen points one must observe in order to accurately estimate the policy at any other point. Intuitively, it is a reasonable complexity measure for sequential decision making settings with partial information feedback such as RL.

We show several new relationships between the policy eluder dimension and other well-established complexity measures for learning theory:
\begin{itemize}
    \item First, we show a qualitative equivalence: namely finiteness of the policy eluder dimension is equivalent to finiteness of both the star number and threshold dimension, which are two learning theoretic quantities that characterize pool-based active learning and online learning respectively (\pref{thm:equivalence}).
    \item Next, we compare policy eluder dimension with a classical notion of dimensional complexity called the sign rank. In \pref{thm:separation} we show that there exist function classes with small policy eluder dimension but large sign rank. An important implication of this result is that sample complexity guarantees for RL stated in terms of the eluder dimension go far beyond what is previously known for linear or generalized linear function approximation settings.
    \item Lastly, we connect the policy eluder dimension, the star number, and the threshold dimension with the spanning capacity in \pref{thm:bounds-on-spanning}.
\end{itemize}
The rest of this thesis predominantly focuses on the coverage-based complexity measures described in \pref{sec:coverage-conditions}. We include \pref{chap:eluder} for several reasons: a huge body of literature has proposed to study the sample complexity of RL in terms of the eluder dimension, so the results in this section help us understand those bounds, and we believe that understanding the role of the complexity measures discussed herein (eluder dimension, star number, threshold dimension, and generalized rank) for policy learning and other sequential decision making problems is an exciting and fruitful direction for future work.

\begin{center}
\rule{0.95\linewidth}{0.5pt}
\end{center}

Now we study the sample complexity of policy learning along three key axes: \emph{environment access}, \emph{coverage conditions}, and \emph{representational conditions}. Each of the following chapters is devoted to a single interaction protocol, for which we will prove upper/lower bounds in terms of the complexity measures based on coverage, assuming various representational conditions on $\Pi$. A summary of our results can be found in \pref{tab:results}.

\newcolumntype{P}[1]{>{\centering\arraybackslash}p{#1}}

\renewcommand{\arraystretch}{1.5} 
\begin{table}[h]
\centering
\resizebox{0.99\linewidth}{!}{%
\begin{tabular}{l|P{5cm}|P{7cm}}
\toprule
 & \textbf{Minimax} & \textbf{Instance-Dependent} \\
\midrule
\rule{0pt}{2em} \textbf{Online RL} &
\begin{tabular}{@{}P{2.9cm} P{1.7cm}@{}}
        \makecell{$\color{Green} \lesssim \poly(\dimRL, K, D)$ \\ Thm.~\ref{thm:sunflower}} & \hspace{-1em}  {\makecell{ $\color{red}  \gtrsim \eps^{-\log \dimRL} $ \\ Thm.~\ref{thm:lower-bound-online} } }
  \end{tabular}
 & \\
\cline{1-2}
\rule{0pt}{2.5em} \textbf{Gen/Local Sim.} &  
\makecell{$\color{Green} \lesssim \dimRL, \quad \color{red} \gtrsim \dimRL,$\\ Thms.~\ref{thm:generative_upper_bound} and \ref{thm:generative_lower_bound}} &
\multirow{-2}{*}{\makecell{$\color{red} \not\lesssim \poly(\ccov)$ \\ Thm.~\ref{thm:lower-bound-coverability}} }\\
\hline
&  
\multirow{2}{*}{
    \begin{tabular}{@{}P{2.2cm} P{2.2cm}@{}}
      {\makecell{\scriptsize Realizable \\[-0.5em] \scriptsize $\pi^\star \in \Pi$}}&
      {\makecell{\scriptsize Agnostic \\[-0.5em] \scriptsize $\pi^\star \not\in \Pi$}}\\
        \makecell{$\color{Green} \lesssim 1$ \\ $\mathsf{AggreVaTe}$} &  \makecell{$\color{red} \gtrsim \dimRL$ \\ Thm.~\ref{thm:lower-bound-expert}}
  \end{tabular}
} & \\
\rule{0pt}{2.5em} \multirow{-2}{*}{\textbf{Imitation Learning}} & & \vspace{-2.5em} \makecell{---} \\
\hline\noalign{\vspace{2mm}}
\multicolumn{2}{c}{} & 
\multirow{3}{*}{
  \begin{tabular}{@{}P{1.8cm}|P{1.8cm}|P{2.8cm}@{}}
      {\makecell{\scriptsize Policy \\[-0.5em] \scriptsize Completeness} } &
      {\makecell{\scriptsize Realizable \\[-0.5em] \scriptsize $\pi^\star \in \Pi$}}&
      {\makecell{\scriptsize Agnostic \\[-0.5em] \scriptsize $\pi^\star \not\in \Pi$}}\\
      \rule{0pt}{2.2em} \multirow{2}{*}{\makecell{$\color{Green} \lesssim \cconc$ \\ \psdp{} }}
       &{\makecell{\textcolor{Dandelion}{\bf{?}}$^\star$ }}  & \makecell{$\color{red} \not\lesssim \poly(\cpush)$ \\ Thm.~\ref{thm:lower-bound-policy-completeness} } \\
       \cline{2-3}
        \rule{0pt}{2em} & \multicolumn{2}{c}{ \makecell{ $\color{Green} \lesssim \poly(\cpush)$ \\ Thm.~\ref{thm:block-mdp-result} (for BMDP)} } \\
  \end{tabular}
} \\
\hline
\rule{0pt}{2.5em} \textbf{\boldsymbol{$\mu$}-Resets} & \makecell{---} &\\
\cline{1-2}
\rule{0pt}{2.5em} \textbf{Hybrid Resets} & \makecell{---} & \\
\bottomrule
\end{tabular}
}
\caption{Summary of our results. The table is organized by interaction protocols well as the type of bound (minimax vs.~instance-dependent). We use \textcolor{Green}{green} to denote upper bounds, and \textcolor{red}{red} to denote lower bounds. Here, $\dimRL$ denotes the spanning capacity, $(K,D)$ are the parameters of the sunflower property, and $\ccov, \cconc, \cpush$ denote coverability, concentrability, and pushforward concentrability respectively. Some remarks: (1) $\mathsf{AggreVaTe}$ \cite{ross2014reinforcement} and \psdp{} \cite{bagnell2003policy} are classical RL algorithms. (2) For readability, we only state the dependence on the aforementioned complexity measures, and polynomial dependence on other problem parameters is omitted. (3) For $\mu$-resets + policy realizability (\textcolor{Dandelion}{\textbf{?}}), we establish sample-inefficiency for \psdp{} and \cpi{} \cite{kakade2002approximately} (see \pref{fig:psdp-lower-bound-simple} and \pref{thm:psdp-lower-bound}), but an information theoretic resolution remains open.} \label{tab:results}
\end{table}

\paragraph{\pref{chap:generative}: Generative Model and Local Simulator.} We explore policy learning when the learner is given access to an interaction protocol with resets (either the generative model or the weaker local simulator). We show that the minimax sample complexity of agnostic PAC RL here is characterized by the spanning capacity of the policy class (\pref{thm:generative_upper_bound} and \ref{thm:generative_lower_bound}), and give further refinements of these bounds for infinite policy classes, stated in terms of the Natarajan dimension of the policy class. Since these minimax results are somewhat pessimistic, we also investigate whether it is possible to adapt to coverability in an instance-dependent fashion. We prove this is impossible: \pref{thm:lower-bound-coverability} shows that there exists a setting with \emph{constant} coverability for which any algorithm must use number of samples scaling with the spanning capacity (which in this case is \emph{exponential} in the horizon). This result formalizes the folklore intuition that ``policy learning methods cannot explore'' and shows a statistical separation with value-based methods, for which it is known that adapting to coverability is possible \cite{xie2022role, mhammedi2024power}. 

\paragraph{\pref{chap:online}: Online RL.} In this chapter, we turn to the most standard and widely considered interaction protocol, online RL.
In \pref{chap:generative}, we showed that spanning capacity characterized the minimax sample complexity of agnostic PAC RL with a generative model. A tempting conjecture is that spanning capacity is also the right characterization in online RL. The lower bound is already clear since online RL is at least as hard as learning with a generative model. But is spanning capacity also sufficient? \pref{thm:lower-bound-online} answers this in the negative: it constructs a specific $\Pi$ for which the minimax sample complexity is \emph{superpolynomial} in the spanning capacity. On the positive side, we introduce and motivate a new structural assumption on the policy class called the \emph{sunflower property} (\pref{def:core_policy}). We show in \pref{thm:sunflower} that statistically-efficient learning is possible for policy classes which have both bounded spanning capacity and satisfy the sunflower property. This is accomplished via a new algorithm called Policy Optimization by Learning $\eps$-Reachable States ($\mathsf{POPLER}$). On a technical level, $\mathsf{POPLER}$ utilizes a new algorithmic technique for policy evaluation called the \emph{policy-specific Markov Reward Process} (\pref{sec:algorithm_description}).

\paragraph{\pref{chap:imitation-learning}: Imitation Learning.} This chapter studies \emph{imitation learning} (IL), a popular paradigm for RL which has been widely applied to robotics, autonomous control, game playing, and more recently, language generation. The basic motivation for IL is that in many RL settings, we already have ``expert'' demonstrations/feedback, and we are not learning from scratch. We study a particular form of \emph{interactive IL}, in which the learner can actively query for the value function of an expert $Q^{\star}(x,a)$ on a given state-action pair. Intuitively, this feedback should help the learner by guiding exploration, thus reducing the statistical dependence on spanning capacity established in \pref{chap:generative}--\ref{chap:online} (which for many policy classes can be quite large). 

Our main result is that the \emph{realizability} of the expert policy plays a crucial role. If the expert policy lies in the policy class $\Pi$, it is known that this feedback can be used efficiently to achieve $\poly( \log \abs{\Pi}, A, H)$ sample complexity (with no dependence on the spanning capacity). However, we show in \pref{thm:lower-bound-expert} that if the expert policy does not lie in $\Pi$, then the learner might need sample complexity which is at least \emph{linear} in spanning capacity.

\paragraph{\pref{chap:mu-reset}: Online RL with Exploratory Resets.} Now we study the $\mu$-reset interaction protocol, where the learner is given additional sampling access to an exploratory reset distribution $\mu = \crl{\mu_h}_{h \in [H]}$. Our starting point is the classical Policy Search by Dynamic Programming (\psdp{}) algorithm \cite{bagnell2003policy}. It is known that \psdp{} is sample-efficient when the reset has bounded concentrability and the policy class satisfies a \emph{policy completeness} condition. A natural question to ask is if either of these requirements be removed.

We have already shown that explicit sampling access to $\mu$ is critical: i.e., \pref{thm:lower-bound-coverability} demonstrated that even with the generative model, the learner cannot adapt to \emph{coverability} (which merely posits the existence of a sampling distribution $\mu$ with good concentrability; see \pref{def:coverability}) and must use sample complexity which scales with the spanning capacity.

Thus, this chapter primarily focuses on relaxing the policy completeness condition:
\begin{itemize}
    \item \emph{Realizable Policy Class:} When the policy class is realizable (a strictly weaker assumption than policy completeness), we show that under slightly stronger assumptions on $\mu$ (pushforward concentrability or admissibility), sample complexity which is exponential in horizon is possible (\pref{thm:psdp-ub-pushforward} and \pref{thm:psdp-ub-admissible}). We further show that this is tight for \psdp{} by giving an algorithm-dependent lower bound in \pref{thm:psdp-lower-bound}. 
    \item \emph{Agnostic Policy Class:} In the fully agnostic setting, we give an information theoretic lower bound (\pref{thm:lower-bound-policy-completeness}) which shows that \emph{no algorithm} can utilize the $\mu$-reset interaction protocol to achieve sample-efficient learning. 
\end{itemize}

\paragraph{\pref{chap:hybrid-resets}: Hybrid Resets.} Motivated by the negative results for the local simulator (\pref{thm:lower-bound-coverability}) and $\mu$-resets (\pref{thm:lower-bound-policy-completeness}), we ask if statistically-efficient policy learning is possible under stronger forms of access to the environment. (Here, by statistical efficiency, we want the sample complexity to depend on instance-dependent notions of complexity rather than the worst-case spanning capacity.) 

Our main result is a new algorithm called Policy Learning for Hybrid Resets (\stochalg{}) that uses hybrid resets (both the local simulator and $\mu$-resets) to learn Block MDPs. Block MDPs are perhaps the simplest setting with large state space but low intrinsic complexity, as well as a stepping stone to more challenging settings, and the aforementioned lower bounds actually take the form of Block MDPs. In \pref{thm:block-mdp-result}, we prove that \stochalg{} achieves sample complexity which scales with the pushforward concentrability. This result highlights the significant power of hybrid resets in unlocking new statistical guarantees which are impossible under weaker forms of access to the environment. On a technical level, \stochalg{} introduces a new algorithmic technique for policy evaluation called the \emph{policy emulator} (\pref{def:policy-emulator}), which we anticipate will find further use. 

\section{Note on Contents of Thesis}
This thesis is based on the following prior publications by the author:
\begin{enumerate}
    \item \cite{li2022understanding}: with Pritish Kamath, Dylan J.~Foster, and Nathan Srebro, published in NeurIPS 2022.
    \item \cite{jia2023agnostic}: with Zeyu Jia, Alexander Rakhlin, Ayush Sekhari, and Nathan Srebro, published in NeurIPS 2023.
    \item \cite{krishnamurthy2025role}: with Akshay Krishnamurthy and Ayush Sekhari, published in COLT 2025.
\end{enumerate}
Some remarks are in order. We have significantly reorganized the material from these three papers into the thesis, so that several chapters contain material from multiple papers. We have also standardized the notation to be consistent across chapters, and introduced additional notation as needed throughout the thesis.

During the course of the Ph.D., the author also worked on other topics which are not central to the theme of this thesis. We briefly mention them here.
\begin{enumerate}
    \setcounter{enumi}{3}
    \item \cite{li2022exponential}: with Junbo Li, Anmol Kabra, Nathan Srebro, Zhaoran Wang, and Zhuoran Yang, published in NeurIPS 2022. This work studies online RL when the transition function follows a exponential family distribution.
    \item \cite{li2022pessimism}: with Cong Ma and Nathan Srebro, published in NeurIPS 2022. This work studies offline linear contextual bandits.
    \item \cite{blum2023dueling}: with Avrim Blum, Meghal Gupta, Naren S.~Manoj, Aadirupa Saha, and Yuanyuan Yang, published in ALT 2024. This work studies a generalization of dueling optimization, where the learner interacts with a monotone oracle that provides ``improvements'' over the points the learner queries.
    \item \cite{li2024optimistic}: with Lin Chen, Adel Javanmard, and Vahab Mirrokni, published in COLT 2024. This work studies a weakly supervised learning problem called Learning from Label Proportions (LLP) and studies various learning rules for LLP, both theoretically and empirically.
\end{enumerate}
Publications 4 and 5 are discussed in more detail in \pref{sec:realizability-approach} in order to situate them within the broader landscape of reinforcement learning theory work.

\chapter{Background and Problem Setup}\label{chap:background}

This chapter provides relevant background concepts which will be used throughout the thesis. We begin in \pref{sec:background-mdp} by reviewing the Markov Decision Process (MDP), which is the central model of study in this thesis. Next, in \pref{sec:interaction-models} we give a taxonomy of \emph{interaction protocols}, which formalize how a learner can interact with an unknown MDP. \pref{sec:pac-rl-basics} states the learning objective of PAC RL and provides a broad overview of typical approaches to solving this objective. Lastly, \pref{sec:coverage-conditions} introduces various notions of \emph{coverage conditions}, which are used to measure the complexity of exploration in an MDP. Of note, here we introduce a new worst-case measure called the \emph{spanning capacity}. 

\section{Markov Decision Process}\label{sec:background-mdp}
We study reinforcement learning (RL) in a finite horizon Markov Decision Process (MDP). 

\paragraph{MDP Notation.} We denote the MDP by the tuple $M = \prn{\statesp, \actionsp, P, R, H, d_1}$, which consists of a state space $\statesp$, action space $\actionsp$ with cardinality $A$, probability transition function $P: \statesp \times \actionsp \to \Delta(\statesp)$, reward function $R: \statesp \times \actionsp \to \Delta([0,1])$, horizon $H \in \bbN$, and initial state distribution $d_1 \in \Delta(\statesp)$. For simplicity we assume that the state space $\statesp$ is layered across time, i.e., $\statesp= \statesp_1 \cup \cdots \cup \statesp_H$ where $\statesp_i \cap \statesp_j = \emptyset$ for all $i \ne j$. Thus, given a state $x \in \statesp$ it can be inferred which layer $x$ belongs to, which we will overload as the function $h: \statesp \to [H]$. Beginning with $x_1 \sim d_1$, an episode proceeds in $H$ steps, where at each time step $h \in [H]$, the learner plays an action $a_h$, the reward is sampled as $r_h \sim R(x_h, a_h)$, and the next state is sampled as $x_{h+1} \sim  P(\cdot \mid x_h, a_h)$. This results in a trajectory $\tau = (x_1, a_1, r_1, \cdots, x_H, a_H, r_H)$. We assume that the rewards are normalized so that $\sum_{h=1}^H r_h \in [0,1]$ almost surely.

\paragraph{Policies and Value Functions.} A \emph{policy} is a function $\pi: \statesp \to \Delta(\actionsp)$. For a deterministic policy $\pi: \statesp \to \actionsp$ we denote $\pi(x_h)$ to be the action that the policy takes when presented with state $x_h$. For a stochastic policy, we denote $\pi(\cdot  \mid  x_h)$ to be the distribution over actions. We use $\En^\pi[\cdot]$ and $\Pr^\pi[\cdot]$ to denote the expectation and probability under the process of running $\pi$ in the MDP $M$. For a (partial) trajectory $\tau_{h_\bot:h_\top} = \prn{ x_{h_\bot}, a_{h_\bot}, r_{h_\bot}, \cdots, x_{h_\top}, a_{h_\top}, r_{h_\top}} $ we denote consistency with a deterministic policy $\pi : \statesp \to \actionsp$ using $\cons$, i.e., $\crl{ \pi \cons \tau_{h_\bot:h_\top} }$ denotes the event that $\pi(x_h) = a_h$ for every $h_\bot \le h \le h_\top$.  

The value function and $Q$-function for a given policy $\pi$ are
\begin{align*}
    V^\pi_h(x) = \En^\pi \brk*{ \sum_{h' = h}^H r_{h'}  \mid  x_h = x}, \quad \text{and} \quad Q^\pi_h(x, a) = \En^\pi \brk*{ \sum_{h' = h}^H r_{h'}  \mid  x_h = x, a_h = a }.
\end{align*}
We let $\pi^\star$ denote an optimal (deterministic) policy which maximizes $Q^\pi(x,a)$ for every $(x,a) \in \statesp \times \actionsp$ simultaneously. Furthermore when clear from the context we denote $V^\pi \coloneqq \En_{x_1 \sim d_1} V^\pi(x_1)$. We also define the \emph{occupancy measures}
\begin{align*}
    d^\pi_h(x,a) \coloneqq \Pr^\pi \brk*{x_h = x, a_h = a} \quad \text{and} \quad d^\pi_h(x) \coloneqq \Pr^\pi \brk*{x_h = x}.
\end{align*}

We assume the learner is given a policy class $\Pi \subseteq \Delta(\actionsp)^\statesp$. For any $h \in [H]$ we let $\Pi_h \subseteq \Delta(\actionsp)^{\statesp_h}$ denote the restriction of the policy class to the states in layer $h$. We define a \emph{partial policy} $\pi_{h_\bot:h_\top}$ to be one that is defined over a contiguous subset of layers $[h_\bot, \cdots, h_\top] \subseteq [H]$, and use $\Pi_{h_\bot:h_\top}$ to denote the set of partial policies defined by $\Pi$. 

In the parlance of statistical learning theory, we say the policy class $\Pi$ satisfies \emph{(policy) realizability} if $\pi^\star \in \Pi$. Otherwise, we say we are in the \emph{agnostic} setting, which imposes no assumptions on nature (i.e., the underlying MDP).  

\paragraph{Block MDPs.} Block MDPs \cite{jiang2017contextual, du2019provably} are a prototypical setting for RL with large state spaces but low intrinsic complexity. In this thesis, we will give several results for Block MDPs, both upper and lower bounds. Formally, a Block MDP is given by the tuple $M = (\statesp, \latentsp, \actionsp, H, \optlatp, \optlatr, \psi)$. Compared to the definition of the MDP, we additionally specify a \emph{latent state space} $\latentsp$ and an \emph{emission function} $\psi: \latentsp \to \Delta(\statesp)$. To avoid confusion we refer to observed states $x \in \statesp$ as \emph{observations}. Typically, we assume the latent state space $\latentsp$ is finite, while the observation space $\statesp$ can be arbitrarily large or infinite. Without loss of generality, we will assume that the initial latent state $s_1$ is fixed and known to the learner. 

The dynamics of the Block MDP take the following form: Starting from an initial latent state $s_1$, an emission $x_1 \sim \emission(s_1)$ is generated. For every layer $h \in [H]$, the latent state evolves according to $s_{h+1} \sim \optlatp(\cdot  \mid  s_h, a_h)$ and the reward is sampled as $r_h \sim \optlatr(s_h, a_h)$. The latent state $s_h$ is never observed by the learner, and instead the learner receives only the observation $x_h \sim \emission(s_h)$.

The emission function $\psi$ satisfies the property of \emph{decodability}, which asserts that for every pair $s \ne s'$, we have $\supp(\emission(s)) \cap \supp(\emission(s')) = \emptyset$. Therefore, we can define the ground-truth decoder function $\optdec: \statesp \to \latentsp$ which maps every observation $x$ to the corresponding latent $s$ from which it was emitted. Under decodability, the observation-level transition function (resp.~reward function) can be written as $P(\cdot  \mid  x_h, a_h) = \emission \circ \optlatp(\cdot \mid  \optdec(x_h), a_h)$ (resp.~$R(x_h, a_h) = \optlatr(\optdec(x_h), a_h)$). A priori, both the emission $\psi$ and the decoder $\optdec$ are unknown to the learner and, in a departure from prior work on Block MDPs~\citep[e.g.,][]{misra2020kinematic}, in policy learning the learner does not have access to a decoder class $\Phi$ containing the true decoder $\optdec$, or an emission class \(\Psi\) containing \(\psi\).  

\clearpage
\section{Taxonomy of Interaction Protocols}\label{sec:interaction-models}
\begin{wrapfigure}{r}{0.4\textwidth}
    \centering
    \vspace{-2em}
    \includegraphics[scale=0.23, trim={0cm 18cm 43cm 0cm}, clip]{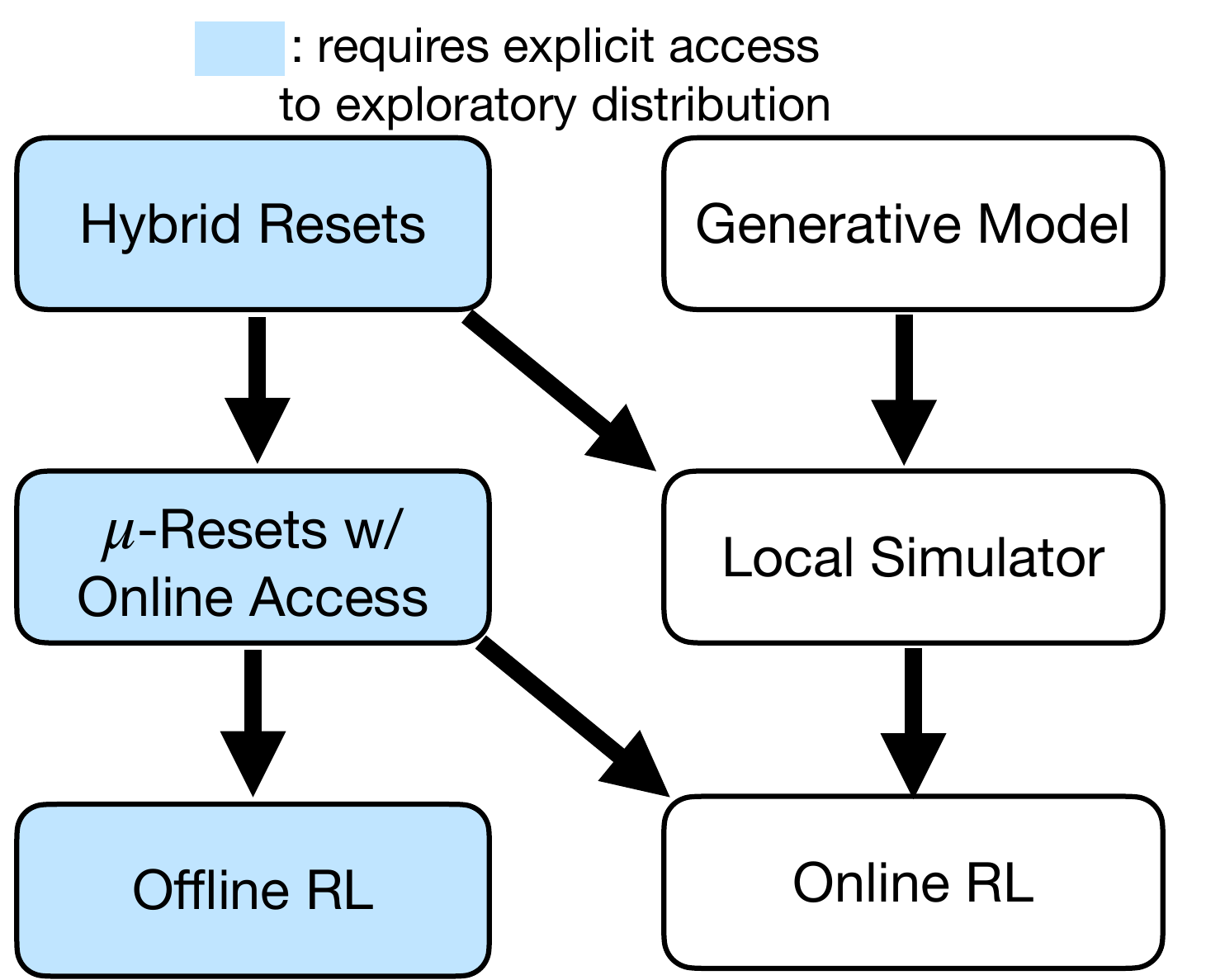}
    \caption{Relationships between interaction protocols. An arrow $A \boldsymbol{\rightarrow} B$ implies that protocol $B$ can be simulated with protocol $A$. \vspace{-2em}}
    \label{fig:interaction-models}
\end{wrapfigure}

Before we can begin to discuss the statistical complexity of RL, we first need to formalize what a ``sample'' is, and how it is collected. In the standard PAC learning framework for supervised learning, each sample is assumed to be i.i.d.~drawn from some underlying distribution $\cD$. 

However, the MDP formalism is considerably richer than supervised learning, and many different ways of collecting data have been proposed and studied in the literature, both in theory and practice. We refer the reader to \pref{fig:interaction-models}, which illustrates the relationships between these interaction protocols which are discussed below.

First, we state the most standard form of access: online reinforcement learning.
\begin{itemize}
\item \textbf{Online RL.} Here, the learner accesses $M$ through the following protocol: in every episode, they can submit any policy $\pi \in \Delta(\actionsp)^\statesp$ and receive a trajectory sampled by running $\pi$ from the initial state distribution $x_1 \sim d_1$. 
\end{itemize}
The next two forms of access represent practical scenarios in which the learner has simulator access by augmenting online RL with reset capabilities.

\begin{itemize}
    \item \textbf{Generative Model.} Also known as a \emph{global simulator}. The learner can query any tuple $(x,a) \in \statesp \times \actionsp$ and receive a sample $(x', r)$ where $x' \sim P(\cdot \mid x,a)$ and $r \sim R(x,a)$.
    \item \textbf{Local Simulator.} In addition to starting from a random initial state $x_1 \sim d_1$, the learner can choose to reset the MDP to any state $x \in \statesp$ which has been previously encountered and then generate a (partial) trajectory starting from this state. The main difference from the generative model is that the local simulator can only reset to previously seen states.
\end{itemize}
Finally, we list several settings in which the learner is also given additional access to some sampling distribution $\mu$ (either over states or state-action pairs). Typically, performance in these settings is measured in terms of ``coverage quality'' of this distribution $\mu$ (to be discussed shortly in \pref{sec:coverage-conditions}). 
\begin{itemize}
    \item \textbf{Offline/Batch RL.} Instead of on-demand sampling access to $M$, the learner receives a dataset $\cD = \crl{\cD_h}_{h=1}^H$ where each $\cD_h$ is comprised of tuples $(x_h,a_h,x'_{h+1},r_h)$ where $(x_h,a_h)$ are i.i.d.~drawn from $\mu_h \in \Delta(\statesp_h \times \actionsp)$, next state $x'_{h+1} \sim P(\cdot \mid x_h,a_h)$, and reward $r_h \sim R(x_h,a_h)$.
    \item \textbf{$\mu$-Resets.} The learner has access to an exploratory reset distribution $\mu = \crl{\mu_h}_{h=1}^H$, and can choose to either generate (partial) trajectories by running policies from the initial state distribution $d_1$ or any of the exploratory resets $\mu_1, \cdots, \mu_H$.
    \item \textbf{\resetmodel.} The learner has access to both an exploratory reset distribution $\mu = \crl{\mu_h}_{h=1}^H$ and a local simulator. 
\end{itemize}

\paragraph{What is a sample?}
For the generative model and offline RL, we define a sample to be an $(x,a,x',r)$ tuple. For the rest of the interaction protocols, we define a sample to be a single episode of interaction with $M$, i.e., a partial trajectory $\tau_{h:H} = (x_h, a_h, r_h, \cdots, x_{H}, a_{H}, r_{H})$ that is obtained by executing some policy on the MDP $M$. Under this convention, simulating one ``sample'' from the local simulator or online RL requires $H$ ``samples'' from the generative model; in this thesis, we are not concerned with $\poly(H)$ factors, so we will sweep this minor inconsistency under the rug.

\section{PAC Reinforcement Learning}\label{sec:pac-rl-basics}
In this section, we present the PAC RL objective, and give a broad overview of approaches to designing algorithms which solve this objective.

In PAC RL, the goal of the learner is to interact with the MDP (through one of the aforementioned interaction protocols) and with probability at least $1-\delta$, output a policy $\estpi: \statesp \to \actionsp$ such that
\begin{align*}
    V^{\estpi} \ge \max_{\pi \in \Pi} V^\pi - \eps. \numberthis \label{eq:pac-rl}
\end{align*}
We call a policy $\estpi$ which satisfies \eqref{eq:pac-rl} an $\eps$-optimal policy.
This objective directly extends the supervised learning objective \eqref{eq:classification-error}. (In fact, supervised learning can be viewed as a simple case of RL with binary actions, horizon 1, and binary rewards.)

\subsection{Basic Results for PAC RL} We begin with the following classical result which shows that PAC RL is statistically intractable in the worst case. 
\begin{proposition}[Informal, No Free Lunch for RL \cite{krishnamurthy2016pac}]\label{prop:nfl-rl}
    There exists a policy class $\Pi$ for which the sample complexity under a generative model is at least $\Omega(\min\crl{A^H, \abs{\Pi}, \abs{\statesp} A} / \eps^2)$. 
\end{proposition} 

Learning with a local simulator or online RL is only harder than learning with a generative model, so the lower bound of \pref{prop:nfl-rl} extends to these interaction protocols as well. \pref{prop:nfl-rl} is the analogue of the classical \emph{No Free Lunch} results in statistical learning theory \citep{shalev2014understanding}; it indicates that without placing further assumptions on the MDP or the policy class \(\Pi\) (e.g., by introducing additional structure or constraining the state/action space sizes, policy class size, or the horizon), sample efficient PAC RL is not possible. 

Indeed, an almost matching upper bound of $\wt{\cO} \prn{ \min\crl{A^H, \abs{\Pi}, \abs{\statesp} A} /\eps^2}$ under online RL access is quite easy to obtain. The $\abs{\Pi}/\eps^2$ guarantee can simply be obtained by iterating over $\pi \in \Pi$, collecting  $\wt{\cO}(1/\eps^2)$ trajectories per policy, and then picking the policy with highest empirical value. The $\abs{\statesp} A/\eps^2$ guarantee can be obtained by running known algorithms for tabular RL \cite{zhang2021reinforcement}. Finally, the $A^H/\eps^2$ guarantee is achieved by the classical importance sampling (IS) algorithm \citep{kearns1999approximate, agarwal2019reinforcement}. Since importance sampling is an important technique that we use and build upon in this thesis, we give a formal description of the algorithm in \pref{alg:is}.

\begin{algorithm}[!htp] 
    \caption{$\mathsf{ImportanceSampling}$}\label{alg:is}
    \begin{algorithmic}[1]
            \Require Policy class $\Pi$, online RL access to $M$.
            \State Collect $n = \cO(A^H \log \abs{\Pi}/\eps^2)$ trajectories by executing $(a_1, \dots, a_H)\sim \unif(\actionsp^H)$. 
            \State Return $\pihat = \argmax_{\pi \in \Pi} \wh{v}^\pi_{\mathrm{IS}}$, where \(\wh{v}^\pi_{\mathrm{IS}} \coloneqq \frac{A^H}{n} \sum_{i=1}^n \indd{\pi \cons \tau^{(i)}} \prn{\sum_{h=1}^H r_h^{(i)}}\).
    \end{algorithmic}
\end{algorithm} 
For every $\pi \in \Pi$, the quantity $\wh{v}^\pi_{\mathrm{IS}}$ is an unbiased estimate of $V^\pi$ with variance $A^H$; the sample complexity guarantee follows by standard concentration bounds \citep[see, e.g.,][]{agarwal2019reinforcement}.

\emph{Remark.} Another common performance metric in RL is the cumulative \emph{regret} of an online RL algorithm interacting with the MDP over $T$ rounds, which is motivated by the notion of regret in online learning or sequential prediction tasks. We focus on the PAC RL objective because we are interested in understanding the sample complexity of RL algorithms under various interaction protocols, not just online RL, and the PAC RL objective gives us an equal playing field for making comparisons. While the two performance metrics can be related, there can be subtle differences \cite{dann2017unifying, jin2018q, menard2020fast, wagenmaker2022beyond}.

In order to circumvent the lower bound of \pref{prop:nfl-rl}, two main approaches can be taken, which we describe next.

\subsection{Approach 1: Assumption of Realizability}\label{sec:realizability-approach}

A popular paradigm for  developing algorithms for MDPs with large state/action spaces is to use additional function approximation to either model the MDP dynamics or optimal value functions. Over the last decade, there has been a long line of work \citep{russo2013eluder, jiang2017contextual, dann2018oracle, sun2019model, du2019provably, wang2020reinforcement, du2021bilinear, foster2021statistical, jin2021bellman, zhong2022gec, foster2023tight} focusing on understanding structural conditions on the function class and the underlying MDP that enable statistically efficient RL. Algorithmically, these model- or value-based methods balance the exploration/exploitation tradeoff via uncertainty quantification, exploration bonuses, and the principle of optimism. They typically optimize surrogate objectives based on Bellman errors rather than directly optimizing the policy performance objective \eqref{eq:pac-rl}.

The universally adopted assumption underpinning all of these works is that of \emph{realizability}: the true model/value function belongs to the chosen class. Conceptually, the primary role of the realizability assumption is to tame \emph{error amplification} due to the sequential nature of the RL problem. Focusing on the class of algorithms based on \emph{value function approximation} \cite{wang2020provably, jin2021bellman, xie2022role, foster2024model}, theoretical guarantees for these algorithms typically requires the function class to satisfy an even stronger representational condition called \emph{Bellman completeness}, which states that the function class $\cF$ is closed under the Bellman operator. One drawback of the Bellman completeness assumption is that it is nonmontone, that is, if $\cF$ is a set of value functions which satisfies Bellman completeness, adding even a single function $\cF' = \cF \cup \crl{f}$ may break Bellman completeness, even though one might expect learning with $\cF'$ to not be significantly more challenging than learning with $\cF$.

In addition to the aforementioned papers, we highlight two additional papers by the author which fall under Approach 1:
\begin{enumerate}
    \item \cite{li2022exponential}: We study a nonlinear setting in RL proposed by \cite{chowdhury2021reinforcement} over continuous state space $\statesp \subseteq \bbR^{d_\statesp}$, when the transition function $P$ has an exponential family form:
        \begin{align*}
            P(x'\mid x,a) \coloneqq q(x') \cdot \exp\prn*{\tri*{\psi(x'), W_0 \phi(x,a)} - Z_{xa}(W_0) }, \numberthis \label{eq:exponential-family}
        \end{align*}
        where $\psi: \statesp \to \bbR^{d_\psi}$ and $\phi: \statesp \times \actionsp \to \bbR^{d_\phi}$ are known feature embeddings, $q: \statesp \to \bbR_{\ge 0}$ is a known base measure, $W_0 \in \bbR^{d_\psi \times d_\phi}$ is an unknown parameter, and $Z_{xa}: \bbR^{d_\psi \times d_\phi} \to \bbR$ is the log partition function that ensures the density integrates to 1. The transition model \eqref{eq:exponential-family} covers both the well-studied linear dynamical system as well as nonlinear extensions \cite{mania2020active, kakade2020information}. Instead of using maximum likelihood estimation (MLE) to estimate the model parameter $W_0$, we use score matching, an unnormalized density estimation technique \cite{hyvarinen2005estimation}. Our main result is an online RL algorithm which achieves sample complexity $\poly(d_\psi, d_\phi, H, 1/\eps)$ (ignoring some polynomial dependence on structural scale parameters). Notably, our guarantee incurs no dependence on the state space size $\abs{\statesp}$, which is infinite.
    \item \cite{li2022pessimism}: We study offline learning of contextual bandits (RL with $H=1$), when the expected reward function takes a linear form:
        \begin{align*}
            \En\brk*{R(x,a)} = \phi(x,a)^\top \theta^\star,
        \end{align*}
        where $\phi: \statesp \times \actionsp \to \bbR^d$ is a known feature mapping and $\theta^\star \in \bbR^d$ is an unknown parameter vector. We present a family of learning rules $\crl{\estpi_p}_{p \ge 1}$ based on the \emph{pessimism principle}, which discounts policies that are less represented/supported in the offline dataset. The learning rules are instantiated via confidence sets with respect to different $\ell_p$ norms, generalizing prior work \cite{xie2021bellman, rashidinejad2021bridging, jin2021pessimism}. We prove tight upper and lower bounds on the minimax sample complexity of these learning rules, and demonstrate that the novel $\estpi_\infty$ learning rule has an \emph{adaptive optimality property}, dominating all other predictors in the family. 
\end{enumerate}

\subsection{Approach 2: Agnostic Policy Learning}
The typical approach taken in statistical learning theory, and the main approach taken in this thesis, is to understand when statistically efficient RL is possible \emph{without} relying on assumptions on nature itself, i.e., obtaining assumption-free/agnostic guarantees. Why should we care about this, when we already have a relatively mature theory for Approach 1? We give some motivation below: 
\begin{itemize}
    \item Every model class $\cM$ or value class $\cF$ can be converted into a realizable policy class $\Pi$ of the same size (simply by selecting the greedy policies for each model/value function), but the converse is not true. Thus, this approach relies on \emph{strictly weaker} assumptions than Approach 1.
    \item In general, realizability-based guarantees do not yield guarantees for \eqref{eq:pac-rl} as $\eps \to 0$. In some cases, even mild misspecification can cause catastrophic breakdown of guarantees \citep{du2019good, lattimore2020learning}.
    \item Practically speaking, verifying realizability can be difficult. Furthermore, in various applications the optimal policy $\optpi$ might have a succinct representation, but the ground truth model $(P,r)$ or the optimal value function $V^\star$ may be highly complex, rendering accurate approximation of the dynamics or value functions infeasible without substantial domain knowledge \cite{dong2020expressivity}.
\end{itemize}
In this thesis, we treat \eqref{eq:pac-rl} as an agnostic objective (making as minimal assumptions on nature as possible), and we present two styles of guarantees: 
\begin{enumerate}
    \item We study the minimax sample complexity of PAC RL. To do so, we introduce several distribution-free complexity measures which only depend on the policy class $\Pi$. This is akin to the classical VC bounds in supervised learning, which depend only on the hypothesis class and make no assumptions on the distribution $\cD$.
    \item We also study instance-dependent bounds in terms of intrinsic measures of complexity that take into account the underlying MDP $M$, the policy class $\Pi$, and (when given) the sampling distribution $\mu$. We carefully establish when it is possible for learning algorithms to adapt to these intrinsic measures. This is akin to Rademacher-style bounds for supervised learning.
\end{enumerate}
Our results give sample complexity guarantees under various interaction protocols of the form
\begin{align*}
    \poly\prn*{ \mathsf{comp}(\dots), H, A, \eps^{-1}, \log(\delta^{-1}), \log \abs{\Pi} },    
\end{align*}
for some appropriately defined $\mathsf{comp}(\dots)$ which depends only on $\Pi$ (in the first case) or on $\Pi$, $M$, and possibly $\mu$ (in the second case). We will not worry too much about the precise polynomial dependence on action space $A$ and horizon $H$, and invest more effort into designing algorithms and proving guarantees which have no explicit dependence on $\abs{\statesp}$. We state our bounds using $\log \abs{\Pi}$ as a proxy for the statistical complexity of $\Pi$. It is not too difficult to extend our results to infinite policy classes, and we do so in some cases (see, e.g., the results in \pref{sec:generative-infinite-policy} and \pref{sec:online-infinite-policy}).

\section{Coverage Conditions}\label{sec:coverage-conditions}

In this section, we collect several complexity measures based on \emph{coverage}, which will be studied throughout the thesis.

\subsection{Concentrability, Coverability, and Pushforward Variants}
Given a sampling distribution $\mu = \crl{\mu_h}_{h=1}^H$, one can measure its quality by how well it covers the state space. This notion of quality is measured by \emph{coverage conditions}, and is well studied in RL (see Bibliographical Remarks in \pref{sec:background-bib}). Roughly speaking, coverage conditions measure the expansiveness of the set of occupancy measures $\crl{d^\pi_h}$ for policies in the given class $\Pi$. We state a classical notion called concentrability, which depends on the sampling distribution, MDP, and policy class.

\begin{definition}[Concentrability]\label{def:concentrability} The concentrability coefficient for a distribution $\mu = \crl{\mu_h}_{h=1}^H$ with respect to policy class $\Pi$ and MDP $M$ is defined as
\begin{align*}
    \cconc(\mu; \Pi, M) \coloneqq \sup_{\pi \in \Pi, h \in [H]} \nrm*{\frac{d^\pi_h}{\mu_h}}_\infty.
\end{align*}
When clear from the context we denote the concentrability coefficient as $\cconc$.
\end{definition} 

We also define a strengthening of concentrability called pushforward concentrability.

\begin{definition}[Pushforward Concentrability]\label{def:exploratory-pushforward-distribution}
The pushforward concentrability coefficient for a distribution $\mu = \crl{\mu_h}_{h\in[H]}$ with respect to MDP $M$ is 
\begin{align*}
    \cpush(\mu; M) \coloneqq \max_{h \in [H]} \sup_{(x,a,x') \in \statesp_{h-1} \times \actionsp \times \statesp_{h}} \frac{P(x' \mid x,a)}{\mu_h(x')}.
\end{align*}
When clear from the context we denote the pushforward concentrability coefficient as $\cpush$.
\end{definition}
Note that unlike concentrability, pushforward concentrability only depends on the distribution $\mu$ and the MDP $M$, and does not depend on the policy class $\Pi$. It is known that the pushforward concentrability coefficient is always an upper bound on the concentrability coefficient for any distribution $\mu$, but concentrability can be arbitrarily smaller \cite{xie2021batch}. 

Even without explicit access to a sampling distribution $\mu$, one can define a notion of \emph{coverability}, which merely posits the existence of a good sampling distribution, and thus lower bounds the concentrability coefficient for any distribution $\mu$. Coverability is an intrinsic property that depends on the underlying MDP and the policy class. The pushforward variant of coverability can also be defined (see \pref{def:pushforward-coverability}), but is not central to the results in this thesis.

\begin{definition}[Coverability]\label{def:coverability}
The coverability coefficient for policy class $\Pi$ and MDP $M$ is defined as
\begin{align*}
\ccov(\Pi, M)\coloneqq~ \inf_{\mu_1, \dots \mu_H \in \Delta(\statesp \times \actionsp)} \sup_{\pi \in \Pi, h \in [H]} ~ \nrm*{ \frac{d^\pi_h}{\mu_h} }_\infty 
         =~ \max_{h \in [H]}~ \sum_{s, a} \sup_{\pi \in \Pi} d^\pi_h(x,a). \numberthis \label{eq:cov-defn}
\end{align*} 
When clear from the context we denote the coverability coefficient as $\ccov$.
\end{definition}

The last equality of \eqref{eq:cov-defn}, which says that coverability is equivalent to a notion of cumulative reachability, is nontrivial, and it is shown in Lemma 3 of \cite{xie2022role}. 

\emph{Remark.} We state the original state-action variant of coverability as defined in \cite{xie2022role}, where the infimum is taken over distributions over state-action pairs. This is inconsistent with \pref{def:concentrability} and \ref{def:exploratory-pushforward-distribution}, but this is minor, since the state-action variant is related to the state variant of coverability (where the infimum is taken over distributions over states) by a factor of $A$. 

\subsection{Spanning Capacity}\label{sec:C-pi}
Now we introduce the \emph{spanning capacity}, which is the worst-case (over all MDPs defined over fixed state/action spaces and horizon) value of coverability (\pref{def:coverability}). It is solely a structural property of the policy class $\Pi$ itself. Our presentation follows that of \cite{jia2023agnostic}, which first defines it as a notion of maximum reachability over deterministic MDPs, and then shows the equivalence with worst-case coverability (see \pref{lem:coverability}).

\paragraph{Preliminaries.} In this section, we restrict ourselves to deterministic policy classes $\Pi \subseteq \actionsp^\statesp$. We set up some notation: define $\cMsto$ as the set of all (stochastic and deterministic) MDPs of horizon $H$ over the state space $\statesp$ and action space $\actionsp$, and $\cMdet \subset \cMsto$ to denote the set of all MDPs with both deterministic transitions and rewards (in particular, such MDPs have a fixed initial state). In deterministic MDP $M \in \cMdet$, we say $(x,a)$ is \emph{reachable} by $\pi \in \Pi$  if $(x,a)$ lies on the trajectory obtained by running $\pi$ on $M$. 

\paragraph{Definition and Examples.} The spanning capacity measures the ``complexity'' of a given policy class $\Pi$ as the maximum number of state-action pairs which are reachable by some $\pi \in \Pi$ in any \emph{fixed deterministic} MDP.

\begin{definition}[Spanning Capacity]\label{def:dimension}
Fix a deterministic MDP $M \in \cMdet$. We define the \emph{cumulative reachability} at layer $h\in[H]$, denoted by $\creach_h(\Pi, M)$, as
\begin{align*}
    \creach_h(\Pi, M) \coloneqq \abs{ \crl{(x,a) : (x,a)~\text{is  reachable by } \Pi ~\text{at layer \(h\)} } }.
\end{align*}
Denote $\dimRL_h(\Pi)\coloneqq \max_{M \in \cMdet} \creach_h(\Pi, M)$. We define the \emph{spanning capacity} of $\Pi$ as
\begin{align*}
    \dimRL(\Pi) \coloneqq \max_{h \in [H]} \dimRL_h(\Pi).
\end{align*}
\end{definition}

A natural interpretation of the spanning capacity is that it represents the largest ``needle in a haystack'' that can be embedded in a deterministic MDP using the policy class $\Pi$. To see this, let $(M^\star, h^\star)$ be the MDP and layer which witnesses $\dimRL(\Pi)$, and let $\crl{(x_{h^\star}^{(i)},a_{h^\star}^{(i)})}_{i=1}^{\dimRL(\Pi)}$ be the set of  state-action pairs reachable by \(\Pi\) in $M^\star$ at layer $h^\star$. Then one can hide a reward of 1 on one of these state-action pairs; since every trajectory visits a single $(x_{h^\star}^{(i)}, a_{h^\star}^{(i)})$ at layer $h^\star$, we need at least $\dimRL(\Pi)$ samples in order to discover which state-action pair has the hidden reward. Note that in the (agnostic) PAC RL objective \eqref{eq:pac-rl}, we only need to worry about the states that are reachable using \(\Pi\), even though the \(h^\star\) layer may have other non-reachable states and actions with possibly larger rewards.

To build intuition, we first look at some simple examples with bounded spanning capacity:

\begin{itemize}
    \item \textbf{Contextual Bandits:} Consider the standard formulation of contextual bandits (i.e., RL with $H = 1$). For any policy class \(\Picb\), since $H=1$, the largest deterministic MDP we can construct has a single state $x_1$ and at most $A$ actions available on $x_1$, so $\dimRL(\Picb) \le A$.
    \item \textbf{Tabular MDPs:} Consider tabular RL with the policy class \(\Pitab = \actionsp^\statesp\) consisting of all deterministic policies on the underlying state space. Depending on the relationship between $\abs{\statesp},A$ and $H$, we have two possible bounds on $\dimRL(\Pitab) \le \min\crl{A^H, \abs{\statesp} A}$. If the state space is exponentially large in $H$, then it is possible to construct a full $A$-ary ``tree'' such that every $(x,a)$ pair at layer $H$ is visited, giving us the $A^H$ bound. However, if the state space is small, then the number of $(x,a)$ pairs available at any layer $H$ is trivially bounded by $\abs{\statesp} A$.
    \item \textbf{Bounded Cardinality Policy Classes:} For any policy class \(\Pi\), we always have that $\dimRL(\Pi) \le \abs{\Pi}$, since in any deterministic MDP, in any layer \(h \in [H]\), each $\pi \in \Pi$ can visit at most one new $(x,a)$ pair. Thus, for policy classes  \(\abs{\Pismall}\) with small cardinality
 (e.g. \(\abs{\Pismall} = \poly(H, A)\)),  the spanning capacity is also bounded. 
\end{itemize}
Before proceeding, we note that in light of this discussion, the spanning capacity is always bounded for any policy class \(\Pi\).
\begin{proposition}\label{prop:dimension-ub}
For any policy class $\Pi$, we have $\dimRL(\Pi) \le \min\crl{A^H, \abs{\Pi}, \abs{\statesp} A}$.
\end{proposition}

\pref{prop:dimension-ub} recovers the worst-case upper and lower bounds discussed in \pref{sec:pac-rl-basics}. We next present several policy classes for which spanning capacity is substantially smaller than upper bound of \pref{prop:dimension-ub}. For these policy classes we set the state/action spaces to be $\statesp_h = \crl{x_{(i,h)} : i \in [K]}$ for all $h\in [H]$ and $\actionsp = \crl{0,1}$, respectively. The bounds on spanning capacity are shown via induction. Detailed calculations are deferred to \pref{sec:examples-policy-classes}.

\begin{itemize}
    \item \textbf{Singletons:} We define $\Pisingleton \coloneqq \crl{\pi_{(i,h)}: i \in [K], h \in [H]}$, where $\pi_{(i,h)}$ takes action 1 on state $x_{(i,h)}$ and 0 everywhere else. We have $\dimRL(\Pisingleton) = H+1$, since once we fix a deterministic MDP, there are at most $H$ states where we can split from the trajectory taken by the policy which always plays $a=0$, so therefore the maximum number of $(x,a)$ pairs reachable at layer $h \in [H]$ is $h+1$. 
    \item \textbf{$\boldsymbol{\ell}$-tons}: This is a natural generalization of singletons. We define
$\PiLton \coloneqq \crl{\pi_I : I \subset \statesp, \abs{I} \le \ell}$, where the policy \(\pi_I\) is defined s.t.~\(\pi_I(x) = \ind{x \in I}\) for any \(x \in \statesp\). Here, $\dimRL(\PiLton) = \Theta(H^\ell)$.
    \item \textbf{1-Active Policies}: We define \(\Pioneactive\) to be the class of policies which can take both possible actions on a single state $x_{(1,h)}$ in each layer $h$, but on other states $x_{(i, h)}$ for $i \ne 1$ must take action 0. Formally, \(\Pioneactive \ldef{} \crl{\pi_b \mid b \in \crl{0, 1}^H}\), where for any bitstring \(b \in \crl{0, 1}^H\) the policy \(\pi_b\) is defined such that  \( \pi_b(x) = b[h] \) if \(x = x_{(1, h)} \), and $\pi_b(x) = 0$ otherwise. Here, $\dimRL(\Pioneactive) = \Theta(H)$.
    \item \textbf{All-Active Policies}: We define $\PiJactive \coloneqq \crl{\pi_b \mid b \in \crl{0, 1}^H}$, where for any bitstring \(b \in \crl{0, 1}^H\) the policy \(\pi_b\) is defined such that  \( \pi_b(x) = b[h] \) if \(x = x_{(j, h)} \), and $\pi_b(x) = 0$ otherwise. We let $\Piactive \coloneqq \bigcup_{j = 1}^K \PiJactive$. Here, $\dimRL(\Piactive) = \Theta(H^2)$.
\end{itemize}

\paragraph{Relating Spanning Capacity to Coverability.} 
It is straightforward from \pref{def:coverability} that spanning capacity equals worst-case coverability when we maximize over deterministic MDPs, since for any deterministic MDP, $\sup_{\pi \in \Pi} d^\pi_h(x,a) = \ind{(x,a)~\text{is  reachable by } \Pi ~\text{at layer \(h\)}}$. The next lemma shows that spanning capacity is \emph{exactly} worst-case coverability even when we maximize over the larger class of stochastic MDPs. As a consequence, there always exists a deterministic MDP that witnesses worst-case coverability.

\begin{lemma}\label{lem:coverability}
For any policy class $\Pi$, we have $\sup_{M \in \cMsto} \ccov(\Pi, M) =   \dimRL(\Pi)$.
\end{lemma}

\begin{proof}[Proof of \pref{lem:coverability}]
Fix any $M \in \cMsto$, as well as $h \in [H]$. We claim that
\begin{align*}
    \Gamma_h \coloneqq \sum_{x_h \in \statesp_h, a_h \in \actionsp_h} \sup_{\pi \in \Pi} d^\pi_h(x_h,a_h; M) \le \dimRL_h(\Pi).\numberthis \label{eq:lem-coverability-eq}
\end{align*}
Here, $d^\pi_h(x_h,a_h; M)$ is the state-action visitation distribution of the policy \(\pi\) on MDP $M$.

We first set up additional notation. Let us define a \emph{prefix} as any tuple of pairs of the form
\begin{align*}
    (x_1, a_1, x_2, a_2, \dots, x_k, a_k) \quad \text{or} \quad (x_1, a_1, x_2, a_2, \dots, x_{k}, a_{k}, x_{k+1}).
\end{align*}
We will denote prefix sequences as $(x_{1:k}, a_{1:k})$ or $(x_{1:k+1}, a_{1:k})$ respectively. For any prefix $(x_{1:k}, a_{1:k})$ (similarly prefixes of the type $(x_{1:k+1}, a_{1:k})$) we let $d^\pi_h(x_h, a_h \mid  (x_{1:k}, a_{1:k}) ; M)$ denote the conditional probability of reaching $(x_h, a_h)$ under policy $\pi$ given one observed the prefix $(x_{1:k}, a_{1:k})$ in MDP $M$, with $d^\pi_h(x_h, a_h \mid  (x_{1:k}, a_{1:k}) ; M) = 0$ if $\pi \not\cons (x_{1:k}, a_{1:k})$ or $ \pi \not\cons (x_h, a_h)$.

In the following proof, we assume that the start state $x_1$ is fixed, but this is without any loss of generality, and the proof can easily be adapted to hold for stochastic start states.

Our strategy will be to explicitly compute the quantity $\Gamma_h$ in terms of the dynamics of $M$ and show that we can upper bound it by a ``derandomized'' MDP $M'$ which maximizes reachability at layer $h$. Let us unroll one step of the dynamics:
\begin{align*}
    \Gamma_h
    &\coloneqq \sum_{x_h \in \statesp_h, a_h \in \actionsp} \sup_{\pi \in \Pi} d^\pi_h(x_h, a_h; M) \\
    &\overset{(i)}{=} \sum_{x_h \in \statesp_h, a_h \in \actionsp} \sup_{\pi \in \Pi} d^\pi_h(x_h, a_h \mid  x_1 ; M) , \\
    &\overset{(ii)}{=} \sum_{x_h \in \statesp_h, a_h \in \actionsp} \sup_{\pi \in \Pi} \crl*{ \sum_{a_1 \in \cA} d^\pi_h(x_h, a_h\mid x_1, a_1; M) } \\
    &\overset{(iii)}{\leq} \sum_{a_1 \in \actionsp} \sum_{x_h \in \statesp_h, a_h \in \cA} \sup_{\pi \in \Pi}  d^\pi_h(x_h, a_h \mid  x_1, a_1; M).
\end{align*}
The equality $(i)$ follows from the fact that $M$ always starts at $x_1$. The equality $(ii)$ follows from the fact that $\pi$ is deterministic, so there exists exactly one $a' = \pi(x_1)$ for which $d^\pi_h(x_h, a_h\mid x_1, a'; M) = d^\pi_h(x_h, a_h\mid x_1 ; M)$, with all other $a'' \ne a'$ satisfying $d^\pi_h(x_h, a_h\mid x_1, a''; M) = 0$. The inequality $(iii)$ follows by swapping the supremum and the sum.

Continuing in this way, we can show that
\begin{align*}
\Gamma_h &= \sum_{a_1 \in \actionsp} \sum_{x_h \in \statesp_h, a_h \in \cA} \sup_{\pi \in \Pi} \crl*{ \sum_{x_2 \in \statesp_2} P(x_2 \mid x_1, a_1) \sum_{a_2 \in \cA} d^\pi_h(x_h, a_h \mid  (x_{1:2}, a_{1:2}); M) } \\
&\le \sum_{a_1 \in \actionsp} \sum_{x_2 \in \statesp_2} P(x_2 \mid x_1, a_1) \sum_{a_2 \in \cA} \sum_{x_h \in \statesp_h, a_h \in \cA} \sup_{\pi \in \Pi}  d^\pi_h(x_h, a_h\mid (x_{1:2}, a_{1:2}); M)  \\
&\hspace{0.5in} \vdots \\
&\le \sum_{a_1 \in \actionsp} \sum_{x_2 \in \statesp_2} P(x_2 \mid x_1, a_1) \sum_{a_2 \in \cA} \dots \sum_{x_{h-1} \in \statesp_{h-1}} P(x_{h-1}| x_{h-2}, a_{h-2})  \\
&\hspace{2.0in} \times \sum_{a_{h-1} \in \actionsp} \sum_{x_h \in \statesp_h, a_h \in \cA} \sup_{\pi \in \Pi}  d^\pi_h(x_h, a_h\mid (x_{1:h-1}, a_{1:h-1}); M).
\end{align*}
Now we examine the conditional visitation $d^\pi_h(x_h, a_h\mid (x_{1:h-1}, a_{1:h-1}); M)$. Observe that it can be rewritten as
\begin{align*}
    d^\pi_h(x_h, a_h\mid (x_{1:h-1}, a_{1:h-1}); M) = P(x_h|x_{h-1}, a_{h-1}) \cdot \ind{\pi \cons (x_{1:h}, a_{1:h})}.
\end{align*}
Plugging this back into the previous display, and again swapping the supremum and the sum, we get that
\begin{align*}
\Gamma_h &\le \sum_{a_1 \in \actionsp} \dots \sum_{x_{h} \in \statesp_{h}} P(x_{h}| x_{h-1}, a_{h-1}) \sum_{a_{h} \in \cA} \sup_{\pi \in \Pi}   \ind{\pi \cons (x_{1:h}, a_{1:h})} \\
&= \sum_{a_1 \in \actionsp} \dots \sum_{x_{h} \in \statesp_{h}} P(x_{h}| x_{h-1}, a_{h-1}) \sum_{a_{h} \in \cA}   \ind{\exists \pi \in \Pi: \pi \cons (x_{1:h}, a_{1:h})}
\end{align*}
Our last step is to derandomize the stochastic transitions in the above stochastic MDP, simply by taking the sup over the transition probabilities:
\begin{align*}
&\Gamma_h \le \sum_{a_1 \in \actionsp} \sup_{x_2 \in \statesp_2} \sum_{a_2 \in \cA} \dots \sup_{x_{h} \in \statesp_{h}} \sum_{a_{h} \in \cA}   \ind{\exists \pi \in \Pi: \pi \cons (x_{1:h}, a_{1:h})} = \dimRL_h(\Pi).
\end{align*}
The right hand side of the inequality is exactly the definition of $\dimRL_h(\Pi)$, thus proving \eqref{eq:lem-coverability-eq}. In particular, the above process defines the deterministic MDP which maximizes the cumulative reachability at layer $h$. Taking the maximum over $h$ as well as supremum over $M$, we see that $\sup_{M \in \cMsto} \ccov(\Pi, M) \le \dimRL(\Pi)$. Furthermore, from the definitions we have
\begin{align*}
    \dimRL(\Pi) = \sup_{M \in \cMdet} \ccov(\Pi, M) \le \sup_{M \in \cMsto} \ccov(\Pi, M).
\end{align*}
This concludes the proof of \pref{lem:coverability}.
\end{proof}

\subsection{Examples of Policy Classes}\label{sec:examples-policy-classes}
We provide calculations for the spanning capacity for the examples considered in \pref{sec:C-pi}. 

\paragraph{$\ell$-tons.}
In the following, we will denote $\Pi_\ell \coloneqq \PiLton$. We will first prove that $\dimRL(\Pi_\ell)\le 2H^{\ell}$. To show this, we will prove that $\dimRL_h(\Pi_\ell)\le 2h^\ell$ by induction on $H$. When $H = 1$, the class is a subclass of the contextual bandit class, hence we have $\dimRL_1(\Pi_{\ell})\le 2$. Next, suppose $\dimRL_{h-1}(\Pi_\ell)\le 2(h-1)^\ell$. Fix any deterministic MDP and call the first state $x_1$. Policies taking $a=1$ at $x_1$ can only take $a=1$ on at most $\ell-1$ states in the following layers. Such policies reach at most $\dimRL_{h-1}(\Pi_{\ell-1})$ states in layer $h$. Policies taking $a=0$ at $x_1$ can only take $a=1$ on at most $\ell$ states in the following layers. Such policies reach at most $\dimRL_{h-1}(\Pi_{\ell})$ states in layer $h$.  Hence we obtain
$$\dimRL_{h}(\Pi_{\ell})\le \dimRL_{h-1}(\Pi_{\ell-1}) + \dimRL_{h-1}(\Pi_{\ell})\le 2(h-1)^{\ell-1} + 2(h-1)^{\ell}\le 2h^{\ell}.$$
This finishes the proof of the induction hypothesis. Based on the induction argument, we get
$$\dimRL(\Pi_\ell) = \max_{h\in [H]}\dimRL_h(\Pi_\ell)\le 2H^\ell.$$

\paragraph{$1$-Active Policies.}
We will first prove that $\dimRL(\Pioneactive)\le 2H$. For any deterministic MDP, we use $\bar{\statesp}_h$ to denote the set of states reachable by $\Pioneactive$ at layer $h$. We will show that $\bar{\statesp}_h \le 2h$ by induction on $h$. For $h = 1$, this holds since any deterministic MDP has only one state in the first layer. Suppose it holds at layer $h$. Then, we have
$$|\bar{\statesp}_{h+1}|\le |\{(x, \pi(x)):x\in\bar{\statesp}_h, \pi\in \Pi\}|.$$
Note that policies in $\Pioneactive$ must take $a=0$ on every $x \notin \{x_{(1,1)}, x_{(1,2)}, \cdots, x_{(1, H)}\}$. Hence $|\{(x, \pi(x)) ~|~x\in\bar{\statesp}_h, \pi\in \Pi\}|\le |\bar{\statesp}_h| + 1\le h+1$. Thus, the induction argument is complete. As a consequence we have $\dimRL_{h}(\Pi)\le 2h$ for all $h$, so
$$\dimRL(\Pioneactive) = \max_{h\in [H]} \dimRL_h(\Pioneactive)\le 2H.$$

\paragraph{All-Active Policies.} For any deterministic MDP, there is a single state $x_{(j,1)}$ in the first layer. Any policy which takes $a=1$ at state $x_{(j,1)}$ must belong to $\PiJactive$. Hence such policies can reach at most $\dimRL_{h-1}(\PiJactive)$ states in layer $h$. For polices which take action $0$ at state $h$, all these policies will transit to a fixed state in layer $2$. Hence such policies can reach at most $\dimRL_{h-1}(\Piactive)$ states at layer $h$. Therefore, we get
$$\dimRL_h(\Piactive) \le \dimRL_{h-1}(\Piactive) + \max_j \dimRL_{h-1}(\PiJactive)\le \dimRL_{h-1}(\Piactive) + 2(h-1).$$
By telescoping, we get
$$\dimRL_h(\Piactive)\le h(h-1),$$
which indicates that
$$\dimRL(\Piactive) = \max_{h\in [H]} \dimRL_h(\Piactive)\le H(H-1).$$

\paragraph{Discussion.} We now discuss some observations/extensions of these calculations.

\emph{Parameter Shared vs.~Product Policy Classes.}
Many of the policy class examples we have considered exhibit \emph{parameter sharing}, meaning that the behaviors of policies have dependencies across layers. In particular, such policy classes cannot be written as product classes $\Pi = \Pi_1 \times \Pi_2 \times \cdots \times \Pi_H$. Any parameter-shared policy class $\Pi$ can be written as a subset of a product policy class $\Pi'$ with $\log \abs{\Pi'} \le H \log \abs{\Pi}$. However, spanning capacity can increase arbitrarily (consider the singleton policy class, where the spanning capacity of the enlarged product class is $\dimRL(\Pisingleton') = A^H$). 

For product classes, we have the following inductive characterization of spanning capacity, which is more or less obvious given our previous calculations.
\begin{proposition}\label{prop:product-spanning-cap}
    Let $\Pi$ be any product policy class with $A=2$. Denote
    \begin{align*}
        M_h &\coloneqq \abs*{\crl*{(x,a) \in \statesp_h \times \actionsp : \exists \pi \in \Pi_h, \pi(x) = a}}, \\
        N_h &\coloneqq \abs*{ \crl*{x \in \statesp_h: \exists \pi, \pi' \in \Pi_h, \pi(x) \ne \pi'(x)} }. 
    \end{align*}
    By convention set $\dimRL_0(\Pi) = 1$. For all layers $h \in [H]$, we have the inductive relationship 
    \begin{align*}
        \dimRL_h(\Pi) = \min \crl*{M_h, 2\dimRL_{h-1}(\Pi), \dimRL_{h-1}(\Pi) + N_h}. 
    \end{align*}
\end{proposition}

\emph{Policy Classes for Continuous State Spaces.} In some cases, it is possible to construct policy classes over continuous state spaces that have bounded spanning capacity. For example, consider $\Pisingleton$, which is defined over a discrete (but large) state space. We can extend it to a continuous state space $\statesp_h = \crl{x_{(z, h)} : z \in \bbR}$ for all $h \in [H]$, action space $\actionsp = \crl{0,1}$, and policy class
\begin{align*}
    \wt{\Pisingleton} \coloneqq \crl*{\pi_{(i, h')} : \pi_{(i, h')} (x_{(z, h)}) = \ind{z \in [i, i+1) \text{ and } h = h'}, i\in \bbN, h' \in [H]}.
\end{align*}
Essentially, we have expanded each state to be an interval on the real line. Using the same reasoning, we have the bound $\dimRL(\wt{\Pisingleton}) = H+1$. One can also generalize this construction to the policy class $\wt{\PiLton}$ and preserve the same value of spanning capacity.

However, in general, this expansion to continuous state spaces may blow up the spanning capacity. Consider a similar modification to $\Pioneactive$ (again, with the same new state space):
\begin{align*}
    \wt{\Pioneactive} \coloneqq \crl*{\pi_b :  \pi_b(x_{(z,h)}) = \begin{cases}
        b[h] &\text{ if } z \in [0,1)\\
        0 &\text{ otherwise,}
\end{cases} \text{where}~ b \in \crl{0,1}^H}.
\end{align*}
While $\dimRL(\Pioneactive) = \Theta(H)$, it is easy to see that $\dimRL(\wt{\Pioneactive}) = 2^H$ since one can construct a $H$-layer deterministic tree using states in $[0,1)$ as every $(x,a)$ pair at layer $H$ will be reachable by $\wt{\Pioneactive}$. This is because $\wt{\Pioneactive}$ is a product class, so the spanning capacity is given by \pref{prop:product-spanning-cap}.

\section{Bibliographical Remarks}\label{sec:background-bib}
\paragraph{Interaction Protocols.}
Formalizing the notion of \emph{sample complexity} for RL under different interaction protocols dates back to the early work \cite{kearns1998finite, kearns1999approximate, kearns2002sparse}; see also Kakade's thesis \cite{kakade2003sample} for a historical exposition.
The generative model was introduced by Kearns, Mansour, and Ng \cite{kearns1999approximate}. 
The distinction between global/local simulation access is more contemporary: the terminology ``local'' seems to have originated in \cite{weisz2021query} (although many early algorithms for the generative model actually operate under the more restricted local simulation access setting \cite[e.g.,][]{kearns1999approximate}. Simulators are used broadly in practice: for example, they are used in robotics control tasks \cite{todorov2012mujoco, akkaya2019solving, tassa2018deepmind} and game playing \cite{bellemare2012investigating, silver2017mastering, silver2018general}. We refer the reader to Appendix A of \cite{mhammedi2024power} for an exemplary survey of related works on local simulators. Offline RL is first studied in \cite{ernst2005tree, riedmiller2005neural} under the name ``batch RL'', and we refer the reader to the excellent overview \cite{levine2020offline}.
The $\mu$-reset model was introduced by Kakade and Langford \cite{kakade2002approximately}.
The hybrid resets interaction protocol was introduced and studied in the author's work \cite{krishnamurthy2025role}. 

\paragraph{Coverage Conditions.} Coverage conditions have been extensively studied in RL. In offline RL, many works study the concentrability coefficient \cite[see, e.g.,][]{munos2003error,munos2007performance, munos2008finite, chen2019information, foster2021offline, jia2024offline} as well as weaker notions such as single-policy concentrability \cite{jin2021pessimism, rashidinejad2021bridging, zhan2022offline} and conditions based on value-function approximation \cite{chen2019information, xie2021bellman, cheng2022adversarially, wang2020statistical, zanette2021provable, li2022pessimism}. In addition, under the $\mu$-reset model, the standard assumption made is on bounded concentrability coefficient, sometimes called the \emph{distribution mismatch coefficient} in policy optimization literature \cite{agarwal2021theory}. Pushforward concentrability is first studied in \cite{xie2021batch}. More recently, \cite{xie2022role} introduced the notion of \emph{coverability coefficient} and study it for online RL access with value function approximation. Coverability (and the related pushforward variant) is further studied in the papers \cite{amortila2024scalable, amortila2024harnessing, amortila2024reinforcement, jia2023agnostic, mhammedi2024power}. Various models such as tabular MDPs, linear MDPs, low-rank MDPs, and exogenous MDPs are known to satisfy coverability.

\paragraph{Block MDPs.} Block MDPs are a canonical model for studying reinforcement learning with large state spaces but low intrinsic complexity. In particular, Block MDPs are known to satisfy low (pushforward) coverability \cite{mhammedi2024power}, implying that reset distributions exist which satisfy low (pushforward) concentrability. They have been studied in a long line of work \cite{jiang2017contextual, du2019provably, misra2020kinematic, zhang2022efficient, uehara2021representation, mhammedi2023representation}. Recently, \cite{amortila2024reinforcement} study a more general setting of RL with latent dynamics which covers the Block MDP as a special case. A common theme among these works is that standard online access to $M$ is assumed, and the assumption of \emph{decoder realizability} is made, i.e., that the learner is given access to a class $\Phi$ such that $\optdec \in \Phi$, with the achievable bounds scaling with $\log \abs{\Phi}$. Under standard online access, a minimax lower bound of $\log \abs{\Phi}$ can be obtained by reduction to supervised learning. In contrast, this thesis studies how to achieve sample-efficient learning without decoder realizability but with \emph{stronger forms of access} to $M$. Our bounds replace the dependence on $\log \abs{\Phi}$ (which in the worst case can scale with $\abs{\statesp}$) with dependence on $\log \abs{\Pi}$, which can be arbitrarily smaller. 

\chapter{Policy Eluder Dimension}\label{chap:eluder}
This chapter studies a candidate complexity measure for policy learning in RL called the \emph{policy eluder dimension.} Roughly speaking, the policy eluder dimension is a worst-case complexity measure that only depends on the policy class $\Pi$, and it is sequential in nature. After introducing the quantity in \pref{sec:intro-eluder}, we provide several new insights on the policy eluder dimension and its relationship to other learning theoretic quantities:
\begin{enumerate}
    \item In \pref{sec:qualitative-equivalence}, we elucidate a fundamental connection between the policy eluder dimension and two other well-studied learning theoretic quantities: (1) the \emph{star number}, a quantity which characterizes the minimax label complexity of pool-based active learning, and (2) the \emph{threshold dimension}, a quantity that characterizes the regret of online learning. Specifically, we show that finiteness of policy eluder dimension is equivalent to finiteness of both star number and threshold dimension. 
    \item In \pref{sec:eluder-vs-rank}, we show separations between policy eluder dimension and \emph{sign-rank}, a classical notion of dimension complexity which corresponds to the smallest dimension required to embed the input space so that all policies in $\Pi$ are realizable as halfspaces. 
    \item In \pref{sec:eluder-spanning}, we study the relationship between policy eluder dimension and spanning capacity, which was introduced in \pref{def:dimension}. 
\end{enumerate}
Ultimately, the rest of this thesis focuses on the coverage based complexity measures described in \pref{sec:coverage-conditions}. We include this presentation of the policy eluder dimension for several reasons: a huge body of literature (see Bibliographical Remarks in \pref{sec:eluder-bibliographic-remarks}) has proposed to study the sample complexity of RL in terms of the eluder dimension, so the results in this section help us understand those bounds. Furthermore, we believe that understanding the role of the complexity measures discussed herein (eluder dimension, star number, threshold dimension, and generalized rank) for sequential decision making problems is an exciting and fruitful direction for future work. 

\section{Preliminaries}\label{sec:intro-eluder}
In a very influential paper, \citet{russo2013eluder} introduced the notion of \emph{eluder dimension} for a real-valued function class $\cF$ and used it to analyze the regret of algorithms (based on the Upper Confidence Bound (UCB) and Thompson Sampling paradigms) for the multi-armed bandit problem with function approximation. Informally speaking, the eluder dimension characterizes the longest sequence of {\em adversarially chosen} points one must observe in order to accurately estimate the function value at any other point.
We state a variant of the original definition, proposed by \citet{foster2020instance}, that is never larger and is sufficient to analyze all the applications of eluder dimension in the literature.\footnote{The main difference is that the original definition of \cite{russo2013eluder} asks for witnessing \emph{pairs} of functions $f_i, f_i' \in \cF$, while the presented definition restricts $f_i' = f^\star$. We refer the reader to \cite{foster2020instance} for a more detailed discussion.}

\begin{definition}[Eluder Dimension]\label{def:eluder}
For any function class $\cF\subseteq \bbR^\cX$, $f^\star: \cX \to \bbR$, and scale $\eps \ge 0$, the \emph{exact eluder dimension} $\eEdim_{f^\star}(\cF, \eps)$ is the largest $m\in \bbN$ such that there exists $(x_1, f_1),\dots, (x_m, f_m) \in \cX\times \cF$ satisfying for all $i \in [m]$:
\begin{equation}\label{eq:eluder-def}
\abs*{f_i(x_i) - f^{\star}(x_i)} > \eps, \quad \text{and} \quad \sum_{j < i}~\prn*{f_i(x_j) - f^{\star}(x_j)}^2 \le \eps^2.
\end{equation}
The \emph{eluder dimension (with respect to $f^\star$)} is defined as $\Edim_{f^\star}(\cF, \eps) \coloneqq \sup_{\eps' \ge \eps} \eEdim_{f^\star}(\cF, \eps')$, and we also denote $\Edim(\cF, \eps) \coloneqq \sup_{f^\star \in \cF}\Edim_{f^\star}(\cF, \eps).$
\end{definition}

\pref{def:eluder} is a scale-sensitive notion of complexity for real-valued function classes, as it depends on the scale parameter $\eps \ge 0$. In agnostic policy learning, we have a policy class which maps the state space $\statesp$ to a finite action set $\actionsp$. Therefore, we specialize the definition to a scale-insensitive/combinatorial variant of eluder dimension called the \emph{policy eluder dimension} \cite{mou2020sample, foster2020instance}. For clarity of presentation, in the rest of this chapter we focus on the binary action setting (with $\actionsp = \crl{-1, +1}$) where the function class $\Pi \subseteq \crl{-1, +1}^\statesp$.

\begin{definition}[Policy Eluder Dimension]\label{def:policy-eluder}
    For any policy class $\Pi \subseteq \crl{-1, +1}^\statesp$, policy $\optpi : \statesp \to \crl{-1, +1}$, we define the \emph{policy eluder dimension} (with respect to $\optpi$), denoted $\Edim_{\optpi}(\Pi)$, as the largest $m \in \bbN$ such that there exists $(x_1, \pi_1),\dots, (x_m, \pi_m) \in \statesp \times \Pi$ satisfying for all $i\in[m]$:
\begin{align*}
    \pi_i(x_i) \ne \optpi(x_i), \quad \text{and}\quad \text{for all} \ j<i: \quad \pi_i(x_j) = \optpi(x_j).
\end{align*}
We also denote $\Edim(\Pi) \coloneqq \sup_{\optpi \in \Pi} \Edim_{\optpi}(\Pi)$.
\end{definition}
 
\section{A Qualitative Equivalence}\label{sec:qualitative-equivalence}

In this section, we relate the policy eluder dimension to two other learning-theoretic quantities, the star number and the threshold dimension. First we provide definitions, and we refer the reader to \pref{fig:eluder-star-thresh} for an illustration of the differences between these complexity measures.

\begin{definition}[Star Number]\label{def:star-number}
    For any policy class $\Pi \subseteq \crl{-1, +1}^\statesp$, policy $\optpi : \statesp \to \crl{-1, +1}$, we define the \emph{star number} (with respect to $\optpi$), denoted $\Sdim_{\optpi}(\Pi)$, as the largest $m \in \bbN$ such that there exists $(x_1, \pi_1),\dots, (x_m, \pi_m) \in \statesp \times \Pi$ satisfying for all $i\in[m]$:
\begin{align*}
    \pi_i(x_i) \ne \optpi(x_i), \quad \text{and}\quad \text{for all} \ j \ne i: \quad \pi_i(x_j) = \optpi(x_j).
\end{align*}
We also denote $\Sdim(\Pi) \coloneqq \sup_{\optpi \in \Pi} \Sdim_{\optpi}(\Pi)$.
\end{definition}
\pref{def:star-number} comes from \cite{hanneke2015minimax}, who give tight upper and lower bounds on the label complexity of \emph{pool-based active learning} via the star number $\Sdim(\Pi)$ and show that almost every previously proposed complexity measure for active learning takes a worst case value equal to the star number. Roughly speaking, the star number corresponds to the number of ``singletons'' one can embed in a function class; that is, the maximum number of functions that differ from a base function $\optpi$ at exactly one point among a subset of the domain $\{x_1,\dots, x_m\} \subseteq \cX$. 

\begin{definition}[Threshold Dimension]\label{def:threshold-dim}
    For any policy class $\Pi \subseteq \crl{-1, +1}^\statesp$, policy $\optpi : \statesp \to \crl{-1, +1}$, we define the \emph{threshold dimension} (with respect to $\optpi$), denoted $\Tdim_{\optpi}(\Pi)$, as the largest $m \in \bbN$ such that there exists $(x_1, \pi_1),\dots, (x_m, \pi_m) \in \statesp \times \Pi$ satisfying for all $i\in[m]$:
\begin{align*}
    \text{for all}~ k \ge i: \quad \pi_k(x_k) \ne \optpi(x_k), \quad \text{and}\quad \text{for all} \ j < i: \quad \pi_i(x_j) = \optpi(x_j).
\end{align*}
We also denote $\Tdim(\Pi) \coloneqq \sup_{\optpi \in \Pi} \Tdim_{\optpi}(\Pi)$.
\end{definition}
The threshold dimension was introduced by \citet{hodges1997shorter}, and has gained attention due to its role in proving an equivalence relationship between private PAC learning and online learning \citep[see, e.g.,][]{alon2019private, bun2020equivalence}. We slightly generalize the definition of \cite{alon2019private} to allow for any base function $\optpi$, in the spirit of the other definitions. A classical result in model theory provides a link between the threshold dimension and Littlestone dimension (which we denote $\mathsf{Ldim}$), a quantity which is both necessary and sufficient for online learnability \cite{ben2009agnostic, alon2021adversarial}. In particular, results by \citet{shelah1990classification} and \citet{hodges1997shorter} show that for any binary-valued $\Pi$ and any $\optpi \in \Pi$:
\begin{align*}
    \floor{\log \Tdim_{\optpi}(\Pi)} \le \mathsf{Ldim} \le 2^{\Tdim_{\optpi}(\Pi)}.
\end{align*}
A combinatorial proof of this result can be found in Thm.~3 of \cite{alon2019private} (they prove the result for $\optpi \equiv -1$, but it is easy to extend their proof to hold for any $\optpi$). Thus, finiteness of threshold dimension is necessary and sufficient for online learnability (albeit in a weaker, ``qualitative'' sense).

\begin{figure}[t!]
\begin{center}
\newcommand{\drawgrid}{%
	\foreach \i in {0,...,6} {
		\draw[black] (\mx+\i*\xgap-0.5*\xgap, \my+0.5*\ygap) -- (\mx+\i*\xgap-0.5*\xgap, \my-5.5*\ygap);
		\draw[black] (\mx-0.5*\xgap,\my-\i*\ygap+0.5*\ygap) -- (\mx+5.5*\xgap, \my-\i*\ygap+0.5*\ygap);
	}
	\foreach \i in {1,...,6} {
		\node at (\mx-\xgap, \my-\i*\ygap+\ygap) {\scriptsize$x_{\i}$};
		\node at (\mx+\i*\xgap-\xgap, \my+\ygap) {\scriptsize$\pi_{\i}$};
	}
}

\begin{tikzpicture}
	\def \xgap{0.45}
	\def \ygap{0.45}
	
	\def \mx{0}
	\def \my{0}
	\drawgrid
	\node at (\mx+0*\xgap, \my-0*\ygap) {1};
	\node at (\mx+0*\xgap, \my-1*\ygap) {*};
	\node at (\mx+0*\xgap, \my-2*\ygap) {*};
	\node at (\mx+0*\xgap, \my-3*\ygap) {*};
	\node at (\mx+0*\xgap, \my-4*\ygap) {*};
	\node at (\mx+0*\xgap, \my-5*\ygap) {*};
	\node at (\mx+1*\xgap, \my-0*\ygap) {0};
	\node at (\mx+1*\xgap, \my-1*\ygap) {1};
	\node at (\mx+1*\xgap, \my-2*\ygap) {*};
	\node at (\mx+1*\xgap, \my-3*\ygap) {*};
	\node at (\mx+1*\xgap, \my-4*\ygap) {*};
	\node at (\mx+1*\xgap, \my-5*\ygap) {*};
	\node at (\mx+2*\xgap, \my-0*\ygap) {0};
	\node at (\mx+2*\xgap, \my-1*\ygap) {0};
	\node at (\mx+2*\xgap, \my-2*\ygap) {1};
	\node at (\mx+2*\xgap, \my-3*\ygap) {*};
	\node at (\mx+2*\xgap, \my-4*\ygap) {*};
	\node at (\mx+2*\xgap, \my-5*\ygap) {*};
	\node at (\mx+3*\xgap, \my-0*\ygap) {0};
	\node at (\mx+3*\xgap, \my-1*\ygap) {0};
	\node at (\mx+3*\xgap, \my-2*\ygap) {0};
	\node at (\mx+3*\xgap, \my-3*\ygap) {1};
	\node at (\mx+3*\xgap, \my-4*\ygap) {*};
	\node at (\mx+3*\xgap, \my-5*\ygap) {*};
	\node at (\mx+4*\xgap, \my-0*\ygap) {0};
	\node at (\mx+4*\xgap, \my-1*\ygap) {0};
	\node at (\mx+4*\xgap, \my-2*\ygap) {0};
	\node at (\mx+4*\xgap, \my-3*\ygap) {0};
	\node at (\mx+4*\xgap, \my-4*\ygap) {1};
	\node at (\mx+4*\xgap, \my-5*\ygap) {*};
	\node at (\mx+5*\xgap, \my-0*\ygap) {0};
	\node at (\mx+5*\xgap, \my-1*\ygap) {0};
	\node at (\mx+5*\xgap, \my-2*\ygap) {0};
	\node at (\mx+5*\xgap, \my-3*\ygap) {0};
	\node at (\mx+5*\xgap, \my-4*\ygap) {0};
	\node at (\mx+5*\xgap, \my-5*\ygap) {1};
	
	\node at (\mx+2.5*\xgap, \my-6.5*\ygap) {Eluder sequence};
	
	\def \mx{4.5}
	\drawgrid
	\node at (\mx+0*\xgap, \my-0*\ygap) {1};
	\node at (\mx+0*\xgap, \my-1*\ygap) {0};
	\node at (\mx+0*\xgap, \my-2*\ygap) {0};
	\node at (\mx+0*\xgap, \my-3*\ygap) {0};
	\node at (\mx+0*\xgap, \my-4*\ygap) {0};
	\node at (\mx+0*\xgap, \my-5*\ygap) {0};
	\node at (\mx+1*\xgap, \my-0*\ygap) {0};
	\node at (\mx+1*\xgap, \my-1*\ygap) {1};
	\node at (\mx+1*\xgap, \my-2*\ygap) {0};
	\node at (\mx+1*\xgap, \my-3*\ygap) {0};
	\node at (\mx+1*\xgap, \my-4*\ygap) {0};
	\node at (\mx+1*\xgap, \my-5*\ygap) {0};
	\node at (\mx+2*\xgap, \my-0*\ygap) {0};
	\node at (\mx+2*\xgap, \my-1*\ygap) {0};
	\node at (\mx+2*\xgap, \my-2*\ygap) {1};
	\node at (\mx+2*\xgap, \my-3*\ygap) {0};
	\node at (\mx+2*\xgap, \my-4*\ygap) {0};
	\node at (\mx+2*\xgap, \my-5*\ygap) {0};
	\node at (\mx+3*\xgap, \my-0*\ygap) {0};
	\node at (\mx+3*\xgap, \my-1*\ygap) {0};
	\node at (\mx+3*\xgap, \my-2*\ygap) {0};
	\node at (\mx+3*\xgap, \my-3*\ygap) {1};
	\node at (\mx+3*\xgap, \my-4*\ygap) {0};
	\node at (\mx+3*\xgap, \my-5*\ygap) {0};
	\node at (\mx+4*\xgap, \my-0*\ygap) {0};
	\node at (\mx+4*\xgap, \my-1*\ygap) {0};
	\node at (\mx+4*\xgap, \my-2*\ygap) {0};
	\node at (\mx+4*\xgap, \my-3*\ygap) {0};
	\node at (\mx+4*\xgap, \my-4*\ygap) {1};
	\node at (\mx+4*\xgap, \my-5*\ygap) {0};
	\node at (\mx+5*\xgap, \my-0*\ygap) {0};
	\node at (\mx+5*\xgap, \my-1*\ygap) {0};
	\node at (\mx+5*\xgap, \my-2*\ygap) {0};
	\node at (\mx+5*\xgap, \my-3*\ygap) {0};
	\node at (\mx+5*\xgap, \my-4*\ygap) {0};
	\node at (\mx+5*\xgap, \my-5*\ygap) {1};
	\node at (\mx+2.5*\xgap, \my-6.5*\ygap) {Star sequence};
	
	\def \mx{9}
	\drawgrid
	\node at (\mx+0*\xgap, \my-0*\ygap) {1};
	\node at (\mx+0*\xgap, \my-1*\ygap) {1};
	\node at (\mx+0*\xgap, \my-2*\ygap) {1};
	\node at (\mx+0*\xgap, \my-3*\ygap) {1};
	\node at (\mx+0*\xgap, \my-4*\ygap) {1};
	\node at (\mx+0*\xgap, \my-5*\ygap) {1};
	\node at (\mx+1*\xgap, \my-0*\ygap) {0};
	\node at (\mx+1*\xgap, \my-1*\ygap) {1};
	\node at (\mx+1*\xgap, \my-2*\ygap) {1};
	\node at (\mx+1*\xgap, \my-3*\ygap) {1};
	\node at (\mx+1*\xgap, \my-4*\ygap) {1};
	\node at (\mx+1*\xgap, \my-5*\ygap) {1};
	\node at (\mx+2*\xgap, \my-0*\ygap) {0};
	\node at (\mx+2*\xgap, \my-1*\ygap) {0};
	\node at (\mx+2*\xgap, \my-2*\ygap) {1};
	\node at (\mx+2*\xgap, \my-3*\ygap) {1};
	\node at (\mx+2*\xgap, \my-4*\ygap) {1};
	\node at (\mx+2*\xgap, \my-5*\ygap) {1};
	\node at (\mx+3*\xgap, \my-0*\ygap) {0};
	\node at (\mx+3*\xgap, \my-1*\ygap) {0};
	\node at (\mx+3*\xgap, \my-2*\ygap) {0};
	\node at (\mx+3*\xgap, \my-3*\ygap) {1};
	\node at (\mx+3*\xgap, \my-4*\ygap) {1};
	\node at (\mx+3*\xgap, \my-5*\ygap) {1};
	\node at (\mx+4*\xgap, \my-0*\ygap) {0};
	\node at (\mx+4*\xgap, \my-1*\ygap) {0};
	\node at (\mx+4*\xgap, \my-2*\ygap) {0};
	\node at (\mx+4*\xgap, \my-3*\ygap) {0};
	\node at (\mx+4*\xgap, \my-4*\ygap) {1};
	\node at (\mx+4*\xgap, \my-5*\ygap) {1};
	\node at (\mx+5*\xgap, \my-0*\ygap) {0};
	\node at (\mx+5*\xgap, \my-1*\ygap) {0};
	\node at (\mx+5*\xgap, \my-2*\ygap) {0};
	\node at (\mx+5*\xgap, \my-3*\ygap) {0};
	\node at (\mx+5*\xgap, \my-4*\ygap) {0};
	\node at (\mx+5*\xgap, \my-5*\ygap) {1};
	\node at (\mx+2.5*\xgap, \my-6.5*\ygap) {Threshold sequence};
\end{tikzpicture}

%
%
%
%
\caption{Illustration of witnessing sequences of length $6$ for policy eluder dimension, star number and threshold dimension with respect to $\optpi \equiv 0$ (we use the $0$/$1$ representation of policies for clarity). `*' in the eluder witness sequence refers to a free value, either 0 or 1.}
\label{fig:eluder-star-thresh}
\end{center}
\end{figure}

\paragraph{Main Result.} Now we are ready to state our main result, which says that for any binary-valued policy class, finiteness of the policy eluder dimension is \emph{equivalent} to finiteness of both star number and threshold dimension. 

\begin{theorem}\label{thm:equivalence}
    For any policy class $\Pi \subseteq \crl{-1, +1}^\statesp$ and any $\optpi \in \Pi$, the following holds:
\begin{align*}
    \max\{\Sdim_{\optpi}(\Pi), \Tdim_{\optpi}(\Pi) \} \le \Edim_{\optpi}(\Pi) \le 4^{\max\{ \Sdim_{\optpi}(\Pi), \Tdim_{\optpi}(\Pi) \} }.
\end{align*}
\end{theorem}

The proof of \pref{thm:equivalence} is deferred to \pref{sec:thm-equivalence}. 

\pref{thm:equivalence} has an exponential gap between the upper and lower bounds. Can we improve either of the inequalities? The lower bound cannot be improved, by considering the simple examples over $\statesp = [n]$ of the singleton class $\Pi_\mathrm{sing}\coloneqq \{x \mapsto \ind{x = i} \mid i \in [n+1]\}$ and the threshold class $\Pi_\mathrm{thresh} \coloneqq \{x \mapsto \ind{x \ge i} \mid i \in [n+1]\}$. We also show that the upper bound cannot be improved, for example, to $\Edim(\Pi) \le \mathrm{poly}(\Sdim(\Pi), \Tdim(\Pi))$ in general.

\begin{theorem}\label{thm:tightness-ub}
For every $N > 0$, there exists a function class $\Pi_N$ such that $\Edim(\Pi_N) = N$ and $\max\{\Sdim(\Pi_N), \Tdim(\Pi_N)\} < c \cdot \log_2 N$, where $c> 1/2$ is some absolute numerical constant.
\end{theorem}

The proof of \pref{thm:tightness-ub} is found in \pref{sec:proof-tightness-ub}.

\subsection{Proof of \pref{thm:equivalence}}\label{sec:thm-equivalence}
\paragraph{Proof of Lower Bound.} The lower bound is straightforward from the definition, since any sequence of $(x_1,\pi_1), \dots, (x_m, \pi_m)$ that witness $\Sdim_{\optpi}(\Pi) = m$ or $\Tdim_{\optpi}(\Pi) = m$ must also be valid ``eluder sequences''; so the eluder dimension can only be larger.

\paragraph{Proof of Upper Bound.} Fix $\Pi$ and $\optpi \in \Pi$. Let $(x_1, \pi_1), \dots, (x_m, \pi_m) \in \statesp \times \Pi$ be the sequence which witnesses $\Edim_{\optpi}(\Pi) = m$. To prove the bound, we will show that there exists a subset of size at least $k \ge \log_4 m$ which witnesses $\Sdim_{\optpi} = k$ or $\Tdim_{\optpi} = k$.

This follows from a connection to Ramsey theory. A visualization of the proof is depicted in \pref{fig:proof-illustration}. Recall that diagonal Ramsey number $R(k,k)$ is defined as the smallest $m$ such that every red-blue labeling of the edges of the graph $K_m$ contains a monochromatic subgraph $K_k$.

In addition, define $E(k,k)$ as the smallest $m$ such that any eluder sequence $(x_1, \pi_1), \dots, (x_m, \pi_m)$ contains a subsequence $(x_{i_1}, \pi_{i_1}), \dots, (x_{i_k}, \pi_{i_k})$ which witnesses $\Sdim_{\optpi}(\Pi) \ge k$ or $\Tdim_{\optpi}(\Pi) \ge k$. We claim that $E(k,k) \le R(k,k)$.

To see this, note that there exists a bijection between colorings of $K_m$ and eluder sequences. Every eluder sequence $(x_1, \pi_1), \dots, (x_m, \pi_m)$ can be used to construct a red-blue coloring of $K_m$ as follows. For every edge $e_{ij}$ for $i > j \in [m]$, we color it red if $\pi_j(x_i) = \optpi(x_i)$ and blue otherwise. Observe that if there exists a subsequence $(x_{i_1}, \pi_{i_1}), \dots, (x_{i_k}, \pi_{i_k})$ which witnesses $\Sdim_{\optpi} \ge k$, then the subgraph comprised of the vertices $i_1, \dots, i_k$ in the coloring of $K_m$ must be monochromatic red. Likewise, if a subsequence witnesses $\Tdim_{\optpi} \ge k$, then the subgraph must be monochromatic blue. Thus, if $m$ is such that if every coloring of $K_m$ induces a monochromatic coloring $K_k$, then for any eluder sequence, we can always find a subsequence that witnesses $\Sdim_{\optpi} \ge k$ or $\Tdim_{\optpi} \ge k$. This shows that $E(k,k) \le R(k,k)$.

The proof concludes by applying the classical bound $R(k,k) \le 4^{k}$ \cite[see, e.g.,][]{mubayi2017survey}. 

This completes the proof of \pref{thm:equivalence}.\qed

\begin{figure}[t]
\centering
\subfigure[Eluder matrix]{
\includegraphics[scale=0.45, trim={0cm 27cm 55cm 0cm},clip]{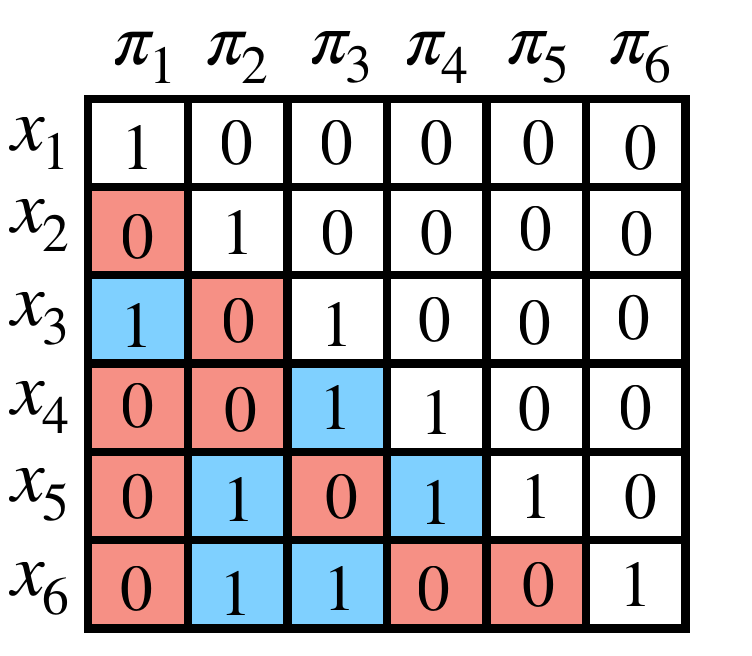}
}
\hspace{1em}
\subfigure[Ramsey graph]{
\includegraphics[scale=0.45, trim={0cm 28cm 55cm 0cm},clip]{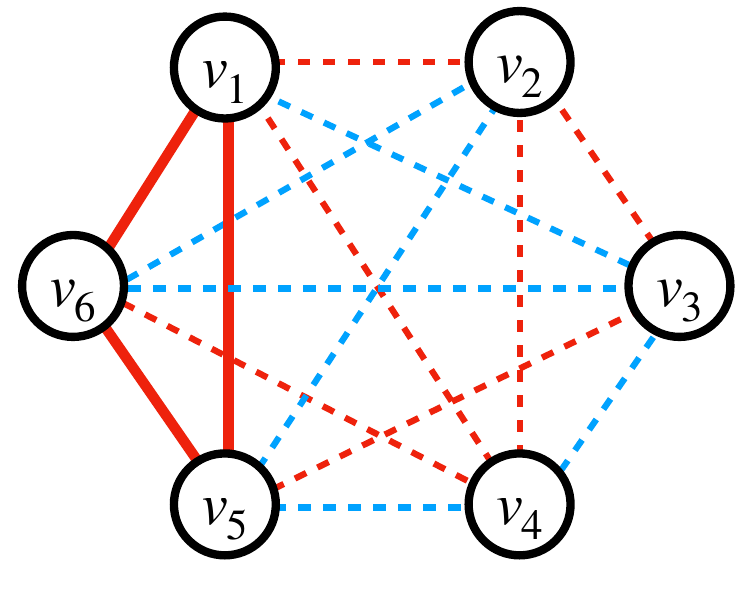}
}
\caption{An example illustrating the connection between the upper bound and Ramsey numbers. Left: a sequence $\{(x_1, \pi_1), \dots, (x_6, \pi_6)\}$ witnessing $\Edim_0(\Pi) = 6$, represented in matrix form. (We switch to 0/1-valued function classes for clarity.) Right: In the corresponding graph $K_6$, we color the graph edges $e_{ij}$ to be \textcolor{red}{red} if $\pi_j(x_i) = 0$ and \textcolor{blue}{blue} if $\pi_j(x_i) = 1$. Since $R(3,3) = 6$, we are guaranteed a subgraph $K_3$ which is monochromatic; in this example, the subgraph is given by the vertices $\{v_1, v_5, v_6\}$. Red subgraphs indicate sequences which witness $\Sdim_0(\Pi)$; blue subgraphs witness $\Tdim_0(\Pi)$. In this case, $\{(x_1, \pi_1), (x_5, \pi_5), (x_6, \pi_6)\}$ witnesses $\Sdim_0(\Pi) \ge 3$.}
\label{fig:proof-illustration}
\end{figure}

\subsection{Proof of \pref{thm:tightness-ub}}\label{sec:proof-tightness-ub}
We will construct $\Pi_N \subseteq \crl{-1,+1}^{[N]}$ randomly, such that $\abs{\Pi_N} = N+1$ and $\Edim_{1}(\Pi_N) = N$. Note that it is equivalent to define an $N\times (N+1)$ sign matrix $B$, representing the values of $\statesp \times \Pi$ with entry $B_{ij} = \pi_j(x_i)$. Let $B$ be randomly drawn according to the following distribution:
\begin{align*}
    B_{ij} \sim \begin{cases}
    1 & i < j,\\
    -1 & i = j,\\
    \mathrm{Rad}(1/2) & i > j.
    \end{cases}
\end{align*}
By construction, with probability 1, $(x_1, \pi_1), \dots, (x_N, \pi_N)$ is a valid sequence witnessing $\Edim(\Pi_N) = \Edim_1(\Pi_N) = N$.

We now have to argue that there exists some sign matrix $C \in \{-1,+1\}^{N\times (N+1)}$ such that the equivalent function class $\Pi_N$ has small threshold dimension and star number. We use the following two lemmas to simplify the requirement that $\Pi_N$ have small threshold dimension and star number with respect to all base functions $\optpi\in \Pi$ to just considering the base functions $\optpi \equiv -1$ and $\optpi \equiv 1$.

\begin{lemma}\label{lem:dim-ub}
For any $\Pi$ and $\dim\in \{\Edim, \Tdim, \Sdim\}$, we have  $\dim(\Pi) \le \dim_{1}(\Pi) + \dim_{-1}(\Pi)$. 
\end{lemma}

\begin{lemma}\label{lem:thres-ub}
For any $\Pi$, $\Tdim(\Pi) \le 2 \Tdim_{1}(\Pi)$.
\end{lemma}

We set up some additional notation. Denote ${\bf I} = (i_1, i_2, \dots, i_k)$ and ${\bf J} = (j_1, j_2, \dots, j_k)$ to be $k$-length sequences of distinct elements from $[N]$. For any two sequences ${\bf I}$, ${\bf J}$, we use $(x_{\bf I}, \pi_{\bf J}) \coloneqq ((x_{i_1}, \pi_{j_1}), \dots, (x_{i_k}, \pi_{j_k}))$. We define valid star sequences to be any $(x_{\bf I}, \pi_{\bf J})$ which witness $\Sdim_1(\Pi_N) = k$ or $\Sdim_{-1}(\Pi_N) = k$, and we define valid threshold sequences to be any $(x_{\bf I}, \pi_{\bf J})$ which witnesses $\Tdim_1(\Pi_N) = k$.

Define the random variable $X_k$ to be the number of valid star or threshold sequences, i.e.,~$X_k \coloneqq \abs{ \crl{ ({\bf I}, {\bf J}) \colon (x_{\bf I}, \pi_{\bf J}) \text{ is valid star or threshold sequence} } }.$ By linearity of expectation, we have
\begin{align*}
    \En[X_k] &\le N^{k}(N+1)^{k} \cdot \max_{ {\bf I}, {\bf J} \subset [N] } p_{\bf{I}, \bf{J}}, \\
    &\qquad \text{where } p_{\bf{I}, \bf{J}} \coloneqq \mathbb{P}\left[ (x_{\bf I}, \pi_{\bf J}) \text{ is a valid star or threshold sequence} \right].
\end{align*}
Now we apply the following lemma to upper bound the expectation.

\begin{lemma}\label{lem:prob-upper-bound}
For all $\bf{I}$, $\bf{J}$, we have $ p_{\bf{I}, \bf{J}} \le 3 \cdot 2^{-k(k-1)/2}$.
\end{lemma}

We apply \pref{lem:prob-upper-bound} to the previous display. When $k = \Omega(\log_2 N)$, we have $\En[X_k] < 1$. Therefore, there must exist an $N\times (N+1)$-sized sign matrix $C$ such that the corresponding $\Pi_N \subseteq \crl{-1,+1}^{[N]}$ has $\Sdim_1(\Pi_N) < O(\log_2 N)$, $\Sdim_{-1}(\Pi_N) < O(\log_2 N)$, and $\Tdim_1(\Pi_N) < O(\log_2 N)$, but $\Edim_{1}(\Pi_N) = N$. By \pref{lem:dim-ub} and \ref{lem:thres-ub}, this concludes the proof of \pref{thm:tightness-ub}.\qed

It remains to prove the auxiliary lemmas used in the proof of \pref{thm:tightness-ub}.

\begin{proof}[Proof of \pref{lem:dim-ub}]
We prove the result for star number; the result for eluder dimension and threshold dimension can be shown with a similar argument. Fix any $\optpi \in \Pi$, and let $(x_1, \pi_1), \dots, (x_m, \pi_m)$ denote the sequence which witnesses $\Sdim_{\optpi}(\Pi) = m$. Let $I_{+} \subseteq [m]$ denote the indices $i$ for which $\optpi(x_i) = 1$ and $I_{-} \subseteq [m]$ denote the indices $i$ for which $\optpi(x_i) = -1$. By definition of star number, $\pi_j(x_j) = -1$ for every $j\in I_{+}$ and $\pi_k(x_j) = \optpi(x_j) = 1$ for every $j\ne k \in I_{+}$. Thus we know that $\{(x_k, \pi_k): k\in I_{+}\}$ is a valid sequence which witnesses $\Sdim_{1}(\Pi) \ge \abs{I_{+}}$. Similarly $\{(x_k, \pi_k): k\in I_{-}\}$ is a valid sequence which witnesses $\Sdim_{-1}(\Pi) \ge \abs{I_{-}}$. Thus we have shown that $\Sdim_{\optpi}(\Pi) = m \le \Sdim_{1}(\Pi) + \Sdim_{-1}(\Pi)$. Taking the supremum on the LHS yields the claim.
\end{proof}
\begin{proof}[Proof of \pref{lem:thres-ub}]
 Fix any $\optpi \in \Pi$, and let $(x_1, \pi_1), \dots, (x_m, \pi_m)$ denote the sequence which witnesses $\Tdim_{\optpi}(\Pi) = m$. Again let $I_{+} \subseteq [m]$ denote the indices $i$ for which $\optpi(x_i) = 1$ and $I_{-} \subseteq [m]$ denote the indices $i$ for which $\optpi(x_i) = -1$. Either $\abs{I_{+}} \ge m/2$ or $\abs{I_{-}} \ge m/2$. We break into cases.

\textit{Case 1.} If $\abs{I_{+}} \ge m/2$, then taking the subsequence indexed by $I_{+}$ already shows that $\Tdim_1(\Pi) \ge m/2 = \Tdim_{\optpi}(\Pi)/2$, and we are done.

\textit{Case 2.} If $\abs{I_{-}} \ge m/2$, then let us consider the subsequence indexed by $I_{-}$. We reindex it to call it $(x_1, \pi_1), \dots, (x_k, \pi_k)$, where $k = \abs{I_{-}}$. Observe that the sequence
\begin{align*}
(x_k, \optpi), (x_{k-1}, \pi_k), (x_{k-2}, \pi_{k-1}), \dots, (x_1, \pi_2)
\end{align*}
witnesses $\Tdim_1(\Pi) \ge k \ge m/2 = \Tdim_{\optpi}(\Pi)/2$.

Thus in both cases we have shown that $\Tdim_1(\Pi) \ge \Tdim_{\optpi}(\Pi)/2$; taking the supremum yields the claim.
\end{proof}

\begin{proof}[Proof of \pref{lem:prob-upper-bound}]
 Fix any $\bf{I}$, $\bf{J}$ to be $k$-length subsequences of $[N]$. In order for $(x_{\bf I}, \pi_{\bf J})$ to be a valid star sequence w.r.t.~$\optpi(x) = 1$, the following properties of the matrix $B$ must hold:

\begin{enumerate}
    \item For every $r \in [k]$, $B_{i_r, j_r} = -1$.
    \item For every $r, s \in [k]$ such that $r\ne s$, $B_{i_r, j_s} = 1$.
\end{enumerate}

 In order for $(x_{\bf I}, \pi_{\bf J})$ to be a valid star sequence w.r.t.~$\optpi(x) = -1$, we just flip the values in the above two properties.

Likewise, in order for $(x_{\bf I}, \pi_{\bf J})$ to be a valid threshold sequence w.r.t.~$\optpi(x)=1$, the following properties of the matrix $B$ must hold:
\begin{enumerate}
    \item For every $r \in [k]$, $B_{i_r, j_r} = -1$.
    \item For every $r, s \in [k]$ such that $r < s$, $B_{i_r, j_s} = 1$.
    \item For every $r, s \in [k]$ such that $r \ge s$, $B_{i_r, j_s} = -1$.
\end{enumerate}

First, we will prove that the probability that $(x_{\bf I}, \pi_{\bf J})$ is a valid star sequence w.r.t.~$\optpi(x)=1$, as well as the probability that $(x_{\bf I}, \pi_{\bf J})$ is a valid threshold sequence w.r.t.~$\optpi(x)=1$ are both $2^{-k(k-1)/2}$. For any $r \in [k]$, if $i_r <  j_r$, then by construction of $B$ we know that $\pi_{j_r}(x_{i_r}) = 1$, so $(x_{\bf I}, \pi_{\bf J})$ cannot be a valid star sequence or threshold sequence w.r.t~$\optpi(x)=1$. Henceforth, assume $i_r \ge j_r$ for all $r\in[k]$. Now define indices $r_1, r_2, \dots, r_k$ as the permutation of $[k]$ such that $i_{r_1}> i_{r_2}> \dots > i_{r_k}$. For any $p < q$, we have $i_{r_p} > i_{r_q} \ge j_{r_q}$. Thus for every $p < q$, the corresponding entry $(i_{r_p}, j_{r_q})$ is sampled from $\mathrm{Rad}(1/2)$. In order for $(x_{\bf I}, \pi_{\bf J})$ to be a valid star sequence, all of these must take the value of 1; likewise in order for  $(x_{\bf I}, \pi_{\bf J})$ to be a valid threshold sequence, all of these must take the value of $-1$. Since there are $k(k-1)/2$ of these, we have the desired result.

Now we bound the probability that $(x_{\bf I}, \pi_{\bf J})$ is a valid star sequence w.r.t.~$\optpi(x)=-1$. Note that because the definition of star number is permutation-invariant, we can assume that $i_1 > i_2 > \dots > i_k$ without loss of generality. Consider the pair $(i_1, j_1)$. We require $B_{i_1, j_1} = 1$, so either $i_1 > j_1$ or $i_1 < j_1$. Since we require $B_{i_2, j_1} = -1$, we cannot have $i_1 < j_1$, so we must have $i_1 > j_1$. Using a similar argument, we must have $i_r > j_r$ for all $r\in[k-1]$. Thus, there must be at least $k(k-1)/2$ random entries in the submatrix given by $({\bf I}, {\bf J})$, all of which must take value $-1$ in order for $(x_{\bf I}, \pi_{\bf J})$ is a valid star sequence w.r.t.~$\optpi(x)=-1$.
By union bound we get $p_{\bf{I}, \bf{J}} \le 3 \cdot 2^{-k(k-1)/2}$, thus proving the result.
\end{proof}

\section{Eluder Dimension vs.~Sign Rank}\label{sec:eluder-vs-rank}
In this section, we investigate the comparison of the policy eluder dimension with the sign rank of the policy class, which is a classical notion of \emph{dimension complexity} that has been studied extensively in combinatorics, learning theory, and communication complexity \citep[see, e.g.,][and references therein]{alon85geometrical, forster02upp, arriaga06algorithmic,alon2016sign}. Our main result is a separation between policy eluder and sign rank: namely there exists a function class with constant policy eluder dimension but infinite sign rank. We then explore the consequences of this result for sample complexity bounds in literature which utilize the scale-sensitive eluder dimension (\pref{def:eluder}).

\subsection{Definition and Main Result}
We state the definition of sign rank.
\begin{definition}\label{def:sign-rank}
    Let $\Pi \subseteq \crl{-1, +1}^\statesp$. We define the sign rank of $\Pi$, denoted $\signrank(\Pi)$, to be the smallest dimension $d$ for which there exists mappings $\phi: \statesp \to \bbR^d$ and $w: \Pi \to \bbR^d$ such that
    \begin{align*}
        \text{for all}~(x, \pi) \in \statesp \times \Pi:\quad \pi(x) = \sign(\tri{ w(\pi), \phi(x) } ), \numberthis \label{eq:sign-rank}
    \end{align*}
    or $\infty$ if no such $d$ exists.
\end{definition}

How does the policy eluder dimension of a class $\Pi$ relate to its sign rank? In one direction, there is an easy-to-see separation: namely, linear predictors in $\bbR^3$ (which trivially have $\signrank=3$) have infinite star number and threshold dimension (and therefore infinite policy eluder dimension by \pref{thm:equivalence}). However, we ask if the other separation is also possible: can we construct a $\Pi$ which has small policy eluder but large $\signrank$? 

We provide an explicit \emph{exponential} separation using the parity class.
\begin{proposition}\label{prop:parity-eluder-ub}
For $\statesp = \crl{-1,+1}^d$ and $\Fparity \coloneqq \crl{\pi_S : \pi_S(x) = \prod_{i\in S} x_i \mid S \subseteq [d]}$, it holds that $\signrank(\Fparity) \ge 2^{d/2}$, but $\Edim(\Fparity) \le d$.
\end{proposition}
\begin{proof}[Proof of \pref{prop:parity-eluder-ub}]
The first part is a well known result that $\signrank(\Fparity) \ge 2^{d/2}$ \citep{forster02upp}.

Now we show the second part. For any $x \in \crl{-1,+1}^d$ consider its $0$-$1$ representation $\wt{x} \in \bbF_2^d$ (representing $+1$ by $0$ and $-1$ by $1$). All functions in $\Fparity$ can be simply viewed as linear functions over $\bbF_2$. Namely, any parity function is indexed by a vector $a \in \bbF_2^d$, with $\pi_a(x) := (-1)^{\tri{a, \wt{x}}}$.

Fix any $\optpi \in \Fparity$. Suppose that $\Edim_{\optpi}(\Fparity) = m$, witnessed by $(x_1, \pi_{a_1}), \ldots, (x_m, \pi_{a_m}) \in \crl{-1,+1}^d$ and $\optpi = \pi_{a^{\star}}$. We have
\begin{align*}
    f_{a_{i}}(x_i) \ne \pi_{a^{\star}}(x_i),\quad \text{and} \quad f_{a_{i}}(x_j) = \pi_{a^{\star}}(x_j) \text{ for all } j < i.
\end{align*}
Thus, we have $\tri{a_{i}-a^{\star}, \wt{x}} = 0$ for all $\wt{x} \in \bbF_2\text{-}\mathrm{span}(\crl{\wt{x}_1, \ldots, \wt{x}_{i-1}})$.
But $\tri{a_{i}-a^{\star}, \wt{x}_i} = 1$ and hence $\wt{x}_i$ is linearly independent of $\crl{\wt{x}_1, \ldots, \wt{x}_{i-1}}$ over $\bbF_2^d$.
Thus, $\crl{\wt{x}_1, \ldots, \wt{x}_m}$ are all linearly independent over $\bbF_2^d$, and hence $m \le d$.
\end{proof}

We are also able to show stronger (but nonconstructive) separations by extending the probabilistic techniques of \cite{alon2016sign}, who provided similar separations for VC dimension versus $\signrank$.

\begin{theorem}\label{thm:separation} For every $N > 0$, there exists a policy class $\Pi_N \subseteq \crl{-1, +1}^{[N]}$ such that $\Edim_1(\Pi_N) = 4$ and $\signrank(\Pi_N) \ge \Omega(N^{1/9}/\log N)$, where $1$ is shorthand for the all 1s function.
\end{theorem}
The proof of \pref{thm:separation} is deferred to \pref{sec:proof-thm-separation}.
It is straightforward to replace the reference policy $\optpi \equiv 1$ with any fixed reference policy $\optpi: [N]\to \{-1,+1\}$. 

The careful reader might notice that we do not prove the existence of a policy class where $\Edim(\Pi)$ is constant and the $\signrank$ is infinite; instead we prove the weaker statement that a policy class exists with $\Edim$ w.r.t.~any \emph{fixed} policy $\optpi$ is bounded. We conjecture that the stronger statement holds; see \pref{sec:proof-thm-separation} for more details.

\subsection{Discussion}
We now provide some motivation and context for these separation results with sign rank.

In the paper \cite{russo2013eluder}, they established upper bounds on eluder dimension (\pref{def:eluder}) for (i)~function classes for which inputs have finite cardinality (the ``tabular'' setting), (ii)~linear functions over $\bbR^d$ of bounded norm, and (iii)~generalized linear functions over $\bbR^d$ of bounded norm, with any activation that has derivatives bounded away from $0$.
Apart from these function classes (and those that can be embedded into these), understanding of eluder dimension had been limited.
Indeed, one might wonder the following:
\begin{center}
    \textit{Are all function classes with bounded eluder dimension just generalized linear models?}
\end{center}
Answering this question has substantial ramifications on the scope of prior work. An answer of ``yes'' would imply that the results in the literature which give sample complexity or regret guarantees in terms of eluder dimension do not go beyond already-established regret guarantees for (generalized) linear settings~\citep[see, e.g.,][]{filippi10, li17, wang2019optimism, kveton2020randomized}. An answer of ``no'' can be construed as a positive result for RL theory, as it would indicate that existing (and future) results which use the eluder dimension apply to a richer set of function classes than generalized linear models.

To answer this question, the author's paper \cite{li2022understanding} first formally defines what it means for a function class to be written as a generalized linear model (GLM). Informally, for an activation $\sigma : \bbR \to \bbR$ and a function class $\cF \subseteq \bbR^\statesp$, we define the \textbf{\boldmath $\sigma$-rank} to be the smallest dimension $d$ needed to express every function in $\cF$ as a generalized linear function in $\bbR^d$ with activation $\sigma$. Paraphrasing Prop.~3 of \cite{li2022understanding}, when $\cF$ is $\crl{-1, +1}$-valued we have
\begin{align*}
    \signrank(\cF) \le \mathsf{id}\text{-}\dc(\cF),\quad \text{where}~\mathsf{id}(z) = z.
\end{align*}
A similar lower bound holds for nonlinear activations (i.e., GLMs).
Therefore, the separation results (\pref{prop:parity-eluder-ub} and \pref{thm:separation}) give examples where the answer to this question is ``no'', since the function classes described therein have small eluder dimension but cannot be written as GLMs with small dimension. Intuitively, the $\sigma$-rank captures the \emph{best possible} upper bound on eluder dimension that the results from \cite{russo2013eluder} can give for a given $\cF$ by treating it as a GLM with activation $\sigma$. We can conclude that sample complexity bounds via the eluder dimension are far more powerful than similar bounds for linear or GLM settings. 

\subsection{Proof of \pref{thm:separation}}\label{sec:proof-thm-separation}
For simplicity, let us consider only function classes of size $N$. It is equivalent to reason about $N\times N$ sign matrices which define the values that $\statesp \times \Pi$ take. We slightly abuse notation to define the $\signrank$ of an $N\times N$ matrix $S$ to be
\begin{align*}
    \signrank(S) \coloneqq \crl*{\mathrm{rank}(M) : M\in \bbR^{N\times N}, \ \sign(M_{ij}) = S_{ij} \text{ for all } i,j\in[N]}.
\end{align*}
We also define $\Edim_1(S)$ similarly: $\Edim_1(S)$ is the maximum $k$ such that we can find two $\bf{I}$, $\bf{J}$ which are $k$-length subsequences of $[N]$ such that the matrix $S$ restricted to $\bf{I}$, $\bf{J}$ has $-1$ on the diagonal and $+1$ above the diagonal.

The proof relies on the following key lemma, which bounds the number of matrices with $\signrank$ at most $r$.

\begin{lemma}[Lemma 22 of \cite{alon2016sign}]\label{lem:num-sign-matrices}
Let $r \le N/2$. The number of $N\times N$ sign matrices with sign rank at most $r$ does not exceed $2^{O(rN\log N)}$.
\end{lemma}

In order to prove the result, we use a probabilistic argument to show that there must exist many distinct $N\times N$ matrices with $\Edim_1 = 4$.

\begin{lemma}\label{lem:eluder-matrices-count}
The number of $N\times N$ sign matrices with $\Edim_1 \le 4$ is at least $2^{\Omega(N^{10/9})}$. 
\end{lemma}

The above lemmas imply that there must exist at least one sign matrix with $\Edim_1(S) \le 4$ and $\signrank(S) \ge \Omega(N^{1/9}/\log N)$. This proves \pref{thm:separation}, assuming \pref{lem:eluder-matrices-count} which we now prove.\qed

\begin{proof}[Proof of \pref{lem:eluder-matrices-count}]
 Define the set of $E_5$-light matrices as the set of $5\times 5$ sign matrices which are always $+1$ above the diagonal. We claim there exists an $N\times N$ sign matrix $C$ which (1) contains no $E_5$-light matrices and (2) has at least $\Omega(N^{10/9})$ entries that are $+1$. Such a matrix $C$ has $\Edim_1(C) \le 4$; moreover changing any $+1$ to $-1$ in $C$ will not increase $\Edim_1$.

Let $B$ be a random $N\times N$ sign matrix with each entry $+1$ with probability $p \coloneqq 1/(2N^{8/9})$. Define the random variable
\begin{align*}
    X \coloneqq \text{(\# $+1$'s in $C$)} - \text{(\# $E_5$-light matrices in $C$)}.
\end{align*}
Then $\En[X] \ge N^2 p - N^{10} p^{10} \ge \Omega(N^{10/9})$. Take some matrix with value of $X$ at least the expectation and change a $+1$ to a $-1$ in every $E_5$-light matrix to get $C$. Since this does not affect the value of $X$, we know that the resulting matrix $B$ has $\Omega(N^{10/9})$ entries that are $+1$.

Since changing any $+1$ to $-1$ in $C$ will not increase $\Edim_1$, we see that there are at least $2^{\Omega(N^{10/9})}$ distinct sign matrices with $\Edim \le 4$. This concludes the proof of \pref{lem:eluder-matrices-count}.
\end{proof}

\paragraph{A stronger separation?}
Our result does not fully show the separation between eluder dimension and $\signrank$. In the randomized construction used in \pref{lem:eluder-matrices-count}, it could be the case that the matrix $C$ we pick satisfies $\Edim_1(C) \le 4$, but there could be some other $\optpi$ (that is, another column of $C$) such that $\Edim_{\optpi}(C) = \omega(1)$. 
We conjecture that the stronger separation result also holds:

\begin{conjecture}
There exists absolute constants $k \in \bbN$ and $c > 0$ such that the following hold. For every $N > 0$, there exists a function class $\Pi_N \subseteq \crl{-1, +1}^{[N]}$ such that $\Edim(\Pi_N) \le k$ and $\signrank(\Pi_N) \ge \Omega(N^{c}/\log N)$.
\end{conjecture}

In light of \pref{thm:equivalence}, it suffices to show that there exists some function class $\Pi_N$ such that $\max\{\Sdim(\Pi_N), \Tdim(\Pi_N) \} \le k$ and $\signrank(\Pi_N) \ge \Omega(N^{c}/\log N)$.

Consider the easier problem of showing the separation for just threshold dimension. Here there is no difficulty. The key step is to apply \pref{lem:thres-ub} to reduce the problem to showing the result with respect to a \emph{single} $\optpi$. It follows as a corollary of \pref{thm:separation} that there exists a function class $\Pi_N$ such that $\Tdim(\Pi_N) \le 8$ and $\signrank(\Pi_N) \ge \Omega(N^{1/9}/\log N)$. (Using a more direct analysis, it is possible to improve the constants 8 and $1/9$.)

Showing the separation for star number (and eluder dimension) is a different story. We do not have a direct analogue of \pref{lem:thres-ub} for star number and eluder dimension. The weaker \pref{lem:dim-ub} allows us to reduce to showing the separation for two functions; but it is unclear how to leverage this reduction to extend the construction in the proof of \pref{lem:eluder-matrices-count}.

\section{Connection to Spanning Capacity}\label{sec:eluder-spanning}

Now we investigate the relationship between the spanning capacity and previously discussed complexity measures (policy eluder dimension, star number, and threshold dimension).

For clarity of presentation, we make the following assumptions, which are useful but not required. For every $h \in [H]$ we denote the state space at layer $h$ as $\statesp_h \coloneqq \crl{x_{(j,h)} : j \in [K]}$ for some $K \in \bbN$. We will restrict ourselves to binary action spaces $\cA = \crl{-1,+1}$. In addition, we assume that all policy class $\Pi$ under consideration satisfy the following stationarity assumption (in other words, it has parameter sharing across the state indices $j \in [K]$).
\begin{assumption}\label{ass:stationary}
The policy class $\Pi$ satisfies \emph{stationarity}: for every $\pi \in \Pi$ we have 
\begin{align*}
    \pi(x_{(j,1)}) = \pi(x_{(j,2)}) = \cdots = \pi(x_{(j,H)}) \quad \text{for every}~j\in [K].
\end{align*} 
For any $\pi \in \Pi$ and $j \in [K]$, we use $\pi(j)$ as a shorthand for the value of $\pi(x_{(j,h)})$ on every $h$.
\end{assumption} 

Now we state our main result.

\begin{theorem}\label{thm:bounds-on-spanning}
For any policy class $\Pi$ satisfying \pref{ass:stationary} we have
\begin{align*}
  \max\crl*{\min\crl*{\Sdim(\Pi), H + 1},\min\crl*{ 2^{\floor{ \log_2 \Tdim(\Pi)}}, 2^H}} \le \dimRL(\Pi) \le 2^{\Edim(\Pi)}.
\end{align*}
\end{theorem}

We give several remarks on \pref{thm:bounds-on-spanning}. The proof is deferred to the following subsection. 

It is interesting to understand to what degree we can improve the bounds in \pref{thm:bounds-on-spanning}. On the lower bound side, we note that each of the terms individually cannot be sharpened:
\begin{itemize}
    \item For the singleton class $\Pisingleton \coloneqq \crl{\pi_i(j) \mapsto \ind{j = i} : i \in [K]}$ we have $\dimRL(\Pisingleton) = \min\crl{K, H + 1}$ and $\Sdim(\Pisingleton) = K$.
    \item For the threshold class $\Pi_\mathrm{thres} \coloneqq \crl{\pi_i(j) \mapsto \ind{j \ge i} : i \in [K]}$, when $K$ is a power of two, it can be shown that $\dimRL(\Pi_\mathrm{thres}) = \min\crl{K, 2^H}$ and $\Tdim(\Pi_\mathrm{thres}) = 
    K$.
\end{itemize}
For the upper bound, our bound is exponential in $\Edim(\Pi)$; in light of \pref{thm:equivalence}, there can be a huge gap between the lower bound and upper bound. It is interesting to see whether the upper bound can be improved (possibly, to scale polynomially with $\Edim(\Pi)$, or more directly in terms of some function of $\Sdim(\Pi)$ and $\Tdim(\Pi)$. We do not have any examples of policy classes for which the upper bound improved upon what is shown by \pref{prop:dimension-ub}---for a good reason: it was later shown by \citet[][Thm.~11]{hanneke2024star} that $\Edim(\Pi) \ge \log_2(\Pi)$. Thus, the upper bound in \pref{thm:bounds-on-spanning} cannot offer any improvement at all! We include the result and its proof because it may be a stepping stone to showing refinements in future work.

Lastly, we remark that the lower bound of $\dimRL(\Pi) \ge \min\crl{\Omega(\Tdim(\Pi)), 2^H}$ is especially noteworthy since lower bounds for policy learning in terms of $\dimRL(\Pi)$ (to be shown later on) directly generalize prior work that shows the class of linear policies cannot be learned with $\poly(H)$ sample complexity \citep[e.g.,][]{du2019good}, since linear policies even in 2 dimensions have infinite threshold dimension.

\subsection{Proof of \pref{thm:bounds-on-spanning}}

We will prove each bound separately.

\textbf{Star Number Lower Bound.} 
Let $\bar{\pi}\in \Pi$ and the sequence $(j_1, \pi_1), \dots, (j_N, \pi_N)$ witness $\Sdim(\Pi) = N$. We construct a deterministic MDP $M$ for which the cumulative reachability at layer $h_\mathrm{max} \coloneqq \min\crl{N, H}$ (\pref{def:dimension}) is at least $\min\crl{N, H+1}$. The transition dynamics of $M$ are as follows; we will only specify the transitions until $h_\mathrm{max} - 1$ (afterwards, the transitions can be arbitrary).
\begin{itemize}
    \item The starting state of $M$ at layer $h=1$ is $x_{(j_1, 1)}$.
    \item (On-Chain Transitions): For every $h < h_\mathrm{max}$,
    \begin{align*}
        P(x' \mid x_{(j_h, h)}, a) &= \begin{cases}
            \ind{x' = x_{(j_{h+1}, h+1)}} & \text{if}~a = \bar{\pi}(x_{(j_h, h)}), \\[0.5em]
            \ind{x' = x_{(j_{h}, h+1)}} & \text{if}~a \ne \bar{\pi}(x_{(j_h, h)}).
        \end{cases}
    \end{align*}
    \item (Off-Chain Transitions): For every $h < h_\mathrm{max}$, state index $\tilde{j} \ne j_h$, and action $a \in \cA$,
    \begin{align*}
        P(x' \mid x_{(\tilde{j}, h)}, a) =
            \ind{x' = x_{(\tilde{j}, h+1)}}. 
    \end{align*}
\end{itemize}
We now compute the cumulative reachability at layer $h_\mathrm{max}$. If $N \le H$, the the number of $(x,a)$ pairs that $\Pi$ can reach in $M$ is $N$ (namely the pairs $(x_{(j_1, N)}, 1), \cdots, (x_{(j_N, N)}, 1)$). On the other hand, if $N > H$, then the number of $(x,a)$ pairs that $\Pi$ can reach in $M$ is $H+1$ (namely the pairs $(x_{(j_1, H)}, 1), \cdots, (x_{(j_H, H)}, 1), (x_{(j_H, H)}, 0)$). Thus we have shown that $\dimRL(\Pi) \geq \min \crl{N, H+1}$.

\textbf{Threshold Dimension Lower Bound.}
Let $\bar{\pi}\in \Pi$ and the sequence $(j_1, \pi_1), \dots, (j_N, \pi_N)$ witness $\Tdim(\Pi) = N$. We define a deterministic MDP $M$ as follows. Set $h_\mathrm{max} = \min\crl{\floor{\log_2 N}, H}$. Up until layer $h_\mathrm{max}$, the MDP will be a full binary tree of depth $h_\mathrm{max}$; afterward, the transitions will be arbitrary. It remains to assign state labels to the nodes of the binary tree (of which there are $2^{h_\mathrm{max}} - 1 \le N$). We claim that it is possible to do so in a way so that every policy $\pi_\ell$ for $\ell \in [2^{h_\mathrm{max}}]$ reaches a different state-action pair at layer $h_\mathrm{max}$. Therefore the cumulative reachability of $\Pi$ on $M$ is at least $2^{h_\mathrm{max}} = \min \crl{ 2^{\floor{\log_2 N}}, 2^H }$ as claimed.

It remains to prove the claim. The states of $M$ are labeled $j_2, \cdots, j_{2^{h_\mathrm{max}}}$ according to the order they are traversed using inorder traversal of a full binary tree of depth $h_\mathrm{max}$ \citep{cormen2022introduction}. One can view the MDP $M$ as a binary search tree where the action 0 corresponds to going left and the action 1 corresponds to going right. Furthermore, if we imagine that the leaves of the binary search tree at depth $h_\mathrm{max}$ are labeled from left to right with the values $1.5, 2.5, \cdots, 2^{h_\mathrm{max}} + 0.5$, then it is clear that for any $\ell \in [2^{h_\mathrm{max}}]$, the trajectory generated by running $\pi_{\ell}$ on $M$ is exactly the path obtained by searching for the value $\ell + 0.5$ in the binary search tree. Thus we have shown that the cumulative reachability of $\Pi$ on $M$ is the number of leaves at depth $h_\mathrm{max}$, thus proving the claim.

\textbf{Eluder Dimension Upper Bound.} 
Let $\Edim(\Pi) = N$. We only need to prove this statement when $N \le H$, as otherwise the statement already follows from \pref{prop:dimension-ub}. Let $(M^\star, h^\star)$ be the MDP and layer which witness $\dimRL(\Pi)$. Also denote $x_1$ to be the starting state of $M^\star$. For any state $x$, we denote $\mathrm{child}_0(x)$ and $\mathrm{child}_1(x)$ to be the states in the next layer which are reachable by taking $a=0$ and $a=1$ respectively.

For any reachable state $x$ at layer $h$ in the MDP $M^\star$ we define the function $f(x)$ as follows. For any state $x$ at layer $h^\star$, we set $f(x) \coloneqq 1$ if the state-action pairs $(x,0)$ and $(x,1)$ are both reachable by $\Pi$; otherwise we set $f(x) \coloneqq 0$. For states in layers $h < h^\star$ we set 
\begin{align*}
    f(x) \coloneqq \begin{cases}
        \max \crl{f(\mathrm{child}_0(x)), f(\mathrm{child}_1(x))} + 1 &\text{if both $(x,0)$ and $(x,1)$ are reachable by $\Pi$}, \\
        f(\mathrm{child}_0(x)) &\text{if only $(x,0)$ is reachable by $\Pi$},\\
        f(\mathrm{child}_1(x)) &\text{if only $(x,1)$ is reachable by $\Pi$}.
    \end{cases}
\end{align*}

We claim that for any state $x$, the contribution to $\dimRL(\Pi)$ by policies that pass through $x$ is at most $2^{f(x)}$. We prove this by induction. Clearly, the base case of $f(x) = 0$ or $f(x) = 1$ holds. If only one child of $x$ is reachable by $\Pi$ then the contribution to $\dimRL(\Pi)$ by policies that pass through $x$ equal to the contribution to $\dimRL(\Pi)$ by policies that pass through the child of $x$. If both children of $x$ are reachable by $\Pi$ then the contribution towards $\dimRL(\Pi)$ by policies that pass through $x$ is upper bounded by the sum of the contribution towards $\dimRL(\Pi)$ by policies that pass through the two children, i.e.~it is at most $2^{f(\mathrm{child}_0(x))} + 2^{f(\mathrm{child}_1(x))} \le 2^{f(x)}$. This concludes the inductive argument.

Now we bound $f(x_1)$. Observe that the quantity $f(x_1)$ counts the maximum number of layers $h_1, h_2, \cdots, h_L$ that satisfy the following property: there exists a trajectory $\tau = (x_1, a_1, \cdots, x_H, a_H)$ for which we can find $L$ policies $\pi_1, \pi_2, \cdots, \pi_L$ so that each policy $\pi_\ell$ when run on $M^\star$ (a) reaches $x_{h_\ell}$, and (b) takes action $\pi_\ell(x_{h_\ell}) \ne a_{h_\ell}$. Thus, by definition of the eluder dimension we have $f(x_1) \le N$. Therefore, we have shown that the cumulative reachability of $\Pi$ in $M^\star$ is at most $2^N$. \qed

\section{Bibliographical Remarks}\label{sec:eluder-bibliographic-remarks}
First, we remark that our presentation mostly focused on results pertaining to the \emph{policy eluder dimension}; we refer the interested reader to the paper \cite{li2022understanding} for additional results on the (scale-sensitive) eluder dimension. We now highlight several related works here.

\paragraph{Bounds on Eluder Dimension.} Several papers provide bounds on eluder dimension for various function classes. The original bounds on tabular, linear, and generalized linear functions were proved by \cite{russo2013eluder} (and later generalized by \cite{osband14}). \cite{mou2020sample} provides several bounds for the policy eluder dimension, mostly focusing on linear function classes. The bound in part (ii) of \pref{prop:parity-eluder-ub} was also calculated by \cite[][Prop.~3]{mou2020sample}.
When the function class lies in an RKHS, \cite{huang2021short} show that the eluder dimension is equivalent to the notion of \emph{information gain} \cite{srinivas2009gaussian, russo2016information}, which can be seen as an infinite dimensional generalization of the fact that the eluder dimension for linear functions over $\bbR^d$ is $\tilde{\Theta}(d)$. We also point the reader to \cite{hanneke2024star}, which provides several complementary insights on the relationship between the policy eluder dimension, star number, and threshold dimension (for a discussion on \pref{thm:equivalence}, see their Sec.~D).

\paragraph{Applications of Eluder Dimension.} The main application of the eluder dimension is to design algorithms and prove regret guarantees for contextual bandits and reinforcement learning. A few examples include the papers \cite{wen13,osband14, wang2020statistical, ayoub20, du20, foster2020instance, jin2021bellman, feng2021provably, ishfaq2021randomized, huang2021towards, mou2020sample}. While the majority of papers prove upper bounds via eluder dimension, \cite{foster2020instance} provided lower bounds for contextual bandits in terms of eluder dimension, if one is hoping for instance-dependent regret bounds. In addition, several works observe that eluder dimension sometimes does not characterize the sample complexity, as the guarantee via eluder dimension can be too loose \cite{jin2021bellman,foster2021statistical}. The works \cite{jia2024does, pacchiano2024second} show that eluder dimension plays a fundamental role in characterizing variance-dependent regret bounds for contextual bandits.

Beyond the online RL setting, eluder dimension has been applied to risk sensitive RL \cite{fei2021risk}, Markov games \cite{huang2021towards, jin2021power}, representation learning \cite{xu2021representation}, active online learning \cite{chen2021active}, imitation learning \cite{sekhari2023contextual, sekhari2023contextual, hanneke2025reliable}, and online learning with safety constraints \cite{sridharan2024online}.

\chapter{Generative Model and Local Simulator}\label{chap:generative}
In this chapter, we study the sample complexity of policy learning with simulator access.
\begin{enumerate}
    \item In \pref{sec:generative-spanning-capacity-bounds} we study the minimax sample complexity. We prove upper/lower bounds which show that bounded spanning capacity $\dimRL(\Pi)$ is necessary and sufficient for policy learning with a generative model. These results paint a relatively complete picture for the minimax sample complexity of learning any policy class $\Pi$, in the generative model, up to an $H\cdot \log \abs{\Pi}$ factor. 
    \item In \pref{sec:generative-infinite-policy} we sketch refinements to these bounds for infinite policy classes. Specifically, we show that (1) the $\log \abs{\Pi}$ dependence in the upper bound can be replaced by a suitable complexity for multiclass learning such as Natarajan dimension; and (2) a corresponding lower bound in terms of Natarajan dimension can also be established.
    \item In \pref{sec:generative-adapting-coverability} we investigate whether it is possible to adapt to the intrinsic difficulty of exploration as measured by the coverability coefficient, which can be much smaller than the (worst-case) spanning capacity. We show that this is impossible in general.
\end{enumerate}

\emph{Remark on Generative Model vs.~Local Simulator.} In this chapter, all upper bounds use local simulator access, while all lower bounds hold against the generative model. Therefore, while this distinction is worth keeping in mind, it is not critical for our results.

\section{Spanning Capacity is Necessary and Sufficient}\label{sec:generative-spanning-capacity-bounds}

In this section, we show that for any policy class, the \Compname{} characterizes the minimax sample complexity for agnostic PAC RL under generative model. We formalize the minimax sample complexity as follows.

\begin{definition}\label{def:generative-minimax}
Fix any $(\eps, \delta) \in (0,1)$ as well as a policy class $\Pi$. We denote the minimax sample complexity $\ngen(\Pi; \eps, \delta)$ to be the smallest $n \in \bbN$ for which there exists an algorithm $\alg$ which for any MDP $M$ collects at most $\ngen(\Pi; \eps, \delta)$ samples using generative access to $M$ and returns an $\eps$-optimal policy for $M$ with probability at least $1-\delta$.
\end{definition}

\subsection{Upper Bound}
\begin{theorem}[Upper Bound for Generative Model]
\label{thm:generative_upper_bound}
For any policy class $\Pi$, the minimax sample complexity $(\eps,\delta)$-PAC learning $\Pi$ is at most
\begin{align*}
    \ngen(\Pi; \eps, \delta) \le  \cO \prn*{\frac{H \cdot \dimRL(\Pi)}{\eps^2} \cdot \log \frac{\abs{\Pi}}{\delta}  }.
\end{align*}
\end{theorem}

\begin{proof}[Proof of \pref{thm:generative_upper_bound}]
 
We show that, with minor changes, the $\mathsf{TrajectoryTree}$ algorithm of \cite{kearns1999approximate} attains the guarantee in \pref{thm:generative_upper_bound}. The pseudocode can be found in \pref{alg:trajectory-tree}. The key modification is \pref{line:trajectory-tree-sample}: we simply observe that only $(x,a)$ pairs which are reachable by some $\pi \in \Pi$ in the current tree $\wh{T}_i$ need to be sampled (in contrast, the original algorithm of \cite{kearns1999approximate} samples all $A^H$ transitions).
   
\begin{algorithm}[!htp]
\caption{$\mathsf{TrajectoryTree}$ \citep{kearns1999approximate}}\label{alg:trajectory-tree}
    \begin{algorithmic}[1]
            \Require Policy class $\Pi$, generative access to the MDP $M$, number of samples $n$
            \State Initialize dataset of trajectory trees $\cD = \emptyset$.
            \For{$i=1, \dots, n$}
            \State Initialize trajectory tree $\wh{T}_i = \emptyset$.
            \State Sample initial state $x_1^{(i)} \sim \mu$.
            \While{$\mathrm{True}$}   \hfill
            \algcomment{Sample transitions and rewards for a trajectory tree}
            \State{Find any unsampled $(x,a)$ s.t.~$(x,a)$ is reachable in $\wh{T}_i$ by some $\pi \in \Pi$.}\label{line:trajectory-tree-sample}
            \If{no such $(x,a)$ exists} \textbf{break}
            \EndIf
            \State Sample $x' \sim P(\cdot\mid x,a)$ and $r \sim R(x,a)$ \label{line:generative-query}
            \State Add tuple $(x,a,r,x')$ to $\wh{T}_i$.
            \EndWhile
            \State $\cD \gets \cD \cup \wh{T}_i$.
            \EndFor
            \For{$\pi \in \Pi$}
            \hfill
            \algcomment{Policy evaluation}
            \State Set $\wh{V}^\pi \gets \frac{1}{n} \sum_{i=1}^n \wh{v}^\pi_i$, where $\wh{v}^\pi_i$ is the cumulative reward of $\pi$ on $\wh{T}_i$.
            \EndFor
            \State \textbf{Return} $\wh{\pi} \gets \argmax_{\pi \in \Pi} \wh{V}^\pi$.
    \end{algorithmic}
\end{algorithm}

Fix any $\pi \in \Pi$. For every trajectory tree $i \in [n]$, the algorithm has collected enough transitions so that $\wh{v}_i^\pi$ is well-defined, by \pref{line:trajectory-tree-sample} of the algorithm. By the sampling process, it is clear that the values $\crl{\wh{v}_i^\pi}_{i \in [n]}$ are i.i.d.~generated. We claim that they are unbiased estimates of $V^\pi$. Observe that one way of defining $V^\pi$ is the expected value of the following process:
\begin{enumerate}[label=\((\arabic*)\)]
    \item
For every $(x, a) \in \statesp \times \cA$, independently sample a next state $x' \sim P(\cdot\mid x,a)$ and a reward $r \sim R(x,a)$ to define a deterministic MDP $\wh{M}$.
\item Return the value $\wh{v}^\pi$ to be the value of $\pi$ run on $\wh{M}$.
\end{enumerate}
Define the law of this process as $\overline{\cQ}$. The sampling process of $\mathsf{TrajectoryTree}$ (call the law of this process $\cQ$) can be viewed as sampling the subset of $\wh{M}^\mathrm{det}$ which is reachable by some $\pi \in \Pi$. Thus, we have
\begin{align*}
V^\pi = \En_{\wh{M}\sim\overline{\cQ}}\brk*{ \wh{v}^\pi } = \En_{\wh{T}_\sim \cQ} \brk*{ \En \brk*{\wh{v}^\pi \mid{} \wh{T}} } = \En_{\wh{T}_\sim \cQ} \brk*{ \wh{v}^\pi } ,
\end{align*}
where the second equality is due to the law of total probability, and the third equality is due to the fact that $\wh{v}^\pi$ is measurable with respect to the trajectory tree $\wh{T}$. Thus, $\crl{\wh{v}_i^\pi}_{i \in [n]}$ are unbiased estimates of $V^\pi$.

Therefore, by Hoeffding's inequality (\pref{lem:hoeffding}) we see that $\abs{V^\pi - \wh{V}^\pi} \le \sqrt{ \tfrac{\log (2/\delta)}{2n} }$. Applying union bound we see that when the number of trajectory trees exceeds $n \gtrsim \tfrac{\log (\abs{\Pi}/\delta)}{\eps^2}$, with probability at least $1-\delta$, for all $\pi \in \Pi$, the estimates satisfy $\abs{V^\pi - \wh{V}^\pi} \le \eps/2$. Thus, the $\mathsf{TrajectoryTree}$ algorithm returns an $\eps$-optimal policy. Since each trajectory tree uses at most $H \cdot \dimRL(\Pi)$ queries to the generative model, we have the claimed sample complexity bound. This concludes the proof of \pref{thm:generative_upper_bound}.
\end{proof}

\subsection{Lower Bound}

\begin{theorem}[Lower Bound for Generative Model] \label{thm:generative_lower_bound}
For any policy class $\Pi$, the minimax sample complexity $(\eps,\delta)$-PAC learning $\Pi$ is at least
\begin{align*}
    \ngen(\Pi; \eps, \delta) \ge  \Omega \prn*{\frac{\dimRL(\Pi)}{\eps^2} \cdot \log \frac{1}{\delta}  }.
\end{align*}
\end{theorem}

\begin{proof}[Proof of \pref{thm:generative_lower_bound}]
    
Fix any worst-case deterministic MDP $M^\star$ which witnesses $\dimRL(\Pi)$ at layer $h^\star$. Since $\dimRL(\Pi)$ is a property depending on the dynamics of $M^\star$, we can assume that $M^\star$ has zero rewards. We can also assume that the algorithm knows $M^\star$ and $h^\star$ (this only makes the lower bound stronger). We construct a family of instances $\cM^\star$ where all the MDPs in $\cM^\star$ have the same dynamics as $M^\star$ but different nonzero rewards at the reachable $(x,a)$ pairs at layer $h^\star$.

Observe that we can embed a multi-armed bandit instance with $\dimRL(\Pi)$ arms using the class $\cM^\star$. The value of any policy $\pi \in \Pi$ is exactly the reward that it receives at the \emph{unique} $(x,a)$ pair in layer $h^\star$ that it reaches. Any $(\eps,\delta)$-PAC algorithm that works over the family of instances $\cM^\star$ must return a policy $\estpi$ that reaches an $(x,a)$ pair in layer $h^\star$ with near-optimal reward. Furthermore, in the generative model setting, the algorithm can only receive information about a single $(x,a)$ pair. Thus, such a PAC algorithm must also be able to PAC learn the best arm for multi-armed bandits with $\dimRL(\Pi)$ arms. Therefore, we can directly apply existing PAC lower bounds which show that the sample complexity of $(\eps, \delta)$-PAC learning the best arm for $K$-armed multi-armed bandits is at least $\Omega(\tfrac{K}{\eps^2} \cdot \log \tfrac{1}{\delta})$ \citep[see, e.g.,][]{mannor2004sample}.
\end{proof}

\section{Extension to Infinite Policy Classes}\label{sec:generative-infinite-policy}

In this section, we extend the results from \pref{sec:generative-spanning-capacity-bounds} to the infinite policy classes, using concepts from multiclass PAC learning (for a refresher, we refer the reader to \pref{app:multiclass-learning}). We will state our results in terms of the Natarajan dimension (\pref{def:natarajan}), which is a generalization of the VC dimension used to study multiclass learning. We note that the results in this section could be stated in terms of other complexity measures from multiclass learning such as the graph dimension and DS dimension \citep[see, e.g.,][]{natarajan1989learning, shalev2014understanding, daniely2014optimal, brukhim2022characterization}.

\subsection{Upper Bound Extension}
We modify the analysis of the $\mathsf{TrajectoryTree}$ to account for infinite policy classes and replace the dependence on $\log \abs{\Pi}$ with the Natarajan dimension $\natarajan(\Pi)$ (and additional log factors). Recall that our analysis of $\mathsf{TrajectoryTree}$ required us to prove a uniform convergence guarantee for the estimate $\wh{V}^\pi$: with probability at least $1-\delta$, for all $\pi \in \Pi$, we have $\abs{\wh{V}^\pi - V^\pi} \lesssim \eps$. We previously used a union bound over $\abs{\Pi}$, which gave us the $\log \abs{\Pi}$ dependence. Now we sketch an argument to replace the $\log \abs{\Pi}$ with $\natarajan(\Pi)$. 

Let $\cT$ be the set of all possible trajectory trees. We introduce the notation $v^\pi: \cT \to \bbR$ to denote the function that takes as input a trajectory tree $\wh{T}$ (for example, as sampled by $\mathsf{TrajectoryTree}$) and returns the value of running $\pi$ on it. Then we can rewrite the desired uniform convergence guarantee:
\begin{align*}
    \text{w.p.~at least}~1-\delta, \quad \sup_{\pi \in \Pi} ~ \abs*{ \frac{1}{n} \sum_{i=1}^n v^\pi(\wh{T}_i) - \En\brk*{ v^\pi(\wh{T}) } } \le \eps. \numberthis\label{eq:unif-convergence-statement}
\end{align*}
In light of \pref{lem:uniform-convergence}, we will compute the pseudodimension for the function class $\cV^\Pi = \crl{v^\pi: \pi \in \Pi}$. Define the subgraph class
\begin{align*}
    \cV^{\Pi,+} \coloneqq \crl*{(\wh{T}, \theta) \mapsto \ind{v^\pi(\wh{T}) \le \theta} : \pi \in \Pi} \subseteq \crl{0,1}^{\cT \times \bbR}.
\end{align*}
By definition, $\mathrm{Pdim}(\cV^\Pi) = \mathrm{VC}(\cV^{\Pi,+})$. Fix any $Z = \crl{(\wh{T}_1, \theta_1), \cdots, (\wh{T}_d, \theta_d)} \in (\cT \times \bbR)^d$. In order to show that $\mathrm{VC}(\cV^{\Pi,+}) \le d$ for some value of $d$ it suffices to prove that the projection of $\cV^{\Pi,+}$ onto $Z$ satisfies $\abs{\cV^{\Pi,+}\big|_Z} < 2^d$.

For any index $t\in [d]$, we also denote $\pi(\vec{x}_i) \in \cA^{\le H \dimRL(\Pi)}$ to be the vector of actions selected by $\pi$ on all $\Pi$-reachable states in $\wh{T}_i$ (of which there are at most $H \cdot \dimRL(\Pi)$). We claim that
\begin{align}\label{eq:bound-pseudo}
    \abs[\Big]{\cV^{\Pi,+}\big|_{Z}} \le \abs*{ \crl*{\prn*{\pi(\vec{x}_1), \cdots, \pi(\vec{x}_d) } : \pi \in \Pi} } =: \abs[\Big]{\Pi\big|_Z}. 
\end{align}
This holds because once the $d$ trajectory trees are fixed, for every $\pi \in \Pi$, the value of the vector $\cV^{\Pi, +}\big|_{Z} \in \crl{0,1}^d$ is determined by the trajectory that $\pi$ takes in every trajectory tree. This in turn is determined by the assignment of actions to every reachable state in all the $d$ trajectory trees, of which there are at most $\dimRL(\Pi) \cdot H \cdot d$ of. Therefore, we can upper bound the size of $\cV^{\Pi, +}\big|_{Z}$ by the number of ways any $\pi \in \Pi$ assign actions to every state among the trajectory trees $\wh{T}_1, \cdots, \wh{T}_d$.

Applying \pref{lem:sauer-natarajan} to Eq.~\eqref{eq:bound-pseudo}, we get that 
\begin{align*}
    \abs[\Big]{\cV^{\Pi,+}\big|_{Z}} \le \prn*{\frac{H \dimRL(\Pi) d \cdot e \cdot  (A+1)^2}{2 \natarajan(\Pi)}}^{\natarajan(\Pi)}. 
\end{align*}
For the choice of $d = \wt{\cO}\prn*{\natarajan(\Pi)}$, the previous display is at most $2^d$, thus proving the bound on $\pseudo(\cV^\Pi)$. Lastly, the bound can be plugged back into \pref{lem:uniform-convergence} to get a bound on the error of $\mathsf{TrajectoryTree}$: the uniform convergence bound in Eq.~\eqref{eq:unif-convergence-statement} holds using
\begin{align*}
    n = \wt{\cO}\prn*{ H \dimRL(\Pi) \cdot \frac{\natarajan(\Pi) + \log \frac{1}{\delta}}{\eps^2}} \quad \text{samples}.
\end{align*}
This in turn yields a guarantee on $\wh{\pi}$ returned by $\mathsf{TrajectoryTree}$, via a standard argument.

\subsection{Lower Bound Extension}\label{sec:generative-lower-bound-extend}
Now we address the lower bound. Observe that it is possible to achieve a lower bound that depends on $\natarajan(\Pi)$ with the following construction. First, identify the layer $h\in [H]$ such that the witnessing set $C$ contains the maximal number of states in $\statesp_h$; by the pigeonhole principle there must be at least $\natarajan(\Pi)/H$ such states in layer $h$. Then, we construct an MDP which ``embeds'' a hard multiclass learning problem at layer $h$ over these states. A lower bound of $\Omega \prn*{ \tfrac{\natarajan(\Pi)}{H\eps^2} \cdot \log \tfrac{1}{\delta} }$ follows from \pref{thm:multiclass-fun}.

By combining \pref{thm:generative_lower_bound} with the above idea we get the following corollary. 

\begin{corollary}[Lower Bound for Generative Model with Infinite Policy Classes] 
For any policy class $\Pi$, the minimax sample complexity $(\eps,\delta)$-PAC learning $\Pi$ is at least
\begin{align*}
     \ngen(\Pi; \eps, \delta) \ge  \Omega \prn*{\frac{\dimRL(\Pi) + \natarajan(\Pi)/H}{\eps^2} \cdot \log \frac{1}{\delta}}.
\end{align*}
\end{corollary}
Since the generative model setting is easier than online RL, this lower bound also extends to the online RL setting.

\emph{Remark.} Our lower bound is additive in $\dimRL(\Pi)$ and $\natarajan(\Pi)$; we do not know if it is possible to strengthen this to be a product of the two factors, as is shown in the upper bound extension.

\section{Adapting to Coverability is Statistically Intractable}\label{sec:generative-adapting-coverability}
We have shown that the minimax (i.e., worst case over all MDPs) sample complexity for learning with a generative model or local simulator is characterized by the spanning capacity $\spancap(\Pi)$. Unfortunately, spanning capacity is exponentially large for many policy classes of interest (such as linear policies, see \pref{sec:eluder-spanning}) and can be arbitrarily larger than coverability. Naturally, one might wonder if it is possible to achieve instance-dependent guarantees which adapt to coverability $\ccov(\Pi, M)$.

Our main result in this section is that in general, it is not possible to adapt to coverability, even with generative access.

\begin{theorem}\label{thm:lower-bound-coverability} For any $H \in \bbN$, there exists a policy class $\Pi$ of size $2^H$ and a family of MDPs $\cM$ over a  state space of size $2^{O(H)}$, binary action space, and horizon $H$ such that every $M \in \cM$ satisfies $\ccov(\Pi, M) = 2$ and $\Pi$ is policy complete\footnote{See \pref{def:policy-completeness}. Policy completeness is a strong \emph{representational condition} on $\Pi$ which asserts that the policy class $\Pi$ is closed under the policy improvement operator. Further discussion is deferred to \pref{chap:mu-reset}.} for $M$, so that any proper deterministic algorithm that returns a $1/8$-optimal policy must use at least $2^{\Omega(H)}$ generative access samples for some MDP in $\cM$. 
\end{theorem}

The proof is deferred to \pref{sec:lower-bound-generative-access}. For now, we discuss some implications of \pref{thm:lower-bound-coverability} as well as the key ideas behind the construction. 

\pref{thm:lower-bound-coverability} shows that even under the strongest interaction protocol and the strongest representational condition on $\Pi$, the mere existence of a good exploratory distribution $\mu$ is insufficient for policy learning. In other words, it formalizes the folklore intuition that ``policy learning methods cannot explore''. Prior work \cite[Proposition 4.1 of][]{agarwal2020pc} suggests that policy gradient methods may fail to explore due to vanishing gradients; \pref{thm:lower-bound-coverability} shows that this is not an algorithmic limitation of policy gradient methods but an information theoretic barrier. Furthermore, there is no contradiction between \pref{thm:lower-bound-coverability} and works which imbue policy gradient methods with exploration capabilities \cite{agarwal2020pc, zanette2021cautiously, liu2024optimistic}, since the latter impose stronger dynamics and/or function approximation assumptions.

Additionally, \pref{thm:lower-bound-coverability} reveals a strict separation between policy-based RL and value-based RL with a local simulator. Under the stronger assumption that the learner has access to a $Q$-function class $\cF \in [0,1]^{\statesp \times \actionsp}$ satisfying value function realizability ($Q^\star \in \cF$), \cite[Theorem 3.1 of ][]{mhammedi2024power} gives an algorithm that achieves sample complexity $\poly(\ccov, H, \log \abs{\cF}, 1/\eps, \log 1/\delta)$. Again, this is not in contradiction with our result because in \pref{thm:lower-bound-coverability}, the implicitly defined value function class $\cF$ has cardinality which is double-exponential in $H$.   

\paragraph{Key Ideas for \pref{thm:lower-bound-coverability}.} The lower bound construction takes the form of \emph{rich observation combination locks}, which are Block MDP variants of the classic combination lock construction \cite{sutton2018reinforcement}. At a high level, the latent transitions of these instances are given by a combination lock parameterized by an unknown open-loop policy $\optpi \in \Pi_\mathsf{open}$; taking the optimal policy $\optpi$ gives the learner reward of 1, while deviation from $\optpi$ at any layer gives the learner reward of zero. Also, the emission function $\emission$ for each state is supported on an exponentially large set which is a-priori unknown to the learner (hence the name ``rich observations''). Such constructions have appeared in previous lower bounds for online RL \cite{sekhari2021agnostic, jia2023agnostic}. The classic combination lock can easily be solved in $\poly(H)$ samples using tabular RL approaches which use the principle of \emph{optimism in the face of uncertainty}---when the learner sees a previously observed state $x_h$, they explore by trying out a new action $a_h$ since it could potentially lead to higher reward. However, the addition of rich observations makes the problem statistically intractable, since it is likely that the learner always sees new observations, so they cannot identify what latent state they are in or when they have deviated from $\optpi$ in a given episode.

\begin{figure}[ht!]
    \centering
\includegraphics[scale=0.32, trim={0cm 19.5cm 32cm 0cm}, clip]{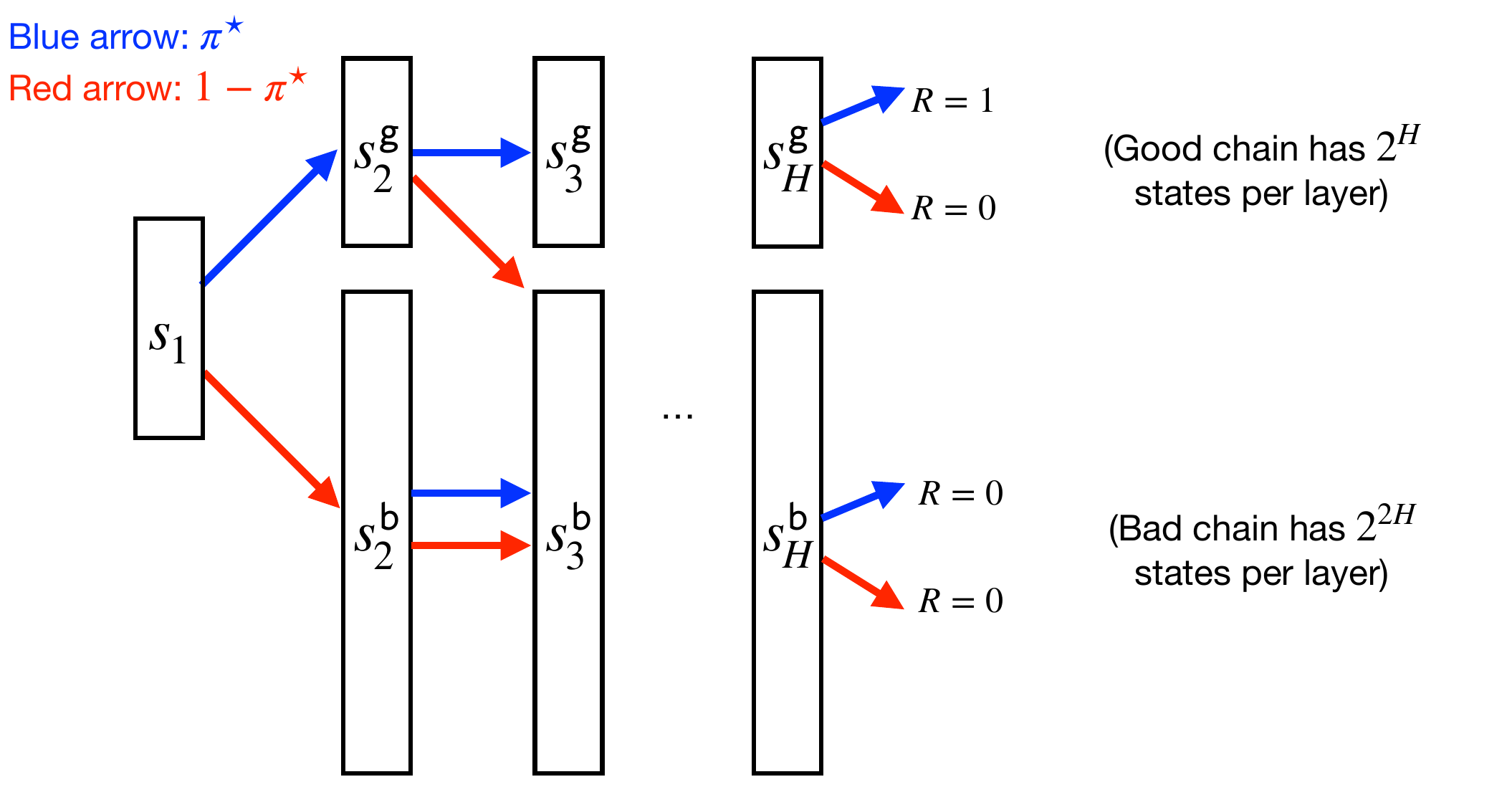}
    \caption{Construction used for proof of \pref{thm:lower-bound-coverability}.} 
    \label{fig:lb1}
\end{figure}

Since the rich observation combination lock is a Block MDP, it naturally satisfies small coverability. (Furthermore, exploratory distributions $\mu$ can be constructed which satisfy small concentrability.) Therefore, it is a natural starting point for proving lower bounds in our setting. 
An example of the lower bound construction can be found in \pref{fig:lb1}. In order to prevent the learner from using the more powerful generative model, the lower bound construction has unbalanced emission supports: namely for all $h \ge 2$, the support of $\emission(s_h^\mathsf{g})$ is of size $2^H$, while the support of $\emission(s_h^\mathsf{b})$ is of size $2^{2H}$. Intuitively, the learner receives little information unless they can sample from $(s_H^\mathsf{g}, \optpi_H)$ and receive reward of 1. Since the emission support for $s_h^\mathsf{g}$ is exponentially smaller than that of $s_h^\mathsf{b}$, unless the learner guesses $\exp(H)$ times with the generative model, it is likely that they only receive observations sampled from $s_h^\mathsf{b}$. Stated in a different way, it is not possible for the learner to construct an exploratory distribution $\mu$ which has $\cconc = \poly(H)$, even using $\poly(H)$ adaptive queries to the generative model. Thus, the generative model provides no real additional power over the online RL setting, for which we know $2^{\Omega(H)}$ lower bounds \cite{sekhari2021agnostic}.

\subsection{Proof of \pref{thm:lower-bound-coverability}}\label{sec:lower-bound-generative-access}


\paragraph{Lower Bound Construction.} We define a family of Block MDPs $\cM = \crl{M_{\optpi, \optdec}}_{\optpi \in \Pi, \optdec \in \Phi}$ which are parameterized by an optimal policy $\optpi \in \Pi$ and a decoding function $\optdec \in \Phi$ (to be described). 

\begin{itemize}
    \item \textbf{Policy Class:} We define the policy class $\Pi$ to be open loop policies:
\begin{align*}
    \Pi \coloneqq \crl{\pi: \forall x \in \statesp_h, \pi_h(x) \equiv a_h, (a_1, \cdots, a_H) \in \actionsp^H}.
\end{align*}
For partial policies $\pi_{h_\bot:h_\top} \in \Pi_{h_\bot:h_\top}$ we sometimes drop the subscript $h_\bot:h_\top$ if clear from context. Observe that any policy can be identified with the sequence of actions taken at every layer $h$. We also overload equality to compare partial policies $\pi_{h_\bot: h_\top}$ with complete policies $\pi'_{1:H}$, i.e., we write $\pi = \pi'$ iff $\pi_{h_\bot: h_\top} = \pi'_{h_\bot: h_\top}$.
    \item \textbf{Latent MDP:} The latent state space $\latentsp$ is layered where each $\latentsp_h \coloneqq \crl{ s_h^\sgood, s_h^\sbad}$ is comprised of a good and bad state. We abbreviate the state as $\crl{\sgood, \sbad}$ if the layer $h$ is clear from context. The action space $\actionsp = \crl{0,1}$. The starting state is always $\sgood$. Let $\optpi \in \Pi$ be any policy, which can be represented by a vector in $(\pi^\star_1, \cdots, \pi^\star_H) \in \crl{0,1}^H$. The latent transitions/rewards of an MDP parameterized by $\optpi \in \Pi$ are given by the standard combination lock. For every $h \in [H]$: 
    \begin{align*}
        \optlatp(\cdot  \mid  s, a) = \begin{cases}
            \delta_{s_{h+1}^\sgood} & \text{if}~s = s_h^\sgood ~\text{and}~a = \pi^\star_h\\
            \delta_{s_{h+1}^\sbad}& \text{otherwise}.
        \end{cases} \quad \text{and} \quad  \optlatr(s,a) = \ind{s = s_H^\sgood, a = \pi^\star_H}.
    \end{align*} 
    \item \textbf{Rich Observations:} The observation state space $\statesp$ is layered where each $\cX_h \coloneqq \crl{x_h^{(1)}, \cdots, x_h^{(m)}}$ with $m = 2^{2H}$. The decoding function class $\Phi$ is the collection of all decoders which for every $h \ge 2$ assigns $s_h^\sgood$ to a subset of $\statesp_h$ of size $2^H$ and $s_h^\sbad$ to the rest:
    \begin{align*}
        \Phi &\coloneqq \crl*{\optdec: \statesp \mapsto \latentsp :~\forall~ x_1 \in \statesp_1,~\optdec(x_1) = \sgood,  \text{ and } \forall~h \ge 2, \abs*{ \crl*{x \in \statesp_h: \optdec(x) = \sgood}} = 2^H},\\
             &\text{so that}~\abs{\Phi} = \binom{2^{2H}}{2^H}^{H-1} = 2^{2^{\wt{O}(H)}}.
    \end{align*}
    In the MDP parameterized by $\optdec \in \Phi$, the emission for every $s \in \latentsp$ is
    \begin{align*}
        \emission(s) = \unif\prn*{\crl*{x \in \statesp_h: \optdec(x) = s}}.
    \end{align*}
    
\end{itemize}

Now we establish several facts about the lower bound construction. Fix any $M = M_{\optpi, \optdec}$.
\begin{enumerate}
    \item Since $M$ is a Block MDP with 2 latent states per layer, $\ccov(\Pi, M) = 2$.
    \item The class $\Pi$ satisfies policy completeness with respect to $M$. To see this, fix any layer $h \in [H]$ and partial policy $\pi \in \Pi_{h+1:H}$. We have:
    \begin{align*}
        \forall~(x,a) \in \statesp_h \times \actionsp:\quad Q^\pi(x,a) = \ind{\optdec(x) = s_h^\sgood, a = \pi^\star_h, \pi = \pi^\star_{h+1:H}}.
    \end{align*}
    Therefore in \pref{def:policy-completeness} we can take $\wt{\pi}_h \coloneqq \pi^\star_h$, which satisfies $\wt{\pi}_h \in \argmax_{a \in \actionsp} Q^\pi(x,a)$ for all $x \in \statesp_h$.
\end{enumerate}

\paragraph{Notation for Algorithms.}  We use $\alg$ to denote a deterministic algorithm that collects $T$ samples, i.e., full-length episodes. For technical convenience, we will suppose that $\alg$ sequentially queries the generative model by looping over layers, i.e., it queries $(X_1, A_1) \in \statesp_1 \times \actionsp$, then $(X_2, A_2) \in \statesp_2 \times \actionsp$, etc. This only increases the sample complexity of $\alg$ by a factor of $H$, which is negligible since we will show that $\alg$ requires $\exp(H)$ samples.

For any $t \in [T]$ we define $\cF_{t-1}$ to be the sigma-field of everything observed in the first $t-1$ episodes. We further define for any $h \in [H]$ the filtration $\cF_{t,h-1}$ to be the sigma-field of everything observed in the first $t-1$ episodes as well as the first $h-1$ steps of the $t$-th sample. For the generative model, $\cF_{t,h-1} \coloneqq \sigma( \cF_{t-1}, \crl{\prn{X_{t,i}, A_{t,i}, R_{t,i}, X'_{t,i}}}_{i \le h-1} )$, where $R_{t,i}$ and $X'_{t,i}$ are the reward and transition that is returned by the environment. The tuple $(X_{t,h}, A_{t,h})$ is measurable with respect to $\cF_{t,h-1}$ (since $\alg$ is deterministic).

\paragraph{Sample Complexity Lower Bound.} We use an interactive variant of Le Cam's Convex Hull Method (\pref{thm:interactive-lecam-cvx}) to prove our lower bound. First we need to instantiate the parameter space. We will let $\Theta \coloneqq \crl{(\optpi, \optdec):~\optpi \in \Pi, \optdec \in \Phi}$ so that $\cM = \crl{M_\theta}_{\theta \in \Theta} = \crl{M_{\optpi, \optdec}}_{\optpi \in \Pi, \optdec \in \Phi}$. We further denote the subsets 
\begin{align*}
    \Theta_0 &\coloneqq \crl{(\optpi, \optdec):~\optpi \in \Pi~\text{s.t.}~\optpi_H = 0, \optdec \in \Phi} \\
    \Theta_1 &\coloneqq \crl{(\optpi, \optdec):~\optpi \in \Pi~\text{s.t.}~\optpi_H = 1, \optdec \in \Phi}
\end{align*}
The observation space $\cY$ is defined as the set of observations over $T$ rounds as well as returned proper policy for an algorithm interacting with the MDP, i.e.,
\begin{align*}
    \cY \coloneqq (\statesp \times \actionsp \times \statesp \times [0,1])^{HT} \times \Pi.
\end{align*}

For an observation $y \in \cY$ we define the final returned policy as $y^\pi$. The loss function is given by
\begin{align*}
    L((\optpi, \phi), y) \coloneqq \ind{\optpi \ne y^\pi}. 
\end{align*}
Then we have for any $y \in \cY$, $(\optpi_0, \phi_0) \in \Theta_0$, and $(\optpi_1, \phi_1) \in \Theta_1$ that
\begin{align*}
    L((\optpi_0, \phi_0), y) +  L((\optpi_1, \phi_1), y) \ge 1 \coloneqq 2\Delta,
\end{align*}
since the last bit of $y^\pi$ can be either 0 or 1, thus only matching exactly one of $\optpi_0$ and $\optpi_1$.

Now we are ready to apply \pref{thm:interactive-lecam-cvx}. We get that for any $\alg$, we must have
\begin{align*}
    \sup_{(\optpi, \optdec) \in \Pi \times \Phi} \En_{Y \sim \Pr^{M_{\optpi, \optdec}, \alg}} \brk*{V^\star - V^{\estpi}} &= \sup_{(\optpi, \optdec) \in \Pi \times \Phi} \En_{Y \sim \Pr^{M_{\optpi, \optdec}, \alg}} \brk*{1 - \ind{\optpi = Y^\pi} } \\
    &= \sup_{(\optpi, \optdec) \in \Pi \times \Phi} \En_{Y \sim \Pr^{M_{\optpi, \optdec}, \alg}} \brk*{ L((\optpi, \phi), Y) } \\
    &\ge \frac{1}{4} \cdot \max_{\nu_0 \in \Delta(\Theta_0), \nu_1 \in \Delta(\Theta_1)} \prn*{1 - \dtv\prn*{\Pr^{\nu_0, \alg}, \Pr^{\nu_1, \alg}}} \\
    &\ge \frac{1}{4} \cdot \prn*{1 - \dtv\prn*{\Pr^{\unif(\Theta_0), \alg}, \Pr^{\unif(\Theta_1), \alg}}}. 
\end{align*}
It remains to compute an upper bound $\dtv\prn*{\Pr^{\unif(\Theta_0), \alg}, \Pr^{\unif(\Theta_1), \alg}}$ which holds for any $\alg$. This is accomplished by the following lemma.

\begin{lemma}\label{lem:tv-bound-generative} Let $T = 2^{O(H)}$. For any deterministic $\alg$ that adaptively collects $HT$ samples via generative access, we have
\begin{align*}
    \dtv\prn*{\Pr^{\unif(\Theta_0), \alg}, \Pr^{\unif(\Theta_1), \alg}} \le \frac{T^4H}{2^{H-9}}.
\end{align*}
\end{lemma}
Plugging in \pref{lem:tv-bound-generative}, we conclude that for any $\alg$ that collects $2^{cH}$ samples for sufficiently small constant $c > 0$ must be $1/8$-suboptimal in expectation. This concludes the proof of \pref{thm:lower-bound-coverability}.\qed

\subsection{Proof of \pref{lem:tv-bound-generative} (TV Distance Calculation for \pref{thm:lower-bound-coverability})}

Let us denote $\nu_0 \coloneqq \unif(\Theta_0)$ and $\nu_1 \coloneqq \unif(\Theta_1)$. By the TV distance chain rule (\pref{lem:chain-rule-tv}) we have
\begin{align*}
     & \dtv\prn*{\Pr^{\nu_0, \alg}, \Pr^{\nu_1, \alg}} \\
     &\le \sum_{t=1}^T \sum_{h=1}^H \En^{\nu_0, \alg} \brk*{\dtv \prn*{\Pr^{\nu_0, \alg} \brk*{ X_{t,h}, A_{t,h} \mid \cF_{t,h-1} }, \Pr^{\nu_1, \alg} \brk*{ X_{t,h}, A_{t,h} \mid \cF_{t,h-1} }}} \\
     &\qquad + \En^{\nu_0, \alg} \brk*{\dtv \prn*{\Pr^{\nu_0, \alg} \brk*{ X_{t,h}', R_{t,h} \mid X_{t,h}, A_{t,h}, \cF_{t,h-1} }, \Pr^{\nu_1, \alg} \brk*{ X_{t,h}', R_{t,h} \mid X_{t,h}, A_{t,h}, \cF_{t,h-1} }}} \\
     &= \sum_{t=1}^T \sum_{h=1}^H \En^{\nu_0, \alg} \brk*{\dtv \prn*{\Pr^{\nu_0, \alg} \brk*{ X_{t,h}', R_{t,h} \mid \cF_{t,h-1} }, \Pr^{\nu_1, \alg} \brk*{ X_{t,h}', R_{t,h} \mid \cF_{t,h-1} }}} \\
     &= \underbrace{ \sum_{t=1}^T \sum_{h=1}^{H-1} \En^{\nu_0, \alg} \brk*{\dtv \prn*{\Pr^{\nu_0, \alg} \brk*{ X_{t,h}' \mid  \cF_{t,h-1} }, \Pr^{\nu_1, \alg} \brk*{ X_{t,h}' \mid  \cF_{t,h-1} }}} }_{\text{transition TV distance}} \\
     &\qquad + 
     \underbrace{ \sum_{t=1}^T \En^{\nu_0, \alg} \brk*{\dtv \prn*{\Pr^{\nu_0, \alg} \brk*{ R_{t,H} \mid  \cF_{t,H-1} }, \Pr^{\nu_1, \alg} \brk*{ R_{t,H} \mid  \cF_{t,H-1} }}} }_{\text{reward TV distance}}.
\end{align*}
The first equality follows from the fact that the TV distance for the distribution over state-action pairs $(X_{t,h}, A_{t,h})$ is zero since $(X_{t,h}, A_{t,h})$ is measurable with respect to $\cF_{t,h-1}$. The second equality follows because the rewards only come at the last layer in every MDP instance.

We now show how to bound each term separately.

\paragraph{Transition TV Distance.} For the transition TV distance, we have the following computation for all $t \in [T], h \in [H-1]$:
\begin{align*}
    \hspace{2em}&\hspace{-2em} \En^{\nu_0, \alg} \brk*{\dtv \prn*{ \Pr^{\nu_0, \alg} \brk*{ X_{t,h}' \mid  \cF_{t, h-1} }, \Pr^{\nu_1, \alg} \brk*{X_{t,h}' \mid  \cF_{t, h-1} } } } \\
    &\overset{(i)}{\le} \En^{\nu_0, \alg} \brk*{\dtv \prn*{ \Pr^{\nu_0, \alg} \brk*{ X_{t,h}' \mid  \cF_{t, h-1} }, \unif(\cX_{h+1}) } } \\
    &\qquad + \En^{\nu_0, \alg} \brk*{\dtv \prn*{ \Pr^{\nu_1, \alg} \brk*{ X_{t,h}' \mid  \cF_{t, h-1} }, \unif(\cX_{h+1}) } } \\ 
    &\overset{(ii)}{\le} \frac{t}{2^{H-3}}. \numberthis \label{eq:ub-transition-calculation}
\end{align*}
The inequality $(i)$ follows by triangle inequality and the inequality $(ii)$ uses \pref{lem:tv-distance-from-uniform-generative}.

\paragraph{Reward TV Distance.} 
We can compute that
\begin{align*}
    \hspace{2em}&\hspace{-2em} \En^{\nu_0, \alg} \brk*{\dtv \prn*{\Pr^{\nu_0, \alg} \brk*{ R_{t,H} \mid  \cF_{t,H-1} }, \Pr^{\nu_1, \alg} \brk*{ R_{t,H} \mid  \cF_{t,H-1} }}} \\
    &\overset{(i)}{\le} \En^{\nu_0, \alg} \brk*{\dtv \prn*{\Pr^{\nu_0, \alg} \brk*{ R_{t,H} \mid  \cF_{t,H-1} },  \delta_0 }} + \En^{\nu_0, \alg} \brk*{\dtv \prn*{\Pr^{\nu_1, \alg} \brk*{ R_{t,H} \mid  \cF_{t,H-1} },  \delta_0 }} \\
    &\overset{(ii)}{=} \En^{\nu_0, \alg} \brk*{\Pr^{\nu_0, \alg} \brk*{ R_{t,H} = 1 \mid  \cF_{t,H-1} } } + \En^{\nu_0, \alg} \brk*{ \Pr^{\nu_1, \alg} \brk*{ R_{t,H} = 1 \mid  \cF_{t,H-1} } } \\
    &\overset{(iii)}{\le} t \cdot \frac{T^2H}{2^{H-8}}. \numberthis\label{eq:reward-tv}
\end{align*}
The inequality $(i)$ follows by triangle inequality, while $(ii)$ uses the fact that the rewards are in $\crl{0,1}$. Lastly, $(iii)$ follows by \pref{lem:reward-bound-lb1}.

\paragraph{Final Bound.} Thus, combining Eqs.~\eqref{eq:ub-transition-calculation} and \eqref{eq:reward-tv} we can conclude that:
\begin{align*}
    \dtv\prn*{\Pr^{\nu_0, \alg}, \Pr^{\nu_1, \alg}} \le \frac{T^2H}{2^{H-3}} + \frac{T^4H}{2^{H-8}} \le \frac{T^4H}{2^{H-9}}.
\end{align*}
This concludes the proof of \pref{lem:tv-bound-generative}.\qed

\begin{lemma}[Transition TV Distance for Construction in \pref{thm:lower-bound-coverability}]\label{lem:tv-distance-from-uniform-generative}
For any $t \in [T], h \in [H-1]$, we have
\begin{align*}
    \nrm*{\Pr^{\nu_0, \alg}\brk*{X_{t,h}'  \mid X_{t,h}, A_{t,h}, \cF_{t, h-1}} - \unif(\statesp_{h+1}) }_1 &\le \frac{t}{2^{H-2}}, \\
    \nrm*{\Pr^{\nu_1, \alg}\brk*{X_{t,h}'  \mid X_{t,h}, A_{t,h}, \cF_{t, h-1}} - \unif(\statesp_{h+1})}_1 &\le \frac{t}{2^{H-2}}.
\end{align*}
\end{lemma}
\begin{proof}[Proof of \pref{lem:tv-distance-from-uniform-generative}]
We prove the bound for $\nu_0$, since the proof for $\nu_1$ is identical. Denote the ``annotated'' sigma-field
\begin{align*}
    \cF_{t, h-1}' = \sigma \Bigg( &\cF_{t, h-1}, X_{t,h}, A_{t,h}, \crl*{\optdec(X): X \in \cF_{t, h-1} \cup \crl{X_{t,h}} }, \\
    &\crl*{ \ind{A = \optpi(X)}: (X,A) \in \cF_{t, h-1} \cup \crl{X_{t,h}, A_{t,h}} } \Bigg) 
\end{align*}
to be the sigma-field which also includes the latent state labels for all of the seen observations as well as whether the actions taken followed $\optpi$ or not. Let us denote $\ell = \phi(X_{t,h}') \in \crl{\sgood, \sbad}$ to be the latent state of the next observation. Observe that the label $\ell$ is measurable with respect to $\cF_{t,h-1}'$ since the filtration $\cF_{t-1}'$ includes $\optdec(X_{t,h})$ as well as $\ind{A_{t,h} = \optpi(X_{t,h})}$. Furthermore denote $\cX_\mathsf{obs}$ to denote the total number of observations that we have encountered already in layer $h+1$ and $\cX_\mathsf{obs}^\ell$ to denote the observations we have encountered whose latent state is $\ell$.

Under the uniform distribution over decoders, the assignment of the remaining observations is equally likely. Therefore we can write the distribution of $X_{t,h}'$ as: 
\begin{align*}
    \text{if}~\ell = \sgood:&\quad \Pr^{\nu_0, \alg} \brk*{X_{t,h}' = x  \mid  \cF_{t,h-1}'} = \begin{cases}
        \frac{1}{2^H} &\text{if}~x \in \cX_\mathsf{obs}^\ell \\
        0 &\text{if}~x \in \cX_\mathsf{obs} - \cX_\mathsf{obs}^\ell \\
        \frac{1}{2^H} \cdot \frac{2^H - \abs{\cX_\mathsf{obs}^\ell}}{2^{2H} - \abs{\cX_\mathsf{obs}}} &\text{if}~x \in \cX_{h+1} - \cX_\mathsf{obs}
    \end{cases}\\
    \text{if}~\ell = \sbad:&\quad \Pr^{\nu_0, \alg} \brk*{X_{t,h}' = x  \mid  \cF_{t,h-1}'} = \begin{cases}
        \frac{1}{2^{2H}-2^{H}} &\text{if}~x \in \cX_\mathsf{obs}^\ell \\
        0 &\text{if}~x \in \cX_\mathsf{obs} - \cX_\mathsf{obs}^\ell \\
        \frac{1}{2^{2H}-2^{H}} \cdot \frac{2^{2H}-2^{H} - \abs{\cX_\mathsf{obs}^\ell}}{2^{2H} - \abs{\cX_\mathsf{obs}}} &\text{if}~x \in \cX_{h+1} - \cX_\mathsf{obs}
    \end{cases}
\end{align*}
We elaborate on the calculation for the last probability in each case. Suppose $\ell = \sgood$. Then for any $x \in \cX_{h+1} - \cX_\mathsf{obs}$ which has not been observed yet we assign $\optdec(x) = \ell$ in
\begin{align*}
    &\binom{2^{2H} - \abs{\cX_\mathsf{obs}} -1}{2^{H} -\abs{\cX_\mathsf{obs}^\ell} - 1 } \text{ ways out of } \binom{2^{2H} - \abs{\cX_\mathsf{obs}}}{2^{H} -\abs{\cX_\mathsf{obs}^\ell } } \text{ assignments.} \\
    &\quad \Longrightarrow \optdec(x) = \sgood \text{ with probability } \frac{2^{H} -\abs{\cX_\mathsf{obs}^\ell}}{2^{2H} - \abs{\cX_\mathsf{obs}}}.
\end{align*}
For each assignment where $\optdec(x) = \sgood$ we will select it with probability $1/2^H$ since the emission is uniform, giving us the final probability as claimed. A similar calculation can be done for the case where $\ell = \sbad$.

Therefore we can calculate the final bound that
\begin{align*}
    \nrm*{\Pr^{\nu_0, \alg}\brk*{X_{t,h}'  \mid  \cF_{t, h-1}'} - \unif(\statesp_{h+1}) }_1 &= \sum_{x \in \cX_{h+1}} \abs*{\Pr^{\nu_0, \alg}\brk*{X_{t,h}' = x \mid  \cF_{t, h-1}'} - \frac{1}{2^{2H}}} \\
    &\le \begin{cases}
        \frac{\abs{\cX_\mathsf{obs}^\sgood}}{2^H} + \frac{\abs{\cX_\mathsf{obs}^\sbad} }{2^{2H}} + \abs*{\frac{2^H - \abs{\cX_\mathsf{obs}^\sgood}}{2^{H}} - \frac{2^{2H} - \abs{\cX_\mathsf{obs}}}{2^{2H}}} &\text{if}~\ell = \sgood, \\[0.5em]
        \frac{\abs{\cX_\mathsf{obs}^\sbad}}{2^{2H}- 2^H} + \frac{\abs{\cX_\mathsf{obs}^\sgood} }{2^{2H}} + \abs*{\frac{2^{2H} - 2^H - \abs{\cX_\mathsf{obs}^\sbad}}{2^{2H} - 2^H} - \frac{2^{2H} - \abs{\cX_\mathsf{obs}}}{2^{2H}}} &\text{if}~\ell = \sbad.
    \end{cases}\\
    &\le \frac{4 \cdot \abs{\cX_\mathsf{obs}}}{2^H} \le \frac{4 t}{2^H}.
\end{align*}
Since $\Pr^{\nu_0, \alg}\brk*{X_{t,h}'  \mid  X_{t,h}, A_{t,h}, \cF_{t, h-1}} = \En^{\nu_0, \alg} \Pr^{\nu_0, \alg}\brk{X_{t,h}'  \mid  \cF_{t, h-1}'}$, we have by convexity of TV distance and Jensen's inequality,
\begin{align*}
    \nrm*{\Pr^{\nu_0, \alg}\brk*{X_{t,h}'  \mid X_{t,h}, A_{t,h}, \cF_{t, h-1}} - \unif(\statesp_{h+1}) }_1 \le \frac{4 t}{2^H},
\end{align*}
which concludes the proof of \pref{lem:tv-distance-from-uniform-generative}.
\end{proof}

\begin{lemma}[Reward Bound for Construction in \pref{thm:lower-bound-coverability}]\label{lem:reward-bound-lb1} 
Let $T \le 2^H$. For any $t \in [T]$:
    \begin{align*}
        \En^{\nu_0, \alg} \brk*{\Pr^{\nu_0, \alg} \brk*{ R_{t,H} = 1 \mid  \cF_{t,H-1} } } &\le t \cdot \frac{HT^2}{2^{H-7}}.\\
        \En^{\nu_0, \alg} \brk*{ \Pr^{\nu_1, \alg} \brk*{ R_{t,H} = 1 \mid  \cF_{t,H-1} } } &\le t \cdot \frac{HT^2}{2^{H-7}}.
    \end{align*}
\end{lemma}

\begin{proof}[Proof of \pref{lem:reward-bound-lb1}]

To show the proof, we use induction to show that the probability of see nonzero reward remains small throughout the entire execution of $\alg$.

\paragraph{Peeling Off Bad Events.} First, we will peel off a couple ``bad'' events which occur with low probability:
\begin{itemize}
    \item Let $\eventfresh$ be the event that every freshly sampled observation (i.e., querying the generative model on some observation $X_{t,h} \notin \cF_{t,h-1}$) in any layer $h \ge 2$ has a bad label:
    \begin{align*}
        \eventfresh \coloneqq \crl*{\phi(X_{t,h}) = \sbad \text{ for every } t \in [T], h \ge 2, X_{t,h} \notin \cF_{t,h-1}}.
    \end{align*}
    We will show in \pref{lem:fresh-states} that due to the unbalanced sizes of every layer and the uniform distribution over decoders, $\eventfresh$ must occur with high probability. This captures the intuition that the generative model affords no additional power over local simulation, since data generated from states with a bad label are not informative, and with high probability all fresh samples have a bad label. 
    \item Let $\eventnew$ be the event that every sampled transition is a new observation that has never been seen before:
    \begin{align*}
        \eventnew \coloneq \crl*{X_{t,h}' \notin \cF_{t,h-1} \text{ for every }t \in [T], h \in [H] }.
    \end{align*} 
    We show in \pref{lem:repeated-transitions} that due to the large state space in every layer, $\eventnew$ also occurs with high probability, therefore capturing the intuition that transitions are not informative for learning the optimal policy $\optpi$.
\end{itemize}

We can compute that:
\begin{align*}
    \hspace{2em}&\hspace{-2em} \En^{\nu_0, \alg} \brk*{\Pr^{\nu_0, \alg} \brk*{ R_{t,H} = 1 \mid \cF_{t, H-1} } } \\
                &\le \Pr^{\nu_0, \alg}\brk*{\eventfresh^c} + \Pr^{\nu_0, \alg}\brk*{\eventnew^c} + \En^{\nu_0, \alg} \brk*{\ind{\eventfresh \wedge \eventnew} \Pr^{\nu_0, \alg} \brk*{ R_{t,H} = 1 \mid \cF_{t, H-1} } }\\
     &\le \frac{HT\cdot  2^H}{2^{2H}-2T} + \frac{HT^2}{2^H} + \En^{\nu_0, \alg} \brk*{\ind{\eventfresh \wedge \eventnew} \Pr^{\nu_0, \alg} \brk*{ R_{t,H} = 1 \mid \cF_{t, H-1} } }. \\
     &\le \frac{HT^2}{2^{H-2}} + \En^{\nu_0, \alg} \brk*{\ind{\eventfresh \wedge \eventnew} \Pr^{\nu_0, \alg} \brk*{ R_{t,H} = 1 \mid \cF_{t, H-1} } }, \numberthis \label{eq:peeled-reward-eq}
\end{align*}
where the second line uses \pref{lem:fresh-states} and \pref{lem:repeated-transitions}, and the last line uses the fact that $T = 2^{O(H)}$. 

We will show inductively that under the distribution $\Pr^{\nu_0, \alg}$, rewards are nonzero with exponentially small (in $H$) probability. Then we use this bound to prove the final guarantee.

\paragraph{Inductive Claim.} Let $\cE_{R,t}$ be the event that after the $t$-th episode, all of the observed rewards are zero, i.e., $\cE_{R,t} \coloneqq \crl{R_{t',H} = 0 \text{ for all } t' \le t}$. We claim that 
\begin{align*}
    \Pr^{\nu_0, \alg} \brk*{\cE_{R,t}^c \wedge \eventfresh \wedge \eventnew} \le  t \cdot \frac{HT^2}{2^{H-7}}. \numberthis \label{claim:claim-reward}
\end{align*}
We show this via induction. The base case of $t=0$ trivially holds. Now suppose that \pref{claim:claim-reward} holds for at episode $t-1$. We show that it holds at episode $t$. We calculate that
\begin{align*}
    \hspace{2em}&\hspace{-2em}    \Pr^{\nu_0, \alg} \brk*{ \cE_{R,t}^c \wedge \eventfresh \wedge \eventnew}\\ 
    &\overset{(i)}{\le} \Pr^{\nu_0, \alg} \brk*{ \cE_{R,t-1}^c \wedge \eventfresh \wedge \eventnew} + \En^{\nu_0, \alg} \brk*{\ind{\cE_{R,t-1} \wedge \eventfresh \wedge \eventnew} \Pr^{\nu_0, \alg} \brk*{R_{t,H} = 1 \mid \cF_{t, H-1}} } \\
    &\overset{(ii)}{\le} (t-1) \cdot \frac{HT^2}{2^{H-7}} + \En^{\nu_0, \alg} \brk*{\ind{\cE_{R,t-1} \wedge \eventfresh \wedge \eventnew} \Pr^{\nu_0, \alg} \brk*{R_{t,H} = 1 \mid \cF_{t, H-1}} } \numberthis \label{eq:induction-one-step}
\end{align*}
Here, inequality $(i)$ uses the fact that if we see zero reward in the first $t-1$ episodes, $\cE_{R,t}^c$ can only happen if $R_{t,H} = 1$; inequality $(ii)$ uses the inductive hypothesis.

Now we will provide a bound on the reward distribution. We can calculate that
\begin{align*}
    \hspace{2em}&\hspace{-2em}    \Pr^{\nu_0, \alg} \brk*{ R_{t,H} = 1 \mid \cF_{t, H-1} }\\ 
    &= \sum_{\phi \in \Phi} \Pr^{\nu_0, \alg} \brk*{ R_{t,H} = 1 \mid \phi, \cF_{t, H-1} } \Pr^{\nu_0, \alg} \brk*{ \phi \mid\cF_{t, H-1} } \\
    &\overset{(i)}{=} \sum_{\phi \in \Phi} \ind{ \phi(X_{t,H}) = \sgood \text{ and } A_{t,H} = 0} \Pr^{\nu_0, \alg} \brk*{ \phi \mid \cF_{t, H-1} } \\
    &\le \sum_{\phi \in \Phi} \ind{ \phi(X_{t,H}) = \sgood } \Pr^{\nu_0, \alg} \brk*{ \phi \mid \cF_{t, H-1} } =  \Pr^{\nu_0, \alg} \brk*{ \phi(X_{t,H}) = \sgood \mid \cF_{t, H-1} }. \numberthis \label{eq:reward-upper-bound-induction}
\end{align*}
For $(i)$ we use the fact that the event $R_{t,H}=1$ is measurable with respect to $\phi$ and $\cF_{t, H-1}$.

\paragraph{Dataset as a DAG.}
To further bound Eq.~\eqref{eq:reward-upper-bound-induction}, we take the following viewpoint: for any $t\in[T], h \in [H]$, the collected dataset $\cF_{t,h}$ can be viewed as directed acyclic graph (DAG) with set of vertices given by the observations in $\cF_{t,h}$. In this DAG, the edges are labeled with $A \in \crl{0,1}$, and we draw an edge $X \to X'$ with label $a$ if the sample $(X,A,X')$ exists in the dataset $\cF_{t,h}$. For any observation $x \in \statesp$ and filtration $\cF$, we define the root-layer operation $\rootlayer(x \mid \cF)$ to be minimum layer $h$ for which there exists some path in the DAG representation of $\cF$ from some $X_h \to x$ with $X_h \in \cF$. If $x \notin \cF$, we have the convention that $\rootlayer(x \mid \cF) = h(x)$. We also denote $\mathsf{Root}(x \mid \cF)$ to be any observation $X_h \in \cF \cup \crl{x}$ which witnesses $\rootlayer(x \mid \cF) = h$.

We can further calculate that
\begin{align*}
    \hspace{2em}&\hspace{-2em} \En^{\nu_0, \alg} \brk*{ \ind{\cE_{R,t-1} \wedge \eventfresh \wedge \eventnew} \Pr^{\nu_0, \alg} \brk*{ \phi(X_{t,H}) = \sgood \mid \cF_{t, H-1} } } \\ 
    &\le  \En^{\nu_0, \alg} \brk*{ \ind{\cE_{R,t-1} \wedge \eventfresh \wedge \eventnew} \Pr^{\nu_0, \alg} \brk*{ \substack{\text{exists a path $X_1 \to X_{t,H}$ in $\cF_{t,H-1}$} \\ \text{labeled by $\optpi_{1:H-1}$ }} \mid \cF_{t, H-1}} }. \numberthis \label{eq:exists-path}
\end{align*}
The inequality is shown as follows: if $\rootlayer(X_{t,H} \mid \cF_{t,H-1}) \ge 2$, then event $\eventfresh$ guarantees that any observation $X_h \in \cF_{t,H-1}$ which witnesses the value of $\rootlayer$ has a bad label $\phi(X_h) = \sbad$, so therefore we must also have $\phi(X_{t,H}) = \sbad$. Otherwise, if $\rootlayer(X_{t,H} \mid \cF_{t,H-1}) = 1$, then $\phi(X_{t,H}) = \sgood$ implies that the path $X_1 \to X_{t,H}$ which witnesses $\rootlayer = 1$ must be labeled by $\optpi_{1:H-1}$.

\paragraph{Analyzing the Posterior of $\optpi$.} To bound Eq.~\eqref{eq:exists-path}, we apply chain rule and a change of measure argument. 
\begin{align*}
    \hspace{2em}&\hspace{-2em} \ind{\cE_{R,t-1} \wedge \eventfresh \wedge \eventnew}\Pr^{\nu_0, \alg} \brk*{ \substack{\text{exists a path $X_1 \to X_{t,H}$ in $\cF_{t,H-1}$} \\ \text{labeled by $\optpi_{1:H-1}$ }} \mid \cF_{t, H-1}} \\
    &= \ind{\cE_{R,t-1} \wedge \eventfresh \wedge \eventnew} \\
    &\qquad \times \sum_{\pi \in \Pi_{1:H-1}} \Pr^{\nu_0, \alg} \brk*{ \substack{\text{exists a path $X_1 \to X_{t,H}$ in $\cF_{t,H-1}$} \\ \text{labeled by $\pi$ }} \mid \cF_{t, H-1}}  \cdot \Pr^{\nu_0, \alg} \brk*{ \optpi = \pi \mid \cF_{t, H-1}} \\
    &\le \frac{HT^2}{2^{H-6}} + \frac{1}{\abs{\Pi_{1:H-1} } } \sum_{ \pi \in \Pi_{1:H-1} } \Pr^{\nu_0, \alg} \brk*{ \substack{\text{exists a path $X_1 \to X_{t,H}$ in $\cF_{t,H-1}$} \\ \text{labeled by $\pi$ }} \mid \cF_{t, H-1}} \\
    &\le \frac{HT^2}{2^{H-7}}. \numberthis \label{eq:final}
\end{align*}
The first inequality follows by the calculation in \pref{lem:posterior-of-optimal-generative}, and the second inequality follows because there are at most $T$ paths in the DAG representation of $\cF_{t,H-1}$.

\paragraph{Completing Induction for \pref{claim:claim-reward}.} By combining Eqs.~\eqref{eq:induction-one-step}--\eqref{eq:final} we see that as long as $T \le 2^{O(H)}$, then
\begin{align*}
    \Pr^{\nu_0, \alg} \brk*{\cE_{R,t}^c} \le (t-1) \cdot \frac{HT^2}{2^{H-7}} + \frac{HT^2}{2^{H-7}} = t \cdot \frac{HT^2}{2^{H-7}}.
\end{align*}
This proves the claim.

\paragraph{Final Bounds for \pref{lem:reward-bound-lb1}.} To prove the first inequality, we have directly by \pref{claim:claim-reward}
\begin{align*}
\En^{\nu_0, \alg} \brk*{\Pr^{\nu_0, \alg} \brk*{ R_{t,H} = 1 \mid \cF_{t, H-1} } } \le \Pr^{\nu_0, \alg} \brk*{\cE_{R, t}^c} \le t \cdot \frac{HT^2}{2^{H-7}}.
\end{align*}
To prove the second inequality in the lemma statement, we can get a similar bound as Eq.~\eqref{eq:peeled-reward-eq}:
\begin{align*}
    \hspace{2em}&\hspace{-2em} \En^{\nu_0, \alg} \brk*{\Pr^{\nu_1, \alg} \brk*{ R_{t,H} = 1 \mid \cF_{t, H-1} } } \\
    &\le \frac{HT^2}{2^{H-2}} + \En^{\nu_0, \alg} \brk*{\ind{\eventfresh \wedge \eventnew }\cdot
    \Pr^{\nu_1, \alg} \brk*{ R_{t,H} = 1 \mid \cF_{t, H-1} } } \\
    &\le \frac{HT^2}{2^{H-2}} + \Pr^{\nu_0, \alg} \brk*{ \cE_{R, t-1}^c \wedge \eventfresh \wedge \eventnew} \\
    &\qquad + \En^{\nu_0, \alg} \brk*{\ind{\cE_{R,t-1}^c \wedge \eventfresh \wedge \eventnew }\cdot
    \Pr^{\nu_1, \alg} \brk*{ R_{t,H} = 1 \mid \cF_{t, H-1} } } \\
    &\le \frac{HT^2}{2^{H-2}} + (t-1) \frac{HT^2}{2^{H-5}} + \En^{\nu_0, \alg} \brk*{\ind{\cE_{R,t-1}^c \wedge \eventfresh \wedge \eventnew }\cdot
    \Pr^{\nu_1, \alg} \brk*{ R_{t,H} = 1 \mid \cF_{t, H-1} } },
\end{align*}
and from here one can replicate the above argument to get a bound on this quantity. The details are omitted. This concludes the proof of \pref{lem:reward-bound-lb1}.
\end{proof}

\begin{lemma}\label{lem:fresh-states}
    $\Pr^{\nu_0, \alg} \brk{\eventfresh^c} \le \frac{HT\cdot  2^H}{2^{2H}-2T}$.
\end{lemma}
\begin{proof}
Let us consider the set $\cI$ (which is a random variable that depends on the interaction of $\alg$ with $\nu_0$):
\begin{align*}
    \cI = \crl{(t,h): X_{t,h} \notin \cF_{t,h-1} }.
\end{align*}
We have
\begin{align*}
    \Pr^{\nu_0, \alg} \brk*{ \eventfresh^c } &\le \En^{\nu_0, \alg} \brk*{ \sum_{t=1}^T \sum_{h=2}^H \ind{(t,h) \in \cI \text{ and } \optdec(X_{t,h}) = \sgood}} \\
    &=  \sum_{t=1}^T \sum_{h=2}^H \En^{\nu_0, \alg} \brk*{\Pr\brk*{ (t,h) \in \cI \text{ and } \optdec(X_{t,h}) = \sgood \mid \cF_{t,h-1}} }\\
    &\le \sum_{t=1}^T \sum_{h=2}^H \En^{\nu_0, \alg} \brk*{\Pr\brk*{ \optdec(X_{t,h}) = \sgood \mid \cF_{t,h-1}, (t,h) \in \cI} }. \numberthis\label{eq:ess-upper-bound}
\end{align*}
Now we will bound the quantity $\Pr\brk*{ \optdec(X_{t,h}) = \sgood \mid \cF_{t,h-1}, (t,h) \in \cI}$ for any $t \in [T]$, $h \ge 2$. Consider the annotated filtration
\begin{align*}
    \cF_{t,h-1}' \coloneqq \sigma\prn*{ \cF_{t, h-1}, \crl*{\phi(X): X \in \cF_{t,h-1}} }
\end{align*}
which includes the decoder label for all observations seen thus far. We compute that for any $t\in[T], h \ge 2$:
\begin{align*}
    \Pr\brk*{ \optdec(X_{t,h}) = \sgood \mid \cF_{t,h-1}', (t,h) \in \cI} = \frac{2^H - \abs{\crl{X \in \cF_{t,h-1} : \optdec(X) = \sgood}}}{2^{2H} - 2t+1}, \numberthis\label{eq:new-state-upper-bound}
\end{align*}
since once we have fixed the value of the decoder on the $2t-1$ seen examples at layer $h$, the label of a new state is uniform over all remaining possibilities.

Continuing the calcuation from Eq.~\eqref{eq:ess-upper-bound}:
\begin{align*}
    \Pr^{\nu_0, \alg} \brk*{ \eventfresh^c }
    &\le \sum_{t=1}^T \sum_{h=2}^H \En^{\nu_0, \alg} \brk*{\Pr\brk*{ \optdec(X_{t,h}) = \sgood \mid \cF_{t,h-1}, (t,h) \in \cI} } \\
    &=  \sum_{t=1}^T \sum_{h=2}^H \En^{\nu_0, \alg} \brk*{ \En \brk*{ \Pr\brk*{ \optdec(X_{t,h}) = \sgood \mid \cF_{t,h-1}', (t,h) \in \cI} \mid \cF_{t,h-1}, (t,h) \in \cI } } \\
    &\le \sum_{t=1}^T \sum_{h=2}^H \frac{2^H}{2^{2H}-2T} \le \frac{HT \cdot 2^H}{2^{2H}-2T}.
\end{align*}
The second inequality uses Eq.~\pref{eq:new-state-upper-bound}. This completes the proof of \pref{lem:fresh-states}.
\end{proof}

\begin{lemma}\label{lem:repeated-transitions}
    $\Pr^{\nu_0, \alg} \brk{\eventnew^c} \le \frac{HT^2}{2^H}$.
\end{lemma}
\begin{proof}
Any sampled transition $X_{t,h}'$ has probability at most $T/2^H$ of being a repeated state (which is maximized if $X_{t,h}'$ has a good label and we have already sampled $T$ such observations from that given latent). Applying union bound over $T(H-1)$ transition samples gives us the final bound.
\end{proof}

\begin{lemma}[Posterior of $\optpi$]\label{lem:posterior-of-optimal-generative} 
    Fix any $t \in [T]$. Then 
    \begin{align*}
        \ind{\cE_{R,t-1} \wedge \eventfresh \wedge \eventnew} \cdot \nrm*{\Pr^{\nu_0, \alg} \brk*{\optpi = \cdot \mid \cF_{t, H-1}} - \unif(\Pi_{1:H-1})}_1 \le \frac{HT^2}{2^{H-6}}.
    \end{align*} 
\end{lemma}
\begin{proof}
In what follows all of the probabilities are taken with respect to $\Pr^{\nu_0, \alg}$. We can compute that
\begin{align*}
    \hspace{2em}&\hspace{-2em} \ind{\cE_{R,t-1} \wedge \eventfresh \wedge \eventnew} \cdot \nrm*{\Pr\brk*{\pi^\star = \cdot \mid \cF_{t, H-1} } - \unif(\Pi_{1:H-1})}_1 \\
    &= \ind{\cE_{R,t-1} \wedge \eventfresh \wedge \eventnew} \cdot \sum_{\pi \in \Pi_{1:H-1}} \abs*{ \Pr\brk*{\pi^\star = \pi \mid \cF_{t, H-1} } - \frac{1}{2^{H-1}} }\\
    &= \ind{\cE_{R,t-1} \wedge \eventfresh \wedge \eventnew} \cdot 2\sum_{\pi \in \Pi_{1:H-1}} \brk*{ \Pr\brk*{\pi^\star = \pi \mid \cF_{t, H-1} } - \frac{1}{2^{H-1}} }_{+} \\
    &= \ind{\cE_{R,t-1} \wedge \eventfresh \wedge \eventnew} \cdot \frac{2}{2^{H-1}} \sum_{\pi \in \Pi_{1:H-1}} \brk*{ \frac{ \Pr\brk*{\cF_{t, H-1} \mid \pi^\star = \pi } }{  \Pr\brk*{  \cF_{t, H-1} } } - 1 }_{+}\\
    &\le 2 \max_{\pi \in \Pi_{1:H-1}} \brk*{ \frac{\ind{\cE_{R,t-1} \wedge \eventfresh \wedge \eventnew} \cdot \Pr\brk*{\cF_{t, H-1} \mid \pi^\star = \pi } }{  \Pr\brk*{  \cF_{t, H-1} } } - 1 }_{+}. \numberthis \label{eq:posterior-ub-for-generative-model} 
\end{align*}
Now we will provide explicit calculations for the conditional distribution of $\cF_{t, H-1}$ for every choice of optimal policy $\pi \in \Pi_{1:H-1}$. Fix any $\cF_{t, H-1}$ such that $R_{i,H} = 0$ for all $i \in [t-1]$ and no repeated transitions (otherwise we can trivially upper bound Eq.~\eqref{eq:posterior-ub-for-generative-model} by 0). By chain rule we have
\begin{align*}
    \Pr\brk*{\cF_{t, H-1} \mid \pi^\star = \pi} = \prn*{ \prod_{i=1}^{t} \prod_{h=1}^{H-1} \Pr\brk*{X'_{i,h} \mid \optpi = \pi, \cF_{i, h-1} } } \times \prod_{i=1}^{t-1} \Pr\brk*{ R_{i,H} \mid \optpi = \pi, \cF_{i, H-1} }.
\end{align*}
We bound the transition and reward probabilities separately using \pref{claim:transition} and \pref{claim:reward}.
\begin{claim}\label{claim:transition}
Fix any $i \in [t]$ and $h \in [H-1]$. We have for every $\pi \in \Pi$:
\begin{align*}
    \Pr\brk*{X'_{i,h} \mid \optpi = \pi, \cF_{i, h-1} } \in \frac{1}{2^{2H}} \cdot \brk*{ \prn*{1 -  \frac{T}{2^H}}, \prn*{1 +  \frac{T}{2^H}} }.
\end{align*}
\end{claim}

To prove this claim, we can compute that
\begin{align*}
    \hspace{2em}&\hspace{-2em}\Pr\brk*{ X'_{i,h} \mid \optpi = \pi, \cF_{i, h-1} } \\
                &= \sum_{\ell \in \crl{\sgood, \sbad}} 
    \Pr\brk*{ X'_{i,h} \mid \optpi = \pi, \cF_{i, h-1}, \phi_h(X_{i,h}) = \ell} 
    \Pr\brk*{\phi_h(X_{i,h}) = \ell \mid \optpi = \pi, \cF_{i, h-1}}
\end{align*}
\underline{\emph{Case 1:  if $A_{i,h} = \optpi_h$.}} If we started in a good state then we would transition to the good state, so
\begin{align*}
    \hspace{2em}&\hspace{-2em}\Pr\brk*{ X'_{i,h} \mid \optpi = \pi, \cF_{i, h-1} }\\
    &= \Pr \brk*{ \phi_h(X_{i,h}) = \mathsf{g} \mid \optpi = \pi, \cF_{i, h-1}} 
    \cdot \frac{ \Pr \brk*{ \phi_{h+1}(X_{i,h}') = \mathsf{g} \mid \optpi = \pi, \cF_{i, h-1}, \phi_h(X_{i,h}) = \sgood }}{2^{2H} - 2^H} \\ 
    &\quad + \Pr \brk*{ \phi_h(X_{i,h}) = \sbad \mid \optpi = \pi, \cF_{i, h-1}} \cdot \frac{ \Pr \brk*{ \phi_{h+1}(X_{i,h}') = \sbad \mid \optpi = \pi, \cF_{i, h-1}, \phi_h(X_{i,h}) = \sbad} }{2^{H}} 
\end{align*}
\underline{\emph{Case 2:  if $A_{i,h} \ne \optpi_h$.}} In this case we know that regardless of the label of $X_{i,h}$ we transition to a bad state, so
    \begin{align*}
        \hspace{2em}&\hspace{-2em}        \Pr\brk*{ X'_{i,h} \mid \optpi = \pi, \cF_{i, h-1} }\\
        &= \Pr \brk*{ \phi_h(X_{i,h}) = \mathsf{g} \mid \optpi = \pi, \cF_{i, h-1}} 
        \cdot \frac{ \Pr \brk*{ \phi_{h+1}(X_{i,h}') = \sbad \mid \optpi = \pi, \cF_{i, h-1}, \phi_h(X_{i,h}) = \sgood  }}{2^{2H} - 2^H} \\
        &\quad + \Pr \brk*{ \phi_h(X_{i,h}) = \mathsf{b} \mid \optpi = \pi, \cF_{i, h-1}} \cdot \frac{ \Pr \brk*{ \phi_{h+1}(X_{i,h}') = \sbad \mid \optpi = \pi, \cF_{i, h-1}, \phi_h(X_{i,h}) = \sbad} }{2^{2H} - 2^H}
    \end{align*}
Either way, applying \pref{lem:latent-calculation-generative} concludes the proof of \pref{claim:transition}.

\begin{claim}\label{claim:reward}
    Fix any $i \in [t-1]$. We have for every $\pi \in \Pi$:
    \begin{align*}
        \ind{\eventfresh} \Pr\brk*{ R_{i,H} = 1 \mid \optpi = \pi, \cF_{i, H-1} } \le \ind{  \substack{\text{exists a path $X_1 \to X_H$ in $\cF_{t,H-1}$} \\ \text{labeled by $\pi$ }} }.
\end{align*}
\end{claim}

To prove this claim, we use casework.

\underline{\emph{Case 1: if $\rootlayer(X_{i,H} \mid \cF_{i, H-1}) \ge 2$.}} Then we must have
\begin{align*}
    \ind{\eventfresh} \cdot \Pr\brk*{ R_{i,H} = 1 \mid \optpi = \pi, \cF_{i, H-1} } &\le \ind{\eventfresh} \Pr\brk*{ \phi(\mathsf{Root}(X_{i,H})) = \mathsf{g} \mid \optpi = \pi, \cF_{i, H-1} } = 0.
\end{align*}
The equality holds because $\eventfresh \Rightarrow \crl{\phi(\mathsf{Root}(X_{i,H})) = \sbad}$. This proves \pref{claim:reward} in this case.

\underline{\emph{Case 2: if $\rootlayer(X_{i,h} \mid \cF_{i, H-1}) = 1$.}} In this case we can compare the path witnessing $\rootlayer = 1$ with the labeling $\optpi$, and we get
\begin{align*}
    \Pr\brk*{ R_{i,H} = 1 \mid \optpi = \pi, \cF_{i, H-1} } \le \ind{  \substack{\text{exists a path $X_1 \to X_H$ in $\cF_{t,H-1}$} \\ \text{labeled by $\pi$ }} }.
\end{align*}
This concludes the proof of \pref{claim:reward}.

With \pref{claim:transition} and \pref{claim:reward} in hand, we return to the analysis of the posterior $\Pr\brk*{\cF_{t, H-1}\mid \pi^\star = \pi}$. Letting $O \coloneqq t(H-1)$ be the number of transitions we observe in $\cF_{t, H-1}$, we get that
\begin{align*}
    \Pr\brk*{\cF_{t, H-1}\mid \pi^\star = \pi} \le \frac{1}{2^{2H \cdot O}} \prn*{1 +  \frac{T}{2^H}}^{HT}. \numberthis \label{eq:ft-upper}
\end{align*}
We also have the lower bound that
\begin{align*}
    \Pr\brk*{  \cF_{t, H-1}} \ge \frac{1}{2^{H-1}} \sum_{\pi \in \Pi_{1:H-1}} \Pr\brk*{\cF_{t, H-1}\mid \pi^\star = \pi} \ge \frac{2^{H-1} - T}{2^{H-1}}  \cdot \frac{1}{2^{2H \cdot O}} \prn*{1 -  \frac{T}{2^H}}^{HT}, \numberthis \label{eq:ft-lower}
\end{align*}
where the last inequality follows because for any filtration $\cF_{t, H-1}$ we must have 
\begin{align*}
    \ind{  \substack{\text{no path $X_1 \to X_H$ in $\cF_{t,H-1}$} \\ \text{labeled by $\pi$ }} } = 1
\end{align*}
for at least $2^{H-1} - T$ such policies in $\Pi_{1:H-1}$.

Putting Eq.~\eqref{eq:ft-upper} and \eqref{eq:ft-lower} together we get that
\begin{align*}
    \ind{\cE_{R,t-1} \wedge \eventfresh \wedge \eventnew} \cdot \frac{\Pr\brk*{\cF_{t, H-1}\mid \pi^\star = \pi } }{  \Pr\brk*{  \cF_{t, H-1}} } \le  \prn*{1 +  \frac{T}{2^{H-2}}}^{2HT+1},
\end{align*}
which in turn using Eq.~\eqref{eq:posterior-ub-for-generative-model} implies that
\begin{align*}
    \hspace{2em}&\hspace{-2em} \ind{\cE_{R,t-1} \wedge \eventfresh \wedge \eventnew} \cdot \nrm*{\Pr\brk*{\pi^\star = \cdot \mid \cF_{t, H-1} } - \unif(\Pi_{1:H-1})}_1 \\
    &\le 2 \max_{\pi \in \Pi_{1:H-1}} \brk*{ \frac{\ind{\cE_{R,t-1} \wedge \eventfresh \wedge \eventnew} \cdot \Pr\brk*{\cF_{t, H-1} \mid \pi^\star = \pi } }{  \Pr\brk*{  \cF_{t, H-1} } } - 1 }_{+} \\
    &\le  2  \prn*{\prn*{1 +  \frac{T}{2^{H-2}}}^{2HT+1} -1 } \le \frac{2HT^2 + T}{2^{H-3}} \exp \prn*{\frac{2HT^2 + T}{2^{H-2}}} \le \frac{HT^2}{2^{H-6}}.
\end{align*}
We use the numerical inequalities $1+y \le \exp(y)$ and $\exp(y) - 1 \le y \exp y$. This concludes the proof of \pref{lem:posterior-of-optimal-generative}.
\end{proof}

\begin{lemma}\label{lem:latent-calculation-generative}
    Let $\cF$ be any filtration of $HT$ generative model samples as well as annotations $\phi(x)$ for a subset of observations $x \in \cF$. Let $\pi \in \Pi_{1:H-1}$ be any policy. Fix any $h \ge 2$, and let $x_\mathrm{new} \in \cX_h - \cF$. Then
    \begin{align*}
        \abs*{ \Pr^{\nu_0, \alg} \brk*{  \phi(x_\mathrm{new}) = \sgood    \mid \cF, \optpi = \pi }  - \prn*{1 - \frac{1}{2^H}}} &\le \frac{T}{2^{H}},\\
        \abs*{ \Pr^{\nu_0, \alg} \brk*{ \phi(x_\mathrm{new}) = \sbad  \mid \cF, \optpi = \pi } - \frac{1}{2^H}} &\le \frac{T}{2^{H}}.
    \end{align*}
\end{lemma}

\begin{proof}[Proof of \pref{lem:latent-calculation-generative}]

Let us denote $\cF'$ to be the completely annotated $\cF$ which includes all labels $\crl{\phi(X): X \in \cF}$. We will show that the conclusion of the lemma applies to every completion $\cF'$, and since 
\begin{align*}
    \Pr^{\nu_0, \alg} \brk*{  \phi(x_\mathrm{new}) = \cdot   \mid \cF, \optpi = \pi } = \En^{\nu_0, \alg} \brk*{ \Pr^{\nu_0, \alg} \brk*{  \phi(x_\mathrm{new}) = \cdot   \mid \cF', \optpi = \pi } \mid \cF, \optpi = \pi},
\end{align*}
this will imply the result by Jensen's inequality and convexity of $\abs{\cdot}$.

We calculate the good label probability:
\begin{align*}
    \Pr^{\nu_0, \alg} \brk*{  \phi(x_\mathrm{new}) = \sgood    \mid \cF', \optpi = \pi } = \frac{2^H - \abs{\crl*{X \in \cF: \phi(X) = \sgood}}}{2^{2H} - \abs{\cF}}.
\end{align*}
For the lower bound we have
\begin{align*}
    \frac{2^H - \abs{ \crl*{X \in \cF: \phi(X) = \sgood}}}{2^{2H} - \abs{\cF}} \ge \frac{2^H - T}{2^{2H} } =  \frac{1}{2^H}\cdot \prn*{ 1 - \frac{T}{{2^{H}}} }.
\end{align*}
For the upper bound we have
\begin{align*}
    \frac{2^H - \abs{\crl*{X \in \cF: \phi(X) = \sgood}}}{2^{2H} - \abs{\cF}} \le \frac{2^H}{2^{2H} - T} =  \frac{1}{2^H}\cdot \prn*{ 1- \frac{T}{2^{2H}} }^{-1} \le \frac{1}{2^H}\cdot  \prn*{ 1 + \frac{T}{{2^{H}}} },
\end{align*}
which holds as long as $T \le 2^{H}$. Combining both upper and lower bounds proves the lemma for the good label. The calculation for $ \phi(x_\mathrm{new}) = \sbad$ is similar, so we omit it. This concludes the proof of \pref{lem:latent-calculation-generative}.
\end{proof}

\section{Bibliographical Remarks}
\paragraph{Exponential Lower Bounds for RL.} We highlight the influential paper \cite{du2019good}, which shows an exponential (in horizon $H$) lower bound for policy learning with a generative model when the policy class $\Pi$ is comprised of \emph{linear policies} which take the form $\pi_\theta(x) \coloneqq \argmax_{a \in \actionsp} \phi(x,a)^\top \theta$. In fact, our lower bound of \pref{thm:generative_lower_bound} recovers these results as a special case, since it is a well known fact that linear policies have infinite threshold dimension, and spanning capacity is lower bounded by threshold dimension (\pref{thm:bounds-on-spanning}).

\paragraph{Policy Eluder Dimension.}
\cite{mou2020sample} study policy learning with a generative model when the policy class has bounded policy eluder dimension (\pref{def:policy-eluder})---see their Thm.~1. However, their result requires additional assumptions: that the policy class contains the optimal policy and that the optimal value function has a gap.

\chapter{Online RL}\label{chap:online}

Now we study the online RL interaction protocol. An overview of our results is provided in \pref{sec:online-overview}. First, we show in \pref{sec:online-lower-bound} that unlike the generative model, bounded spanning capacity by itself is insufficient for PAC RL with online RL interaction. Next, to circumvent this lower bound, we introduce and study a new structural complexity measure called the sunflower property in \pref{sec:online-upper-bound}. We give a positive result on the sample complexity of PAC RL for policy classes which satisfy bounded spanning capacity and the sunflower property. We conclude with several open problems in \pref{sec:online-open-problems}. Proofs of our main results are deferred to \pref{sec:online-deferred-proofs}.

\section{Overview}\label{sec:online-overview}
We define $\non(\Pi; \eps, \delta)$ to be the minimax sample complexity of online RL (c.f.~\pref{def:generative-minimax}). 
\begin{definition}\label{def:online-minimax}
Fix any $(\eps, \delta) \in (0,1)$ as well as a policy class $\Pi$. We denote the minimax sample complexity $\non(\Pi; \eps, \delta)$ to be the smallest $n \in \bbN$ for which there exists an algorithm $\alg$ which for any MDP $M$ collects at most $\non(\Pi; \eps, \delta)$ samples using online RL access to $M$ and returns an $\eps$-optimal policy for $M$ with probability at least $1-\delta$.
\end{definition}

\paragraph{A Tempting Conjecture.} The central question we ask in this chapter is:
\begin{center}
    \textit{Is $\non(\Pi; \eps, \delta)$ always $\Theta(\poly(\dimRL(\Pi)))$ for every policy class $\Pi$?}
\end{center}
The lower bound is already clear. Since online RL is at least as hard as learning with a generative model, by \pref{thm:generative_lower_bound}, we know that $\non(\Pi; \eps, \delta) \ge \ngen(\Pi; \eps, \delta) \ge \Omega(\dimRL(\Pi))$. But is spanning capacity also sufficient? In other words, can we design an algorithm which always achieves $\poly(\dimRL(\Pi))$ sample complexity with online RL access? 

Note that the $\mathsf{TrajectoryTree}$ algorithm used to prove \pref{thm:generative_upper_bound} crucially relied on the ability to perform local resets. If the MDP has deterministic transitions (and fixed initial state $x_1$), then online RL access can be used to perform local resets, since for any state $x$ which the learner has previously observed, we know the exact sequence of actions which will land the learner in that state again. In stochastic MDPs over a large state space, the learner might never see the same state again, and it seems unlikely that the $\mathsf{TrajectoryTree}$ algorithm can be adapted to online RL access.

We now state the key results in this chapter which address this central question.
\begin{enumerate}
    \item In \pref{thm:lower-bound-online}, we give a lower bound which refutes the conjecture: namely, we construct a policy class $\Pi$ for which the minimax sample complexity is at least
        \begin{align*}
        \non(\Pi; \eps, 1/8) = \eps^{-\Omega\prn{\log \dimRL(\Pi)}} .
        \end{align*}
    In particular, this rules out any algorithm that achieves $\poly(\dimRL(\Pi))$ sample complexity. The lower bound exemplifies the challenge of policy learning in large state spaces with online interaction: since the learner is unlikely to to see the same state more than once, they cannot perform (optimistic) exploration to reduce the uncertainty about the optimal policy. On a technical level, \pref{thm:lower-bound-online} relies on a probabilistic construction of certain ``block-free'' matrices, which may also be of independent interest.
\item Interestingly, for all of the policy class examples considered in \pref{sec:C-pi}, it \emph{is} possible to show $\poly(\dimRL(\Pi))$ sample complexity. In \pref{sec:online-upper-bound}, we show that all of these policy class examples satisfy a general structural condition which we call the \emph{sunflower property}. The sunflower property has two parameters $K$ and $D$. For these examples we have $K, D = \poly(H)$ (in fact, the policy class from \pref{thm:lower-bound-online} is the only construction we know of which has $\dimRL(\Pi) = \poly(H)$ but $K,D = 2^{\poly(H)}$). Intuitively, the sunflower property allows us to leverage the two main algorithmic tools available: \emph{tabular algorithms} which adaptively explore the environment when encountering previously-seen states, and \emph{importance sampling}, which enables sample-efficient estimation of many policies simultaneously.

    In \pref{thm:sunflower}, we design a new algorithm called $\mathsf{POPLER}$ which exploits the sunflower structure of a policy class to achieve a sample complexity which is
    \begin{align*}
        \poly \prn*{ \dimRL(\Pi), K, D, \eps^{-1}, \log \abs{\Pi}, \delta^{-1} }.
    \end{align*}
    To perform policy evaluation over a large policy class $\Pi$, $\mathsf{POPLER}$ utilizes a new technical tool called the \emph{policy-specific Markov Reward Process (MRP)}, which for every $\pi \in \Pi$ collapses the original MDP onto a smaller tabular MRP. The transitions/rewards of every policy-specific MRP can be simultaneously estimated via an approach that interpolates between tabular algorithms and importance sampling. 
\end{enumerate}

\section{Spanning Capacity is Insufficient}\label{sec:online-lower-bound}
We prove a somewhat surprising negative result showing that bounded \Compname{} by itself is insufficient to characterize the minimax sample complexity in online RL. In particular, we provide an example for which we have a \emph{superpolynomial} (in \(H\)) lower bound on the number of trajectories needed for learning in online RL, that is not captured by any polynomial function of \Compname{}. This implies that, contrary to RL with a generative model, one can not hope for \(\non(\Pi; \eps, \delta) = \wt{\Theta}\prn{\poly\prn{\dimRL(\Pi), H, \eps^{-1}, \log \delta^{-1}}}\) in online RL.

\begin{theorem}[Lower Bound for Online RL]
\label{thm:lower-bound-online}
Fix any sufficiently large $H$. Let $\eps \in \prn{1/2^{\cO(H)},\cO(1/H)}$ and $\ell \in \crl{2, \dots, H}$ such that $1/\eps^\ell \le 2^H$. There exists a policy class $\Piell$ of size $\cO(1/\eps^\ell)$ with $\dimRL(\Piell) \le \cO(H^{4\ell+2})$ and a family of MDPs $\cM$ with state space $\statesp$ of size $2^{\cO(H)}$, binary action space, and horizon $H$ such that: for any $(\eps, 1/8)$-PAC algorithm, 
there exists an MDP $M \in \cM$ for which the algorithm must collect at least
\begin{align*}
    \Omega\prn*{\min \crl*{\frac{1}{\eps^\ell}, 2^{H/3} } } \quad \text{online trajectories in expectation.}
\end{align*}
\end{theorem}
Informally speaking, the lower bound shows that there exists a policy class \(\Pi\) for which \(\non(\Pi; \eps, \delta) = \eps^{-\Omega(\log_H \dimRL(\Pi))} \). In order to interpret this theorem, we can instantiate choices of $\epsilon = 1/2^{\sqrt{H}}$ and $\ell = \sqrt{H}$ to show an explicit separation.
\begin{corollary} For any sufficiently large $H$, there exists a policy class $\Pi$ with $\dimRL(\Pi) = 2^{\cO(\sqrt{H} \log H)}$ such that for any $(1/2^{\sqrt{H}}, 1/8)$-PAC algorithm, there exists an MDP for which the algorithm must collect at least $2^{\Omega(H)}$ online trajectories in expectation.
\end{corollary}

\pref{thm:lower-bound-online} shows that (1) online RL is \emph{strictly harder} than RL with generative access, and (2) online RL for stochastic MDPs is \emph{strictly harder} than online RL for MDPs with deterministic transitions. 

\emph{Remark.} The seminal work of \citet{foster2021statistical} provides a unified complexity measure called Decision-Estimation Coefficient (DEC) that characterizes the complexity of model-based RL. Given a realizable model class $\cM$, they give an upper bound of $\mathsf{DEC} \cdot \log \abs{\cM}$. However, in our policy learning setup, this guarantee is too loose to be useful, since the size of the implicit model class is at least $\log \abs{\cM} \propto \abs{\statesp}$. Followup works \cite{foster2024model, chen2024assouad, chen2025decision} study whether the $\log \abs{\cM}$ dependence can be improved. We observe that \pref{thm:lower-bound-online} rules out an improvement of $\mathsf{DEC} \cdot \log \abs{\Pi_\cM}$, where $\Pi_\cM$ is the induced policy class. This follows because it is know that coverability is an upper bound on the DEC \cite{xie2022role}, and our lower bound rules out $\poly(\ccov, \log \abs{\Pi})$ guarantees.

\begin{figure}[!t]
    \centering
\includegraphics[scale=0.275, trim={2cm 1.5cm 5cm 0.5cm}, clip]{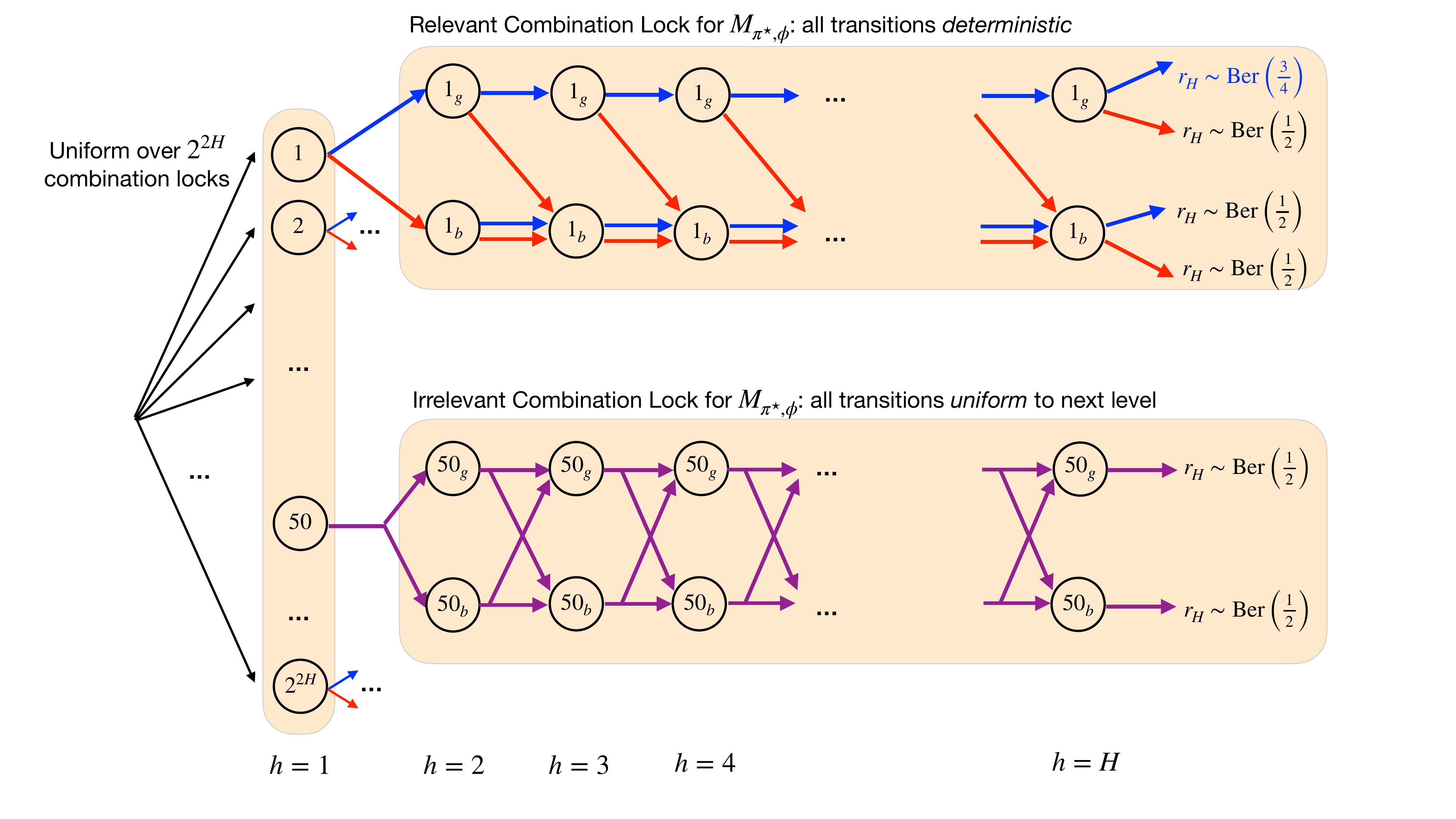}
    \caption{Illustration of the lower bound from \pref{thm:lower-bound-online}. \textcolor{blue}{Blue} arrows represent taking the action $\pistar(x)$, while \textcolor{red}{red} arrows represent taking the action $1- \pistar(x)$. \textcolor{purple}{Purple} arrows denote uniform transition to the states in the next layer, regardless of action. The MDP $M_{\pistar, \phi}$ is a uniform distribution of $2^{2H}$ combination locks of two types. In the \emph{relevant} combination locks (such as Lock 1 in the figure), following $\pistar$ keeps one in the ``good'' chain and gives reward of $\mathrm{Ber}(3/4)$ in the last layer, while deviating from $\pistar$ leads one to the ``bad'' chain and gives reward of $\mathrm{Ber}(1/2)$. In \emph{irrelevant} combination locks (such as Lock 50 in the figure), the next state is uniform regardless of action, and all rewards at the last layer are $\mathrm{Ber}(1/2)$.}
    \label{fig:lower-bound-idea}
\end{figure}

\paragraph{Proof Sketch for \pref{thm:lower-bound-online}.} We defer the full proof of \pref{thm:lower-bound-online} to \pref{sec:proof-lower-bound-online} and sketch the main ideas here.  An illustration of an MDP in the family $\cM$ can be found in \pref{fig:lower-bound-idea}.

The basic building block for our lower bound is the combination lock, a prototypical construction used in prior works \citep{krishnamurthy2016pac,du2019provably, sekhari2021agnostic}. The MDPs we construct are essentially a uniform distribution over $2^{2H}$  different combination locks. In order to receive positive feedback, the learner must play a sequence of $H$ correct actions in a particular combination lock. Intuitively, figuring out this sequence requires multiple revisits to the same lock. However, under online access, it is  unlikely that the learner will get to see the same lock multiple times unless they use an exponential number of samples. Note that under generative access, this is not an issue, since the learner can reset to any state they like.

In more detail, each hard MDP $M_{\pistar, \phi} \in \cM$ is parameterized by a policy $\pistar \in \Piell$ (which is optimal for that MDP) and a decoder $\phi: \statesp \mapsto \crl{\textsc{Good}, \textsc{Bad}}$. In the MDP $M_{\pistar, \phi}$ there will be a \emph{planted set} of ``relevant'' combination locks on which running $\pistar$ achieves $\ber(3/4)$ reward; on the rest of the combination locks any policy $\pi \in \cA^\statesp$ achieves $\ber(1/2)$ reward. Since the planted set is an $\eps$-fraction of the total, the learner must solve an \(\Omega(\epsilon)\)-fraction of the relevant locks in order to find an $O(\eps)$-optimal policy. However, as we have established, the learner will never see the same combination lock multiple times, so their only hope is to try to identify $\pistar$ through alternative means (e.g.~via elimination).

In the vanilla combination lock, it becomes easy to learn $\pistar$ via trajectory data, since once the learner observes a jump to the ``bad'' chain, they can immediately eliminate many candidate policies. Our construction utilizes a \emph{contextual} variant of the combination lock which minimizes information leakage about $\pistar$ from transition data. This is formalized by the decoder $\phi$, which randomly assigns states to be in the ``good'' chain and the ``bad'' chain. In this way, the learner, upon a single visit to a certain lock, cannot know when they have jumped to the ``bad'' chain (or even whether they are in a planted combination lock or not!) unless they can identify whether the reward at level $H$ in that combination lock is $\ber(1/2)$ instead of $\ber(3/4)$.

The last key to the puzzle is to prove that there exists such a policy class $\Piell$ which satisfies these properties yet still has bounded \Compname{}. We reduce this problem to showing the existence of certain ``block-free'' binary matrices, whose existence is shown using the probabilistic method. Further discussion on this is deferred to \pref{sec:remarks-zaran}.

\emph{Remark.} On a technical level, the lower bound argument relies on a stopping-time argument pioneered in the works \cite{garivier2019explore, domingues2021episodic, sekhari2021agnostic}. It is possible to prove the result using the machinery of interactive lower bounds \cite{chen2024assouad}, as is done in several other results in this thesis (e.g., \pref{thm:lower-bound-coverability} and \ref{thm:lower-bound-policy-completeness}). 

\section{Spanning Capacity + Sunflower Property Enable Statistically Efficient Learning}\label{sec:online-upper-bound}
The lower bound in \pref{thm:lower-bound-online} suggests that further structural assumptions on \(\Pi\)  are needed for statistically efficient agnostic PAC RL under the online interaction protocol. Essentially, the lower bound construction provided in \pref{thm:lower-bound-online} is hard to learn because any two distinct policies \(\pi, \pi' \in \Pi\) can differ substantially on a large subset of states (of size at least \(\epsilon \cdot 2^{2H}\)). Thus, we cannot hope to learn ``in parallel'' via a low variance importance sampling strategy that utilizes extrapolation to evaluate all policies $\pi \in \Pi$. 

In the sequel, we consider the following \coreset{} property to rule out such problematic scenarios, and show how bounded \Compname{} along with the \coreset{} property enable sample-efficient PAC RL with the online interaction protocol. The \coreset{} property only depends on the state space, action space, and policy class, and is independent of the transition dynamics and rewards of the underlying MDP. We first define a petal, a key ingredient of a sunflower.  

\begin{definition}[Petal] 
\label{def:petal_policy}
For a policy set \(\bar{\Pi}\), and states \(\bar{\statesp} \subseteq \statesp\), a policy \(\pi\) is said to be a \(\bar{\statesp}\)-\textit{petal} on \(\bar{\Pi}\) if for all \(1 \le h \leq h' \leq H\), and partial trajectories $\tau = (x_h, a_h, \cdots, x_{h'}, a_{h'})$ that are consistent with $\pi$: either \(\tau\) is also consistent with some \(\pi' \in \bar{\Pi}\), or there exists \(i \in (h, h']\) s.t.~$x_i\in \bar{\statesp}$. 
\end{definition}  

Informally, \(\pi\) is a \(\bar{\statesp}\)-petal on \(\bar{\Pi}\) if any trajectory that can be obtained using \(\pi\) can either also be obtained using a policy in \(\bar{\Pi}\)  or must pass through \(\bar{\statesp}\). Thus, any policy that is a \(\bar{\statesp}\)-petal on \(\bar{\Pi}\) can only differentiate from \(\bar{\Pi}\) in a structured way. A policy class is said to be a sunflower if it is a union of petals as defined below: 

\begin{definition}[Sunflower]  
\label{def:core_policy} 
    A policy class $\Pi$ is said to be a \((K, D)\)-\coreset{} if there exists a set \(\Picore\) of Markovian policies with $|\Picore|\le K$ such that for every policy $\pi\in \Pi$ there exists a set $\statesp_\pi \subseteq \statesp$, of size at most \(D\), so that \(\pi\) is an \(\statesp_\pi\)-petal on \(\Picore\). 
\end{definition}



Our next theorem provides a sample complexity bound for agnostic PAC RL for~policy classes that have \((K, D)\)-sunflower structure. This bound is obtained via a new exploration algorithm called $\mathsf{POPLER}$ that takes as input the set \(\Picore\) and corresponding petals \(\crl{\statesp_\pi}_{\pi \in \Pi}\) and leverages importance sampling as well as reachable state identification techniques to simultaneously estimate the value of every policy in \(\Pi\). Algorithm details are deferred to \pref{sec:algorithm_description}.

\begin{theorem} \label{thm:sunflower} 
Let \(\epsilon, \delta > 0\). Suppose the policy class \(\Pi\) satisfies \(\dimRL(\Pi)\) and is a \((K, D)\)-\sunflower. Then, for any MDP \(M\), with probability at least \(1 - \delta\), $\mathsf{POPLER}$ (\pref{alg:main}) succeeds in returning a policy \(\wh \pi\) that satisfies  \(V^{\wh \pi} \geq \max_{\pi \in \Pi} V^\pi - \epsilon\), after collecting 
\begin{align*}
\widetilde{\cO}\prn*{\prn*{\frac{1}{\epsilon^2} + \frac{HD^6 \dimRL(\Pi)}{\epsilon^4}} \cdot K^2 \log\frac{|\Pi|}{\delta}} \quad \text{online trajectories in \(M\).} 
\end{align*}
\end{theorem}
 The proof of \pref{thm:sunflower}, and the corresponding hyperparameters in $\mathsf{POPLER}$ needed to obtain the above bound, can be found in \pref{sec:upper_bound_main}. Before diving into the algorithm and proof details, let us highlight several key aspects of the above sample complexity bound: 

\begin{enumerate}[label=\(\bullet\)]
    \item Note that a class \(\Pi\) may be a \((K, D)\)-\coreset{} for many different choices of \(K\) and \(D\). For example, every policy class trivially satisfies the sunflower property with
        \begin{align*}
            \prn*{K= \min\crl{A^H, \abs{\Pi}}, D=0} \quad \text{or} \quad \prn*{K=0, D=\abs{\statesp}}.
        \end{align*}
        These extremes correspond to a vanilla importance sampling approach or a tabular approach.
        Barring computational issues, one can enumerate over all choices of $\Picore$ and $\crl{\statesp_\pi}_{\pi \in \Pi}$ and check if $\Pi$ is a \((K, D)\)-\coreset{} for that choice of $K = \abs{\Picore}$ and $D = \max_{\pi \in \Pi} \abs{\statesp_\pi}$. Since our bound in \pref{thm:sunflower} scales with \(K\) and \(D\), we are free to choose \(K\) and \(D\) to minimize the corresponding sample complexity bound. In this sense, the sunflower property maps out a \emph{pareto frontier of algorithms} which interpolate between the vanilla importance sampling and tabular approaches. 
    \item In order to achieve polynomial sample complexity in \pref{thm:sunflower}, both \(\dimRL(\Pi)\) and \((K, D)\) are required to be \(\poly(H, \log\abs{\Pi})\). All of the policy classes considered in \pref{sec:C-pi} have the sunflower property, with both \(K, D = \poly(H)\), and thus our sample complexity bound apply to all these classes. See \pref{sec:sunflower-property-examples} for details. 
    \item Notice that for \pref{thm:sunflower} to hold, we need both bounded \Compname{} as well as the sunflower structure on the policy class with bounded \((K, D)\). Thus, one may wonder if we can obtain a similar polynomial sample complexity guarantee in online RL  under weaker assumptions. In \pref{thm:lower-bound-online}, we already showed that bounded \(\dimRL(\Pi)\) alone is not sufficient to obtain polynomial sample complexity in online RL. Likewise, as we show in \pref{sec:sunflower-property-examples}, sunflower property with bounded \((K, D)\) alone is also not sufficient for polynomial sample complexity, and hence both assumptions cannot be individually removed. However, it is an interesting question if there is some other structural assumption that combines both spanning capacity and the sunflower property, and is both sufficient and necessary for agnostic PAC learning in online RL. See \pref{sec:online-open-problems} for further discussions on this. 
\end{enumerate}

\paragraph{How does the sunflower property enable sample-efficient learning?}
Intuitively, the \coreset{} property captures the intuition of simultaneous estimation of all policies \(\pi \in \Pi\) via importance sampling (IS), and allows control of both bias and variance. Let \(\pi\) be a \(\statesp_\pi\)-petal on \(\Picore\). Any trajectory $\tau \cons \pi$ that avoids \(\statesp_\pi\) must be consistent with some policy in \(\Picore\), and will thus be covered by the data collected using \(\pi' \sim \unif(\Picore)\). Thus, using IS with variance scaling with \(K\), one can create a biased estimator for \(V^\pi\), where the bias is \emph{only due} to trajectories that pass through \(\statesp_\pi\). There are two cases. 
\begin{itemize}
    \item  If every state in \(\statesp_\pi\) has small reachability under \(\pi\), i.e.~$d^\pi(x) \ll \eps$ for every \(x \in \statesp_\pi\), then the IS estimate will have a low bias (linear in $\abs{\statesp_\pi}$), so we can compute \(V^\pi\) up to error at most \(\epsilon \abs{\statesp_\pi}\).
    \item On the other hand, if $d^\pi(x)$ is large for some $x \in \statesp_\pi$, it is possible to explicitly control the bias that arises from trajectories passing through them since there are at most $D$ of them. We accomplish this by explicit exploration starting from this state $x$. 
\end{itemize}

\subsection{Algorithm and Proof Ideas}  \label{sec:algorithm_description}

$\mathsf{POPLER}$ (\pref{alg:main}) takes as input a policy class $\Pi$, as well as sets \(\Picore\) and \(\crl{\statesp_{\pi}}_{\pi \in \Pi}\), which can be computed beforehand by enumeration. $\mathsf{POPLER}$ has two phases: a \emph{state identification phase}, where it finds ``petal'' states $x \in \bigcup_{\pi \in \Pi} \statesp_\pi$ that are reachable with sufficiently large probability; and an \emph{evaluation phase} where it computes estimates $\wh{V}^\pi$ for every $\pi \in \Pi$. It uses three subroutines $\datacollector$, $\estreach$, and $\evaluate$, whose pseudocodes are stated in \pref{sec:algorithm_details}.

\paragraph{Key Tool: Policy-Specific Markov Reward Process (MRP).} MRPs are key technical tools used by our main algorithm.  An MRP  \(\MRPsign = \mathrm{MRP}(\statesp, P, R, H, x_\top, x_\bot)\) is defined over the state space \(\statesp\) with start state \(x_\top\) and end state \(x_\bot\),  for trajectory length \(H + 2\). Without loss of generality, we assume that \(\crl{x_\top, x_\bot} \in \statesp\). The transition kernel is denoted by \(P: \statesp \times \statesp \to [0,1] \), such that for any $x \in \statesp$, $\sum_{x'} P_{x \to x'} = 1$. The reward kernel is denoted \(R: \statesp \times \statesp \to \Delta([0,1])\). Throughout, we use the notation $\rightarrow$ to signify that the transitions and rewards are defined along the edges of the MRP. At an intuitive level, an MRP is an MDP with singleton action space, i.e.~the actions have no effect. For technical reasons, the state space in the MRPs is not layered.

Now we elaborate on the MRP construction, which is the key technical tool used in both the identification and evaluation phases of the algorithm. To build intuition, let us consider a fixed policy $\pi \in \Pi$ with petal states $\statesp_\pi$ and define a population version of policy-specific MRP. In particular, let \(\statesp_\pi^+ = \statesp \cup \crl{x_\bot, x_\top}\), and associated with $\pi$ define $\MRPsign^\pi = \mathrm{MRP}(\statesp^+_{\pi}, P^\pi, R^\pi, H, x_\top, x_\bot)$  which essentially compresses the transition and reward information in the original MDP relevant to the policy \(\pi\). 

For any states \(x \in \statesp_\pi \cup \crl{x_\top}\) and $x' \in \statesp_\pi \cup \crl{x_\bot}$ residing in different layers $h < h'$ in the underlying MDP,\footnote{For the simplicity of analysis, we slightly abuse the notation and assume that all trajectories in the underlying MDP start at \(x_\top\) at time step \(0\) and terminate at \(x_\bot\) at time step \(H+1\), and do not observe \(x_\bot\) and \(x_\top\) in between from time steps \(h = 1, \dots, H\). However, recall that \(x_\bot\) and \(x_\top\) are not part of the original layered state space \(\statesp\) for the MDP. } we define: 

\begin{itemize}
    \item \textbf{Transition $\bm{P_{x\to x'}^\pi}$} as:
    \begin{align*} 
        P_{x\to x'}^\pi \coloneqq \bbP^\pi \brk*{\substack{ \text{$\tau_{h:h'}$ goes from \(x\) to \(x'\)} \\ \text{without passing through any other $x'' \in \statesp_\pi$} } \mid \text{$\tau_h = x$}}.
    \end{align*} 

    \item \textbf{Rewards $\bm{R_{x\to x'}^\pi}$} as:
       \begin{align*}
        R_{x\to x'}^\pi \coloneqq \En^\pi \brk*{ R(\tau_{h:h'})\ind{\substack{ \text{$\tau_{h,h'}$ goes from \(x\) to \(x'\)} \\ \text{without passing through any other $x'' \in \statesp_\pi$ } }} \mid \text{$\tau_h = x$}}.
    \end{align*}
\end{itemize}

where $\tau_{h:h'}$ denotes a partial trajectory from layer $h$ to $h'$, and $R(\tau_{h:h'})$ denotes the cumulative rewards from layer $h$ to $h'$ along the partial trajectory $\tau_{h:h'}$.  Furthermore, 
\(P_{x_\bot \to x_\bot}^\pi = 1\) and \(R_{x_\bot \to x_\bot}^\pi = 0\).


The key technical benefit of using policy-specific MRPs is that the value $V^\pi$ for the policy \(\pi\) in the original MDP is identical to the value of policy-specific MRP \(\MRPsign^\pi\) (starting from \(x_\bot\)). Thus, if one knew the transitions and rewards in $\MRPsign^\pi$, one could calculate the value of the policy \(\pi\) via dynamic programming on $\MRPsign^\pi$. Of course, we do not know these quantities, so we must estimate them by interacting with the original MDP. A naive approach is to simply run $\pi$ many times to get estimates for each transition and reward---but since we want to estimate $V^\pi$ for every $\pi \in \Pi$ simultaneously, this approach would incur an $\abs{\Pi}$ dependency in the sample complexity. Instead, our algorithm uses importance sampling to estimate transitions and rewards in the corresponding $\MRPsign^\pi$ for many policies simultaneously.

We next describe the key algorithmic ideas, as well as the empirical estimation of policy-specific MRPs, in the two phases of \pref{alg:main}. 

\begin{algorithm}[!t] 
    \caption{{\bf{P}}olicy {\bf{OP}}timization by {\bf{L}}earning $\boldsymbol\eps$-{\bf{R}}eachable States ($\mathsf{POPLER}$)}  \label{alg:main} 
    \begin{algorithmic}[1]  
            \Require Policy class \(\Pi\), Sets \(\Picore\) and \(\crl{\statesp_\pi}_{\pi \in \Pi}\),  Parameters \(K, D, n_1, n_2, \epsilon, \delta\). 
            \State Define an additional start state \(x_\top\) (at \(h = 0\)) and end state \(x_\bot\) (at \(h = H+1\)). 
           \State Initialize $\reachablestates = \crl{x_\top}$, $\cT \leftarrow \crl{(x_\top, \mathrm{Null})}$; for every \(\pi \in \Pi\), set \(\statesp_\pi^+ \ldef{} \statesp_\pi \cup \crl{x_\top, x_\bot}\). 
           \label{line:initialization}
        \State $\mD_\top \leftarrow \datacollector(x_\top, \mathrm{Null}, \Picore, n_1)$
        \vspace{0.5em}
\\
\algcommentbig{Identification of Petal States that are Reachable with \(\Omega(\epsilon/D)\) Probability}  
             \While{$\mathsf{Terminate}=\mathsf{False}$}  \label{line:while_loop}     \hfill 
             \State Set \(\mathsf{Terminate}=\mathsf{True}\).
     \For{$\pi\in \Pi$} \label{line:while_loop_start}            
                \State Compute the sets: \label{line:S_sets_comp} 
                \vspace{-1em}
                \begin{align*}
                    \text{already-explored reachable states:}\quad \SHPalg{\pi} &= \statesp^+_\pi \cap \reachablestates\\
                    \text{remaining states:}\quad \SRemalg{\pi} &= \statesp_{\pi} \setminus (\SHPalg{\pi}\cup \crl{x_\bot})
                \end{align*}
                \vspace{-1.75em}
                 \State Estimate the policy-specific MRP $\widehat \MRPsign^\pi_{\reachablestates}$ according to \eqref{eq:emp-transitions} and \eqref{eq:emp-rewards}. \label{line:est-policy-specific-mrp}
            \For{$\bar{x}\in\SRemalg{\pi}$} \label{line:state_identification} 
                \State Estimate prob of reaching \(\bar{x}\) under \(\pi\): $\widehat{d}^\pi (\bar{x})\leftarrow \estreach(\statesp_\pi^+, \widehat \MRPsign^\pi_{\reachablestates}, \bar{x})$. \label{line:DP_solver_search}  \label{line:DP}
                \If{$\widehat{d}^\pi(\bar{x})\ge \nicefrac{\epsilon}{6D}$}   \label{line:test_and_add}
                    \State Update $\reachablestates \gets \reachablestates \cup \crl{\bar{x}}$, $\cT\leftarrow \cT \cup \crl{(\bar{x}, \pi)}$ and set  \(\mathsf{Terminate}=\mathsf{False}\). 
                    \label{line:add_s} 
                    \State Collect dataset $\mD_{\bar{x}}\leftarrow \datacollector(\bar{x}, \pi, \Picore, n_2)$. \label{line:fresh_dataset}
                \EndIf 
            \EndFor 
        \EndFor \label{line:while_loop_end}
  \EndWhile
    \\
        \vspace{0.5em}
        \algcommentbig{Policy Evaluation and Optimization} 
        \For{$\pi\in \Pi$} 
            \State $\hV^\pi\leftarrow \evaluate(\Picore, \reachablestates, \{\mD_x\}_{x \in \reachablestates}, \pi)$. \label{line:evaluate}
        \EndFor
        \State \textbf{Return} $\widehat{\pi} \in \argmax_\pi \hV^\pi$. \label{line:return}
    \end{algorithmic}
\end{algorithm} 

\subsubsection*{State Identification Phase} 
The goal of the state identification phase is to discover all such petal states that are reachable with probability  $\Omega(\eps/D)$. The algorithm proceeds in a loop and sequentially grows the set $\cT$, which contains tuples of the form \((x, \pi_x)\), where $x \in \bigcup_{\pi \in \Pi} \statesp_\pi$ is a sufficiently reachable petal state (for some policy) and \(\pi_x\) denotes a policy that reaches \(x\) with probability \(\Omega(\eps/D)\). We also denote $\reachablestates \coloneqq \crl*{x: (x, \pi_x) \in \cT}$ to denote the set of reachable states in \(\cT\). Initially, $\cT$ only contains a dummy start state $x_\top$ and a null policy. We will collect data using the $\datacollector$ subroutine that: for a given $(x, \pi_x) \in \cT$, first run $\pi_x$ to reach state $x$, and if we succeed in reaching $x$, restart exploration by sampling a policy from $\unif(\Pi_\mathrm{exp})$. Note that $\datacollector$ will be sample-efficient for any \((x, \pi_x)\) since \(\Omega(\epsilon/D)\) fraction of the trajectories obtained via \(\pi_x\) are guaranteed to reach $x$ (by definition of \(\pi_x\) and construction of set $\cT$). Initially, we run $\datacollector$ using $\unif(\Pi_\mathrm{exp})$ from the start, where we slightly abuse the notation and assume that all trajectories in the MDP start at the dummy state $x_\top$ at time step \(h = 0\). 

In every loop, the algorithm attempts to find a new petal state $\bar{x}$ for some $\pi \in \Pi$ that is guaranteed to be $\Omega(\eps/D)$-reachable by $\pi$. This is accomplished by constructing a (estimated and partial) version of the policy-specific MRP using the datasets collected up until that loop (\pref{line:est-policy-specific-mrp}). In particular given a policy \(\pi\) and a set \( \SHPalg{\pi} = \statesp^+_\pi \cap \reachablestates\),  we construct $\widehat \MRPsign^\pi_{\reachablestates} = \mathrm{MRP}(\statesp^+_{\pi}, \wh P^\pi, \wh R^\pi, H, x_\top, x_\bot)$ which essentially compresses our empirical knowledge of the original MDP relevant to the policy \(\pi\). In particular, for any states \(x \in \statesp_\pi \cup \crl{x_\top}\) and $x' \in \statesp_\pi \cup \crl{x_\bot}$ residing in different layers $h < h'$ in the underlying MDP, we define: 
\begin{itemize}[label=\(\bullet\)]
    \item \textbf{Transition $\bm{\wh P_{x \to x'}^\pi}$} as:
    \begin{align*} 
        \widehat{P}_{x\to x'}^\pi = 
	    \frac{1}{|\mD_x|}\sum_{\tau\in \mD_x}\frac{\ind{\pi\cons\tau_{h:h'}}}{\tfrac{1}{|\Picore|} \sum_{\pi'\in \Picore}\ind{\pi' \cons \tau_{h:h'}}}\ind{ \substack{ \text{$\tau_{h:h'}$ goes from \(x\) to \(x'\)} \\ \text{without passing through any other \(x'' \in \statesp_\pi\)}}}.
	\numberthis\label{eq:emp-transitions}
    \end{align*}
    \item \textbf{Transition $\bm{\wh R_{x \to x'}^\pi}$} as: 
    \begin{align*} 
        \widehat{R}_{x\to x'}^\pi = \frac{1}{|\mD_x|}\sum_{\tau\in \mD_x}\frac{R(\tau_{h:h'}) \cdot \ind{\pi\cons\tau_{h:h'}}}{\tfrac{1}{|\Picore|}\sum_{\pi'\in \Picore}\ind{\pi' \cons \tau_{h:h'}}}\ind{ \substack{ \text{$\tau_{h:h'}$ goes from \(x\) to \(x'\)} \\ \text{without passing through any other \(x'' \in \statesp_\pi\)}}}. \numberthis\label{eq:emp-rewards}
    \end{align*}   
\end{itemize}

Clearly, the above definition implies that $\widehat{P}_{x\to x'}^\pi = 0$ and  $\widehat{R}_{x\to x'}^\pi = 0$ for any $x \notin \SHPalg{\pi}$ since \(\cD_x\) would be empty corresponding to these unexplored states. Furthermore, 
\(P_{x_\bot \to x_\bot}^\pi = 1\) and \(R_{x_\bot \to x_\bot}^\pi = 0\). 


Note that since the set $\reachablestates$ is changing in each iteration of the loop as the algorithm collects more data, the policy-specific MRP $\widehat \MRPsign^\pi_{\reachablestates}$ also changes in every iteration of the loop --- in particular, more and more transitions/rewards are assigned nonzero values due to new states being added to $\reachablestates$. More details on policy-specific MRPs is given in \pref{sec:algorithm_details}. 

The key advantage of constructing the empirical versions of policy-specific MRPs is that they allow us to explore and find new leaf states in \(\statesp_\pi\) which are reachable with probability at least \(\Omega(\nicefrac{\epsilon}{6D})\). In particular, using standard dynamic programming (subroutine $\estreach$), we can check whether a candidate petal $\bar{x}$ is reachable with decent probability by $\pi$ (lines \ref{line:DP}-\ref{line:test_and_add}); If it is, then we add $(\bar{x}, \pi)$ to the set $\cT$ and collect a fresh dataset using $\datacollector$ (lines \ref{line:add_s}-\ref{line:fresh_dataset}). Crucially, the importance sampling technique enables us to be sample efficient, since the same dataset $\cD_x$ can be used to evaluate transitions/rewards in Eqs.~\eqref{eq:emp-transitions} and \eqref{eq:emp-rewards} for multiple $\pi \in \Pi$ for which $x$ is a petal state. Furthermore, the number of such datasets we collect must be bounded---each $(x, \pi_x) \in \cT$ contributes $\Omega(\eps/D)$  to cumulative reachability, but since cumulative reachability is bounded from above by $\dimRL(\Pi)$ (\pref{lem:coverability}), we know that $\abs{\cT} \le \cO(D \cdot \dimRL(\Pi)/\eps)$.


\subsubsection*{Evaluation Phase} 
Next, $\mathsf{POPLER}$ moves to the evaluation phase. Using the collected data, it executes the $\evaluate$ subroutine for every $\pi \in \Pi$ to get estimates $\wh{V}^\pi$ (\pref{line:evaluate}) corresponding to \(V^\pi\). For a given \(\pi \in \Pi\), the $\evaluate$ subroutine also constructs an empirical policy-specific MRP $\widehat \MRPsign^\pi_{\reachablestates}$ for every \(\pi \in \Pi\) and computes the value of \(\pi\) via dynamic programming on $\widehat\MRPsign^\pi_{\reachablestates}$. While the returned estimate $\wh{V}^\pi$ is biased, in the complete proof, we will show that the bias is negligible since it is now only due to the states in the petal $\statesp_\pi$ which are \emph{not} $\Omega(\eps/D)$-reachable. Thus, we can guarantee that $\wh{V}^\pi$ closely estimates \(V^\pi\) for every $\pi \in \Pi$, and therefore $\mathsf{POPLER}$ returns a near-optimal policy.

\subsection{Sunflower Property Examples}\label{sec:sunflower-property-examples}
We will show that the examples in \pref{sec:C-pi} satisfy the sunflower property (\pref{def:core_policy}) with small \(K\) and \(D\). At the end, we also give an example of a policy class which satisfies the sunflower property but has large spanning capacity, thus indicating that sunflower property itself is insufficient for sample-efficient RL. 

\paragraph{Tabular MDP.}
We choose $\Picore = \{\pi_a: \pi_a(x) = a, a\in \mA\}$ to be the set of policies which play the constant $a$ for each $a \in \cA$  and $\statesp_{\pi} = \statesp$ for every $\pi\in \Pi$, then any partial trajectory which satisfies the condition in \pref{def:core_policy} is of the form $(x_h, a_h)$, which is consistent with $\pi_{a_h}\in \Picore$. Hence $\Pi$ is a $(A, \abs{\statesp})$-\sunflower.

\paragraph{Contextual Bandit.}
We choose $\Picore = \{\pi_a: \pi_a(x)\equiv a, a\in\mA\}$, and $\statesp_{\pi} = \emptyset$ for every $\pi\in\Pi$, then any partial trajectory which satisfies the condition in \pref{def:core_policy} is in the form $(x,a)$, which is consistent with $\pi_a\in \Picore$. Hence $\Pi$ is a $(A, 0)$-\sunflower.

\paragraph{$H$-Layer Contextual Bandit.} 
We choose 
$$\Picore = \{\pi_{a_1, \cdots, a_H}: \pi_{a_1, \cdots, a_H}(x_h)\equiv a_h, a_1, \cdots, a_H\in\mA\}$$
and $\statesp_{\pi} = \emptyset$ for every $\pi\in\Pi$, then any partial trajectory which satisfies the condition in \pref{def:core_policy} is in the form $(x_1, a_1, \cdots, x_H, a_H)$, which is consistent with $\pi_{a_1, a_2, \cdots, a_H}\in \Picore$. Hence $\Pi$ is a $(A^H, 0)$-\sunflower.

\paragraph{$\ell$-tons.}
We choose 
$$\Picore = \{\pi_0\}\cup\{\pi_h:1\le h\le H\},$$
where $\pi_0(x)\equiv 0$, and \(\pi_h\) chooses the action \(1\) on all the states at layer \(h\), i.e., $\pi_h(x)\coloneqq \ind{x\in\statesp_h}$. For every $\pi\in \Pi_\ell$, we choose $\statesp_\pi$ to be the states for which $\pi(x) = 1$ (there are at most $\ell$ such states). Fix any partial trajectory $\tau = (x_{h}, a_h \cdots, x_{h'}, a_{h'})$ which satisfies $\pi\cons \tau$. Suppose that for all $i \in (h, h']$, $x_i\not\in\statesp_\pi$. Then we must have $a_i = 0$ for all $i \in (h, h']$. Hence $\pi_h \cons \tau$ (if $a_h = 1$) or $\pi_0 \cons \tau$ (if $a_h = 0$), and $\tau$ is consistent with some policy in $\Picore$.  Therefore, $\Pi_\ell$ is an $(H+1, \ell)$-\sunflower.

\paragraph{$1$-Active Policies.} 
We choose $\statesp_\pi = \{x_{(1,1)}, x_{(1,2)}, \cdots, x_{(1, H)}\}$ for all $\pi\in \Pi$ as well as 
$$\Picore = \{\pi_0\}\cup\{\pi_h:1\le h\le H\},$$
where $\pi_0(x)\equiv 0$ and $\pi_h(x)\coloneqq \ind{x\in\statesp_h}$. Fix any partial trajectory $\tau = (x_{h}, a_h, \cdots, x_{h'}, a_{h'})$ which satisfies $\pi\cons \tau$. If we have $i \in (h, h']$, $x_i\not\in\statesp_\pi$, then we must have $a_i = 0$. Thus, $\pi_h\cons \tau$ (if $a_h = 1$) or $\pi_0\cons \tau$ (if $a_h = 0$), so $\tau$ is consistent with some policy in $\Picore$. Therefore, $\Pioneactive$ is a $(H+1, H)$-\sunflower.

\paragraph{All-Active Policies.} We choose $\statesp_\pi = \{x_{(j,1)}, \cdots, x_{(j,H)}\}$ for all $\pi\in \PiJactive$, as well as 
$$\Picore = \{\pi_0\}\cup\{\pi_h:1\le h\le H\},$$ 
where $\pi_0(x)\coloneqq 0$ and $\pi_h(x)\coloneqq \ind{x\in\statesp_h}$. Fix any partial trajectory $\tau = (x_{h}, a_h \cdots, x_{h'}, a_{h'})$ which satisfies $\pi\cons \tau$. If we have $i \in (h, h']$, $x_i\not\in\statesp_\pi$, then we must have $a_i = 0$. Thus, $\pi_h\cons \tau$ (if $a_h = 1$) or $\pi_0\cons \tau$ (if $a_h = 0$), so $\tau$ is consistent with some policy in $\Picore$. Therefore, $\Piactive$ is a $(H+1, H)$-\sunflower.

\paragraph{Sunflower property is insufficient by itself.}
We give an example of a policy class $\Pi$ for which the sunflower property holds for $K,D = \poly(H)$ but $\dimRL(\Pi) = 2^H$. Therefore, in light of \pref{thm:generative_lower_bound}, the sunflower property by itself cannot ensure statistically efficient agnostic PAC RL in the online access model. 

The example is as follows: Consider a binary tree MDP with $2^H-1$ states and action space $\cA = \crl{0,1}$. The policy class $\Pi$ will be able to get to every $(x,a)$ pair in layer $H$. To define the policies, we consider each possible trajectory $\tau = (x_1, a_1, \cdots, x_H, a_H)$ and let: 
\begin{align*} 
    \Pi \coloneqq \crl*{\pi_{\tau} : \pi_\tau(x) = \begin{cases}
        a_i &\text{if } x_i \in \tau,\\
        0 &\text{otherwise}, 
    \end{cases}}.
\end{align*}
Thus it is clear that $\dimRL(\Pi) = 2^H$, but the sunflower property holds with $K=1$, $D=H$ by taking $\Picore = \crl{\pi_0}$ (the policy which always picks $a=0$).

\section{Open Problem: Sample Complexity of Online RL}\label{sec:online-open-problems}

The results in this chapter establish when learning is possible in online RL, but fall short of giving a complete characterization of learnability. We highlight two concrete directions for progress on this question.

\begin{question}
    Are policy classes with bounded spanning capacity ``almost learnable'': can we show a sample complexity guarantee of the form $\non(\Pi; \eps, \delta) \lesssim \eps^{-O(\log \dimRL(\Pi))}$?
\end{question}

In conjunction with the other results in this thesis, this would establish that the minimax sample complexity of agnostic policy learning with online RL interaction (ignoring other problem parameters) is:
\begin{align*}
    \frac{\dimRL(\Pi)}{\eps^2} \cdot \log \frac{1}{\delta} \lesssim \non(\Pi; \eps, \delta) \lesssim \frac{1}{\eps^{O(\log \dimRL(\Pi)}} \cdot \log \frac{\abs{\Pi}}{\delta},
\end{align*}
and that neither the upper bound nor the lower bound can be improved significantly in terms of the dependence on $\dimRL(\Pi)$. We suspect that one might be able to extend the algorithmic ideas of our upper bound in \pref{thm:sunflower} to discover an MDP-dependent sunflower core/petal set. It would also be interesting to try to improve the lower bound: as we discuss in \pref{sec:remarks-zaran}, the current construction via ``block-free'' matrices cannot be readily improved, but alternative constructions might yield stronger lower bounds.


\begin{question}
    Is (some variant of) the sunflower property necessary for online RL? More formally, for any policy class $\Pi$, is the sample complexity in online RL at least
    \begin{align*}
        \min_{K,D \in \bbN}~ \crl{K+D} \quad \text{such that $\Pi$ is a $(K,D)$-sunflower}?
    \end{align*}
\end{question}

We have strong evidence: this statement is true for every example considered in this thesis.

\section{Deferred Proofs}\label{sec:online-deferred-proofs}
\subsection{Proof of \pref{thm:lower-bound-online}}\label{sec:proof-lower-bound-online}

Now we prove \pref{thm:lower-bound-online}. We restate the theorem below with the precise constants:

\begin{reptheorem}{thm:lower-bound-online}[Lower Bound for Online RL]
Let $h_0 \in \bbN$ and $c \in (0,1)$ be universal constants. Fix any $H \ge h_0$. Let $\eps \in (1/2^{cH},1/(100H))$ and $\ell \in \crl{2, \dots, H}$ such that $1/\eps^{\ell} \le 2^H$. There exists a policy class $\Piell$ of size $1/(6\eps^\ell)$ with $\dimRL(\Piell) \le O(H^{4\ell+2})$ and a family of MDPs $\cM$ with state space $\statesp$ of size $H\cdot 2^{2H+1}$, binary action space, horizon $H$ such that: for any $(\eps/16, 1/8)$-PAC algorithm, there exists an $M \in \cM$ in which the algorithm has to collect at least
\begin{align*}
    \min\crl*{ \frac{1}{120 \eps^\ell}, 2^{H/3 - 3} } \quad \text{online trajectories in expectation.}
\end{align*}
\end{reptheorem}

\subsubsection{Construction of State Space, Action Space, and Policy Class}

\paragraph{State and Action Spaces.} We define the state space $\statesp$. In every layer $h \in [H]$, there will be $2^{2H+1}$ states. The states will be paired up, and each state will be denoted by either $\jof{h}$ or $\jpof{h}$, so $\statesp_h = \crl{\jof{h}: j \in [2^{2H}]} \cup \crl{\jpof{h}: j \in [2^{2H}]}$. For any state $x \in \statesp$, we define the \emph{index} of $x$, denoted $\idx(x)$ as the unique $j\in [2^{2H}]$ such that $x \in \crl{\jof{h}}_{h\in[H]} \cup \crl{\jpof{h}}_{h\in[H]}$. In total there are $H \cdot 2^{2H+1}$ states. The action space is $\cA = \crl{0,1}$.

\paragraph{Policy Class.} For the given $\eps$ and $\ell \in \crl{2, \dots, H}$, we show via a probabilistic argument the existence of a large policy class $\Piell$ which has bounded \Compname{} but is hard to explore. We state several properties in \pref{lem:piell-properties} which will be exploited in the lower bound.

We introduce some additional notation. For any $j \in [2^{2H}]$ we denote
\begin{align*}
    \Piell_j \coloneqq \crl{\pi \in \Piell: \exists h \in [H], \pi(\jof{h}) = 1},
\end{align*}
that is, $\Piell_j$ are the policies which take an action $a=1$ on at least one state with index $j$.

We also define the set of \emph{relevant state indices} for a given policy $\pi \in \Piell$ as
\begin{align*}
    \Jrel^\pi \coloneqq \crl{j \in [2^{2H}]: \pi \in \Piell_j}.
\end{align*}
For any policy $\pi$ we denote $\pi(j_{1:H}) \coloneqq (\pi(\jof{1}), \dots, \pi(\jof{H})) \in \crl{0,1}^H$ to be the vector that represents the actions that $\pi$ takes on the states  in index $j$. The vector $\pi(j'_{1:H})$ is defined similarly.

\begin{lemma}\label{lem:piell-properties} Let $H$, $\eps$, and $\ell$ satisfy the assumptions of \pref{thm:lower-bound-online}. There exists a policy class $\Piell$ of size $N = 1/(6 \eps^\ell)$ which satisfies the following properties.
\begin{itemize}
    \item[(1)] For every $j \in [2^{2H}]$ we have $\abs{\Piell_j} \in  [\eps N/2, 2\eps N]$. 
    \item[(2)] For every $\pi \in \Pi$ we have $\abs{ \Jrel^\pi } \ge \eps/2 \cdot 2^{2H}$.
    \item[(3)] For every $\pi \in \Piell_j$, the vector $\pi(j_{1:H})$ is unique and always equal to $\pi(j'_{1:H})$.
    \item[(4)] Bounded \Compname{}: $\dimRL(\Piell) \le c \cdot H^{4\ell+2}$ for some universal constant $c > 0$.
\end{itemize}
\end{lemma}

\subsubsection{Construction of MDP Family}
The family $\cM = \crl{M_{\pistar,\phi}}_{\pistar \in \Piell, \phi \in \Phi}$ will be a family of MDPs which are indexed by a policy $\pistar$ as well as a \emph{decoder} function $\phi: \statesp \mapsto \crl{\textsc{good}, \textsc{bad}}$, which assigns each state to be ``good'' or ``bad'' in a sense that will be described later on. An example construction of an MDP $M_{\pistar, \phi}$ is illustrated in \pref{fig:lower-bound-idea}. For brevity, the bracket notation used to denote the layer that each state lies in has been omitted in the figure.

\paragraph{Decoder Function Class.} The decoder function class $\Phi$ will be the set of all possible mappings which for every $j\in [2^{2H}]$ and $h \ge 2$ assign exactly one of $\jof{h}$ or $ \jpof{h}$ to the label $\textsc{Good}$ (where the other is assigned to the label $\textsc{Bad}$). There are $(2^{H-1})^{2^{2H}}$ such functions. The label of a state will be used to describe the transition dynamics. Intuitively, a learner who does not know the decoder function $\phi$ will not be able to tell if a certain state has the label  $\textsc{Good}$ or $\textsc{Bad}$ when visiting that state for the first time.

\paragraph{Transition Dynamics.} The MDP $M_{\pistar, \phi}$ will be a uniform distribution over $2^{2H}$ combination locks $\crl{\mathsf{CL}_j}_{j \in [2^{2H}]}$ with disjoint states. More formally, $x_1 \sim \unif(\{\jof{1}\}_{j\in[2^{2H}]})$. From each start state $\jof{1}$, only the $2H-2$ states corresponding to index $j$ at layers $h\ge 2$ will be reachable in the combination lock $\mathsf{CL}_j$.

In the following, we will describe each combination lock $\mathsf{CL}_j$, which forms the basic building block of the MDP construction.
\begin{itemize}[label=\(\bullet\)]
    \item \textbf{Good/Bad Set.} At every layer $h\in [H]$, for each $\jof{h}$ and $\jpof{h}$, the decoder function $\phi$ assigns one of them to be $\textsc{Good}$ and one of them to be $\textsc{Bad}$. We will henceforth denote $\jgof{h}$ to be the good state and $\jbof{h}$ to be the bad state. Observe that by construction in Eq.~\eqref{eq:pi-construction}, for every $\pi \in \Piell$ and $h\in [H]$ we have $\pi(\jgof{h}) = \pi(\jbof{h})$.
    \item \textbf{Dynamics of $\mathsf{CL}_j$, if $j\in \Jrel^\pistar$.} Here, the transition dynamics of the combination locks are deterministic. For every $h\in [H]$,
    \begin{itemize}
        \item On good states $\jgof{h}$ we transit to the next good state iff the action is $\pistar$: \begin{align*}
            P(x'~|~\jgof{h}, a) &= \begin{cases}
            \ind{x' = \jgof{h+1}}, &\text{if } a=\pistar(\jgof{h}) \\
            \ind{x' = \jbof{h+1}}, &\text{if } a\ne\pistar(\jgof{h}).
            \end{cases}
        \end{align*}
        \item On bad states $\jbof{h}$ we always transit to the next bad state:
        \begin{align*}
            P(x'~|~\jbof{h}, a) = \ind{x' = \jbof{h+1}}, \quad \text{for all } a \in \cA.
        \end{align*}
    \end{itemize}
    \item \textbf{Dynamics of $\mathsf{CL}_j$, if $j\notin \Jrel^\pistar$.} If $j$ is not a relevant index for $\pistar$, then the transitions are uniformly random regardless of the current state/action. For every $h\in [H]$,
    \begin{align*}
    P(\cdot ~|~ \jgof{h},a) = P(\cdot ~|~ \jbof{h}, a) = \unif \prn*{ \crl{\jgof{h+1}, \jbof{h+1}} }, \quad \text{for all } a \in \cA.
\end{align*}
    \item \textbf{Reward Structure.} The reward function is nonzero only at layer $H$, and is defined as
    \begin{align*}
        R(x,a) = \mathrm{Ber}\prn*{ \frac{1}{2} + \frac{1}{4} \cdot \mathbbm{1}\crl{\pistar \in \Piell_j} \cdot \mathbbm{1}\crl{x = \jgof{H}, a = \pistar(\jgof{H})} }
    \end{align*}
    That is, we get $3/4$ whenever we reach the $H$-th good state for an index $j$ which is relevant for $\pistar$, and $1/2$ reward otherwise.
\end{itemize}

\paragraph{Reference MDPs.} We define several reference MDPs.
\begin{itemize}[label=$\bullet$]
    \item In the reference MDP $M_0$, the initial start state is again taken to be the uniform distribution, i.e., $x_1 \sim \unif(\{\jof{1}\}_{j\in[2^{2H}]})$, and all the combination locks behave the same and have uniform transitions to the next state along the chain: for every $h\in [H]$ and $j \in [2^{2H}]$,
    \begin{align*}
    P(\cdot ~|~ \jof{h},a) = P(\cdot ~|~ \jpof{h}, a) = \unif \prn*{ \crl{\jof{h+1}, \jpof{h+1}} }, \quad \text{for all } a \in \cA.
\end{align*}
    The rewards for $M_0$ are $\mathrm{Ber}(1/2)$ for every $(x,a) \in \statesp_H\times \cA$.
    \item For any decoder $\phi \in \Phi$, the reference MDP $M_{0, \pistar, \phi}$ has the same transitions as $M_{\pistar, \phi}$ but the rewards are $\mathrm{Ber}(1/2)$ for every $(x,a) \in \statesp_H\times \cA$.
\end{itemize}

\subsubsection{Proof of \pref{thm:lower-bound-online}}
We are now ready to prove the lower bound using the construction of the MDP family $\cM$.

\paragraph{Value Calculation.} Consider any $M_{\pistar,\phi} \in \cM$. For any policy $\pi \in \cA^{\statesp}$ we use $V_{\pistar,\phi}(\pi)$ to denote the value of running $\pi$ in MDP $M_{\pistar,\phi}$. By construction we can see that
\begin{align*}
    V_{\pistar,\phi}(\pi) = \frac{1}{2} + \frac{1}{4} \cdot \Pr_{\pistar,\phi}\brk*{ \idx(x_1) \in \Jrel^\pistar \text{ and } \pi(\idx(x_1)_{1:H}) = \pistar(\idx(x_1)_{1:H}) },  \numberthis\label{eq:value-of-pi}
\end{align*}
where in the above, we defined for any \(x_1\), \(\pi(\idx(x_1)_{1:H})\ = \pi(j_{1:H}) = (\pi(j[1]), \dots, \pi(j[H]))\), where \(j\) denotes \(\idx(x_1)\). Informally speaking, the second term counts the additional reward that $\pi$ gets for solving a combination lock rooted at a relevant state index $\idx(x_1) \in \Jrel^\pistar$. By Property (2) and (3) of \pref{lem:piell-properties}, we additionally have $V_{\pistar,\phi}(\pistar) \ge 1/2 + \eps/8$, as well as $V_{\pistar,\phi}(\pi) = 1/2$ for all other $\pi\ne \pistar\in \Piell$. 

By Eq.~\eqref{eq:value-of-pi}, if $\pi$ is an $\eps/16$-optimal policy on $M_{\pistar,\phi}$ it must satisfy
\begin{align*}
    \Pr_{\pistar,\phi} \brk*{ \idx(x_1) \in \Jrel^\pistar \text{ and } \pi(\idx(x_1)_{1:H}) = \pistar(\idx(x_1)_{1:H}) } \ge \frac{\eps}{4}.
\end{align*}

\paragraph{Averaged Measures.} We define the following measures which will be used in the analysis.
\begin{itemize}[label=$\bullet$]
    \item Define $\Pr_\pistar[\cdot] = \frac{1}{\abs{\Phi}} \sum_{\phi \in \Phi} \Pr_{\pistar, \phi}[\cdot]$ to be the averaged measure where we first pick $\phi$ uniformly among all decoders and then consider the distribution induced by $M_{\pistar, \phi}$.
    \item Define the averaged measure $\Pr_{0, \pistar}[\cdot] = \frac{1}{\abs{\Phi}} \sum_{\phi \in \Phi} \Pr_{0, \pistar, \phi}[\cdot]$ where we pick $\phi$ uniformly and then consider the distribution induced by $M_{0, \pistar, \phi}$.
\end{itemize}
For both averaged measures the expectations $\En_{\pistar}$ and $\En_{0, \pistar}$ are defined analogously.

\paragraph{Algorithm and Stopping Time.}
Recall that an algorithm $\alg$ is comprised of two phases. In the first phase, it collects some number of trajectories by interacting with the MDP in episodes. We use $\st$ to denote the (random) number of episodes after which $\alg$ terminates. We also use $\alg_t$ to denote the intermediate policy that the algorithm runs in round $t$ for $t \in [\st]$. In the second phase, $\alg$ outputs\footnote{We present the lower bound for the class of deterministic algorithms that output a deterministic policy. However, all the arguments could be extended to stochastic algorithms.} a policy $\pihat$. We use the notation $\alg_f: \crl{\tau^{(t)}}_{t \in [\st]} \mapsto \cA^{\statesp}$  to denote the second phase of $\alg$ which outputs $\pihat$ as a measurable function of collected data.

For any policy $\pistar$, decoder $\phi$, and dataset $\cD$ we define the event
\begin{align*}
    \cE(\pistar, \phi, \alg_f(\cD)) := \crl*{ \Pr_{\pistar,\phi}\brk*{ \idx(x_1) \in \Jrel^\pistar~ \text{and}~ \alg_f(\cD)(\idx(x_1)_{1:H}) = \pistar(\idx(x_1)_{1:H}) } \ge \frac{\eps}{4} }.
\end{align*}
The event $\cE(\pistar, \phi, \alg_f(\cD))$ is measurable with respect to the random variable $\cD$, which denotes the collected data.

Under this notation, the PAC learning guarantee on $\alg$ implies that for every $\pistar \in \Piell$, $\phi \in \Phi$ we have
\begin{align*}
    \Pr_{\pistar,\phi} \brk*{ \cE(\pistar, \phi, \alg_f(\cD)) } \ge 7/8.
\end{align*}
Moreover via an averaging argument we also have
\begin{align*}
    \Pr_{\pistar} \brk*{ \cE(\pistar, \phi, \alg_f(\cD)) } \ge 7/8. \numberthis \label{eq:pac-guarantee-exp}
\end{align*}
\paragraph{Lower Bound Argument.} We apply a truncation to the stopping time $\st$. Define $\Tmax := 2^{H/3}$. Observe that if $\Pr_\pistar \brk{\st > \Tmax} > 1/8$ for some $\pistar \in \Piell$ then the lower bound immediately follows, since
\begin{align*}
    \max_{\phi \in \Phi} \En_{\pistar, \phi}[\st] ~>~ \En_{\pistar}[\st] ~\ge~ \Pr_{\pistar}[\st > \Tmax] \cdot \Tmax ~\ge~ \Tmax/8,
\end{align*}
so there must exist an MDP $M_{\pistar, \phi}$ for which $\alg$ collects at least $\Tmax/8 = 2^{H/3-3}$ samples in expectation.

Otherwise we have $\Pr_\pistar \brk{\st > \Tmax} \le 1/8$ for all $\pistar \in \Piell$. This further implies that for all $\pistar \in \Piell$,
\begin{align*}
    \hspace{2em}&\hspace{-2em}\Pr_\pistar \brk*{\st < \Tmax \text{ and } \cE(\pistar, \phi, \alg_f(\cD))} \\
    &= \Pr_\pistar \brk*{\cE(\pistar, \phi, \alg_f(\cD))} - \Pr_\pistar \brk*{\st > \Tmax \text{ and } \cE(\pistar, \phi, \alg_f(\cD))} \\
    &\ge 3/4. \numberthis \label{eq:lower_bound_property}
\end{align*}

However, in the following, we will argue that if Eq.~\eqref{eq:lower_bound_property} holds then \(\alg\) must query a significant number of samples in \(M_0\).

\begin{lemma}[Stopping Time Lemma]\label{lem:stopping-time}
Let $\delta \in (0, 1/8]$. Let $\alg$ be an $(\eps/16,\delta)$-PAC  algorithm. Let $\Tmax \in \bbN$.  Suppose that $\Pr_\pistar \brk*{\st < \Tmax \text{ and } \cE(\pistar, \phi, \alg_f(\cD))} \ge 1-2\delta$ for all $\pistar \in \Piell$. The expected stopping time for $\alg$ on $M_0$ is at least
\begin{align*}
    \En_0 \brk*{\st} \ge \prn*{ \frac{\abs{\Piell}}{2} - \frac{4}{\eps}} \cdot \frac{1}{7} \log \prn*{\frac{1}{ 2 \delta }} - \abs{\Piell} \cdot  \frac{\Tmax^2}{2^{H+3}} \prn*{\Tmax + \frac{1}{7}\log \prn*{\frac{1}{ 2 \delta }}}.
\end{align*}
\end{lemma}

Using \pref{lem:stopping-time} with $\delta = 1/8$ and plugging in the value of $\abs{\Piell}$ and $\Tmax$, we see that
\begin{align*}
    \En_0[\st] &\ge  \prn*{ \frac{\abs{\Piell}}{2} - \frac{4}{\eps}} \cdot \frac{1}{7} \log \prn*{\frac{1}{ 2 \delta }} - \abs{\Piell} \cdot \frac{\Tmax^2}{2^{H+3}}  \prn*{\Tmax + \frac{1}{7}\log \prn*{\frac{1}{ 2 \delta }}} \ge \frac{\abs{\Piell}}{20}.
\end{align*}
For the second inequality, we used the fact that $\ell \ge 2$, $H \ge 10^5$, and $\eps < 1/10^7$.

We have shown that either there exists some MDP $M_{\pistar, \phi}$ for which $\alg$ collects at least $\Tmax/8 = 2^{H/3-3}$ samples in expectation, or $\alg$ must query at least $\abs{\Piell}/20 = 1/(120 \eps^\ell)$ trajectories in expectation in \(M_0\). Putting it all together, the lower bound on the sample complexity is at least
\begin{align*}
    \min\crl*{ \frac{1}{120 \eps^\ell}, 2^{H/3 - 3} }.
\end{align*}
This concludes the proof of \pref{thm:lower-bound-online}.\qed

\subsubsection{Proof of \pref{lem:piell-properties}}\label{sec:construction-policy-class}

To prove \pref{lem:piell-properties}, we first use a probabilistic argument to construct a certain binary matrix $B$ which satisfies several properties, and then construct $\Piell$ using $B$ and verify it satisfies Properties (1)-(4).

\paragraph{Binary Matrix Construction.}

First, we define a block-free property of binary matrices.
\begin{definition}[Block-free Matrices]\label{def:intersubsection-property}
Fix parameters $k, \ell, N, d \in \bbN$ where \(k \leq N\) and \(l \leq d\). We say a binary matrix $B \in \crl{0,1}^{N \times d}$ is $(k, \ell)$-block-free if the following holds: for every $I \subseteq [N]$ with $\abs{I} = k$, and $J \subseteq [d]$ with $\abs{J} = \ell$ there exists some $(i,j)\in I\times J$ with $B_{ij} = 0$.
\end{definition}

In words, matrices which are $(k,\ell)$-block-free do not contain a $k\times \ell$ block of all 1s.

\begin{lemma}\label{lem:matrix-construction}
Fix any $\eps \in (0,1/10)$ and $\ell \in \bbN$. For any
\begin{align*}
    d \in \Big[ \frac{16 \ell \cdot \log(1/\eps)}{\eps}, \frac{1}{20} \cdot \exp\Big( \frac{1}{48 \eps^{\ell-1}} \Big) \Big] ,
\end{align*}
there exists a binary matrix $B \in \crl{0,1}^{N\times d}$ with $N = 1/(6 \cdot \eps^\ell)$ such that:
\begin{enumerate}[label=\((\arabic*)\)]
    \item (Row sum): for every row $i \in [N]$, we have $\sum_{j} B_{ij} \ge \eps d /2$.
    \item (Column sum): for every column $j \in [d]$, we have $\sum_{i} B_{ij} \in [\eps N /2, 2\eps N]$.
    \item The matrix $B$ is $(\ell \log d,\ell)$-block-free.
\end{enumerate}
\end{lemma}

\begin{proof}[Proof of \pref{lem:matrix-construction}.]
The existence of $B$ is proven using the probabilistic method. Let $\wt{B} \in \crl{0,1}^{N\times d}$ be a random matrix where each entry is i.i.d.~chosen to be 1 with probability $\eps$.

By Chernoff bounds (\pref{lem:chernoff}), for every row $i \in [N]$, we have $\Pr \brk{\sum_{j} B_{ij} \le \tfrac{\eps d}{2}} \le \exp\prn{-\eps d /8}$; likewise for every column $j \in [d]$, we have $\Pr \brk{\sum_{j} B_{ij} \notin \brk{\tfrac{\eps N}{2}, 2\eps N}} \le 2\exp\prn{-\eps N /8}$. By union bound, the matrix $\wt{B}$ satisfies the first two properties with probability at least $0.8$ as long as
\begin{align*}
    d \ge (8 \log 10N)/\eps, \quad \text{and} \quad N \ge (8 \log 20d)/\eps.
\end{align*}
One can check that under the choice of $N = 1/(6 \cdot \eps^\ell)$ and the assumption on $d$, both constraints are met.

Now we examine the probability of $\wt{B}$ satisfies the block-free property with parameters $(k \coloneqq \ell \log d, \ell)$. Let $X$ be the random variable which denotes the number of submatrices which violate the block-free property in $\wt{B}$, i.e.,
\begin{align*}
    X = \abs{ \crl{I \times J: I \subset [N], \abs{I} = k, J \subset [d], \abs{J} = \ell, \wt{B}_{ij} = 1 \ \forall \ (i,j)\in I \times J} }.
\end{align*}
By linearity of expectation, we have
\begin{align*}
    \bbE \brk{X} \le N^{k} d^\ell \eps^{k \ell}.
\end{align*}
We now plug in the choice $k = \ell \log d$ and observe that as long as $N \le 1/(2e \cdot \eps^\ell)$ we have $\bbE[X] \le 1/2$. By Markov's inequality, $\Pr [X = 0] \ge 1/2$.

Therefore with positive probability, $\wt{B}$ satisfies all 3 properties (otherwise we would have a contradiction via inclusion-exlusion principle). Thus, there exists a matrix $B$ which satisfies all of the above three properties, proving the result of \pref{lem:matrix-construction}.
\end{proof}

\paragraph{Policy Class Construction.}
For the given $\eps$ and $\ell \in \crl{2, \dots, H}$ we will use \pref{lem:matrix-construction} to construct a policy class $\Piell$ which has bounded \Compname{} but is hard to explore. We instantiate \pref{lem:matrix-construction} with the given $\ell$ and $d = 2^{2H}$, and use the resulting matrix $B$ to construct $\Piell = \crl{\pi_i}_{i\in [N]}$ with $\abs{\Piell} = N = 1/(6\eps^\ell)$.

Recall that we assume that
\begin{align*}
    H \ge h_0, \quad \text{and} \quad \eps \in \brk*{ \frac{1}{2^{c H}} , \frac{1}{100H} }.
\end{align*}
We claim that under these assumptions, the requirement of \pref{lem:matrix-construction} is met:
\begin{align*}
    d = 2^{2H} \in \brk*{ \frac{16 \ell \cdot \log(1/\eps)}{\eps}, \frac{1}{20} \cdot \exp\prn*{ \frac{1}{48 \eps^{\ell-1}} } }.
\end{align*}
For the lower bound, we can check that:
\begin{align*}
    \frac{16 \ell \cdot \log(1/\eps)}{\eps} \le 16 H \cdot cH \cdot 2^{cH} \le 2^{2H},
\end{align*}
where we use the bound $\ell \le H$ and $\eps \ge 2^{-cH}$. The last inequality holds for sufficiently small universal constant $c \in (0, 1)$ and sufficiently large $H \ge h_0$.

For the upper bound, we can also check that
\begin{align*}
    \frac{1}{20} \cdot \exp\prn*{ \frac{1}{48 \eps^{\ell-1}} } \ge \frac{1}{20} \cdot \exp\prn*{ \frac{100H}{48}} \ge 2^{2H},
\end{align*}
where we use the bound $\ell \ge 2$ and $\eps \le 1/(100H)$. The last inequality holds for sufficiently large $H$.
We define the policies as follows: for every $\pi_i \in \Piell$ we set
\begin{align*}
    \text{for every } j \in [2^{2H}]: \quad&\pi_i(\jof{h}) = \pi_i(\jpof{h}) = \begin{cases}
    \bit_h(\sum_{a \le i} B_{aj}) &\text{if} \ B_{ij} =1,\\
    0 & \text{if} \ B_{ij} = 0.
    \end{cases} \numberthis\label{eq:pi-construction}
\end{align*}
The function $\bit_h: [2^H-1] \mapsto \crl{0,1}$ selects the $h$-th bit in the binary representation of the input.

\paragraph{Verifying Properties $(1)-(4)$ of \pref{lem:piell-properties}.}
Properties $(1)-(3)$ are straightforward from the construction of $B$ and $\Piell$, since $\pi_i \in \Piell_j$ if and only if $B_{ij} = 1$. The only detail which requires some care is that we require that $2\eps N < 2^{H}$ in order for Property (3) to hold, since otherwise we cannot assign the behaviors of the policies according to Eq.~\eqref{eq:pi-construction}. However, by assumption, this always holds, since $2\eps N = 1/(3\eps^{\ell-1}) \le 2^H.$

We now prove Property (4) that $\Piell$ has bounded \Compname{}. To prove this we will use the block-free property of the underlying binary matrix $B$.

Fix any deterministic MDP $M^\star$ which witnesses $\dimRL(\Piell)$ at layer $h^\star$. To bound $\dimRL(\Piell)$, we need to count the contribution to $C^\mathsf{reach}_{h^\star}(\Pi; M^\star)$ from trajectories $\tau$ which are produced by some $\pi \in \Piell$ on $M$. We first define a \emph{layer decomposition} for a trajectory $\tau = (x_1, a_1, x_2, a_2, \dots, x_H, a_H)$ as the unique tuple of indices $(h_1, h_2, \dots h_{m})$, where each $h_k \in [H]$, that satisfies the following properties:
\begin{itemize}[label=\(\bullet\)]
    \item The layers satisfy $h_1 < h_2 < \dots < h_m$.
    \item The layer $h_1$ represents the first layer where $a_{h_1} = 1$.
    \item The layer $h_2$ represents the first layer where $a_{h_2} = 1$ on some state $x_{h_2}$ such that
    \begin{align*}
        \idx(x_{h_2}) \notin \crl{ \idx(x_{h_1}) }.
    \end{align*}
    \item The layer $h_3$ represents the first layer where $a_{h_3} = 1$ on some state $x_{h_3}$ such that
    \begin{align*}
        \idx(x_{h_3}) \notin \crl{ \idx(x_{h_1}), \idx(x_{h_2}) }.
    \end{align*}
    \item More generally the layer $h_k$, $k\in[m]$ represents the first layer where $a_{h_k} = 1$ on some state $x_{h_k}$ such that
    \begin{align*}
         \idx(x_{h_k}) \notin \crl{ \idx(x_{h_1}), \dots, \idx(x_{h_{k-1}}) }.
    \end{align*}
    In other words, the layer $h_k$ represents the $k$-th layer for where action is $a=1$ on a new state index which $\tau$ has never played $a=1$ on before.
\end{itemize}
We will count the contribution to $C^\mathsf{reach}_{h^\star}(\Pi; M^\star)$ by doing casework on the length of the layer decomposition for any $\tau$. That is, for every length $m \in \crl{0, \dots, H}$, we will bound $C_{h^\star}(m)$, which is defined to be the total number of $(x,a)$ at layer $h^\star$ which, for some $\pi \in \Piell$, a trajectory $\pi \cons \tau$ that has a $m$-length layer decomposition visits. Then we apply the bound
\begin{align*}
    C^\mathsf{reach}_{h^\star}(\Pi; M^\star) \le \sum_{m=0}^H C_{h^\star}(m). \numberthis\label{eq:contribution-decomp}
\end{align*}
Note that this will overcount, since the same $(x,a)$ pair can belong to multiple different trajectories with different length layer decompositions.

\begin{lemma}\label{lem:contributions}
The following bounds hold:
\begin{itemize}[label=\(\bullet\)]
    \item For any $m \le \ell$, $C_{h^\star}(m) \le H^m \cdot \prod_{k=1}^m (2kH) = \cO(H^{4m})$.
    \item We have $\sum_{m \ge \ell+1} C_{h^\star}(m) \le \cO(\ell \cdot H^{4\ell + 1})$.
\end{itemize}
\end{lemma}
Therefore, applying \pref{lem:contributions} to Eq.~\eqref{eq:contribution-decomp}, we have the bound that
\begin{align*}
    \dimRL(\Piell) \le \prn*{\sum_{m \le \ell} O(H^{4m})} + O(\ell \cdot H^{4\ell+1}) \le O(H^{4\ell+2}).
\end{align*}
This concludes the proof of \pref{lem:piell-properties}.\qed

\begin{proof}[Proof of \pref{lem:contributions}]
All of our upper bounds will be monotone in the value of $h^\star$, so we will prove the bounds for $C_H(m)$. In the following, fix any deterministic MDP $M^\star$.

First we start with the case where $m=0$. The trajectory $\tau$ must play $a=0$ at all times; since there is only one such $\tau$, we have $C_{H}(0) = 1$.

Now we will bound $C_H(m)$, for any $m \in \crl{1, \dots, \ell}$. Observe that there are $\binom{H}{m} \le H^m$ ways to pick the tuple $(h_1, \dots, h_m)$. Now we will fix $(h_1, \dots, h_m)$ and count the contributions to $C_H(m)$ for trajectories $\tau$ which have this fixed layer decomposition, and then sum up over all possible choices of $(h_1, \dots, h_m)$.

In the MDP $M^\star$, there is a unique state $x_{h_1}$ which $\tau$ must visit. In the layers between $h_1$ and $h_2$, all trajectories are only allowed take $1$ on states with index $\idx(x_{h_1})$, but they are not required to. Thus we can compute that the contribution to $C_{h_2}(m)$ from trajectories with the fixed layer decomposition to be at most $2H$. The reasoning is as follows. At $h_1$, there is exactly one $(x,a)$ pair which is reachable by trajectories with this fixed layer decomposition, since any $\tau$ must take $a=1$ at $x_{h_1}$. Subsequently we can add at most two reachable pairs in every layer $h \in \{h_1+1, \dots, h_2-1\}$ due to encountering a state $\jof{h}$ or $\jpof{h}$ where $j = \idx(x_{h_1})$, and at layer $h_2$ we must play $a=1$, for a total of $1 + 2(h_2 - h_1 - 1) \le 2H$. Using similar reasoning the contribution to $C_{h_3}(m)$ from trajectories with this fixed layer decomposition is at most $(2H) \cdot (4H)$, and so on. Continuing in this way, we have the final bound of $\prod_{k=1}^m (2kH)$. Since this holds for a fixed choice of $(h_1, \dots, h_m)$ in total we have $C_H(m) \le H^m \cdot \prod_{k=1}^m (2kH) = \cO(H^{4m})$.

When $m \ge \ell + 1$, observe that the block-free property on $B$ implies that for any $J \subseteq [2^H]$ with $\abs{J} = \ell$ we have $\abs{\cap_{j \in J} \Pi_j} \le \ell \log 2^{2H}$. So for any trajectory $\tau$ with layer decomposition such that $m \ge \ell$ we can redo the previous analysis and argue that there is at most $ \ell \log 2^{2H}$ multiplicative factor contribution to the value $C_H(m)$ due to \emph{all} trajectories which have layer decompositions longer than $\ell$. Thus we arrive at the bound
\begin{align*}
    \sum_{m \ge \ell+1} C_{H}(m) \le \cO(H^{4\ell}) \cdot \ell \log 2^{2H} \le \cO(\ell \cdot H^{4\ell+1}).
\end{align*}
This concludes the proof of \pref{lem:contributions}.
\end{proof}

\subsubsection{Remarks on Tightness of \pref{lem:piell-properties}}\label{sec:remarks-zaran}

At the heart of our lower bound construction is \pref{lem:piell-properties}, which constructs a large policy class of size $\abs{\Piell} \asymp 1/\eps^\ell$ with bounded spanning capacity $\dimRL(\Piell) \lesssim H^{\ell}$. The size of the policy class directly corresponds to the lower bound on sample complexity in \pref{thm:lower-bound-online}. Thus, it is natural to ask: can we improve upon \pref{lem:piell-properties} to construct even larger policy classes? In this section, we explain how the current strategy of constructing $\Piell$ via ``block-free'' matrices (\pref{lem:matrix-construction}) \emph{cannot} yield an improved size of $\Piell$, thus ruling out the most direct way of improving our lower bound.

\paragraph{The Zarankiewicz Problem.} Block-free matrices are studied in the context of a problem in extremal graph theory called the Zarankiewicz problem \cite{zarankiewicz1951problem}.\footnote{We are aware of only one other work in learning theory which relates to the Zarankiewicz problem: \cite{ben2002limitations} uses bounds on $z(N,d;k,\ell)$ to prove separations between sign rank and VC dimension. Coincidentally, we used techniques from later improvements on these results \cite{alon2016sign} to prove our separations between eluder dimension and sign rank in \pref{sec:eluder-vs-rank}.} In the language of graph theory, binary matrices $B \in \crl{0,1}^{N \times d}$ correspond to bipartite graphs $G = \prn{V = [N] \cup [d], E}$, where an edge $e_{ij}$ is drawn if and only if the matrix entry $B_{ij} = 1$. Denote $K_{k,\ell}$ to be a complete bipartite graph with $k$ vertices on the left and $\ell$ vertices on the right. We define the Zarankiewicz function as follows:

\begin{definition}
    The Zarankiewicz function $z(N,d;k,\ell)$ denotes the maximum possible number of edges in a bipartite graph $G$ with vertex set $V = [N] \cup [d]$ which does not contain the subgraph $K_{k,\ell}$.
\end{definition}

Alternatively, $z(N,d;k, \ell)$ denotes the maximum number of 1s in an $N \times d$ binary matrix that is $(k, \ell)$ block-free. The Zarankiewicz problem is to determine tight bounds on $z(N,d;k,\ell)$. The following upper bound is shown in \cite{anderson1980bela}:
\begin{align*}
    z(N,d;k, \ell) \le (k-1)^{1/\ell} (d-\ell+1) N^{1-1/\ell} + (\ell-1)N. \numberthis\label{eq:zaran-bound}
\end{align*}

\paragraph{Implications for Our Construction.} In \pref{lem:matrix-construction}, we construct a binary matrix of size $N \times d$ which is block-free yet has a minimum density of 1s. In fact, either property (1) or (2) implies that the total number of 1s is at least $\eps/2 \cdot Nd$. On the other hand, \eqref{eq:zaran-bound} implies that
\begin{align*}
    \frac{\eps}{2} \cdot Nd &\le (\ell \log d)^{1/\ell} \cdot d \cdot N^{1-1/\ell} + \ell N 
\end{align*}
Since we are interested in the regime where $\ell$ is a constant (and therefore $\ell \ll \eps d$, we can ignore the last term. Rearranging we get that
\begin{align*}
    N &\lesssim \prn*{\frac{\ell \log d}{\eps}}^\ell.
\end{align*}
Thus, it is not possible to significantly increase the size of the policy class $\Piell$ beyond $1/\eps^\ell$ via this construction!

\subsubsection{Proof of \pref{lem:stopping-time}}
The proof of this stopping time lemma follows standard machinery for PAC lower bounds \citep{garivier2019explore, domingues2021episodic, sekhari2021agnostic}. In the following we use $\KL{P}{Q}$ to denote the Kullback-Leibler divergence between two distributions $P$ and $Q$ and $\kl{p}{q} $ to denote the Kullback-Leibler divergence between two Bernoulli distributions with parameters $p,q \in [0,1]$.

For any $\pistar \in \Piell$ we denote the random variable
\begin{align*}
N^\pistar = \sum_{t=1}^{\st \wedge \Tmax} \ind{\alg_t(\idx(x_1)_{1:H}) = \pistar(\idx(x_1)_{1:H}) \text{ and } \idx(x_1) \in \Jrel^\pistar},
\end{align*}
the number of episodes for which the algorithm's policy at round $t \in [\st \wedge \Tmax]$ matches that of $\pistar$ on a certain relevant state of $\pistar$.

In the sequel we will prove upper and lower bounds on the intermediate quantity $\sum_{\pistar \in \Pi} \En_0 \brk{N^\pistar}$ and relate these quantities to $\En_0[\st]$.

\paragraph{Step 1: Upper Bound.} First we prove an upper bound. We can compute that
\begin{align*}
&\sum_{\pistar \in \Pi} \En_0 \brk*{N^\pistar} \\
&= \sum_{t=1}^\Tmax \sum_{\pistar \in \Pi} \En_0 \brk*{\ind{\st > t - 1} \ind{\alg_t(\idx(x_1)_{1:H}) = \pistar(\idx(x_1)_{1:H}) \text{ and } \idx(x_1) \in \Jrel^\pistar}} \\
&= \sum_{t=1}^\Tmax \En_0 \brk*{\ind{\st > t - 1} \sum_{\pistar \in \Pi}
\ind{\alg_t(\idx(x_1)_{1:H}) = \pistar(\idx(x_1)_{1:H}) \text{ and } \idx(x_1) \in \Jrel^\pistar}}  \\
&\overleq{(i)} \sum_{t=1}^\Tmax \En_0 \brk*{\ind{\st > t - 1} } \leq \En_0 \brk*{\st \wedge \Tmax} \leq \En_0 \brk*{\st}. \numberthis \label{eq:sum-upper-bound}
\end{align*}
Here, the first inequality follows because for every  index $j$ and every $\pistar \in \Piell_j$, each $\pistar$ admits a unique sequence of actions (by Property (3) of \pref{lem:piell-properties}), so any policy $\alg_t$ can completely match with at most one of the $\pistar$.

\paragraph{Step 2: Lower Bound.} Now we turn to the lower bound. We use a change of measure argument.
\begin{align*}
\En_0 \brk*{N^\pistar} &\overgeq{(i)} \En_{0, \pistar} \brk*{N^\pistar}  - \Tmax \Delta(\Tmax) \\
&= \frac{1}{\abs{\Phi}} \sum_{\phi \in \Phi} \En_{0, \pistar, \phi} \brk*{N^\pistar}  - \Tmax \Delta(\Tmax) \\
&\overgeq{(ii)} \frac{1}{7}\cdot \frac{1}{\abs{\Phi}} \sum_{\phi \in \Phi}  \KL{\Pr_{0,\pistar, \phi}^{\cF_{\st \wedge \Tmax}}}{\Pr_{\pistar, \phi}^{\cF_{\st \wedge \Tmax}}} - \Tmax \Delta(\Tmax) \\
&\overgeq{(iii)} \frac{1}{7} \cdot \KL{\Pr_{0,\pistar}^{\cF_{\st \wedge \Tmax}}}{\Pr_{\pistar}^{\cF_{\st \wedge \Tmax}}} - \Tmax \Delta(\Tmax)
\end{align*}
The inequality $(i)$ follows from a change of measure argument using \pref{lem:change_measure}, with $\Delta(\Tmax) \coloneqq \Tmax^2/2^{H+3}$. Here, $\cF_{\st\wedge\Tmax}$ denotes the natural filtration generated by the first $\st\wedge \Tmax$ episodes. The inequality $(ii)$ follows from \pref{lem:kl-calculation}, using the fact that $M_{0, \pistar, \phi}$ and $M_{\pistar, \phi}$ have identical transitions and only differ in rewards at layer $H$ for the trajectories which reach the end of a relevant combination lock. The number of times this occurs is exactly $N^\pistar$. The factor $1/7$ is a lower bound on $\kl{1/2}{3/4}$. The inequality $(iii)$ follows by the convexity of KL divergence.

Now we apply \pref{lem:KL-to-kl} to lower bound the expectation for any $\cF_{\st\wedge \Tmax}$-measurable random variable $Z \in [0,1]$ as
\begin{align*}
\En_0 \brk*{N^\pistar} &\geq \frac{1}{7} \cdot \kl{\En_{0,\pistar} \brk*{Z}}{\En_{\pistar} \brk*{Z}} - \Tmax \Delta(\Tmax)\\
&\geq \frac{1}{7} \cdot (1- \En_{0,\pistar} \brk*{Z}) \log \prn*{\frac{1}{1-\En_{\pistar} \brk*{Z}}} - \frac{\log(2)}{7} - \Tmax \Delta(\Tmax),
\end{align*}
where the second inequality follows from the bound $\kl{p}{q} \ge (1-p)\log(1/(1-q)) - \log(2)$ \citep[see, e.g.,][Lemma 15]{domingues2021episodic}.

Now we pick $Z = Z_\pistar \coloneqq \ind{ \st < \Tmax \text{ and } \cE(\pistar, \phi, \alg_f(\cD)) }$ and note that $\En_{\pistar}[Z_\pistar] \geq 1-2\delta$ by assumption. This implies that
\begin{align*}
\En_0 \brk*{N^\pistar} &\geq  (1- \En_{0,\pistar} \brk*{Z_\pistar}) \cdot \frac{1}{7} \log \prn*{\frac{1}{ 2 \delta }} - \frac{\log(2)}{7} - \Tmax \Delta(\Tmax).
\end{align*}
Another application of \pref{lem:change_measure} gives
\begin{align*}
\En_0 \brk*{N^\pistar} &\geq (1- \En_0 \brk*{Z_\pistar}) \cdot \frac{1}{7}  \log \prn*{\frac{1}{ 2 \delta }} - \frac{\log(2)}{7} - \Delta(\Tmax) \prn*{\Tmax + \frac{1}{7}\log \prn*{\frac{1}{ 2 \delta }}}.
\end{align*}
Summing the above over $\pistar\in\Piell$, we get
\begin{align*}
\sum_{\pistar} \En_0 \brk*{N^\pistar} &\geq \prn*{ \abs{\Piell} - \sum_{\pistar}\En_0 \brk*{Z_\pistar} } \cdot \frac{1}{7} \log \prn*{\frac{1}{ 2 \delta }} \\
                                      &\qquad - \abs{\Piell} \cdot \frac{\log(2)}{7}  - \abs{\Piell} \cdot \Delta(\Tmax)  \prn*{\Tmax + \frac{1}{7}\log \prn*{\frac{1}{ 2 \delta }}}. \numberthis\label{eq:sum-lower-bound-1}
\end{align*}

It remains to prove an upper bound on $\sum_{\pistar}\En_0 \brk*{Z_\pistar}$. We calculate that
\begin{align*}
\sum_{\pistar}\En_0 \brk*{Z_\pistar} &= \sum_{\pistar}\En_0 \brk*{ \ind{\st < \Tmax \text{ and } \cE(\pistar, \phi, \alg_f(\cD))  }} \\
&\leq \sum_{\pistar}\En_0 \brk*{ \ind{ \Pr_\pistar\brk*{ \idx(x_1) \in \Jrel^\pistar \text{ and } \alg_f(\cD)(\idx(x_1)_{1:H}) = \pistar(\idx(x_1)_{1:H}) } \ge \frac{\eps}{4} } }   \\
&\leq \frac{4}{\eps} \cdot \En_0 \brk*{\sum_{\pistar} \Pr_\pistar\brk*{\idx(x_1) \in \Jrel^\pistar \text{ and } \alg_f(\cD)(\idx(x_1)_{1:H}) = \pistar(\idx(x_1)_{1:H}) } } \numberthis \label{eq:e0-upper-bound-1}
\end{align*}
The last inequality is an application of Markov's inequality.

Now we carefully investigate the sum. For any $\phi \in \Phi$, the sum can be rewritten as
\begin{align*}
    & \hspace{-0.3in}  \sum_{\pistar} \Pr_{\pistar, \phi} \brk*{ \idx(x_1) \in \Jrel^\pistar \text{ and } \alg_f(\cD)(\idx(x_1)_{1:H}) = \pistar(\idx(x_1)_{1:H}) } \\
    = ~& \sum_{\pistar} \sum_{x_1 \in \statesp_1} \Pr_{\pistar, \phi} \brk*{x_1} \Pr_{\pistar, \phi}\brk*{ \idx(x_1)\in \Jrel^\pistar \text{ and } \alg_f(\cD)(\idx(x_1)_{1:H}) = \pistar(\idx(x_1)_{1:H}) ~\mid~ x_1} \\
    \overeq{(i)} ~& \frac{1}{\abs{\statesp_1}} \sum_{x_1 \in \statesp_1} \sum_{\pistar} \Pr_{\pistar, \phi} \brk*{ \idx(x_1) \in \Jrel^\pistar \text{ and } \alg_f(\cD)(\idx(x_1)_{1:H}) = \pistar(\idx(x_1)_{1:H}) ~\mid~ x_1} \\
    \overeq{(ii)} ~& \frac{1}{\abs{\statesp_1}} \sum_{x_1 \in \statesp_1} \sum_{\pistar} \ind{ \idx(x_1) \in \Jrel^\pistar \text{ and } \alg_f(\cD)(\idx(x_1)_{1:H}) = \pistar(\idx(x_1)_{1:H})}. \numberthis\label{eq:sum-of-indicators}
\end{align*}
The equality $(i)$ follows because regardless of which MDP $M_\pistar$ we are in, the first state is distributed uniformly over $\statesp_1$. The equality $(ii)$ follows because once we condition on the first state $x_1$, the probability is either 0 or 1.

Fix any start state $x_1$. We can write
\begin{align*}
\hspace{2em}&\hspace{-2em} \sum_{\pistar} \ind{ \idx(x_1) \in \Jrel^\pistar \text{ and } \alg_f(\cD)(\idx(x_1)_{1:H})
\pistar(\idx(x_1)_{1:H}) } \\
&= \sum_{\pistar \in \Piell_{\idx(x_1)}} \ind{\alg_f(\cD)(\idx(x_1)_{1:H}) = \pistar(\idx(x_1)_{1:H}) } = 1,
\end{align*}
where the second equality uses the fact that on any index $j$, each $\pistar \in  \Piell_{j}$ behaves differently (Property (3) of \pref{lem:piell-properties}), so $\alg_f(\cD)$ can match at most one of these behaviors. Plugging this back into Eq.~\eqref{eq:sum-of-indicators}, averaging over $\phi \in \Phi$, and combining with Eq.~\eqref{eq:e0-upper-bound-1}, we arrive at the bound
\begin{align*}
    \sum_{\pistar}\En_0 \brk*{Z_\pistar} \le \frac{4}{\eps}.
\end{align*}
We now use this in conjunction with Eq.~\eqref{eq:sum-lower-bound-1} to arrive at the final lower bound
\begin{align*}
\sum_{\pistar} \En_0 \brk*{N^\pistar} &\geq \prn*{ \abs{\Piell} - \frac{4}{\eps} } \cdot \frac{1}{7} \log \prn*{\frac{1}{ 2 \delta }} \\
                                      &\qquad - \abs{\Piell} \cdot \frac{\log(2)}{7}  - \abs{\Piell} \cdot \Delta(\Tmax)  \prn*{\Tmax + \frac{1}{7}\log \prn*{\frac{1}{ 2 \delta }}}. \numberthis\label{eq:sum-lower-bound-2}
\end{align*}

\paragraph{Step 3: Putting it All Together.}
Combining Eqs.~\eqref{eq:sum-upper-bound} and \eqref{eq:sum-lower-bound-2}, plugging in our choice of $\Delta(\Tmax)$, and simplifying we get
\begin{align*}
    \En_0 \brk*{\st} &\geq \prn*{ \abs{\Piell} - \frac{4}{\eps} } \cdot \frac{1}{7} \log \prn*{\frac{1}{ 2 \delta }} - \abs{\Piell} \cdot \frac{\log(2)}{7}  - \abs{\Piell} \cdot \Delta(\Tmax)  \prn*{\Tmax + \frac{1}{7}\log \prn*{\frac{1}{ 2 \delta }}}. \\
    &\ge \prn*{ \frac{\abs{\Piell}}{2} - \frac{4}{\eps}} \cdot \frac{1}{7} \log \prn*{\frac{1}{ 2 \delta }} - \abs{\Piell} \cdot  \frac{\Tmax^2}{2^{H+3}} \prn*{\Tmax + \frac{1}{7}\log \prn*{\frac{1}{ 2 \delta }}}.
\end{align*}
The last inequality follows since $\delta \le 1/8$ implies $\log(1/(2\delta)) \geq 2 \log (2)$.

This concludes the proof of \pref{lem:stopping-time}.\qed

\subsubsection{Change of Measure Lemma}

\begin{lemma}
\label{lem:change_measure}
Let \(Z \in \brk{0, 1}\) be a  \(\cF_{\Tmax}\)-measurable random variable. Then, for every $\pistar \in \Piell$,
\begin{align*}
\abs{\En_0 \brk*{Z}  - \En_{0,\pistar} \brk*{Z}} \leq  \Delta(\Tmax) :=  \frac{\Tmax^2}{2^{H+3}}
\end{align*}
\end{lemma}
\begin{proof}
First, we note that
\begin{align*}
    \abs{ \En_0 \brk*{Z}  - \En_{0,\pistar}\brk*{Z} }\le \dtv\prn*{\Pr_0^{\cF_\Tmax}, \Pr_{0, \pistar}^{\cF_\Tmax}} \le \sum_{t=1}^\Tmax \En_0 \brk*{ \dtv\prn*{ \Pr_0[\cdot|\cF_{t-1}], \Pr_{0,\pistar}[\cdot|\cF_{t-1}] } }.
\end{align*}
Here $\Pr_0[\cdot|\cF_t]$ denotes the conditional distribution of the $t$-th trajectory given the first $t-1$ trajectories. Similarly $\Pr_{0, \pistar}[\cdot|\cF_t]$ is the averaged over decoders condition distribution of the $t$-th trajectory given the first $t-1$ trajectories. The second inequality follows by chain rule of TV distance (\pref{lem:chain-rule-tv}). 

Now we examine each term $ \dtv\prn*{ \Pr_0[\cdot|\cF_{t-1}], \Pr_{0,\pistar}[\cdot|\cF_{t-1}] }$. Fix a history $\cF_{t-1}$ and sequence $x_{1:H}$ where all $x_i$ have the same index. We want to bound the quantity
\begin{align*}
    \abs*{\Pr_{0, \pistar}\brk*{X_{1:H}^{(t)} = x_{1:H} ~\mid~ \cF_{t-1}} - \Pr_{0} \brk*{X_{1:H}^{(t)} =  x_{1:H} ~\mid~ \cF_{t-1} }},
\end{align*}
where it is understood that the random variable $x_{1:H}^{(t)}$ is drawn according to the MDP dynamics and algorithm's policy $\alg_t$ (which is in turn a measurable function of $\cF_{t-1}$).

We observe that the second term is exactly
\begin{align*}
\Pr_{0}\brk*{X_{1:H}^{(t)} =  x_{1:H} ~\mid~ \cF_{t-1} } = \frac{1}{\abs{\statesp_1}} \cdot \frac{1}{2^{H-1}},
\end{align*}
since the state $x_1$ appears with probability $1/\abs{\statesp_1}$ and the transitions in $M_0$ are uniform to the next state in the combination lock, so each sequence is equally as likely.

For the first term, again the state $x_1$ appears with probability $1/\abs{\statesp_1}$. Suppose that $\idx(x_1) \notin \Jrel^\pistar$. Then the dynamics of $\Pr_{0,\pi^\star, \phi}$ for all $\phi \in \Phi$ are exactly the same as $M_0$, so again the probability in this case is $1/(\abs{\statesp_1}2^{H-1})$. Now consider when $\idx(x_1) \in \Jrel^\pistar$. At some point $\wh{h} \in [H+1]$, the policy $\alg_t$ will deviate from $\pistar$ for the first time (if $\alg_t$ never deviates from $\pistar$ we set $\wh{h} = H+1)$. The layer $\wh{h}$ is only a function of $x_1$ and $\alg_t$ and does not depend on the MDP dynamics. The correct decoder must assign $\phi(x_{1:\wh{h}-1}) = \textsc{Good}$ and $\phi(x_{\wh{h}: H}) = \textsc{Bad}$, so therefore we have
\begin{align*}
    \Pr_{0, \pistar} \brk*{X_{1:H}^{(t)} =  x_{1:H} ~\mid~ \cF_{t-1} }
    &= \Pr_{0, \pistar}\brk*{\phi(x_{1:\wh{h}-1}) = \textsc{Good} \text{ and }\phi(x_{\wh{h}: H}) = \textsc{Bad} ~\mid~ \cF_{t-1} }
\end{align*}
If $x_1 \notin \cF_{t-1}$, i.e., we are seeing $x_1$ for the first time, then the conditional distribution over the labels given by $\phi$ is the same as the unconditioned distribution: 
\begin{align*}
    \Pr_{0, \pistar}\brk*{\phi(x_{1:\wh{h}-1}) = \textsc{Good} \text{ and }\phi(x_{\wh{h}: H}) = \textsc{Bad} ~\mid~ \cF_{t-1} } = \frac{1}{\abs{\statesp_1}} \cdot \frac{1}{2^{H-1}}.
\end{align*}
Otherwise, if $x_1 \in \cF_{t-1}$ then we bound the conditional probability by 1.
\begin{align*}
    &\Pr_{0, \pistar} \brk*{X_{1:H}^{(t)} =  x_{1:H} ~\mid~ \cF_{t-1} }  \le \frac{1}{\abs{\statesp_1}}.
\end{align*}
Putting this all together we can compute
\begin{align*}
    \Pr_{0, \pistar} \brk*{X_{1:H}^{(t)} =  x_{1:H} ~\mid~ \cF_{t-1} }  \quad \begin{cases}
         = \frac{1}{\abs{\statesp_1}} \cdot \frac{1}{2^{H-1}} &\text{if}\quad \idx(x_1) \notin \Jrel^\pistar, \\[0.5em]
        = \frac{1}{\abs{\statesp_1}} \cdot \frac{1}{2^{H-1}} &\text{if}\quad \idx(x_1) \in \Jrel^\pistar \text{ and } x_1 \notin \cF_{t-1}, \\[0.5em]
        \le \frac{1}{\abs{\statesp_1}} &\text{if}\quad \idx(x_1) \in \Jrel^\pistar \text{ and } x_1 \in \cF_{t-1}, \\[0.5em]
        = 0 &\text{otherwise.}
    \end{cases}
\end{align*}

Therefore we have the bound
\begin{align*}
\abs*{\Pr_{0, \pistar}\brk*{X_{1:H}^{(t)} = x_{1:H} ~\mid~ \cF_{t-1}} - \Pr_{0} \brk*{X_{1:H}^{(t)} =  x_{1:H} ~\mid~ \cF_{t-1} }}
&\leq \frac{1}{\abs{\statesp_1}}\ind{\idx(x_1) \in \Jrel^\pistar, x_1 \in \cF_{t-1}}.
\end{align*}
Summing over all possible sequences $x_{1:H}$ we have
\begin{align*}
    \dtv\prn*{ \Pr_0[\cdot|\cF_{t-1}], \Pr_{0,\pistar}[\cdot|\cF_{t-1}] } &\le \frac{1}{2} \cdot \frac{(t-1) \cdot 2^{H-1}}{\abs{\statesp_1}},
\end{align*}
since the only sequences $x_{1:H}$ for which the difference in the two measures are nonzero are the ones for which $x_1 \in \cF_{t-1}$, of which there are $(t-1) \cdot 2^{H-1}$ of them.

Lastly, taking expectations and summing over $t=1$ to $\Tmax$ and plugging in the value of $\abs{\statesp_1} = 2^{2H}$ we have the final bound.
\end{proof}

The next lemma is a straightforward modification of \citep[Lemma 5]{domingues2021episodic}, with varying rewards instead of varying transitions.
\begin{lemma}\label{lem:kl-calculation}
Let $M$ and $M'$ be two MDPs that are identical in transition and differ in the reward distributions, denote $r_h(x,a)$ and $r'_h(x,a)$. Assume that for all $(x,a)$ we have $r_h(x,a) \ll r'_h(x,a)$. Then for any stopping time $\st$ with respect to $(\cF^t)_{t\ge 1}$ that satisfies $\Pr_{M}[\st < \infty] = 1$,
\begin{align*}
    \KL{\Pr_{M}^{I_{\st}}}{\Pr_{M'}^{I_{\st}}} = \sum_{x\in\statesp, a \in \cA, h \in [H]} \En_{M}[N^\st_{s,a,h}] \cdot \KL{r_h(x,a)}{r'_h(x,a)},
\end{align*}
where $N^\st_{s,a,h} := \sum_{t=1}^\st \ind{(x_h^{(t)}, A_h^{(t)}) = (x,a) }$ and $I_\st: \Omega \mapsto \bigcup_{t\ge 1} \cI_t : \omega \mapsto I_{\st(\omega)}(\omega)$ is the random vector representing the history up to episode $\st$.
\end{lemma}

\subsection{Proof of \pref{thm:sunflower}}\label{sec:upper_bound_main} 
\subsubsection{Algorithmic Details and Preliminaries}   \label{sec:algorithm_details} 
In this section, we provide the details of the subroutines that do not appear in the main body, in Algorithms \ref{alg:sample}, \ref{alg:dp}, and \ref{alg:eval}. The transitions and  reward functions in \pref{line:reward_calculator} in \pref{alg:eval} are computed using Eqs.~\eqref{eq:empirical_MRP_dynamics} and \eqref{eq:empirical_MRP_rewards}, which are specified below, after introducing additional notation. 

\begin{algorithm}[ht] 
    \caption{$\datacollector$}\label{alg:sample}
    \begin{algorithmic}[1]
            \Require State: \(x\), Reacher policy: \(\pi_x\), Exploration policy set: \(\Picore\), Number of samples: \(n\). 
            \If {$x = x_\top$}  \hfill \algcomment{Uniform sampling for start state \(x_\top\)}  
            \For{$t = 1, \dots, n$}
                \State Sample $\pi' \sim\mathrm{Uniform}(\Picore)$, and run $\pi'$ to collect \(\tau = \prn{x_1, a_1, \cdots, x_H, a_H}\). 
            \State $\mD_x \leftarrow \mD_x\cup\{\tau\}$. 
            \EndFor 
            \Else \hfill \algcomment{\(\pi_x\)-based sampling for all other states \(x \neq x_\top\)}  
        \State Identify the layer \(h\) such that $x\in\statesp_h$. 
        \For{$t = 1, \dots, n$} 
            \State Run $\pi_x$ for first $h-1$ layers and collect trajectory \(\prn{x_1, a_1, \cdots, x_{h-1}, a_{h-1}, x_h}\). 
            \If{$x_h = x$}\label{line:is1} 
                \State Sample $\pi'\sim\mathrm{Uniform}(\Picore)$ and run $\pi'$ to collect \(\prn{x_h, a_h, \cdots, x_H, a_H}\).
                \State $\mD_x\leftarrow \mD_x\cup\{\tau = \prn{x_1, a_1, \cdots, x_H, a_H}\}$. \label{line:is2}
            \EndIf
        \EndFor  
            \EndIf 

        \State \textbf{Return} dataset $\mD_x$.  
    \end{algorithmic}
\end{algorithm} 

 
\begin{algorithm}[ht]
\caption{$\estreach$}\label{alg:dp}
	\begin{algorithmic}[1]
		\Require State space \(\Stab\), MRP \(\MRPsign\),  State $\bar{x}\in\statesp^{\mathrm{tab}}$.
        \State Let $P$ be the transition of $\MRPsign$.
		\State Initialize $V(x) = \ind{x = \bar{x}}$ for all \(x \in \Stab\). 
        \State \textbf{Repeat} \(H + 1\) times:
			\State \(\quad\) For all  $x \in \Stab$, calculate $V(x)\leftarrow \sum_{x'\in\statesp^{\mathrm{tab}}} P_{x\to x'} \cdot V(x').$ \hfill \algcomment{Dynamic Programming} 
		\State \textbf{Return} $V(x_\top)$.
	\end{algorithmic}
\end{algorithm}

\begin{algorithm}[ht]
	\caption{$\evaluate$}\label{alg:eval}
	\begin{algorithmic}[1] 
            \Require Policy set $\Picore$, Reachable states \(\reachablestates\), Datasets $\{\mD_x\}_{x \in \reachablestates}$, Policy $\pi$. 
            \State Compute \(\SHP{\pi} \leftarrow \statesp_\pi^+ \cap \reachablestates\) and \(\Stab = \SHP{\pi} \cup \crl{x_\bot}\). 

		\For{$x, x'$ in $\Stab$}  \hfill \algcomment{Compute transitions and rewards on \(\Stab\)}
            \State Let \(h, h'\) be such that \(x \in \statesp_h\) and \(x' \in \statesp_{h'}\) 
            \If{\(h < h'\)}
           \State Calculate $\widehat{P}_{x\to x'}^\pi, \widehat{r}_{x\to x'}^\pi$ according to \eqref{eq:empirical_MRP_dynamics} and \eqref{eq:empirical_MRP_rewards};\label{line:reward_calculator} 
           \Else 
            \State Set $\widehat{P}_{x\to x'}^\pi \leftarrow 0$, $\widehat{r}_{x\to x'}^\pi \leftarrow 0$.  
           \EndIf
            \EndFor 
		\State Set $\widehat{V}(x) = 0$ for all \(x \in \Stab\).  
            \State \textbf{Repeat} for \(H+1\) times: \hfill \algcomment{Evaluate \(\pi\) by dynamic programming} 
            \State \quad For all $x \in \Stab$, calculate $\widehat{V}(x)\leftarrow \sum_{\Stab} \widehat{P}_{x\to x'}^\pi \cdot \left(\widehat{r}_{x\to x'}^\pi + \widehat{V}(x')\right).$ 
		\State \textbf{Return} $\widehat{V}(x_\top)$.  
	\end{algorithmic}
\end{algorithm}



\paragraph{Additional notation.} Recall that we assumed that the state space \(\statesp = \statesp_1 \times \dots \statesp_H\) is layered. Thus, given a state \(x\), we can infer the layer \(h\) such that \(x \in \statesp_h\). By definition \(x_\top\) belongs to the layer \(h = 0\) and \(x_\bot\) belongs to the layer \(h = H\). In the following, we define additional notation: 

\begin{enumerate}[label=\((\alph*)\)]  

\item \textit{Sets \(\event{x}{x'}{\bar \statesp}\)}: For any set \(\bar \statesp\), and states \(x, x' \in \statesp\), we define \(\event{x}{x'}{\bar \statesp}\) as the set of all the trajectories that go from \(x\) to \(x'\) without passing through any state in \(\bar \statesp\) in between. 

More formally, let state \(x\) be at layer \(h\), and \(x'\) be at layer \(h'\). Then, \(\event{x}{x'}{\bar \statesp}\) denotes the set of all the trajectories \(\tau = (x_1, a_1, \cdots, x_H, a_H)\) that satisfy all of the following:  
\begin{enumerate}[label=\(\bullet\)] 
\item \(x_h = x\), where \(x_h\) is the state at timestep \(h\) in \(\tau\). 
\item \(x_{h'} = x'\), where \(x_{h'}\) is the state at timestep \(h'\) in \(\tau\). 
\item For all \(h < \wt{h} < h'\), the state \(x_{\wt{h}}\), at time step \(\wt h\) in \(\tau\), does not lie in the set \(\bar \statesp\).  
\end{enumerate}

\noindent 
Note that when \(h' \leq h\), we define 
 \(\event{x}{x'}{\bar \statesp} = \emptyset\). Additionally, we define  \(\event{x_\top}{x}{\bar \statesp}\) as the set of all trajectories that go to \(x'\) (from a start state) without going through any state in \(\bar \statesp\) in between. Finally, we define   \(\event{x}{x_\bot}{\bar \statesp}\) as the set of all the trajectories that go from \(x\) at time step \(h\) to the end of the episode without passing through any state in \(\bar \statesp\) in between. 

Furthermore, we use the shorthand $\setall{\pi}{x}{x'} \coloneqq \event{x}{x'}{\statesp_\pi}$ to denote the set of all the trajectories that go from \(x\) to \(x'\) without passing though any leaf state \(\statesp_\pi\). 




\item Using the above notation, for any \(x \in \statesp\) and set \(\bar \statesp \subseteq \statesp\), we define \(\bard{\pi}{x}{\bar{\statesp}}\) as the probability of reaching \(x\) (from a start state)  without passing through any state in \(\bar \statesp\) in between, i.e.~
\begin{align*}
\bard{\pi}{x}{\bar{\statesp}} &= \bbP^{\pi}\brk*{\tau~\text{reaches}~x~\text{without passing through any \(x'\in \bar{\statesp}\) before reaching \(x\)}} \\
&= \bbP^{\pi}\brk*{\tau~\in \event{x_\top}{x}{\bar \statesp}}. \numberthis \label{eq:dbar}
\end{align*}

\end{enumerate}

We next recall the notation of Markov Reward Process and formally define both the population versions of policy-specific MRPs. 

\paragraph{Markov Reward Process (MRP).} A Markov reward process
\begin{align*}
    \MRPsign = \mathrm{MRP}(\statesp, P, R, H, x_\top, x_\bot)
\end{align*}
is defined over the state space \(\statesp\) with start state \(x_\top\) and end state \(x_\bot\),  for trajectory length \(H + 2\). Without loss of generality, we assume that \(\crl{x_\top, x_\bot} \in \statesp\). The transition kernel is denoted by \(P: \statesp \times \statesp \to [0,1] \), such that for any $x \in \statesp$, $\sum_{x'} P_{x \to x'} = 1$; the reward kernel is denoted \(R: \statesp \times \statesp \to \Delta([0,1])\). Throughout, we use the notation $\rightarrow$ to signify that the transitions and rewards are defined along the edges of the MRP. 

A trajectory in $\MRPsign$ is of the form \(\tau = (x_\top, x_1, \cdots, x_{H}, x_\bot)\), where \(x_h \in \statesp\) for all \(h \in [H]\). Furthermore, from any state \(x \in \statesp\), the MRP transitions\footnote{Our definition of Markov Reward Processes (MRP) deviates from MDPs that we considered, in the sense that we do not assume that the state space \(\statesp\) is layered in an MRP. This variation is only adapted to simplify the proofs and the notation.} to another state \(x' \in \statesp\) with probability \(P_{x \rightarrow x'}\), and obtains the rewards \(r_{x \rightarrow x'} \sim R_{x \rightarrow x'}\). Thus,  
\begin{align*}
\bbP^{\MRPsign}[\tau] =  P_{x_\top \rightarrow x_1} \cdot \prn*{\prod_{h=1}^{H-1} P_{x_h \rightarrow x_{h+1}} } \cdot P_{x_h \rightarrow x_\bot},  
\end{align*} 
and the rewards   
\begin{align*} 
R^{\MRPsign}(\tau) = r_{x_{\top} \rightarrow x_1} + \sum_{h=1}^H r_{x_{h} \rightarrow x_{h+1}} + r_{x_{H} \rightarrow x_\bot}. 
\end{align*} 
Furthermore, in all the MRPs that we consider in herein, we have \(P_{x_\bot \rightarrow x_\bot} = 1\) and \(r_{x_\bot \rightarrow x_\bot} = 0\). 

\paragraph{Policy-Specific Markov Reward Processes.} 
A key technical tool in our analysis will be policy-specific MRPs that depend on the set \(\reachablestates\) of the states that we have explored so far. Recall that for any policy \(\pi\), \(\statesp_\pi^+ = \statesp_\pi \cup \crl{x_\top, x_\bot}\), \(\SHP{\pi} = \statesp_\pi^+ \cap \reachablestates\) and \(\SRem{\pi} = \statesp_\pi^+  \setminus (\SHP{\pi} \cup \crl{x_\bot})\). We define the expected and the empirical versions of policy-specific MRPs below; see \pref{fig:MRP} for an illustration. 
\begin{enumerate}[label=\((\alph*)\)]
    \item \textbf{Expected Version of Policy-Specific MRP.} We define
        \begin{align*}
            \MRPsign^\pi_{\reachablestates} = \mathrm{MRP}(\statesp_\pi^+, P^\pi, r^\pi, H, x_\top, x_\bot)
        \end{align*}
        where 

\begin{enumerate}
    \item \textit{Transition Kernel \(P^\pi\):} For any \(x \in \SHP{\pi}\) and \(x' \in \statesp_\pi^+\), we have 
    \begin{align*}
P_{x\to x'}^\pi =  \EE^\pi\left[\ind{\tau\in \setall{\pi}{x}{x'}} \mid x_h = x\right], \numberthis \label{eq:policy_MRP_dynamics} 
\end{align*}
where the expectation above is w.r.t.~the  trajectories drawn using \(\pi\) in the underlying MDP, and \(h\) denotes the time step such that \(x \in \statesp_h\) (again, in the underying MDP). Thus, the transition $P_{x\to x'}^\pi$ denotes the probability of taking policy $\pi$ from $x$ and directly transiting to $x'$ without visiting any other states in $\statesp_\pi$. Furthermore, \(P^\pi_{x \rightarrow x'} = \ind{x' = x_\bot}\) for all \(x \in \SRem{\pi}\cup \{x_\bot\}\). 

\item \textit{Reward Kernel \(r^\pi\):} 
 For any \(x \in \SHP{\pi}\) and \(x' \in \statesp_\pi^+\), we have 
\begin{align*}
r_{x\to x'}^\pi & \coloneqq \EE^\pi\brk*{R(\tau_{h:h'})\ind{\tau\in \setall{\pi}{x}{x'}} \mid x_h = x}, \numberthis \label{eq:policy_MRP_reward} 
\end{align*} 
where \(R(\tau_{h:h'})\) denotes the reward for the partial trajectory \(\tau_{h:h'}\) in the underlying MDP. The reward $r_{x\to x'}^\pi$ denotes the expectation of rewards collected by taking policy $\pi$ from $x$ and directly transiting to $x'$ without visiting any other states in $\statesp_\pi$. Furthermore, \(r^\pi_{x \rightarrow x'} = 0\) for all \(x \in \SRem{\pi}\cup \{x_\bot\}\). 
\end{enumerate}
Throughout the analysis, we use $\Pr^\MRPsign[\cdot] \coloneqq \Pr^{\MRPsign^{\pi}_{\reachablestates}}[\cdot]$ and $\En^\MRPsign[\cdot] \coloneqq \En^{\MRPsign^{\pi}_{\reachablestates}}[\cdot]$ as a shorthand, whenever clear from the context.


\medskip

\item \textbf{Empirical Version of Policy-Specific MRPs.}  Since the learner only has sampling access to the underlying MDP, it can not directly construct the MRP \(\MRPsign^\pi_{\reachablestates}\). Instead, in  \pref{alg:main}, the learner constructs an empirical estimate for \(\MRPsign^\pi_{\reachablestates}\), defined as
    \begin{align*}
        \widehat \MRPsign^\pi_{\reachablestates} = \mathrm{MRP}(\statesp_\pi^+, \widehat P^\pi, \widehat r^\pi, H, x_\top, x_\bot)
    \end{align*}
    where  


\begin{enumerate}[label=\(\bullet\)]
    \item \textit{Transition Kernel \(\widehat P^\pi\):} For any \(x \in \SHP{\pi}\) and \(x' \in \statesp_\pi^+\), we have 
\begin{align*} 
	\widehat{P}_{x\to x'}^\pi = \frac{|\Picore|}{|\mD_x|} \sum_{\tau\in \mD_x}\frac{\ind{\pi\cons\tau_{h:h'}}}{\sum_{\pi'\in \Picore}\ind{\pi' \cons\tau_{h:h'}}}\ind{\tau\in \setall{\pi}{x}{x'} }, \numberthis \label{eq:empirical_MRP_dynamics}
\end{align*}
	where \(\Picore\) denotes the core of the sunflower corresponding to \(\Pi\) and \(\cD_x\) denotes a dataset of trajectories collected via \(\datacollector(x, \pi_x, \Picore, n_2)\). Furthermore, \(\widehat{P}^\pi_{x \rightarrow x'} = \ind{x' = x_\bot}\) for all \(x \in \SRem{\pi}\cup \{x_\bot\}\). 
    \item \textit{Reward Kernel \(\widehat r^\pi\):} 
 For any \(x \in \SHP{\pi}\) and \(x' \in \statesp_\pi^+\), we have 
\begin{align*} 
	\widehat{r}_{x\to x'}^\pi = \frac{|\Picore|}{|\mD_x|} \sum_{\tau\in \mD_x}\frac{\ind{\pi\cons\tau_{h:h'}}}{\sum_{\pi'\in \Picore}\ind{\pi' \cons\tau_{h:h'}}}\ind{\tau\in \setall{\pi}{x}{x'} }R(\tau_{h:h'}), \numberthis \label{eq:empirical_MRP_rewards} 
\end{align*}
where \(\Picore\) denotes the core of the sunflower corresponding to \(\Pi\), \(\cD_x\) denotes a dataset of trajectories collected via \(\datacollector(x, \pi_x, \Picore, n_2)\), and  $R(\tau_{h:h'}) = \sum_{i=h}^{h'-1} r_i$. Furthermore, \(\wh r^\pi_{x \rightarrow x'} = 0\) for all \(x \in \SRem{\pi}\).
\end{enumerate}
The above approximates the MRP given by \pref{eq:emp-transitions} and \pref{eq:emp-rewards}.
\end{enumerate}

\begin{figure}[!t]
    \centering
    \includegraphics[scale=0.44, trim={0cm 12cm 41cm 0cm}, clip]{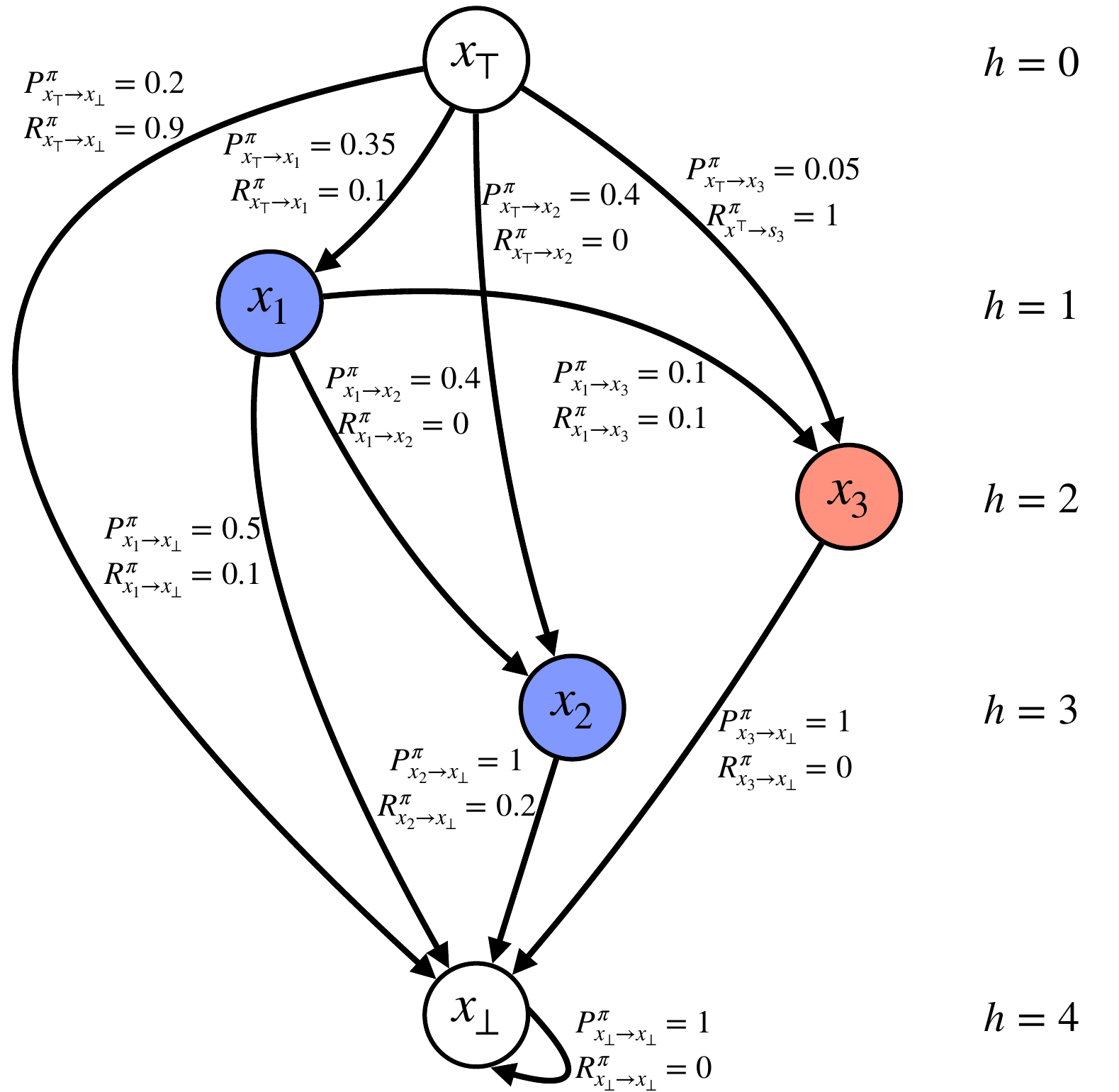}
    \caption{Illustration of an MRP $\MRPsign^\pi_{\reachablestates}$ with $\statesp_\pi = \crl{x_1, x_2, x_3}$ and $\reachablestates = \crl{x_1, x_2}$. In the original MDP $M$, $x_1 \in \statesp_1$, $x_2 \in \statesp_3$, and $x_3 \in \statesp_2$. The edges are labeled with the values of $P^\pi_{x \to x'}$ and $R^\pi_{x \to x'}$. Notice that (1) there are no edges from $x_2 \to x_3$ or $x_3 \to x_2$ because trajectories cannot go from later layers to earlier ones; (2) since $x_3 \notin \reachablestates$, there is no edge from $x_3 \to x_2$, and instead we have $P^\pi_{x_3 \to x_\bot} = 1$ and $R^\pi_{x_3 \to x_\bot} = 0$; (3) $x_\bot$ is an absorbing state with no rewards.} 
    \label{fig:MRP}
\end{figure}

\paragraph{Properties of Trajectories in the Policy-Specific MRPs.} We state several properties of trajectories in the policy-specific MRPs, which will be used in the proofs. Let \(\tau = (x_\top, x_1, \cdots, x_{H}, x_\bot)\) denote a trajectory from either $\MRPsign^\pi_{\reachablestates}$ or $\widehat \MRPsign^\pi_{\reachablestates}$.  
\begin{enumerate}[label=\((\arabic*)\)] 
    \item For some $k \leq H$ we have $x_1, \cdots, x_k \in \statesp_\pi$ and $x_{k+1} = \cdots = x_H = x_\bot$ (if $k = H$ we say the second condition is trivially met). 
    \item Each state in $x_1, \cdots, x_k$ is unique.
    \item Letting $\mathsf{h}(x)$ denote the layer that a state $x \in \statesp_\pi$ is in, we have $\mathsf{h}(x_1) < \cdots < \mathsf{h}(x_k)$.
    \item Either (a) $x_1, \cdots, x_k \in \SHP{\pi}$, or (b) $x_1, \cdots, x_{k-1} \in \SHP{\pi}$ and $x_k \in \SRem{\pi}$.
\end{enumerate}

\paragraph{Parameters Used in \pref{alg:main}.}  Here, we list all the parameters that are used in \pref{alg:main} and its subroutines: 
\begin{align*}
n_1 &= C_1 \frac{(D+1)^4K^2\log(|\Pi|(D+1)/\delta)}{\epsilon^2}, \\
 n_2 &= C_2 \frac{D^3(D+1)^2K^2\log(|\Pi| (D+1)^2/\delta)} {\epsilon^3}, \numberthis \label{eq:alg_parameter}
\end{align*}
where $C_1, C_2 > 0$ are absolute numerical constants, which will be specified later in the proofs. 

\subsubsection{Supporting Technical Results}
We start by stating the following variant of the classical  simulation lemma \citep{kearns2002near, agarwal2019reinforcement}.

\begin{lemma}[Simulation Lemma {\citep[Lemma F.3]{foster2021statistical}}]\label{lem:sim}
Let \(\MRPsign = (\statesp, P, r, H, x_\top, x_\bot)\) be a Markov Reward Process. Then, the empirical version \(\widehat \MRPsign = (\statesp, \widehat P, \widehat r, H, x_\top, x_\bot)\) corresponding to \(\MRPsign\) satisfies: 
	$$|V_{\mathrm{MRP}} - \widehat{V}_{\mathrm{MRP}}|\le \sum_{x\in \statesp} d_{\MRPsign}(x)\cdot \left(\sum_{x'\in\statesp} |P_{x\to x'} - \widehat{P}_{x\to x'}| + \left|r_{x\to x'} - \widehat{r}_{x\to x'}\right|\right),$$
	where $d_{\MRPsign}(x)$ is the probability of reaching $x$ under $\MRPsign$, and \(V_{\mathrm{MRP}}\) and \(\widehat V_{\mathrm{MRP}}\) denotes the value of \(x_\top\) under \(\MRPsign\) and \(\widehat{\MRPsign}\) respectively. 
\end{lemma}


The following technical lemma shows that for any policy \(\pi\), the empirical version of policy-specific MRP closely approximates its expected version. 
\begin{lemma}\label{lem:dynamics_error} 
Let \pref{alg:main} be run with the parameters given in Eq.~\pref{eq:alg_parameter}, and consider any iteration of the while loop in \pref{line:while_loop} with the instantaneous set  \(\reachablestates\). Further, suppose that $|\mD_x|\ge \tfrac{\epsilon n_2}{24D}$ for all $x\in \reachablestates$. Then, with probability at least $1 - \delta$, the following hold: 
\begin{enumerate}[label=\((\alph*)\)] 
\item For all \(\pi \in \Pi\), \(x \in \SHP{\pi}\) and \(x' \in \statesp_\pi \cup \crl{x_\bot}\),
\begin{align*}
	\max\crl*{|P_{x\to x'}^\pi - \widehat{P}_{x\to x'}^\pi|, \left|r_{x\to x'}^\pi - \widehat{r}_{x\to x'}^\pi\right|} &\leq  \frac{\epsilon}{12D(D+1)}. 
\end{align*}
\item For all \(\pi \in \Pi\) and \(x' \in \statesp_\pi \cup \crl{x_\bot}\),
\begin{align*}
		\max \crl{|P_{x_\top\to x'}^\pi - \widehat{P}_{x_\top\to x'}^\pi|, |r^\pi_{x_{\top} \to x'} - \widehat{r}^\pi_{x_{\top} \to x'}|} \le \frac{\epsilon}{12(D+1)^2}.  
\end{align*} 
\end{enumerate}
\end{lemma}

In the sequel, we define the event that the conclusion of \pref{lem:dynamics_error} holds as $\cE_\mathrm{est}$. 

\begin{proof} Fix any \(\pi \in \Pi\). We first prove the bound for $x\in \SHP{\pi}$. Let \(x\) be at layer \(h\). 
Fix any policy \(\pi \in \Pi\), and consider any state \(x' \in \statesp_\pi \cup \crl{x_{\bot}}\), where \(x'\) is at layer \(h'\). 
Note that since \(\Pi\) is a \((K, D)\)-sunflower, with its core \(\Picore\) and petals \({\crl{\statesp_\pi}}_{\pi \in \Pi}\), we must have that any trajectory \(\tau \in \setall{\pi}{x}{x'}\) is also consistent with at least one \(\pi_e \in \Picore\). Furthermore, for any such \(\pi_e\), we have 
	\begin{align*}
		\Pr^{\pi_e}\brk*{\tau_{h:h'} \mid x_h=x} &= \prod_{i=h}^{h'-1} P\brk*{x_{i+1} \mid x_i, \pi_e(x_i), x_h = x} \\ 
		&= \prod_{i=h}^{h'-1} P\brk*{x_{i+1} \mid x_i, \pi(x_i), x_h = x} = \Pr^{\pi} \brk*{\tau_{h:h'} \mid x_h=x},  \numberthis \label{eq:pi-pie} 
	\end{align*} 
	where the second line holds because both \(\pi \cons \tau_{h:h'}\) and \(\pi_e \cons \tau_{h:h'}\). Next, recall from Eq.~\eqref{eq:policy_MRP_dynamics}, that 
\begin{align*}
P_{x\to x'}^\pi &= \EE^\pi \brk*{\ind{\tau \in \setall{\pi}{x}{x'}} \mid x_h = x }.  \numberthis \label{eq:original_MRP_estimate} 
\end{align*}	 

Furthermore, from Eq.~\eqref{eq:empirical_MRP_dynamics}, recall that the empirical estimate \(\widehat{P}_{x\to x'}^\pi \) of \({P}_{x\to x'}^\pi \) is given by :
\begin{align*}
\widehat{P}_{x\to x'}^\pi = \frac{1}{|\mD_x|}\sum_{\tau\in \mD_x}\frac{\ind{ \tau \in \setall{\pi}{x}{x'} }}{\frac{1}{|\Picore|}\sum_{\pi_e\in \Picore}\ind{\pi_e\cons\tau_{h:h'}}},  \numberthis \label{eq:empirical_estimate} 
\end{align*}
where the dataset \(\cD_x\) consists of i.i.d.~samples, and is collected in  lines \ref{line:is1}-\ref{line:is2} in \pref{alg:sample} ($\datacollector$), by first running the policy \(\pi_x\) for \(h\) timesteps and if the trajectory reaches \(x\), then executing \(\pi_e \sim \unif(\Picore)\) for the remaining time steps (otherwise this trajectory is rejected). Let the law of this process be \(q\). We thus note that,
\begin{align*}
\hspace{0.4in} &\hspace{-0.4in} \En_{\tau\sim q} \brk*{\widehat{P}_{x\to x'}^\pi} \\  
&= \En_{\tau \sim q} \brk*{\frac{\ind{\tau \in \setall{\pi}{x}{x'}}}{\frac{1}{|\Picore|}\sum_{\pi_e\in \Picore}\ind{\pi_e\cons\tau_{h:h'}}} \mid x_h = x}    \\ 
&= \sum_{\tau \in \setall{\pi}{x}{x'}} \Pr_q \brk*{\tau_{h:h'} \mid x_h = x} \cdot \frac{1}{\frac{1}{|\Picore|}\sum_{\pi_e\in \Picore} \ind{\pi_e\cons\tau_{h:h'}}}  \\
&\overeq{\proman{1}}  \sum_{\tau \in \setall{\pi}{x}{x'}} \frac{1}{|\Picore|}\sum_{\pi'_e\in\Picore}\ind{\tau \in \setall{\pi'_e}{x}{x'}} \Pr^{\pi'_e} \brk*{\tau_{h:h'}  \mid x_h=x} \\
&\hspace{12em}\cdot \frac{1}{\frac{1}{|\Picore|}\sum_{\pi_e\in \Picore} \ind{\pi_e\cons\tau_{h:h'}}}  \\ 
&\overeq{\proman{2}}  \sum_{\tau \in \setall{\pi}{x}{x'}} \frac{1}{|\Picore|}\sum_{\pi'_e\in\Picore}\Pr^{\pi} \brk*{\tau_{h:h'} \mid x_h=x}  \cdot \frac{\ind{\pi'_e \cons \tau_{h:h'}}}{\frac{1}{|\Picore|}\sum_{\pi_e\in \Picore}  \ind{\pi_e\cons\tau_{h:h'}}}  \\ 
&= \sum_{\tau \in \setall{\pi}{x}{x'}} \Pr^{\pi}\brk*{\tau_{h:h'} \mid x_h=x}   \\  
&\overeq{\proman{3}}  \EE^\pi \brk*{\ind{\tau \in \setall{\pi}{x}{x'}} \mid x_h = x } =  P_{x\to x'}^\pi, 
\end{align*} 
where \(\proman{1}\) follows from the sampling strategy in \pref{alg:sample} after observing \(x_h = x\), and \(\proman{2}\) simply uses the relation \pref{eq:pi-pie} since both \(\pi'_e \cons \tau_{h:h'}\) and \(\pi \cons \tau_{h:h'}\) hold. Finally, in \(\proman{3}\), we use the relation \pref{eq:original_MRP_estimate}.

We have shown that \(\widehat{P}_{x\to x'}^\pi\) is an unbiased estimate of \(P^\pi_{x \to x'}\) for any \(\pi\) and \(x, x' \in \statesp_\pi^+\). Thus, using Hoeffding's inequality (\pref{lem:hoeffding}), followed by a union bound, we get that with probability at least \(1 - \delta/4\), for all \(\pi \in \Pi\), \(x \in \SHP{\pi}\),  and \(x' \in \statesp_\pi \cup \crl{x_\bot}\), 
\begin{align*}
|\widehat{P}_{x\to x'}^\pi - P_{x\to x'}^\pi|\le K\sqrt{\frac{2\log(4 |\Pi|D(D+1)/\delta)}{|\mD_x|}}, 
\end{align*}
where the additional factor of \(K\) appears because for any \(\tau \in \setall{\pi}{x}{x'}\), there must exist some \(\pi_e \in \Picore\) that is also consistent with \(\tau\) (as we showed above), which implies that each of the terms in Eq.~\eqref{eq:empirical_estimate} satisfies the bound a.s.: 
\begin{align*}
\abs*{\frac{\ind{ \tau \in \setall{\pi}{x}{x'} }}{\tfrac{1}{|\Picore|}\sum_{\pi_e\in \Picore}\ind{\pi_e\cons\tau_{h:h'}}}} &\leq \abs{\Picore} = K. 
\end{align*} 

Since $|\mD_x|\ge \tfrac{\epsilon n_2}{24D}$ by assumption, we have 
\begin{align*}
|\widehat{P}_{x\to x'}^\pi - P_{x\to x'}^\pi|\le  K\sqrt{\frac{48D\log(4|\Pi|D(D+1)/\delta)}{\epsilon n_2}}.
\end{align*}

Repeating a similar argument for the empirical reward estimation in Eq.~\pref{eq:empirical_MRP_dynamics}, we get that with probability at least \(1 - \delta/4\), for all \(\pi \in \Pi\), and \(x \in \SHP{\pi}\) and \(x' \in \statesp_\pi \cup \crl{x_\bot}\), we have that 
\begin{align*}
|\widehat{r}_{x\to x'}^\pi - r_{x\to x'}^\pi|\le  K\sqrt{\frac{48D\log(4 |\Pi|D(D+1)/\delta)}{\epsilon n_2}}. 
\end{align*}

Similarly, we can also get for any $\pi\in \Pi$ and $x'\in\statesp_\pi\cup\{x_\bot\}$, with probability at least $1 - \delta/2$, 
\begin{align*}
\max\crl*{|\widehat{r}_{x_\top\to x'}^\pi - r_{x_\top\to x'}^\pi|, |\widehat{P}_{x_\top\to x'}^\pi - P_{x_\top\to x'}^\pi|}  &\le K\sqrt{\frac{2\log(4 |\Pi|(D+1)/\delta)}{|\mD_{x_\top}|}} \\ &= K\sqrt{\frac{2\log(4 |\Pi|(D+1)/\delta)}{n_1}},
\end{align*} 
where the last line simply uses the fact that \(\abs{\cD_{x_\top}} = n_1\). 
%
The final statement is due to a union bound on the above results. This concludes the proof of \pref{lem:dynamics_error}.\end{proof}

\begin{lemma}\label{lem:dbar}
Fix a policy $\pi \in \Pi$ and a set of reachable states $\reachablestates$, and consider the policy-specific MRP $\MRPsign^{\pi}_{\reachablestates}$ (as defined by Eqs.~\eqref{eq:policy_MRP_dynamics} and \eqref{eq:policy_MRP_reward}). Then for any $x\in\SRem{\pi}$, the quantity $\bard{\pi}{x} {\SRem{\pi}} = d^{\MRPsign}(x)$, where $d^{\MRPsign}(x)$ is the occupancy of state $x$ in $\MRPsign^{\pi}_{\reachablestates}$. 
\end{lemma}

\begin{proof}
We use $\bar{\tau}$ to denote a trajectory in $\MRPsign^{\pi}_{\reachablestates}$ and $\tau$ to denote a ``corresponding'' (in a sense which will be described shortly) trajectory in the original MDP $M$. For any $x\in\SRem{\pi}$, we have 
\begin{align*}
    d^{\MRPsign}(x) & = \sum_{\bar{\tau} ~\text{s.t.}~s\in \bar{\tau}} \Pr^{\MRPsign} \brk*{\bar{\tau}} = \sum_{k = 0}^{H-1} \sum_{\bar{x}_{1}, \bar{x}_{2}, \cdots, \bar{x}_{k}\in \SHP{\pi}} \Pr^{\MRPsign} \brk*{\bar{\tau} = (x_{\top}, \bar{x}_{1}, \cdots, \bar{x}_{k}, x, x_\bot, \cdots)}.
\end{align*}
The first equality is simply due to the definition of $d^\MRPsign$. For the second equality, we sum up over all possible sequences which start at $x_\top$, pass through some (variable length) sequence of states $\bar{x}_1,\cdots, \bar{x}_k \in \SHP{\pi}$, then reach $x$ and the terminal state $x_\perp$. By definition of the policy-specific MRP, we know that once the MRP transits to a state $x\in\SRem{\pi}$, it must then transit to $x_\bot$ and repeat $x_\bot$ until the end of the episode.

Now fix a sequence $\bar{x}_1,\cdots, \bar{x}_k \in \SHP{\pi}$. We relate the term in the summand to the probability of corresponding trajectories in the original MDP $M$. To avoid confusion, we let $x_{h_1}, \dots, x_{h_k} \in \SHP{\pi}$ denote the corresponding sequence of states in the original MDP, which are unique and satisfy $h_1 < h_2 < \cdots < h_k$. We also denote $x_{h_{k+1}} = x$.

Using the definition of $\MRPsign^\pi_{\reachablestates}$, we write 
\begin{align*} 
\hspace{0.5in}&\hspace{-0.5in} \Pr^{\MRPsign} \brk*{\bar{\tau} = (x_{\top}, \bar{x}_{1}, \cdots, \bar{x}_{k}, x, x_\bot, \cdots)} \\
&= \prod_{i=1}^k \Pr^{M, \pi}\brk*{ \tau_{h_{i}:h_{i+1}} \in \setall{\pi}{x_{h_i}}{x_{h_{i+1}}} \mid \tau[h_i] =x_{h_i} } \\
&= \Pr^{M, \pi}\brk*{ \forall i\in [k+1],~ \tau[h_i] = x_{h_i} ~\text{and}~ \forall h \in [h_{k+1}] \backslash \crl{h_1, \cdots, h_{k+1}},~ \tau[h] \notin \statesp_\pi}.
\end{align*}
Now we sum over all sequences $\bar{x}_1,\cdots, \bar{x}_k \in \SHP{\pi}$ to get \begin{align*}
&d^\MRPsign(x) \\
     &= \sum_{k=0}^{H-1} \sum_{x_{h_1}, \cdots, x_{h_k} \in \SHP{\pi}} \Pr^{M, \pi}\brk*{ \forall i\in [k+1],~ \tau[h_i] = x_{h_i} ~\text{and}~ \forall h \in [h_{k+1}] \backslash \crl{h_1, \cdots, h_{k+1}},~ \tau[h] \notin \statesp_\pi} \\
    &= \Pr^{M, \pi} \brk*{x \in \tau ~\text{and}~ \forall h \in [h_{k+1}-1],~\tau[h] \notin \SRem{\pi}} \\
    &= \Pr^\pi \brk*{\tau \in \event{x_\top}{x}{\SRem{\pi}}} =  \bard{\pi}{x} {\SRem{\pi}}. 
\end{align*}
The second equality follows from the definition of $\SRem{\pi}$, and the last line is the definition of the $\bar{d}$ notation. This concludes the proof of \pref{lem:dbar}.
\end{proof}

\begin{lemma}\label{lem:identify}
With probability at least $1 - 2\delta$,  any $(\bar{x}, \pi)$ that is added into $\cT$ (in \pref{line:add_s} of \pref{alg:main}) satisfies $d^{\pi}(\bar{x})\ge \nicefrac{\epsilon}{12D}$. 
\end{lemma}
\begin{proof} ~ 
For any $(\bar{x}, \pi)\in\cT$, when we collect $\mD_{\bar{x}}$ in \pref{alg:sample}, the probability that a trajectory will be accepted (i.e.~the trajectory satisfies the ``if'' statement in \pref{line:is1}) is exactly $d^{\pi}({\bar{x}})$. Thus, using Hoeffding's inequality (\pref{lem:hoeffding}), with probability at least $1 - \nicefrac{\delta}{D|\Pi|}$, 
\begin{equation*}
\left|\frac{|\mD_{\bar{x}}|}{n_2} - d^\pi({\bar{x}})\right|\le \sqrt{\frac{2\log(D|\Pi|/\delta)}{n_2}}.\end{equation*} 
Since $|\cT|\le D|\Pi|$, by union bound, the above holds for every $(\bar{x}, \pi) \in \cT$ with probability at least $1-\delta$. Let us denote this event as $\cE_\mathrm{data}$. Under $\cE_\mathrm{data}$, for any $(\bar{x}, \pi)$ that satisfies $d^\pi(\bar{x}) \ge \tfrac{\eps}{12D}$,
\begin{align*}
|\mD_{\bar{x}} | &\ge n_2 d^{\pi}(\bar{x}) - \sqrt{2 n_2 \log(D|\Pi|/\delta)} \ge \frac{\epsilon n_2}{12D} - \frac{\epsilon n_2}{24D} = \frac{\epsilon n_2}{24D}, \numberthis \label{eq:ds}
\end{align*} 
where the second inequality follows by the bound on \(d^\pi(\bar{x})\) and our choice of parameter \(n_2\) in Eq.~\eqref{eq:alg_parameter}. 

In the following, we prove by induction that every $(\bar{x}, \pi)$ that is added into $\cT$ in the while loop from lines \ref{line:while_loop_start}-\ref{line:while_loop_end} in \pref{alg:main} satisfies $d^{\pi}(\bar{x})\ge \tfrac{\epsilon}{12D}$. This is trivially true at initialization when $\cT = \{(x_\top, \mathrm{Null})\}$, since every trajectory starts at the dummy state $x_\top$, for which we have $d^{\mathrm{Null}}(x_\top) = 1$. 

We now proceed to the induction hypothesis. Suppose that in some iteration of the while loop, every tuple $(\bar{x}, \pi) \in \cT$ satisfies $d^{\pi}(\bar{x})\ge \nicefrac{\epsilon}{12D}$, and that \(\prn{ \bar{x}', \pi'}\) is a new tuple that will be added to \(\cT\). We will show that \(\prn{\bar{x}', \pi'}\) will also satisfy $d^{\pi'}({\bar{x}'})\ge \nicefrac{\epsilon}{12D}$. 

Recall that \(\statesp_{\pi'}^+ = \statesp_{\pi'} \cup \crl{x_\top, x_\bot}\), \(\SHP{\pi'} = \statesp_{\pi'}^+ \cap \reachablestates\), and \(\SRem{\pi'} = \statesp_{\pi'}^+ \setminus \SHP{\pi'}\). Let \(\MRPsign^{\pi'}_{\reachablestates} =  \mathrm{MRP}(\statesp_{\pi'}^+, P^{\pi'}, r^{\pi' }, H, x_\top, x_\bot)\) 
    be the policy-specific MRP, where \(P^{\pi'}\) and \(r^{\pi'}\) are defined in Eqs.~\pref{eq:policy_MRP_dynamics} and \pref{eq:policy_MRP_reward} respectively for the policy \(\pi'\). Similarly let $\wh{\MRPsign}^{\pi'}_{\reachablestates} = \mathrm{MRP}(\statesp_{\pi'}^+, \widehat P^{\pi'}, \widehat r^{\pi'}, H, x_\top, x_\bot)$ denote the estimated policy-specific MRP, where \(\wh P^{\pi'}\) and \(r^{\pi'}\) are defined using \pref{eq:empirical_MRP_dynamics} and \pref{eq:empirical_MRP_rewards} respectively. Note that for any state \(x \in \SHP{\pi'}\), the bound in \pref{eq:ds} holds. 

    For the rest of the proof, we assume that the event $\cE_\mathrm{est}$, defined in \pref{lem:dynamics_error}, holds (this happens with probability at least $1-\delta$). By definition of $\cE_\mathrm{est}$, we have
	\begin{equation}\label{eq:error1}|P_{x\to x'}^{\pi'} - \widehat{P}_{x\to x'}^{\pi'}|\le \frac{\epsilon}{12D(D+1)},  \qquad \text{for all} \qquad x'\in \statesp_{\pi'} \cup\{x_\bot\}.
 \end{equation}

Furthermore, note that $\widehat{d}^{\pi'} (\bar{x}')\leftarrow \estreach(\statesp_{\pi'}^+, \wh{\MRPsign}^{\pi'}_{\reachablestates}, \bar{x}')$ since in \pref{alg:dp} we start with \(V(x) = \ind{x = \bar{x}'}\). Furthermore, using  \pref{lem:sim}, we have 
\begin{equation}\label{eq:derror}
\begin{aligned}
    |\widehat{d}^{\pi'}(\bar x') - d^{\MRPsign}(\bar x')| & \le (D+1) \sup_{x\in\SHP{\pi'}, x'\in \statesp_{\pi'}\cup\{x_\bot\}}|\widehat{P}_{x\to x'}^{\pi'} - P_{x\to x'}^{\pi'}| \\ 
        &\le \frac{\epsilon}{12D(D+1)}\cdot (D+1) = \frac{\epsilon}{12D}.
\end{aligned}\end{equation}
where the second inequality follows from \eqref{eq:error1}. Additionally, \pref{lem:dbar} states that \(d^{\MRPsign}(\bar x') = \bard{\pi'}{\bar x'} {\SRem{\pi'}}\). Therefore we obtain
$$|\bard{\pi'}{\bar x'} {\SRem{\pi'}} - \widehat{d}^{\pi'}(\bar x')|\le \frac{\epsilon}{12D}.$$
	
Thus, if the new state-policy pair $(\bar x', \pi')$ is added into $\cT$, we will have 
$$\bard{\pi'}{\bar x'}{\SRem{\pi'}}\ge \frac{\epsilon}{6D} - \frac{\epsilon}{12D} = \frac{\epsilon}{12D}.$$
Furthermore, by definition of $\bar{d}$ we have
$$\bard{\pi'}{\bar x'}{\SRem{\pi'}} = \Pr^{\pi'}\brk*{\tau\in \event{x_\top}{\bar x'}{\SRem{\pi'}}}\le 
\Pr^{\pi'}[\bar x'\in \tau] = d^{\pi'}(\bar x'),$$
so we have proved the induction hypothesis $d^{\pi'}(\bar{x}')\ge \nicefrac{\epsilon}{12D}$ for the next round. This concludes the proof of \pref{lem:identify}.
\end{proof}

The next lemma establishes that \pref{alg:main} will terminate after finitely many rounds, and that after termination will have explored all sufficiently reachable states. 
\begin{lemma}\label{lem:while_terminate} With probability at least \(1 - 2\delta\),  
\begin{enumerate}[label=\((\alph*)\)] 
\item The while loop in \pref{line:while_loop} in \pref{alg:main} will terminate after at most $\tfrac{12HD\dimRL(\Pi)}{\epsilon}$ rounds. 
\item  After the termination of the while loop, for any $\pi\in \Pi$, the remaining states $x\in\SRem{\pi}$ that are not added to \(\reachablestates\)
satisfy $\bard{\pi}{x}{\SRem{\pi}}\le \nicefrac{\epsilon}{4D}$.
\end{enumerate}
\end{lemma}
Notice that according to our algorithm, the same state cannot be added multiple times into $\reachablestates$. Therefore, $|\reachablestates| \le D|\Pi|$, and the maximum number of rounds of the while loop is $D|\Pi|$ (i.e., the while loop eventually terminates).  

\begin{proof}
We prove each part separately.

\begin{enumerate}[label=\((\alph*)\)]
\item 
First, note that from the definition of coverability and \pref{lem:coverability}, we have  
$$\sum_{x\in\statesp}\sup_{\pi\in \Pi}d^{\pi}(x)\le HC^\mathsf{cov}(\Pi; M)\le H\dimRL(\Pi).$$
Furthermore, \pref{lem:identify} states that every $(x, \pi_x)\in \cT$ satisfies $d^{\pi_x}(x)\ge \nicefrac{\epsilon}{12D}$. Thus, at any point in \pref{alg:main}, we have
\begin{align*}
\sum_{x\in\reachablestates}\sup_{\pi\in \Pi}d^{\pi}(x) &\geq \sum_{x\in\reachablestates} d^{\pi_x}(x) \geq \abs{\cT} \cdot \frac{\epsilon}{12D}. 
\end{align*}
Since, \(\reachablestates \subseteq \statesp\), the two bounds indicate that 
$$|\cT| \le \frac{12HD\dimRL(\Pi)}{\epsilon}.$$

Since every iteration of the while loop adds one new $(x, \pi_x)$ to $\cT$, the while loop terminates after at most $\nicefrac{12HD\dimRL(\Pi)}{\epsilon}$ many rounds.

\item  We know that once the while loop has terminated, for every $\pi\in \Pi$ and $\bar{x}\in\SRem{\pi}$, we must have $\widehat{d}^\pi(\bar{x})\le \nicefrac{\epsilon}{6D}$, or else the condition in \pref{line:test_and_add} in \pref{alg:main} is violated. 

Fix any such $(\bar{x},\pi)$ pair. Inspecting the proof of \pref{lem:identify}, we see that 
\begin{align*}
|\bard{\pi}{\bar x} {\SRem{\pi}} - \widehat{d}^{\pi}(\bar x)|\le \frac{\epsilon}{12D}.
\end{align*}
To conclude,  we get 
$$\bard{\pi}{x}{\SRem{\pi}}\le \frac{\epsilon}{6D} + \frac{\epsilon}{12D} = \frac{\epsilon}{4D}.$$
\end{enumerate}
This concludes the proof of \pref{lem:while_terminate}.
\end{proof} 

\begin{lemma} \label{lem:eval_error}
Suppose that the conclusions of Lemmas \ref{lem:dynamics_error} and \ref{lem:while_terminate} hold. Then for every $\pi \in \Pi$, the estimated value $\widehat{V}^\pi$ computed in \pref{alg:main} satisfies
	$$|\widehat{V}^\pi - V^\pi|\le \epsilon.$$ 
\end{lemma}
\begin{proof}
We will break up the proof into two steps. First, we show that for any \(\pi\), the value estimate $\wh{V}^\pi$ obtained using the empirical policy-specific MRP $\wh \MRPsign^\pi_{\reachablestates}$ is close to its value in the  policy-specific MRP $\MRPsign^\pi_{\reachablestates}$, as defined via~\eqref{eq:policy_MRP_dynamics} and \eqref{eq:policy_MRP_reward}. We denote this quantity as $V_{\mathrm{MRP}}^\pi$. Then, we will show that $V_{\mathrm{MRP}}^\pi$ is close to $V^\pi$, the value of the policy \(\pi\) in the original MDP $M$.  

\paragraph{Part 1: $\widehat{V}^\pi$  is close to  $V_{\mathrm{MRP}}^\pi$.} Note that the output $\widehat{V}^\pi$ of \pref{alg:eval} is exactly the value function of MRP $\widehat \MRPsign^\pi_{\reachablestates}$ defined by Eqs.~\eqref{eq:empirical_MRP_dynamics} and \eqref{eq:empirical_MRP_rewards}. When $D = 0$, by part (b) of \pref{lem:dynamics_error}, we obtain
$$|\widehat{V}^{\pi} - V_{\mathrm{MRP}}^\pi| = |\widehat{r}^{\pi}_{x_\top\to x_\bot} - r^\pi_{x_\top\to x_\bot}|\le \frac{\epsilon}{12(D+1)^2}\le \frac{\epsilon}{2}.$$

 When $D \ge 1$, using \pref{lem:dynamics_error}, we have 
\begin{align*}
    &\forall x'\in \statesp_\pi \cup \crl{x_\bot}:\\
    &\qquad |r_{x_\top\to x'}^\pi - \widehat{r}_{x_\top\to x'}^\pi| \le \frac{\epsilon}{12(D+1)^2}, \quad
    |P_{x_\top\to x'}^\pi - \widehat{P}_{x_\top\to x'}^\pi| \le \frac{\epsilon}{12(D+1)^2},\\[0.5em]
    &\forall x\in \SHP{\pi}, x'\in \statesp_\pi^+ \cup \crl{x_\bot}:\\
    &\qquad |r_{x\to x'}^\pi - \widehat{r}_{x\to x'}^\pi| \le \frac{\epsilon}{12D(D+1)}, \quad
    |P_{x\to x'}^\pi - \widehat{P}_{x\to x'}^\pi| \le \frac{\epsilon}{12D(D+1)}. 
\end{align*}
By the simulation lemma (\pref{lem:sim}), we get
\begin{align*}
    |\widehat{V}^{\pi} - V_{\mathrm{MRP}}^\pi| & \le 2(D+2)\max_{s, x'\in\statesp_\pi^+}\left(\left|P_{x\to x'}^\pi - \widehat{P}_{x\to x'}^\pi\right| + \left|r_{x\to x'}^\pi - \widehat{r}_{x\to x'}^\pi\right|\right)\\
    & \le 2(D+2)\left(\frac{\epsilon}{12D(D+1)} + \frac{\epsilon}{12D(D+1)}\right)\le \frac{\epsilon}{2}.
\end{align*}

\paragraph{Part 2: $V_{\mathrm{MRP}}^\pi$ is close to $V^\pi$.} As in the proof of \pref{lem:dbar}, let us consider different trajectories $\bar{\tau}$ that are possible in $\MRPsign^\pi_{\reachablestates}$. We can represent $\bar{\tau} = (x_\top, \bar{x}_1,\cdots, \bar{x}_k, x_\bot, \cdots)$ where the states $\bar{x}_1,\cdots, \bar{x}_k$ are distinct and all except possibly $\bar{x}_k$ belong to $\SHP{\pi}$, and all subsequent states after $x_\bot$ are repetitions of $x_\bot$ until the end of the episode. Let $x_{h_1}, x_{h_2}, \dots, x_{h_k}$ be the same sequence (in the original MDP $M$). Again, we have 
\begin{align*}
    \hspace{1in}&\hspace{-1in}\Pr^\MRPsign[\bar{\tau} = (x_\top, \bar{x}_1,\cdots, \bar{x}_k, x_\bot, \cdots)] \\
    &= \Pr^{\pi}\brk*{ \forall i\in [k],~ \tau[h_i] = x_{h_i}, ~\text{and}~ \forall h \in [H] \backslash \crl{h_1, \cdots, h_{k}},~ \tau[h] \notin \statesp_\pi}, 
\end{align*}
where recall that \( \Pr^\MRPsign\) denotes probability under the $\MRPsign^\pi_{\reachablestates}$, and \(\Pr^{\pi}\) denotes the probability under trajectories drawn according to \(\pi\) in the underlying MDP; \( \EE^\MRPsign\)  and \(\EE^{\pi}\) are defined similarly. 

Furthermore, the expectation of rewards collected in $\MRPsign^\pi_{\reachablestates}$ with trajectories $\bar{\tau}$ is
\begin{align*}
 \hspace{1in}&\hspace{-1in} \EE^{\MRPsign}\brk*{ R[\bar{\tau}] \ind{ (x_\top, \bar{x}_1,\cdots, \bar{x}_k, x_\bot, \cdots) } } \\
= &\EE^{\pi}\brk*{ R[\tau]\ind{ \forall i\in [k],~ \tau[h_i] = x_{h_i}, ~\text{and}~ \forall h \in [H] \backslash \crl{h_1, \cdots, h_{k}},~ \tau[h] \notin \statesp_\pi }}. 
\end{align*}
Next, we sum over all possible trajectories. However, note that the only trajectories that are possible in $M$ whose corresponding trajectories are \emph{not accounted for} in $\MRPsign^\pi_{\reachablestates}$ are precisely those that visit states in $\statesp_\pi$, after visiting some $x_{h_k}$ in the remaining states $\SRem{\pi}$ (since, by construction, the MRP transitions directly to $x_\bot$ after encountering a state in $\SRem{\pi}$). Thus, 
$$V_{\mathrm{MRP}}^\pi = \EE^{\pi}\brk*{ R[\tau] \prn*{\ind{\tau \cap \SRem{\pi} = \emptyset} + \ind{\exists k\in [H]: x_{h_k} \in \SRem{\pi} \text{ and } \forall h > h_k: x_h \notin \statesp_\pi} }},$$ 
where the first term corresponds to trajectories that do not pass through \(\SRem{\pi}\), and the second term corresponds to trajectories that pass through some state in \(\SRem{\pi}\) but then does not go through any other state in \(\statesp_\pi\). On the other hand, 
\begin{align*}
V^\pi = \EE^{\pi}\brk*{ R[\tau]}. 
\end{align*}
Clearly, $V_{\mathrm{MRP}}^\pi \le V^{\pi}$. Furthermore, we also have
\begin{align*}
    V^{\pi} - V_{\mathrm{MRP}}^\pi
    &= \EE^{\pi}\brk*{R[\tau] \ind{\tau \cap \SRem{\pi} \ne \emptyset} - \ind{\exists k\in [H]: x_{h_k} \in \SRem{\pi} \text{ and } \forall h > h_k: x_h \notin \statesp_\pi}}  \\
     &\le \EE^{\pi}\brk*{R[\tau] \ind{\tau \cap \SRem{\pi} \ne \emptyset}} \\ 
    &\le D \cdot \frac{\eps}{4D} = \frac{\eps}{4},
\end{align*}
where the first inequality follows by omitting the second indicator term, and the second inequality follows by taking a  union bound over all possible values of $\SRem{\pi}$ as well as the conclusion of \pref{lem:while_terminate}. 

Putting it all together, we get that 
\begin{align*}
 |\widehat{V}^\pi - V^\pi|\le |V^{\pi} - V_{\mathrm{MRP}}^\pi| + |\widehat{V}^{\pi} - V_{\mathrm{MRP}}^\pi|\le \frac{\epsilon}{4} + \frac{\epsilon}{2} < \epsilon.  
\end{align*}
This concludes the proof of \pref{lem:eval_error}.
\end{proof}

\subsubsection{Proof of \pref{thm:sunflower}} 
We assume the events defined in Lemmas \ref{lem:dynamics_error}, \ref{lem:identify} and \ref{lem:while_terminate} hold (which happens with probability at least $1 - 2 \delta$). With our choices of $n_1, n_2$ in Eq.~\eqref{eq:alg_parameter}, the total number of samples used in our algorithm is at most
$$n_1 + n_2\cdot \frac{12HD\dimRL(\Pi)}{\epsilon} = \widetilde{\cO}\prn*{\prn*{\frac{1}{\epsilon^2} + \frac{HD^6 \dimRL(\Pi)}{\epsilon^4}} \cdot K^2 \log\frac{|\Pi|}{\delta}}.$$
After the termination of the while loop, we know that for any policy \(\pi \in \Pi\) and $x\in\SRem{\pi}$ we have 
$$\bard{\pi}{x}{\SRem{\pi}}\le \frac{\epsilon}{4D}.$$
Therefore, by \pref{lem:eval_error}, we know for every $\pi\in \Pi$, $|\widehat{V}^\pi - V^\pi|\le \epsilon$. Hence the output policy $\widehat{\pi} \in \argmax_\pi \hV^\pi$ satisfies
$$\max_{\pi\in\Pi} V^\pi - V^{\widehat{\pi}}\le 2\epsilon + \hV^\pi - \hV^{\widehat{\pi}}\le 2\epsilon.$$
Rescaling $\epsilon$ by $2\epsilon$ and $\delta$ by $2\delta$ concludes the proof of \pref{thm:sunflower}.\qed
\subsubsection{Extending \pref{thm:sunflower} to Infinite Policy Classes}\label{sec:online-infinite-policy} 
The modified analysis for the online RL upper bound (\pref{thm:sunflower}) proceeds similarly as in \pref{sec:generative-lower-bound-extend}; we sketch the ideas below.

There are two places in the proof of \pref{thm:sunflower} which require a union bound over $\abs{\Pi}$: the event $\cE_\mathrm{est}$ (defined by \pref{lem:dynamics_error}) that the estimated transitions and rewards of the MRPs are close to their population versions, and the event $\cE_\mathrm{data}$ (defined by \pref{lem:identify}) that the datasets collected are large enough. The latter is easy to address, since we can simply modify the algorithm's while loop to break after $\cO\prn*{\tfrac{HD\dimRL(\Pi)}{\epsilon}}$ iterations and union bound over the size of the set $\abs{\cT}$ instead of the worst-case bound on the size $D \abs{\Pi}$. For $\cE_\mathrm{data}$, we follow a similar strategy as the analysis for the generative model upper bound.

Fix a state $x$. Recall that the estimate for the probability transition kernel in the MDP in Eq.~\eqref{eq:empirical_MRP_dynamics} takes the form
\begin{align*}
	\widehat{P}_{x\to x'}^\pi = \frac{1}{|\mD_x|} \sum_{\tau\in \mD_x}\frac{\ind{\pi\cons\tau_{h:h'}}}{\tfrac{1}{|\Picore|}\sum_{\pi'\in \Picore}\ind{\pi_e\cons\tau_{h:h'}}}\ind{\tau\in \setall{\pi}{x}{x'} }.
\end{align*}
(The analysis for the rewards is similar, so we omit it from this proof sketch.)

We set up some notation. Define the function $p^{\pi}_{x\to x'}: (\statesp\times \cA \times \bbR)^H \to [0, \abs{\Picore}]$ as
\begin{align}\label{eq:is-function-def}
    p^\pi_{x\to x'}(\tau) \coloneqq \frac{\ind{\pi\cons\tau_{h:h'}}}{\tfrac{1}{|\Picore|}\sum_{\pi'\in \Picore}\ind{\pi_e\cons\tau_{h:h'}}}\ind{\tau\in \setall{\pi}{x}{x'} },
\end{align}
with the implicit restriction of the domain to trajectories $\tau$ for which the denominator is nonzero. We have $\En[p^\pi_{x\to x'}(\tau)] = P^\pi_{x\rightarrow x'}$. Also let $\Pi_x = \crl*{\pi \in \Pi: x\in \statesp_\pi}$. 

Restated in this notation, our objective is to show the uniform convergence guarantee
\begin{align}\label{eq:uniform-convergence-probabilities}
    \text{w.p.~at least}~ 1-\delta, \quad \sup_{\pi \in \Pi_x, x' \in \statesp_\pi} \Big|\frac{1}{|\mD_x|} \sum_{\tau\in \mD_x} p^\pi_{x\to x'}(\tau) - \En[p^\pi_{x\to x'}(\tau)]\Big| \le \eps.
\end{align}
Again, in light of \pref{lem:uniform-convergence}, we need to compute the pseudodimension for the function class $\cP^{\Pi_x} = \crl*{p^\pi_{x\to x'}: \pi \in \Pi_x,  x' \in \statesp_\pi}$, since these are all possible transitions that we might use the dataset $\cD_s$ to evaluate. Define the subgraph class
\begin{align*}
    \cP^{\Pi_x, +} \coloneqq \crl{(\tau, \theta) \mapsto \ind{p^\pi_{x\to x'}(\tau) \le \theta}: \pi \in \Pi_x, x' \in \statesp_\pi }.
\end{align*}
Fix the set $Z = \crl*{(\tau_1, \theta_1), \cdots, (\tau_d, \theta_d)} \in ((\statesp \times \cA \times \bbR)^H \times \bbR)^d$, where the trajectories $\tau_1, \cdots, \tau_d$ pass through $x$. We also denote $\statesp_Z$ to be the union of all states which appear in $\tau_1, \cdots, \tau_d$. In order to show a bound that $\pseudo(\cP^{\Pi_x}) \le d$ it suffices to prove that $\abs{\cP^{\Pi_x, +}\big|_Z} < 2^d$.

We first observe that
\begin{align*}
    \abs{\cP^{\Pi_x, +}\big|_Z} 
    &\le 1 + \sum_{x' \in \statesp_Z} \abs*{ \crl*{  \prn*{ \ind{p^\pi_{x \to x'}(\tau_1) \le \theta_1 }, \cdots, \ind{p^\pi_{x \to x'}(\tau_d) \le \theta_d }  }  : \pi \in \Pi_x } }.
\end{align*}
The inequality follows because for any choice $x' \notin \statesp_Z$, we have
\begin{align*}
    \prn*{ \ind{p^\pi_{x \to x'}(\tau_1) \le \theta_1 }, \cdots, \ind{p^\pi_{x \to x'}(\tau_d) \le \theta_d }  } = \vec{0},
\end{align*}
no matter what $\pi$ is, contributing at most 1 to the count. Furthermore, once we have fixed $x'$ and the $\crl{\tau_1, \cdots, \tau_d}$ the quantities $\tfrac{1}{|\Picore|}\sum_{\pi'\in \Picore}\ind{\pi_e\cons\tau_{i, h:h'}}$ for every $i \in [d]$ are constant (do not depend on $\pi$), so we can reparameterize $\theta_i' \coloneqq \theta_i \cdot \tfrac{1}{|\Picore|}\sum_{\pi'\in \Picore}\ind{\pi_e\cons\tau_{i, h:h'}} $ to get:
\begin{align*}
\abs{\cP^{\Pi_x, +}\big|_Z}  &\le 1 + \sum_{x' \in \statesp_Z} \abs*{ \crl*{ (  b_1(\pi) , \cdots,
 b_d(\pi) )  : \pi \in \Pi } }, \numberthis\label{eq:projection-bound}\\
&\quad \text{where} \quad b_i(\pi) \coloneqq \ind{\ind{\pi\cons\tau_{i, h:h'}} \ind{\tau_i \in \setall{\pi}{x}{x'} } \le \theta_i'}.
\end{align*}
Now we count how many values the vector $(b_1(\pi), \cdots, b_d(\pi))$ can take for different $\pi \in \Pi_x$. Without loss of generality, we can (1) assume that the $\theta_i'=0$ (since a product of indicators can only take values in $\crl{0,1}$, and if $\theta' \ge 1$ then we must have $b_i(\pi) = 1$ for every $\pi$), and (2) $x' \in \tau_i$ for each $i \in [d]$ (otherwise $b_i(\pi) = 0$ for every $\pi \in \Pi$). So we can rewrite $b_i(\pi) = \ind{\pi\cons\tau_{i, h:h'}} \ind{\tau_i \in \setall{\pi}{x}{x'} }$. For every fixed choice of $x'$ we upper bound the size of the set as:
\begin{align*}
    \hspace{2em}&\hspace{-2em} \abs*{ \crl*{ \prn*{  b_1(\pi) , \cdots,
 b_d(\pi) }  : \pi \in \Pi } } \\ \overleq{\proman{1}}  & \abs*{ \crl*{ \prn*{ \ind{\pi\cons\tau_{1, h:h'}}, \cdots, \ind{\pi\cons\tau_{d, h:h'}} } : \pi \in \Pi }} \\
 &\quad\quad \times \abs*{ \crl*{ \prn*{ \ind{\tau_1 \in \setall{\pi}{x}{x'} }, \cdots, \ind{\tau_d \in \setall{\pi}{x}{x'} } } : \pi \in \Pi }} \\
 \overleq{\proman{2}} & \abs*{ \crl*{ \prn*{\pi(x^{(1)}), \pi(x^{(2)}), \cdots, \pi(x^{(dH)}) } : \pi \in \Pi } } \\
                      &\quad \times \abs*{ \crl*{  \prn*{\ind{x^{(1)} \in \statesp_\pi}, \cdots, \ind{x^{(dH)} \in \statesp_\pi} } : \pi \in \Pi} } \\
 \overleq{\proman{3}} & \prn*{\frac{dH \cdot e (A+1)^2}{2\natarajan(\Pi)}}^{\natarajan(\Pi)} \times \prn{dH}^D. \numberthis\label{eq:b-bound}
\end{align*}
The inequality $\proman{1}$ follows by upper bounding by the Cartesian product. The inequality $\proman{2}$ follows because (1) for the first term, the vector $\prn*{ \ind{\pi\cons\tau_{1, h:h'}}, \cdots, \ind{\pi\cons\tau_{d, h:h'}} }$ is determined by the number of possible behaviors $\pi$ has over all $dH$ states in the trajectories, and (2) for the second term, the vector $\prn*{ \ind{\tau_1 \in \setall{\pi}{x}{x'} }, \cdots, \ind{\tau_d \in \setall{\pi}{x}{x'} } }$ is determined by which of the $dH$ states lie in the petal set for $\statesp_\pi$. The inequality $\proman{3}$ follows by applying \pref{lem:sauer-natarajan} to the first term and Sauer's Lemma to the second term, further noting that every petal $\statesp_\pi$ set has cardinality at most $D$.

Combining Eqs.~\eqref{eq:projection-bound} and \eqref{eq:b-bound} we get the final bound that
\begin{align*}
\abs{\cP^{\Pi_x, +}\big|_Z}  &\le 1 + (dH)^{D+1} \cdot  \prn*{\frac{dH \cdot e (A+1)^2}{2\natarajan(\Pi)}}^{\natarajan(\Pi)}.
\end{align*}
To conclude the calculation, we observe that this bound is $< 2^d$ whenever $d = \wt{\cO}(D + \natarajan(\Pi))$, which we can again use in conjunction with \pref{lem:uniform-convergence} to prove the desired uniform convergence statement found in Eq.~\eqref{eq:uniform-convergence-probabilities}. Ultimately this allows us to replace the $\log \abs{\Pi}$ with $\wt{\cO}(D + \natarajan(\Pi))$ in the upper bound of \pref{thm:sunflower}; the precise details are omitted.

\section{Bibliographical Remarks}

\paragraph{Agnostic RL in Low-Rank MDPs.} \cite{sekhari2021agnostic} explored agnostic PAC RL in low-rank MDPs, and showed that one can perform agnostic learning with respect to any policy class for MDPs that have a small rank. Their results are complementary to the ones established in this chapter. They study how to solve \eqref{eq:pac-rl} without access to a realizable model/value function class. The main difference is that they investigate the low-rank assumption on the underlying MDP dynamics and show how it enables sample-efficient learning for any given policy class, while our results establish what assumptions on the given policy class are necessary and sufficient for agnostic learning \emph{for any MDP}.
\paragraph{Reward-Free RL.} From a technical viewpoint, our algorithm (\pref{alg:main}) share similaries to algorithms developed in the reward-free RL literature \citep{jin2020reward}. In reward-free RL, the goal of the learner is to output a dataset or a set of policies, after interacting with the underlying MDP, that can be later used for planning (with no further interaction with the MDP) for downstream reward functions. The key ideas in our \pref{alg:main}, in particular, that the learner first finds states \(\cI\) that are \(\Omega(\epsilon)\)-reachable and corresponding policies that can reach them, and then outputs datasets \(\crl{\cD_s}_{s \in \cI}\) that can be later used for evaluating any policy \(\pi \in \Pi\), share similarities to algorithmic ideas used in reward-free RL. However, our algorithm strictly generalizes prior works in reward-free RL, and in particular can work with large state-action spaces where the notion of reachability as well as the offline RL objective, is defined w.r.t.~the given policy class. In comparison, prior reward-free RL works compete with the best policy for the underlying MDP, and make structure assumptions on the dynamics, e.g. tabular structure \citep{jin2020reward, menard2021fast, li2023minimax} or linear dynamics \citep{wang2020reward, zanette2020provably, zhang2021reward, wagenmaker2022reward}, to make the problem tractable.

\chapter{Imitation Learning}\label{chap:imitation-learning}

This chapter studies \emph{imitation learning} (IL), a setting which provides the learner with stronger feedback than considered earlier in this thesis. We introduce an interactive imitation learning setup in \pref{sec:il-prelims}, where the learner has access to an ``expert oracle'' that returns value functions of an expert on queried state-action pairs. Our main result, presented in \pref{sec:expert-realizability}, is that the \emph{realizability} of the expert policy plays a crucial role in determining whether the learner can effectively utilize this expert oracle. We show that if the expert policy does not lie in the given policy class $\Pi$, then the additional access of the expert oracle cannot be leveraged by the learner to get improved statistical guarantees.

\section{Preliminaries}\label{sec:il-prelims}

\paragraph{Motivation.} For many policy classes, the spanning capacity may be quite large, and our lower bounds (\pref{thm:generative_lower_bound}, \ref{thm:lower-bound-coverability}, and \ref{thm:lower-bound-online}) demonstrate an unavoidable dependence on the spanning capacity $\dimRL(\Pi)$. Now we investigate whether it is possible to achieve bounds which are independent of $\dimRL(\Pi)$ and instead only depend on $\poly(H, A, \log \abs{\Pi})$ under a stronger feedback model.

The motivation for our feedback model comes from practice. It is uncommon to learn from scratch: often we would like to utilize domain expertise or prior knowledge to learn with fewer samples. For example, during training one might have access to a simulator which can roll out trajectories to estimate the optimal value function $Q^\star$, or one might have access to expert advice/demonstrations. However, this access does not come for free; estimating value functions with a simulator may require significant computation, or the expert might be a human who is providing labels or feedback on the performance of the algorithm. Thus, we consider additional feedback in the form of an expert oracle.

\begin{definition}[Expert Oracle]\label{def:expert-advice-oracle}
An expert oracle $\oracle: \statesp \times \cA \to \bbR$ is a function which given an $(x,a)$ pair as input returns the $Q$ value of some expert policy $\pi_\circ$, denoted $Q^{\pi_\circ}(x,a)$.
\end{definition}

\pref{def:expert-advice-oracle} is a natural formulation for understanding how expert feedback can be used for RL in large state spaces. We do not require $\pi_\circ$ to be the optimal policy (or even the best within the policy class $\Pi$). The objective is to compete with $\pi_\circ$, i.e., with probability at least $1-\delta$, return a policy $\wh{\pi}$ such that $V^{\wh{\pi}} \ge V^{\pi_\circ} - \eps$ using few online interactions with the MDP and calls to $\oracle$. A discussion of how \pref{def:expert-advice-oracle} relates to other imitation learning settings is deferred to the Bibliographical Remarks in \pref{sec:il-bib}.

Any sample efficient algorithm (one which uses at most $\poly(H, A, \log \abs{\Pi}, \eps^{-1}, \delta^{-1})$ online trajectories and calls to the oracle) must use \emph{both} forms of access. The aforementioned lower bounds show that an algorithm which only uses online access to the MDP must use $\Omega(\dimRL(\Pi))$ samples. Likewise, an algorithm which only queries the expert oracle must use $\Omega(\abs{\statesp}A)$ queries because it does not know the dynamics of the MDP, so the best it can do is just learn the optimal action on every state.

\paragraph{Upper Bound Under Realizability.}

Under realizability (namely, $\pi_\circ \in \Pi$), it is known that the dependence on $\dimRL(\Pi)$ can be entirely removed with few queries to the expert oracle.

\begin{algorithm}[ht]
\caption{\textsf{AggreVaTe} \cite{ross2014reinforcement}}\label{alg:aggrevate}
	\begin{algorithmic}[1]
		\Require Oracle Access $\oracle: \statesp \times \cA \to \bbR$, Policy Class $\Pi$, Online RL Access to $M^\star$
        \State Let $\rho_0 = \unif(\Pi)$. Let $\eta \coloneqq \sqrt{8 \log \abs{\Pi}/T}$.
        \State \textbf{for} $t=1,2,\cdots, T$:
			\State \(\quad\) Let $h_t \sim \unif([H])$.
            \State \(\quad\) Sample a state $x_t \sim d^{\rho_t}_{h_t}$.
            \State \(\quad\) Query oracle to get $Q^\star_t(x_t, a) \gets \oracle(x_t, a)$ for all $a \in \cA$.
            \State \(\quad\) Update policy distribution $\rho_{t+1}(\pi) \propto \exp\prn*{\eta \cdot \sum_{t=1}^T Q^\star(x_t, \pi(x_t)) }$
		\State \textbf{Return} $\wh{\rho} = \unif(\crl{ \rho_t}_{t=1}^T)$.
	\end{algorithmic}
\end{algorithm}

\begin{theorem} \label{thm:aggrevate}
    For any $\Pi$ such that $\pi_\circ \in \Pi$, with probability at least $1-\delta$, the $\mathsf{AggreVaTe}$ algorithm (\pref{alg:aggrevate}) 
computes an $\eps$-optimal policy using
\begin{align*}
    n_1 = \cO\prn*{\frac{H^2}{\eps^2} \cdot \log \frac{\abs{\Pi}}{\delta}} ~\text{online trajectories} \quad \text{and} \quad n_2 = \cO\prn*{\frac{H^2A}{\eps^2} \cdot \log \frac{\abs{\Pi}}{\delta}} ~\text{calls to } \oracle.
\end{align*}
\end{theorem}

$\mathsf{AggreVaTe}$ is a reduction to online learning with experts. 
We note that we actually require a slightly weaker, ``local'' oracle: $\mathsf{AggreVaTe}$ only queries the value of $Q^{\pi_\circ}$ on $(x,a)$ pairs which are encountered in online trajectories. Furthermore, it is also possible adapt the algorithm to work with only expert action feedback, i.e., queries of the form $\oracle(x) = \pi_\circ(x)$, by changing the loss to the indicator function $\ind{\pi(x_t) \ne \pi_\circ(x_t)}$. 

For completeness, we prove \pref{thm:aggrevate}; this result can also be found in \cite{ross2014reinforcement, agarwal2019reinforcement}.

\begin{proof}[Proof of \pref{thm:aggrevate}]
Let $\Pi = \crl{\pi_1, \cdots, \pi_N}$. For any $\pi \in \Pi$, let us define the loss function
\begin{align*}
    \bm{\ell}^t_\pi \coloneqq - Q^\star(x_t, \pi(x_t)), \quad \text{ and }\quad \bm{\ell}^t \coloneqq \prn*{\bm{\ell}^t_{\pi_1}, \bm{\ell}^t_{\pi_2}, \cdots, \bm{\ell}^t_{\pi_N}}.
\end{align*}
Note that for any $\pi$ we have $\bm{\ell}^t_\pi \in [-1,0]$. Then the policy distribution update takes the form $\rho_{t+1}(\pi) \propto \exp \prn*{-\eta \sum_{t=1}^T \bm{\ell}^t_\pi}$. We have written the \textsf{AggreVaTe} update rule in the form of the exponential weights algorithm. Now we apply the standard regret bound \cite[see, e.g.,][]{cesa2006prediction}: 
\begin{align*}
    \textsf{Regret} = \sum_{t=1}^T \tri{\rho_t, \bm{\ell}^t} - \min_{i \in [N]} \sum_{t=1}^T \tri{e_i, \bm{\ell}^t} \le \frac{\eta T}{8} + \frac{\log \abs{\Pi}}{\eta} = \sqrt{\frac{T \log \abs{\Pi}}{2}}, \numberthis\label{eq:online-guarantee}
\end{align*}
where the last equality uses the definition of $\eta$.

Now we use \eqref{eq:online-guarantee} to get a guarantee on the final policy $\wh{\rho}$. Let us define
\begin{align*}
    \wt{\bm{\ell}}^t_\pi \coloneqq \En_{h_t \sim \unif([H]), x_t \sim d^{\rho_t}_{h_t}} \brk*{-Q^\star(x_t, \pi(x_t))}, \quad \text{and}\quad  \wt{\bm{\ell}}^t \coloneqq \prn*{\wt{\bm{\ell}}^t_{\pi_1}, \wt{\bm{\ell}}^t_{\pi_2}, \cdots, \wt{\bm{\ell}}^t_{\pi_N}}.
\end{align*}
This can be viewed as the ``expected'' loss we should have incurred in a given round $t \in [T]$. Formally let $\En_{t}[\cdot] \coloneqq \En[\cdot ~|~ \cH_t]$ denote the conditional expectation, which is conditioned on all history up and including the end of iteration $t$. We have $\En_{t-1}[\tri{\rho_t, \bm{\ell}^t}] = \tri{\rho_t,\wt{\bm{\ell}}^t}$ since $\rho_t$ only depends on the history up until time $t-1$. 

By Azuma-Hoeffding (\pref{lem:azuma-hoeffding}) and union bound we get:
\begin{align*}
    \abs*{\frac{1}{T}\sum_{t=1}^T \tri{\rho_t, \bm{\ell}^t} - \frac{1}{T}\sum_{t=1}^T \tri{\rho_t,\wt{\bm{\ell}}^t}} &\le 2 \sqrt{\frac{ \log(1/\delta) }{T} }, \\
    \text{for all}~i \in [N], \quad \abs*{\frac{1}{T}\sum_{t=1}^T \tri{e_i, \bm{\ell}^t} - \frac{1}{T}\sum_{t=1}^T \tri{e_i,\wt{\bm{\ell}}^t}} &\le 2 \sqrt{\frac{ \log(\abs{\Pi}/\delta) }{T} }.
\end{align*}
We use this in conjunction with \eqref{eq:online-guarantee} to get that with probability at least $1-\delta$:
\begin{align*}
    \frac{1}{T}\sum_{t=1}^T \tri{\rho_t,\wt{\bm{\ell}}^t} - \min_{i \in [N]} \sum_{t=1}^T \tri{e_i, \wt{\bm{\ell}}^t} \le \sqrt{\frac{ \log \abs{\Pi}}{2T}} + 4 \sqrt{\frac{ \log(\abs{\Pi}/\delta) }{T} }. \numberthis\label{eq:expected-online-guarantee}
\end{align*}
By Performance Difference Lemma, we know that for any $\pi$,
\begin{align*}
    V^\star - V^{\rho_t} = \sum_{h=1}^H \En_{x_h, a_h \sim d^{\rho_t}_h} \brk*{-A^\star(x_h, a_h)}  = H \cdot \tri{\rho_t, \wt{\bm{\ell}}^t} + \sum_{h=1}^H \En_{x_h, a_h \sim d^{\rho_t}_h} \brk*{V^\star(x_h)}.
\end{align*}
So therefore we have:
\begin{align*}
    V^\star - \frac{1}{T}\sum_{t=1}^T V^{\rho_t} &= \frac{1}{T}\sum_{t=1}^T H\cdot\tri{\rho_t, \wt{\bm{\ell}}^t} + \frac{1}{T}\sum_{t=1}^T \sum_{h=1}^H \En_{x_h, a_h \sim d^{\rho_t}_h} \brk*{Q^\star(x_h, \pi^\star(x_h))}\\
    &= \frac{1}{T}\sum_{t=1}^T H\cdot\tri{\rho_t, \wt{\bm{\ell}}^t} + \max_{i\in [N]} \frac{1}{T}\sum_{t=1}^T \sum_{h=1}^H \En_{x_h, a_h \sim d^{\rho_t}_h} \brk*{Q^\star(x_h, \pi_i(x_h))} \\
    &= \frac{1}{T}\sum_{t=1}^T H\cdot\tri{\rho_t, \wt{\bm{\ell}}^t} -  \min_{i \in [N]} \frac{1}{T} \sum_{t=1}^T H\cdot \tri{e_i, \wt{\bm{\ell}}^t} \lesssim H\sqrt{\frac{\log (\abs{\Pi}/\delta)}{T}}.\numberthis\label{eq:final-aggrevate-bound}
\end{align*}
In the second equality we use realizability, and the inequality uses \eqref{eq:expected-online-guarantee}.

To conclude the proof of \pref{thm:aggrevate}, we set the RHS of \eqref{eq:final-aggrevate-bound} to $\eps$, and observe that in every round $t \in [T]$, $\mathsf{AggreVaTe}$ collects one online episode and calls the oracle $A$ times.
\end{proof}

\section{Lower Bound in Agnostic Setting}\label{sec:expert-realizability}

Realizability of the expert policy used for $\oracle$ is a strong assumption in practice. For example, one might choose to parameterize $\Pi$ as a class of neural networks, but one would like to use human annotators to give expert feedback on the actions taken by the learner; here, it is unreasonable to assume that realizability of the expert policy holds.

We sketch a lower bound in \pref{thm:lower-bound-expert} that shows that without realizability ($\pi_\circ \notin \Pi$), we can do no better than $\Omega(\dimRL(\Pi))$ queries to a generative model or queries to $\oracle$.

\begin{theorem}[informal]\label{thm:lower-bound-expert}
For any $H \in \bbN$, $C \in [2^H]$, there exists a policy class $\Pi$ with $\dimRL(\Pi) = \abs{\Pi} = C$, expert policy $\pi_\circ \notin \Pi$, and family of MDPs $\cM$ with state space $\statesp$ of size $O(2^H)$, binary action space, and horizon $H$ such that any algorithm that returns a $1/4$-optimal policy must either use $\Omega(C)$ queries to a generative model or $\Omega(C)$ queries to the $\oracle$.
\end{theorem}

Before sketching the proof, several remarks are in order.
\begin{itemize}[label=$\bullet$]
    \item By comparing with \pref{thm:aggrevate}, \pref{thm:lower-bound-expert} demonstrates that realizability of the expert policy is crucial for circumventing the dependence on spanning capacity via the expert oracle.
    \item In the lower bound construction of \pref{thm:lower-bound-expert}, $\pi_\circ$ is an optimal policy. Furthermore, the optimal policy is not unique.
    \item While $\pi_\circ \notin \Pi$, the lower bound still has the property that $V^{\pi_\circ} = V^\star = \max_{\pi \in \Pi} V^\pi$; that is, the best-in-class policy $\wt{\pi} \coloneqq \argmax_{\pi \in \Pi} V^{\pi}$ attains the same value as the optimal policy. This is possible because there exist multiple states for which $\wt{\pi}(x) \ne \pi_\circ(x)$, however these states have $d^{\wt{\pi}}(x) = 0$. Thus, we rule out guarantees of the form
        \begin{align*}
            V^{\pi_\circ} - V^{\wh{\pi}} \le ~\underbrace{ \min_{\pi \in \Pi}~ \crl*{ V^{\pi_\circ} - V^\pi } }_{= 0 \text{ in construction}}~  + \eps, \numberthis\label{eq:il-misspecification-suboptimality}
        \end{align*}
        where the misspecification of the policy class $\Pi$ is stated in terms of the suboptimality difference. We mention several related works in \pref{sec:il-bib} which give imitation learning guarantees similar to \eqref{eq:il-misspecification-suboptimality}, but with the misspecification instead measured in terms of some information-theoretic quantity.
\end{itemize}

\begin{proof}[Proof Sketch of \pref{thm:lower-bound-expert}]
    We present the construction as well as intuition for the lower bound, leaving out a formal information-theoretic proof. A formal proof can be obtained using similar arguments as \pref{thm:lower-bound-coverability}.

\paragraph{Construction of MDP Family.} We describe the family of MDPs $\cM$. In every layer, the state space is $\statesp_h = \crl{x_{(j,h)} : j \in [2^h]}$, except at $\statesp_H$ where we have an additional terminating state, $\statesp_H = \crl{x_{(j,h)} : j \in [2^H]} \cup \crl{x_\bot}$. The action space is $\cA = \crl{0,1}$.

The MDP family $\cM = \crl{M_{b, f^\star}}_{b \in \cA^{H-1}, f^\star \in \cA^{\statesp_H}}$ is parameterized by a bit sequence $b \in \cA^{H-1}$ as well as a labeling function $f^\star \in \cA^{\statesp_H}$. The size of $\cM$ is $2^{H-1} \cdot 2^{2^H}$. We now describe the transitions and rewards for any $M_{b, f^\star}$. In the following, let $x_b \in \statesp_{H-1}$ be the state that is reached by playing the sequence of actions $(b[1], b[2], \cdots, b[H-2])$ for the first $H-2$ layers.

\begin{itemize}[label=$\bullet$]
    \item \textbf{Transitions.} For the first $H-2$ layers, the transitions are the same for $M_{b, f^\star} \in \cM$. At layer $H-1$, the transition depends on $b$.
    \begin{itemize}
        \item For any $h \in \crl{1, 2, \dots, H-2}$, the transitions are deterministic and given by a tree process: namely
    \begin{align*}
      P(x' \mid x_{(j,h)}, a) = \begin{cases}
          \ind{x' = x_{(2j -1, h+1)}} &\text{if}~a =0,\\
          \ind{x' = x_{(2j, h+1)}} &\text{if}~a =1.
      \end{cases}
    \end{align*}
        \item At layer $H-1$, for the state $x_b$, the transition is $P(x' \mid x_b, a) = \ind{x' = x_\bot}$ for any $a \in \cA$. For all other states, the transitions are uniform to $\statesp_H$, i.e., for any $x \in \statesp_{H-1} \backslash \crl{x_b}$, $a \in \cA$, the transition is $P(\cdot \mid x, a) = \unif(\statesp_H \backslash \crl{x_\bot})$.
    \end{itemize}
    \item \textbf{Rewards.} The rewards depend on the $b \in \cA^{H-1}$ and $f^\star \in \cA^{\statesp_H}$.
    \begin{itemize}
        \item The reward at layer $H-1$ is $R(x,a) = \ind{x = x_b, a = b[H-1]}$.
        \item The reward at layer $H$ is
        \begin{align*}
            &R(x_\bot, a) = 0 &\text{for any $a \in \cA$,}\\
            &R(x,a) = \ind{a = f^\star(x)} &\text{for any $x \ne x_\bot$, $a \in \cA$.}
        \end{align*}
    \end{itemize}
\end{itemize}

From the description of the transitions and rewards, we can compute the value of $Q^\star(\cdot,\cdot)$.
\begin{itemize}[label=$\bullet$]
    \item {Layers $1, \cdots, H-2$:} For any $x \in \statesp_1 \cup \statesp_2 \cup \cdots \cup \statesp_{H-2}$ and $a \in \cA$, the $Q$-value is $Q^\star(x,a) = 1$.
    \item {Layer $H-1$:} At $x_b$, the $Q$-value is $Q^\star(x_b,a) = \ind{a = b[H-1]}$. For other states $x \in \statesp_{H-1} \backslash \crl{x_b}$, the $Q$-value is $Q^\star(x,a) = 1$ for any $a \in \cA$.
    \item {Layer $H$:} At $x_\bot$, the $Q$-value is $Q^\star(x_\bot,a) = 0$ for any $a \in \cA$. For other states $x \in \statesp_{H} \backslash \crl{x_\bot}$, the $Q$-value is $Q^\star(x,a) = \ind{a = f^\star(x)}$.
\end{itemize}

Lastly, the optimal value is $V^\star = 1$.

\paragraph{Expert Oracle.}
The oracle $\oracle$ returns the value of $Q^\star(x,a)$.

\paragraph{Policy Class.} The policy class $\Pi$ is parameterized by bit sequences of length $H-1$. Denote the function $\mathrm{bin}: \crl{0, 1, \dots, 2^{H-1}} \mapsto \cA^{H-1}$ that returns the binary representation of the input. Specifically,
\begin{align*}
    \Pi \coloneqq \crl*{\pi_b: b \in \crl*{\mathrm{bin}(i) : i \in \crl*{0, 1, \dots, C-1}}},
\end{align*}
where each $\pi_b$ is defined such that $\pi_b(x) \coloneqq b[h]$ if $x \in \statesp_h$, and $\pi_b(x) \coloneqq 0$ otherwise. By construction it is clear that $\dimRL(\Pi) = \abs{\Pi} = C$.

\paragraph{Lower Bound Argument.}
Consider any $M_{b, f^\star}$ where $b \in \crl*{\mathrm{bin}(i) : i \in \crl*{0, 1, \dots, C-1}}$ and $f^\star \in \cA^{\statesp_H}$. There are two ways for the learner to identify a $1/4$-optimal policy in $M_{b, f^\star}$:
\begin{itemize}
    \item Find the value of $b$, and return the policy $\pi_b$, which has $V^{\pi_b} = 1$.
    \item Estimate $\wh{f} \approx f^\star$, and return the policy $\pi_{\wh{f}}$ which picks arbitrary actions for any $x \in \statesp_1 \cup \statesp_2 \cup \dots \cup \statesp_{H-1}$ and picks $\pi_{\wh{f}}(x) = \wh{f}(x)$ on $x \in \statesp_H$.
\end{itemize}

We claim that in any case, the learner must either use many samples from a generative model or many calls to $\oracle$. First, observe that since the transitions and rewards at layers $1, \cdots, H-2$ are known and identical for all $M_{b, f^\star} \in \cM$, querying the generative model on these states does not provide the learner with any information. Furthermore, in layers $1, \cdots, H-2$, every $(x,a)$ pair has $Q^\star(x,a) = 1$, so querying $\oracle$ on these $(x,a)$ pairs also does not provide any information to the learner. Thus, we consider learners which query the generative model or the expert oracle at states in layers $H-1$ and $H$.

In order to identify $b$, the learner must identify which $(x,a)$ pair at layer $H-1$ achieves reward of 1. They can do this either by (1) querying the generative model at a particular $(x,a)$ pair and observing if $r(x,a) = 1$ (or if the transition goes to $x_\bot$); or (2) querying $\oracle$ at a particular $(x,a)$ pair and observing if $Q^\star(x,a) = 0$ (which informs the learner that $x_b = x$ and $b[H-1] = 1-a$). In either case, the learner must expend $\Omega(C)$ queries in total in order to identify $b$.

To learn $f^\star$, the learner must solve a supervised learning problem over $\statesp_H \backslash x_\bot$. They can learn the identity of $f^\star(x)$ by querying either the generative model or the expert oracle on $\statesp_H$. Due to classical supervised learning lower bounds \cite{shalev2014understanding}, learning $f^\star$ requires
\begin{align*}
    \Omega \prn*{ \mathrm{VC}(\cA^{\statesp_{H}}) } = \Omega(2^H) \quad \text{queries.}
\end{align*}
This concludes the proof sketch of \pref{thm:lower-bound-expert}.
\end{proof}

\section{Bibliographical Remarks}\label{sec:il-bib}

The literature on imitation learning is vast; we state several of the most relevant works.

\paragraph{IL in Practice.} Imitation learning has been extensively used in autonomous driving \cite{pomerleau1988alvinn, abbeel2004apprenticeship, bojarski2016end, bansal2018chauffeurnet}, robotic control \cite{finn2017one, stepputtis2020language}, and game playing \cite{ibarz2018reward, vinyals2019grandmaster}. The problem of autoregressive next-token prediction for large language models (LLMs) can be viewed as an imitation learning task \cite{foster2024behavior}.

\paragraph{Offline IL.} Offline approaches to IL require a dataset of expert demonstrations, i.e., trajectories of the form $\tau = (x_1, \pi_\circ(x_1), \cdots, x_H, \pi_\circ(x_H))$. Offline IL is typically solved via the behavior cloning algorithm \cite{ross2010efficient}, which returns the policy which minimizes the 0/1 disagreement with the expert on the dataset via reduction to supervised learning. Other works study behavior cloning with losses other than the indicator loss. \cite{foster2024behavior} study behavior cloning with the log loss and investigate the dependence on horizon for deterministic/stochastic policy classes under the realizability assumption $\pi_\circ \in \Pi$. \cite{rohatgi2025computational} propose a $\rho$-estimator for agnostic offline IL, and show a sample complexity bound where the misspecification is measured in terms of the Hellinger distance $D_\mathsf{H}(\cdot, \cdot)$. Specifically they achieve a guarantee of the form (cf.~\eqref{eq:il-misspecification-suboptimality}):
\begin{align*}
    V^{\pi_\circ} - V^{\wh{\pi}} \lesssim \min_{\pi \in \Pi} D_\mathsf{H}\prn*{\bbP^{\pi}, \bbP^{\pi_\circ} } + \sqrt{\frac{\log \abs{\Pi}/\delta}{n}}, 
\end{align*}
where $\bbP^{\pi}$ denotes the law over trajectories given by running $\pi$ in the MDP.

\paragraph{Interactive IL.}
Interactive settings for IL have been studied in a long line of work, dating back to the seminal results \cite{ross2010efficient,ross2011reduction, ross2014reinforcement}; some additional works include \cite{sun2017deeply, cheng2018convergence, cheng2020policy, yan2021explaining, li2022efficient}. Several works study \emph{active} imitation learning algorithms which are similar to our setting (which measures the number of calls to the oracle $\oracle$). \cite{amortila2022few} study a related expert action oracle under the assumption of linear value functions. They show that in a generative model, with $\poly(d)$ resets and queries to an expert action oracle, one can learn an $\eps$-optimal policy, thus circumventing known hardness results for the linear value function setting. Up to a factor of $A$, one can simulate queries to the expert action oracle by querying $\oracle(x,a)$ for each $a\in \cA$. Therefore, our lower bound in \pref{thm:lower-bound-expert} extends to this weaker setting. \cite{sekhari2023selective} propose an algorithm which selectively queries a stochastic realizable expert, and they prove that this algorithm achieves regret bound scaling with the eluder dimension of the function class.

\paragraph{Tabular RL with Predictions.} We mention a few papers which do not strictly speaking fall under the umbrella of IL, but are relevant for the results in this chapter. \citep{golowich2022can, gupta2022unpacking} study tabular RL with \emph{inexact} value predictions for either the optimal $Q^\star$ or $V^\star$. The main difference with our results is that they assume access to the entire table of values, while we study the agnostic RL setting with a large state space and formalize access to predictions via an expert oracle that is always exactly correct (\pref{def:expert-advice-oracle}).

\chapter{Online RL with Exploratory Resets}\label{chap:mu-reset}

In this chapter, we study the sample complexity of learning under the $\mu$-reset interaction protocol, meaning that the learner is given sampling access to an additional exploratory reset distribution $\mu = \crl{\mu_h}_{h\in[H]}$ and can execute online rollouts from these $\mu_h$ in addition to online rollouts from the starting distribution $d_1$.

To set the stage, in \pref{sec:psdp}, we review a classical algorithm for this setting called Policy Search By Dynamic Programming (\psdp{}), which was introduced by \citet{bagnell2003policy}. We state the classical analysis of \psdp{}, which solves PAC RL under the assumption that the reset satisfies bounded concentrability and the policy class satisfies a strong \emph{policy completeness} assumption. The rest of this chapter is devoted to understanding if it is possible to relax this policy completeness assumption. 
\begin{enumerate}
    \item In \pref{sec:psdp}, we investigate the performance of \psdp{} when the policy class only satisfies \emph{realizability}, and we show that under slightly stronger assumptions, \psdp{} achieves exponential in horizon sample complexity; moreover, we give an algorithm-dependent lower bound which shows this result is tight.
    \item In \pref{sec:mu-reset-lb}, we turn to the fully agnostic setting, and give an information-theoretic lower bound showing that sample-efficient agnostic policy learning is not possible. 
\end{enumerate}
We conclude with open problems in \pref{sec:mu-reset-open-problems}.

\section{Policy Search By Dynamic Programming}\label{sec:psdp}

In this section, we provide a description of the \psdp{} algorithm and analyze its sample complexity. We show the standard upper bound for \psdp{} which has appeared in prior works \cite[e.g.,][]{misra2020kinematic} in \pref{sec:psdp-policy-completeness}. We also prove several new results about \psdp{} when only policy realizability is satisfied: namely if the reset distribution $\mu$ satisfies stronger properties beyond bounded concentrability, we show exponential in $H$ upper bounds in \pref{sec:psdp-upper} as well as a matching lower bound in \pref{sec:psdp-lower}. We also discuss in \pref{sec:psdp-lower} how our lower bounds against \psdp{} also apply against the \cpi{} algorithm.

\subsection{\psdp{} Guarantee Under Policy Completeness}\label{sec:psdp-policy-completeness}

Policy Search By Dynamic Programming (\psdp{}) is a widely studied policy learning algorithm \cite{bagnell2003policy} that relies on $\mu$-reset access.  \psdp{} constructs partial policies $\estpi_{h:H} \in \Pi_{h:H}$, starting from layer $H$, and returns the estimated policy $\estpi_{1:H}$. We provide pseudocode for \psdp{} in \pref{alg:psdp}. 

\begin{algorithm}[!htp]
\caption{\psdp{} \cite{bagnell2003policy}}\label{alg:psdp}
	\begin{algorithmic}[1]
        \Require Reset distributions $\mu = \crl{\mu_h}_{h\in [H]}$, policy class $\Pi$.
        \For {$h = H,\cdots, 1$} 
            \State Initialize dataset $\cD_h = \varnothing$.
            \For {$n$ times}: \hfill \algcomment{Collecting $(x_h, a_h, v_h)$ requires $\mu$-reset access.}
            \State Sample $(x_h, a_h)$ where $x_h \sim \mu_{h}$ and $a_h \sim \unif(\cA)$.
            \State Let $v_h \coloneqq \sum_{h'=h}^H r_{h'}$ be the value of executing $a_h \circ \estpi_{h+1:H}$ from $x_h$.
            \State Set $\cD_h \gets \cD_h \cup \crl{(x_h, a_h, v_h)}$.
        \EndFor
        \State Call CB oracle: $\estpi_h \coloneqq \argmax_{\pi \in \Pi} \frac{1}{n} \sum_{(x_h, a_h, v_h)\in \cD_h} \frac{\ind{a_h = \pi(x_h)}}{A} \cdot v_h$. \label{line:cb-oracle}

        \EndFor
        \State \textbf{Return} $\estpi_{1:H}$.
	\end{algorithmic}
\end{algorithm} 
The classic analysis of \psdp{} requires two key assumptions:
\begin{enumerate}
    \item  An exploration condition of \emph{concentrability} (\pref{def:concentrability}).
    \item A representation condition called \emph{policy completeness}, to be described next.
\end{enumerate}

\paragraph{Policy Completeness.} Completeness assumptions on the function approximator class are often assumed in the study of RL algorithms (see discussion in \pref{sec:realizability-approach}). \psdp{} requires a notion called \emph{policy completeness}, which ensures that the policy class is closed under the policy improvement operator \cite{dann2018oracle, misra2020kinematic}. 

\begin{definition}[Policy Completeness]\label{def:policy-completeness}
A policy class $\Pi$ satisfies policy completeness if for every $\pi \in \Pi$ and $h \in [H]$, there exists a policy $\wt{\pi} \in \Pi$ such that:
\begin{align*}
    \text{for all $x \in \statesp_h$}: \quad \wt{\pi}_h(x) = \argmax_{a \in \actionsp} Q^\pi(x,a).
\end{align*}
\end{definition}
This is a worst-case variant of policy completeness. As we will see, the analysis of \psdp{} only requires a weaker $\ell_1$ variant of policy completeness, and the resulting suboptimality incurs an additive dependence on the $\ell_1$ policy completeness error. Policy realizability (which only asserts that such a $\wt{\pi}$ exists for $\optpi_{h+1:H}$ at every $h \in [H]$) is implied by policy completeness. 

\paragraph{Sample Complexity Guarantee for \psdp.} As a prototypical classical result on policy learning, we now state the guarantee for \psdp{}.

\begin{theorem}\label{thm:psdp-ub}
Suppose the policy class $\Pi$ satisfies policy completeness (\pref{def:policy-completeness}), and the reset distribution $\mu$ satisfies concentrability with parameter $\cconc$. With  probability \(1 - \delta\), \psdp{} finds an $\eps$-optimal policy using $\poly( \cconc, A, H, \eps^{-1}, \log \abs{\Pi}, \log \delta^{-1})$ samples from the reset distribution.
\end{theorem} 

Before we prove \pref{thm:psdp-ub}, we first set up some additional notation. We define an averaged notion of policy completeness; compared to \pref{def:policy-completeness}, this notion is weaker since it only requires completeness to hold in an averaged sense over the reset $\mu$.

For readability, we slightly abuse notation: for $Q$-functions we denote $Q^\pi_h(x,\pi) \coloneqq Q^\pi_h(x, \pi(x))$. Similarly, we sometimes denote rewards as $R(x,\pi) \coloneqq R(x, \pi(x))$ and transitions as $P(\cdot\mid x, \pi) \coloneqq P(\cdot\mid x, \pi(x))$. 

\begin{definition}[Average Policy Completeness]\label{def:averaged-pc} Fix any policy class $\Pi$, as well as exploratory distribution $\mu = \crl{\mu_h}_{h\in[H]}$. For any layer $h \in [H]$ and policy $\estpi \coloneqq \estpi_{h+1:H} \in \Pi_{h+1:H}$ we define the (average) \emph{policy completeness error}, denoted $\epspc: \Pi_{h+1:H} \to \bbR$, as
\begin{align*}
    \epspc(\estpi) \coloneqq \min_{\pi_h \in \Pi_h} ~ \En_{x \sim \mu_h} \brk*{\max_{a \in \cA}~ Q^{\estpi}(x, a) - Q^{\estpi}(x, \pi_h)}.
\end{align*}
\end{definition}
\pref{def:averaged-pc} is similar to previously defined notions of policy completeness \cite{scherrer2014local, agarwal2023variance}. As a point of comparison, Definition 2 of \cite{agarwal2023variance} defines the average policy completeness to be the worst case over the convex hull of suffix policies $\estpi$, i.e. $\epspc \coloneqq \sup_{\estpi \in \mathsf{Conv}(\Pi)} \epspc(\estpi)$, while we define it as a function which takes as input a rollout policy $\estpi$.

\begin{proof}[Proof of \pref{thm:psdp-ub}]
First, we state a standard generalization bound on the contextual bandit oracle invoked in \pref{line:cb-oracle}. With probability at least $1-\delta$, for every $h \in [H]$ the returned policy $\estpi_h$ satisfies
\begin{align*}
    \En_{x \sim \mu_h} \brk*{Q^{\estpi}(x, \estpi_h)} \ge \max_{\pi_h \in \Pi_h} \En_{x \sim \mu_h}\brk*{Q^{\estpi}(x, \pi_h)} - \epsstat, \quad \text{where}\quad \epsstat \coloneqq O\prn*{\sqrt{\frac{A\log (\abs{\Pi}/\delta)}{n} }}.\numberthis\label{eq:cb-eq}
\end{align*}
For every $h \in [H]$, let us define:
\begin{align*}
    \wt{\pi}_h^\star(x) \coloneqq \argmax_{a \in \cA}~ Q^{\estpi}(x,a), \quad \text{and} \quad \wt{\pi}_h \coloneqq \argmax_{\pi_h \in \Pi_h}~ \En_{x \sim \mu_h}\brk*{Q^{\estpi}(x, \pi_h)}
\end{align*}
Then we calculate:
\begin{align*}
    V^\star - V^{\estpi} &\overset{(i)}{=} \sum_{h=1}^H \En_{x \sim d^{\optpi}_h} \brk*{Q^{\estpi}(x, \optpi) - Q^{\estpi}(x, \estpi_h) } \\
    &\overset{(ii)}{\le} \sum_{h=1}^H  \En_{x \sim d^{\optpi}_h} \brk*{Q^{\estpi}(x, \wt{\pi}_h^\star) - Q^{\estpi}(x, \estpi_h) } \\
    &\overset{(iii)}{\le} \sum_{h=1}^H \nrm*{\frac{d^{\optpi}_h}{\mu_h}}_\infty \En_{x \sim \mu_h} \brk*{ Q^{\estpi}(x, \wt{\pi}_h^\star) - Q^{\estpi}(x, \estpi_h) } \\
    &\overset{(iv)}{\le} \cconc \cdot \sum_{h=1}^H \En_{x \sim \mu_h} \brk*{ Q^{\estpi}(x, \wt{\pi}_h^\star) - Q^{\estpi}(x, \estpi_h) }  \\
    &= \cconc \cdot \sum_{h=1}^H \prn*{ \En_{x \sim \mu_h} \brk*{ Q^{\estpi}(x, \wt{\pi}_h^\star) - Q^{\estpi}(x, \wt{\pi}_h) } + \En_{x \sim \mu_h} \brk*{ Q^{\estpi}(x, \wt{\pi}_h) - Q^{\estpi}(x, \estpi_h) } }\\ 
    &\overset{(v)}{\le} H \cconc  \epsstat + \cconc \sum_{h=1}^H  \epspc(\estpi_{h+1:H}).
\end{align*}
Here, $(i)$ follows by the Performance Difference Lemma, $(ii)$ is due to the optimality of $\wt{\pi}^\star_h$, $(iii)$ is due to nonnegativity of $Q^{\estpi}(x, \wt{\pi}_h^\star) - Q^{\estpi}(x, \estpi_h)$, $(iv)$ is due to the definition of $\cconc$, and $(v)$ follows by \pref{def:averaged-pc} and Eq.~\eqref{eq:cb-eq}. Therefore, if the policy completeness error is zero, then we have a bound which is at most $H \cconc \epsstat$, and therefore \psdp{} returns an $\eps$-optimal policy using $\poly(\cconc, A, H, \log\abs{\Pi}, \eps^{-1}, \log \delta^{-1})$ samples.
\end{proof}

\subsection{Upper Bounds for \psdp{} with Policy Realizability}\label{sec:psdp-upper}
The analysis of \psdp{} in \pref{thm:psdp-ub} crucially relied on the policy completeness assumption. Now we ask if policy completeness can be relaxed, i.e.,
\begin{center}
    \textit{Can we still get guarantees for \psdp{} with only realizability?} 
\end{center}

It is easy to see that policy realizability is \emph{insufficient} for \psdp{} even for horizon $H=2$, as shown in \pref{fig:psdp-lower-bound-simple}. Similar to lower bounds for offline RL \cite{foster2021offline}, the construction relies on \emph{overcoverage}, as $\mu$ has nonzero mass on a nonreachable state $\bar{s}_{1}$, which is somewhat unnatural. In the construction, \psdp{} may not even be consistent, since one can take $\gamma$ to be arbitrarily close to 0 so that with constant probability \psdp{} returns a $(1+\gamma)$-suboptimal policy. 

In this section, we circumvent the lower bound and show that if we make stronger assumptions on the reset distribution $\mu$, \psdp{} achieves consistency:
\begin{enumerate}
    \item If $\Pi$ is realizable and the reset $\mu$ has bounded pushforward concentrability (\pref{def:exploratory-pushforward-distribution}), \pref{thm:psdp-ub-pushforward} achieves $(\cpush)^{O(H)}$ sample complexity. 
    \item If $\Pi$ is realizable and the reset $\mu$ is admissible (\pref{def:admissible}) with bounded concentrability, \pref{thm:psdp-ub-admissible} achieves $(\cconc)^{O(H)}$ sample complexity.
\end{enumerate}
The two upper bounds are in general incomparable, as there exist settings in which one achieves a better guarantee than the other. In addition, to the best of our knowledge, neither result is implied by any known results for policy learning---note that the trivial bound of $A^H$ achieved by importance sampling \cite{kearns1999approximate, agarwal2019reinforcement} can be much larger when $\cpush \ll A$.

\begin{figure}[!t]
    \centering
\includegraphics[scale=0.30, trim={0cm 24cm 21cm 0cm}, clip]{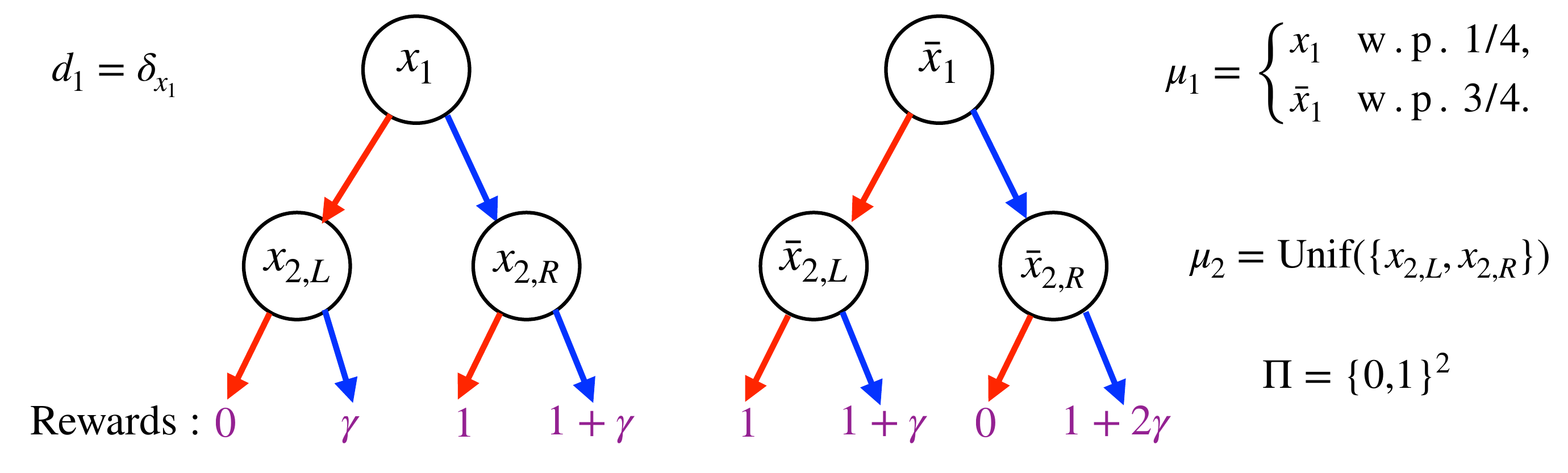}
    \caption{Lower bound for \psdp{} without policy completeness. \textcolor{red}{Red} arrows represent action $0$ and \textcolor{blue}{blue} arrows represent action $1$. In \textcolor{Purple}{purple} we denote the expectation of the stochastic reward. Let $\gamma > 0$ be an arbitrarily small constant. At layer $h=2$, with constant probability, \psdp{} selects $\textcolor{red}{\estpi^{(2)} \gets 0}$ since $\En_{x\sim \mu_2} V^{\pi_0}(x) = 1/2$ and $\En_{x\sim \mu_2} V^{\pi_1}(x) = 1/2 + \gamma$. Conditioned on $\estpi^{(2)} = 0$, we have $\En_{x \sim \mu_1} V^{\pi_0 \circ \estpi^{(2)}}(x) = 3/4$ while $\En_{x \sim \mu_1} V^{\pi_1 \circ \estpi^{(2)}}(x) = 1/4$, so therefore \psdp{} selects $\textcolor{red}{\estpi^{(1)} \gets 0}$. The returned policy $\estpi^{(1)} \circ \estpi^{(2)}$ is $(1+\gamma)$-suboptimal on $d_1$. Note that $\mu = \crl{\mu_1, \mu_2}$ satisfies $\cconc = 4$, and that $\Pi$ satisfies realizability.} 
    \label{fig:psdp-lower-bound-simple}
\end{figure}   

\subsubsection{Policy Realizability + Pushforward Concentrability}

\begin{theorem}\label{thm:psdp-ub-pushforward}
Suppose $\Pi$ is realizable, and the reset $\mu$ satisfies pushforward concentrability with parameter $\cpush$. With high probability, \psdp{} returns an $\eps$-optimal policy using \begin{align*}
    \poly((\cpush)^{H}, A, \log \abs{\Pi}, \eps^{-1}) \quad \text{samples.}
\end{align*}
\end{theorem}

The proof relies on the following lemma, which relates the policy completeness error to the pushforward concentrability coefficient of $\mu$.

\begin{lemma}\label{lem:pc-bound-pushforward}
Fix any layer $h \in [H]$. For any suffix policy $\estpi_{h+1:H}$ we have
\begin{align*}
    \epspc(\estpi_{h+1:H}) \le  \cpush \cdot \En_{x'\sim \mu_{h+1}} \brk*{ V^\star(x') - V^{\estpi}(x') }. 
\end{align*}
\end{lemma}

\begin{proof}[Proof of \pref{lem:pc-bound-pushforward}] 
We have the following computation: 
\begin{align*}
    \hspace{2em}&\hspace{-2em} \epspc(\estpi_{h+1:H}) \\
    &= \min_{\pi_h \in \Pi_h}~\En_{x \sim \mu_h} \brk*{\max_{a \in \cA}~ Q^{\estpi}(x,a) - Q^{\estpi}(x, \pi_h)} \\
    &\le \En_{x \sim \mu_h} \brk*{\max_{a \in \cA}~ Q^{\estpi}(x,a) - Q^{\estpi}(x, \optpi)} \\
    &\le \En_{x \sim \mu_h} \brk*{ Q^{\star}(x,\optpi) - Q^{\estpi}(x, \optpi)} \\
    &= \En_{x \sim \mu_h} \brk*{ r(x,\optpi) + \En_{x' \sim P(\cdot \mid x,\optpi)} V^\star(x')}   - \En_{x \sim \mu_h} \brk*{r(x, \optpi) + \En_{x' \sim P(\cdot \mid x,\optpi)} V^{\estpi}(x')} \\
    &= \En_{x \sim \mu_h, x' \sim P(\cdot \mid x , \optpi)} \brk*{ V^\star(x') - V^{\estpi}(x')}. \numberthis\label{eq:ub-policy-completeness-error} 
\end{align*}
The first inequality is due to the realizability $\optpi \in \Pi$, and the second one is due to the optimality of $\optpi$. 
Now we will perform a change of measure to relate the bound in Eq.~\eqref{eq:ub-policy-completeness-error} to the error of $\estpi$ on the layer $h+1$.
\begin{align*}
    \En_{x \sim \mu_h, x' \sim P(\cdot \mid x, \optpi)} \brk*{ V^\star(x') - V^{\estpi}(x')} &= \En_{x' \sim \mu_{h+1}} \brk*{ \frac{\En_{x \sim \mu_h} P(x' \mid x, \optpi)}{\mu_{h+1}(x')} \cdot \prn*{ V^\star(x') - V^{\estpi}(x')} } \\
    &\le \cpush \cdot \En_{x' \sim \mu_{h+1}} \brk*{  V^\star(x') - V^{\estpi}(x') }, 
\end{align*}
where the inequality uses the nonnegativity of $V^\star(x') - V^{\estpi}(x')$ and the definition of pushforward concentrability. Plugging this back into Eq.~\eqref{eq:ub-policy-completeness-error} proves \pref{lem:pc-bound-pushforward}.
\end{proof}

\begin{proof}[Proof of \pref{thm:psdp-ub-pushforward}]
Using Performance Difference Lemma we have for the learned policy $\estpi \in \Pi$:
\begin{align*}
    V^\star - V^{\estpi} &= \sum_{h=1}^H \En_{x \sim d^{\estpi}_h} \brk*{V^\star(x) - Q^{\star}(x, \estpi_h) } = \sum_{h=1}^H \En_{x \sim \mu_h} \brk*{\frac{d^{\estpi_h}(x)}{\mu_h(x)}\prn*{V^\star(x) - Q^{\star}(x, \estpi_h) }}\\
    &\le \cconc \cdot \sum_{h=1}^H \En_{x \sim \mu_h} \brk*{V^\star(x) - Q^\star(x, \estpi)} \le  \cconc \cdot \sum_{h=1}^H \En_{x \sim \mu_h} \brk*{V^\star(x) - V^{\estpi}(x)}.
\end{align*}
The first inequality uses the fact that $\estpi \in \Pi$ as well as $V^\star(x) \ge Q^{\star}(x, \estpi_h)$, and the second inequality uses the latter fact again. From here, we apply an inductive argument to bound the suboptimality $\En_{x \sim \mu_h} \brk{V^\star(x) - V^{\estpi}(x)}$ for all $h \in [H]$. Fix any $h \in [H]$. We have
\begin{align*}
    \En_{x \sim \mu_h} \brk*{V^\star(x) -  V^{\estpi}(x)}
    &= \En_{x \sim \mu_h} \brk*{Q^\star(x, \optpi) -  Q^{\estpi}(x,\estpi)}\\
    &\le \En_{x \sim \mu_h} \brk*{Q^\star(x, \optpi) - Q^{\estpi}(x,\optpi) + \max_{a} Q^{\estpi}(x,a) -  Q^{\estpi}(x,\estpi)} \\
    &= \En_{x \sim \mu_h, x' \sim P(\cdot \mid x,\optpi)} \brk*{V^\star(x') - V^{\estpi}(x')} + \En_{x \sim \mu_h} \brk*{\max_{a} Q^{\estpi}(x,a) - Q^{\estpi}(x,\estpi)} \\
    &\le \cpush \En_{x' \sim \mu_{h+1}} \brk*{V^\star(x') - V^{\estpi}(x')} + \epsstat + \epspc(\estpi_{h+1:H}) \\
    &\le 2C_\mathrm{push} \cdot \En_{x' \sim \mu_{h+1}} \brk*{ V^\star(x') - V^{\estpi}(x') } + \epsstat, \numberthis\label{eq:recursion-1}
\end{align*}
where the last inequality uses \pref{lem:pc-bound-pushforward}. Recursive application of \pref{eq:recursion-1} and the fact that $\En_{x \sim \mu_H} \brk{V^\star(x) - V^{\estpi}(x)} = \En_{x \sim \mu_H} \brk{r(x, \optpi) - r(x, \estpi_H)} \le \epsstat$ gives us
\begin{align*}
    \En_{x \sim \mu_h} \brk*{V^\star(x) -  V^{\estpi}(x)} \le H \cdot (2\cpush)^{H} \epsstat,
\end{align*}
so therefore the final suboptimality of \psdp{} is at most
\begin{align*}
    V^\star - V^{\estpi} \le \cconc \cdot \sum_{h=1}^H \En_{x \sim \mu_h} \brk*{V^\star(x) - V^{\estpi}(x)} \le H^2 \cdot (2 \cpush)^{H+1} \epsstat.
\end{align*}
Choosing $n = \poly((\cpush)^{H}, A, \log \abs{\Pi}, \eps^{-1})$ so that the right hand side is at most $\eps$ proves the final bound.
\end{proof}

\subsubsection{Policy Realizability + Admissibility + Concentrability}

\begin{definition}
\label{def:admissible}
We say a distribution $\mu$ is admissible if for every $h \in [H]$ there exists some $\pi_b \in \Delta(\Pi)$:
\begin{align*}
    \mu_h(x) = d^{\pi_b}_h(x) \quad \text{for all}~x \in \statesp_h.
\end{align*}

\end{definition}

\begin{theorem}\label{thm:psdp-ub-admissible}
Suppose $\Pi$ is realizable, and the reset $\mu$ (1) satisfies concentrability with parameter $\cconc$, and (2) is admissible. With high probability, \psdp{} finds an $\eps$-optimal policy using $\poly((\cconc)^{H}, A, \log \abs{\Pi}, \eps^{-1})$ samples.
\end{theorem}

To prove \pref{thm:psdp-ub-admissible}, we first establish a few helper lemmas on the errors of the learned policy $\estpi$.

\begin{lemma}\label{lem:transfer} For any layer $h \in [H]$ and admissible distribution $\nu \in \Delta(\statesp_h)$, we have 
\begin{align*}
    \max_{\pi \in \Pi_h}~\En_{x \sim \nu} \brk*{Q^{\estpi}(x, \pi) - Q^{\estpi}(x, \estpi)} \le \cconc \prn*{\epsstat + \epspc(\estpi_{h+1:H})}.
\end{align*}
\end{lemma}

\begin{proof}
We calculate that
\begin{align*}
    \hspace{2em}&\hspace{-2em} \max_{\pi \in \Pi_h}~ \En_{x \sim \nu} \brk*{Q^{\estpi}(x, \pi) - Q^{\estpi}(x, \estpi_h)} \\
    &= \max_{\pi \in \Pi_h} \En_{x \sim \nu} \brk*{Q^{\estpi}(x, \pi) - \max_a Q^{\estpi}(x,a) } + \En_{x \sim \nu} \brk*{\max_a Q^{\estpi}(x,a) - Q^{\estpi}(x, \estpi_h)} \\
    &\le \cconc \cdot \En_{x \sim \mu_h} \brk*{\max_a Q^{\estpi}(x,a) - Q^{\estpi}(x, \estpi_h)} \\
    &= \cconc \cdot \Bigg( \En_{x \sim \mu_h} \brk*{\max_a Q^{\estpi}(x,a)} - \max_{\pi \in \Pi_h} \En_{x \sim \mu_h} \brk*{ Q^{\estpi}(x, \pi)}  \\
    &\hspace{8em} + \max_{\pi \in \Pi_h} \En_{x \sim \mu_h} \brk*{Q^{\estpi}(x, \pi)} - \En_{x \sim \mu_h} \brk*{Q^{\estpi}(x, \estpi_h)} \Bigg) \\
    &\le \cconc \prn*{\epsstat + \epspc(\estpi_{h+1:H})}.
\end{align*}
In the first inequality we use the fact that $\nu$ is admissible, so we can use concentrability to relate the density ratios $\nrm{\nu/\mu}_\infty$.
\end{proof}

\emph{Additional Notation.} In the subsequent analysis, for any distribution $\nu$ we denote $\epspc(\estpi, \nu)$ to be the policy completeness error under distribution $\nu$, i.e.,
\begin{align*}
    \epspc(\estpi, \nu) \coloneqq \min_{\pi_h \in \Pi_h} ~ \En_{x \sim \nu} \brk*{\max_{a \in \cA}~ Q^{\estpi}(x, a) - Q^{\estpi}(x, \pi)}.
\end{align*}
For any partial policy $\pi_{h:t-1}$, we also denote $\nu \circ \pi_{h:t-1} \in \Delta(\statesp_{t})$ to denote the distribution over states in layer $t$ which is achieved by first sampling a state $x_h \sim \nu$ then rolling out with partial policy $\pi_{h:t-1}$.

\begin{lemma}\label{lem:pc-bound}
For any layer $h \in [H]$ and admissible distribution $\nu \in \Delta(\statesp_h)$, we have
\begin{align*}
    \epspc(\estpi_{h+1:H}, \nu) \le (H- h) \cdot \cconc \epsstat + \cconc \cdot \sum_{h' = h+1}^H \epspc(\estpi_{h'+1:H}) 
\end{align*}
\end{lemma}

\begin{proof}
Using the definition of policy completeness we have
\begin{align*}
    \epspc(\estpi_{h+1:H}, \nu) &\le \En_{x \sim \nu} \brk*{\max_a Q^{\estpi}(x,a) - Q^{\estpi}(x, \optpi)} \le \En_{x \sim \nu} \brk*{ Q^{\star}(x,\optpi) - Q^{\estpi}(x, \optpi)}. 
\end{align*}
Now, we apply a recursive argument, which gives us 
\begin{align*}
     \epspc(\estpi_{h+1:H}, \nu) &\le \En_{x \sim \nu} \brk*{ Q^{\star}(x,\optpi) - Q^{\estpi}(x, \optpi)} \\
    &= \En_{x' \sim \nu \circ \optpi} \brk*{ Q^{\star}(x',\optpi) - Q^{\estpi}(x', \estpi)} \\
    &= \En_{x' \sim  \nu \circ \optpi} \brk*{ Q^{\star}(x',\optpi) - Q^{\estpi}(x', \optpi) + Q^{\estpi}(x', \optpi) - Q^{\estpi}(x', \estpi) }
\end{align*}
Because $\nu$ is admissible, so is $\nu \circ \optpi$. Therefore, the second term in the sum is bounded using \pref{lem:transfer}:
\begin{align*}
    \En_{x' \sim  \nu \circ \optpi} \brk*{ Q^{\estpi}(x', \optpi) - Q^{\estpi}(x', \estpi) } &\le \cconc \prn*{ \epspc(\estpi_{h+2:H}) + \epsstat }.
\end{align*}
The first term in the sum can be rewritten as
\begin{align*}
    \En_{x' \sim  \nu \circ \optpi} \brk*{ Q^{\star}(x',\optpi) - Q^{\estpi}(x', \optpi)} = \En_{{x''} \sim  \nu \circ \optpi \circ \optpi} \brk*{ Q^{\star}(x'',\optpi) - Q^{\estpi}(x'', \estpi)}.
\end{align*}
Applying recursion, we get the final bound of
\begin{align*}
    \epspc(\estpi_{h+1:H}, \nu) &\le (H - h) \cdot \cconc \epsstat + \cconc \cdot \sum_{h' = h+1}^H \epspc(\estpi_{h'+1:H}).
\end{align*}
This concludes the proof of \pref{lem:pc-bound}.
\end{proof}

\begin{proof}[Proof of \pref{thm:psdp-ub-admissible}]
We compute the suboptimality as
\begin{align*}
    V^\star - V^{\estpi} &= \sum_{h=1}^H \En_{x \sim d^{\star}_h} \brk*{Q^{\estpi}(x, \optpi) - Q^{\estpi}(x, \estpi) } &\text{(Performance Difference Lemma)} \\
    &\le H \cconc \epsstat + \cconc \prn*{ \sum_{h=1}^H \epspc(\estpi_{h+1:H}) }. &\text{ (\pref{lem:transfer})}
\end{align*}
Now we apply \pref{lem:pc-bound} to show that the policy completeness error can be bounded by the downstream policy completeness errors, using the admissibility of $\mu$.
\begin{align*}
    \hspace{2em}&\hspace{-2em} V^\star - V^{\estpi} \\
    &\le H \cconc \epsstat +  \cconc \sum_{h=1}^H \epspc(\estpi_{h+1:H}) \\
    &= H \cconc \epsstat +  \cconc \cdot \prn*{\epspc(\estpi_{2:H}) + \sum_{h=2}^H \epspc(\estpi_{h+1:H}) } \\
    &\le H \cconc \epsstat +  \cconc \cdot \prn*{\prn{H - 1} \cconc \epsstat + (1+ \cconc) \cdot \sum_{h=2}^H \epspc(\estpi_{h+1:H}) } &\text{(\pref{lem:pc-bound})}\\
    &\lesssim H (\cconc)^2 \epsstat + \prn{1+\cconc}^2 \cdot \prn*{ \sum_{h=2}^H \epspc(\estpi_{h+1:H}) } \\
    &= H (\cconc)^2 \epsstat + \prn{1+\cconc}^2 \cdot \prn*{  \epspc(\estpi_{3:H}) + \sum_{h=3}^H \epspc(\estpi_{h+1:H}) } \\
    &\le H (\cconc)^2 \epsstat \\
    &\hspace{3em} + \prn{1+\cconc}^2 \cdot \prn*{  \prn{H - 2} \cconc \epsstat + \prn{1+\cconc} \sum_{h=3}^H \epspc(\estpi_{h+1:H}) } &\text{(\pref{lem:pc-bound})}\\
    &\lesssim H \prn*{1+\cconc}^3 \epsstat + \prn{1+\cconc}^3 \prn*{ \sum_{h=3}^H \epspc(\estpi_{h+1:H}) }.
\end{align*}
Continuing this way (and observing that $\epspc(\estpi_{H+1} = \varnothing) = 0$) we get a final bound of
\begin{align*}
    V^\star - V^{\estpi} \lesssim H \prn{1+\cconc}^H \epsstat.
\end{align*}
Setting $n = \poly((\cconc)^{H}, A, \log \abs{\Pi}, \eps^{-1})$ makes the RHS at most $\eps$, thus proving \pref{thm:psdp-ub-admissible}.
\end{proof}

\subsection{Lower Bounds for \psdp{} and \cpi{}}\label{sec:psdp-lower}
Now we will show that exponential error compounding is unavoidable for \psdp{} in the absence of policy completeness. \psdp{} relies on a reduction to a contextual bandit oracle. For the lower bound statement, we will assume that $\epsstat > 0$ is a fixed constant and \psdp{} is equipped with a \emph{worst case} oracle $\mathsf{CB}_{\epsstat}$ which for every layer $h \in [H]$ always returns an \emph{arbitrary policy} $\estpi_h$ satisfying
\begin{align*}
    \En_{x \sim \mu_h} \brk*{Q^{\estpi}(x, \estpi_h)} \ge \max_{\pi \in \Pi} \En_{x \sim \mu_h}\brk*{Q^{\estpi}(x, \pi)} - \epsstat.
\end{align*}
Thus, the lower bound statement has the flavor of a statistical query lower bound \cite{kearns1998efficient}, which also assumes a worst-case response up to accuracy $\epsstat$.

\begin{theorem}\label{thm:psdp-lower-bound}
Let $H \ge 2$. Fix any $\epsstat > 0$ and parameter $\cpush \ge 5H$. There exists a tabular MDP $M$ with $S = O(H^2)$ states, $A = O(H)$ actions, and horizon $H$, realizable policy class $\Pi$ of size $2^{\wt{O}(H)}$, and exploratory distribution $\mu$ which is admissible and satisfies pushforward concentrability with parameter $\cpush$, so that \psdp{} equipped with oracle $\mathsf{CB}_{\epsstat}$ returns a policy $\estpi$:
\begin{align*}
    V^\star - V^{\estpi} \ge (\cpush)^{\Omega(H)} \epsstat.
\end{align*}
\end{theorem}

\pref{thm:psdp-lower-bound} is a converse to the positive results of \pref{thm:psdp-ub-pushforward} and \ref{thm:psdp-ub-admissible}, showing that \psdp{} can have exponential in $H$ sample complexity. The lower bound construction in \pref{thm:psdp-lower-bound} as well as the earlier one from \pref{fig:psdp-lower-bound-simple} are given by tabular MDPs. Thus, even simple tabular RL algorithms which use online access to the MDP can solve these constructions! This indicates an \emph{algorithmic limitation} of using dynamic programming to solve policy learning.

Lastly, we also remark that the constructions in \pref{fig:psdp-lower-bound-simple} and \pref{thm:psdp-lower-bound} also apply to \cpi{} \cite{kakade2002approximately}; we refer the reader to \cite[Section 14 of][]{agarwal2019reinforcement} for an exposition of the \cpi{} algorithm. At a high level, the \cpi{} algorithm generates a sequence of policy iterates $\pi^{(1)}, \pi^{(2)}, \cdots$ such that each policy iterate improves upon the previous one and terminates whenever:
\begin{align*}
    \max_{\wt{\pi} \in \Pi}~\En_{x \sim d^{\pi^{(t)}}_\mu} \brk*{Q^{\pi^{(t)}}(x, \wt{\pi}) - Q^{\pi^{(t)}}(x, \pi^{(t)})} \le \eps.
\end{align*}
where $d^{\pi^{(t)}}_\mu$ is the occupancy measure obtained by running the current iterate $\pi^{(t)}$ starting from the reset $\mu$ and $\eps> 0$ is some predefined threshold which represents the accuracy to which \cpi{} solves the policy improvement problem. Thus, if it is not possible to greatly improve (by at least $\eps$) the average $Q$-function by selecting a different policy $\wt{\pi}$, then \cpi{} will terminate. In our constructions, one can check that if we initialize to the all-zeros policy $\pi^{(1)} \equiv 0$, then \cpi{} will terminate immediately even though $\pi^{(1)}$ has constant suboptimality.

In the rest of this section, we will prove \pref{thm:psdp-lower-bound}.

\subsubsection{Lower Bound Construction}
\begin{figure}[!t]
    \centering
\includegraphics[scale=0.31, trim={0.4cm 0cm 15cm 0cm}, clip]{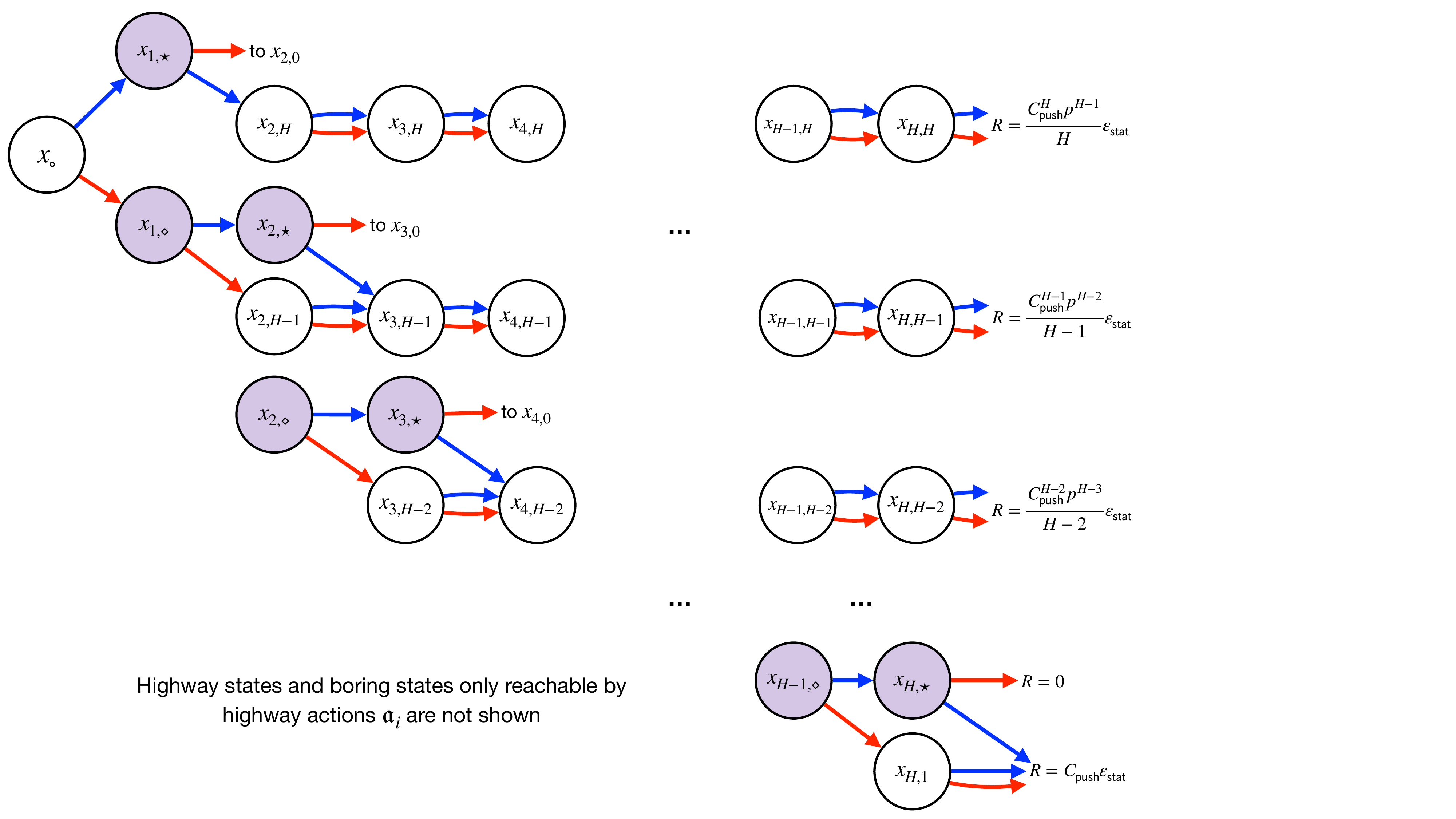}
    \caption{Lower bound construction for \pref{thm:psdp-lower-bound}. To avoid clutter, we do not illustrate the highway states as well as any boring states which are only reachable by taking highway actions $\mathfrak{a}_i$ at layer 0, since their role is only to make sure that the construction satisfies admissibility.} 
    \label{fig:psdp-lb-2}
\end{figure}

Our lower bound construction is illustrated in \pref{fig:psdp-lb-2}.

For notational convenience, we number the layers starting with $h=0$, so that there are $H+1$ layers.

\paragraph{State and Action Spaces.} At $h=0$ there is a single state $x_\circ$ and the action set is $\cA_0 = \crl{0, 1, \mathfrak{a}_1, \cdots, \mathfrak{a}_H}$. For $h \ge 1$, we have 
\begin{align*}
    \statesp_h = \underbrace{ \crl*{x_{h,0}, x_{h,1} \cdots x_{h,H}} }_{\text{$H+1$ boring states}} \cup \underbrace{\crl*{x_{h, \diamond}} \cup  \crl*{x_{h, \star}}  }_{\text{2 special states}} \cup  \underbrace{ \crl*{\bar{x}_{h \to h+1}, \bar{x}_{h\to h+2}, \cdots \bar{x}_{h \to H}} }_{\text{$H-h$ highway states}},
\end{align*}
except for $\cX_H$ which does not have the special state $x_{h, \diamond}$. The action set is $\cA_h = \crl{0,1}$.

\paragraph{Policy Class.} The policy class $\Pi$ is taken to be all open-loop policies over each layer's action space:
\begin{align*}
    \Pi \coloneqq \crl*{\pi: \forall x \in \statesp_h, \pi_h(x) \equiv a_h, (a_0, a_1, \cdots, a_H) \in \prod_{h=0}^H \cA_h}.
\end{align*}

\paragraph{Reset Distribution.} At layer $h=0$ we have $d_0 = \mu_0 = \delta_{x_\circ}$. At layer $h \ge 1$, the distribution $\mu_h$ puts $1/\cpush$ mass on each of the non-diamond states $ \crl{x_{h,0}, x_{h,1} \cdots x_{h,H}} \cup \crl{x_{h, \star}}  \cup \crl{\bar{x}_{h \to h+1}, \bar{x}_{h\to h+2}, \cdots \bar{x}_{h \to H}}$ and the rest on $x_{h, \diamond}$. Therefore $x_{h, \diamond}$ has mass \emph{at least} $p \coloneqq \prn{1 - \frac{2H+1}{\cpush}}$. We have $p > 1/2$ as long as $\cpush \ge 5H$.

\paragraph{Transitions.} At $h=0$, we have
\begin{align*}
    \Pr(\cdot \mid x_\circ, a) = \begin{cases}
        \delta_{x_{1, \diamond}} &\text{if}~a = 0\\
        \delta_{x_{1, \star}} &\text{if}~a = 1\\
        \mu_1 &\text{if}~a = \mathfrak{a}_{1}\\
        \delta_{\bar{x}_{1\to h'}} &\text{if}~a = \mathfrak{a}_{h'}, h' \ge 2.\\
    \end{cases}
\end{align*}
For $h \ge 1$, we have:
\begin{itemize}
    \item \textit{Boring States:} At the boring state $x_{h,i}$, we always transit to the corresponding boring state in the next layer $x_{h+1, i}$ regardless of the action.
    \item \textit{Highway States:} On the highway state $\bar{x}_{h \to h+1}$, we transit to $\mu_{h+1}$ regardless of the action. On highway states $\bar{x}_{h \to h'}$ for $h' > h+1$ we transit to $\bar{x}_{h+1, h'}$ regardless of the action.
    \item \textit{Special States:} We have
    \begin{align*}
        P(\cdot \mid x_{h, \diamond}, a) = \begin{cases}
            \delta_{x_{h+1, H-h}} &\text{if}~a= 0\\
            \delta_{x_{h+1, \star}} &\text{if}~a = 1,
        \end{cases} \quad \text{and} \quad P(\cdot \mid x_{h, \star}, a) = \begin{cases}
            \delta_{x_{h+1, 0}} &\text{if}~a= 0\\
            \delta_{x_{h+1, H-h+1}} &\text{if}~a = 1.
        \end{cases}
    \end{align*}
\end{itemize}

\paragraph{Rewards.} All the rewards are at layer $H$:
\begin{align*}
    R(x_{H, \star}, 0) = 0, \quad R(x_{H, \star}, 1) = \cpush \epsstat, \quad \text{and} \quad \forall~i \in \crl{0, 1, \cdots, H}:~ R(x_{H, i}, \cdot) = 
    \frac{\cpush^{i} p^{i-1}}{i} \epsstat,
\end{align*}

\paragraph{Properties of the Construction.} Now we list several properties of the construction which are more or less immediate to verify.
\begin{enumerate}[(1)]
    \item The state space is of size $O(H^2)$, the action space is of size $O(H)$, and the policy class is of size $(H+2) \cdot 2^{H}$.
    \item Due to the transitions for the highway states, the distribution $\mu$ is admissible at all layers $h \ge 0$.
    \item The minimum probability that $\mu_h$ places on any state $x \in \statesp_h$ is at least $\min\crl{1/\cpush, p} \ge 1/\cpush$, so therefore pushforward concentrability is satisfied with parameter $\cpush$.
    \item The optimal policy is the all-ones policy, $\optpi_h \equiv 1$ for all $h \ge 0$.  Therefore $\Pi$ is realizable.
\end{enumerate}

\subsubsection{Analysis of \psdp{}}
We will show inductively that \psdp{} returns the all-zeros policy $\estpi_h \equiv 0$ for all $h \in [H]$.
\begin{itemize}
    \item At layer $H$, the only state for which the value of taking $a_H=0$ and $a_H=1$ differ is on $x_{H, \star}$, which is sampled under $\mu_H$ with probability $1/\cpush$. The gap between values is $\En_{x \sim \mu_H} \brk*{r(x,1) - r(x,0)} = \epsstat$, so we set $\mathsf{CB}_{\epsstat}$ to return $\estpi_H \equiv 0$.
    \item At layer $H-1$, the two states for which there is a gap in value are the special states $x_{H-1, \diamond}$ and $x_{H-1, \star}$. We can compute that
    \begin{align*}
        \En_{x \sim \mu_{H-1}} \brk*{Q^{\estpi_H}(x, 0) - Q^{\estpi_H}(x, 1)} \ge p \cdot \cpush \epsstat - \frac{1}{\cpush} \cdot \frac{\cpush^2 p\epsstat}{2} = \frac{\cpush p \epsstat}{2} > \epsstat.
    \end{align*}
    Here we use the fact that $\mu_{H-1}(x_{H-1, \diamond}) \ge p > 1/2$, as well as the assumed lower bound on $\cpush$. Therefore, $\mathsf{CB}_{\epsstat}$ must return $\estpi_{H-1} \equiv 0$.
    \item Suppose we are at layer $h$ and for all $h' > h$ \psdp{} selects $\estpi_{h'} \equiv 0$. Then the gap in value between action 0 and action 1 is
    \begin{align*}
        \En_{x \sim \mu_{h}} \brk*{Q^{\estpi_{h+1:H}}(x, 0) - Q^{\estpi_{h+1:H}}(x, 1)} &\ge p \cdot \frac{\cpush^{H-h} p^{H-h-1} \epsstat}{H-h} - \frac{1}{\cpush} \cdot \frac{\cpush^{H-h+1} p^{H-h} \epsstat}{H-h+1} \\
        &= \frac{\cpush^{H-h} p^{H-h} \epsstat}{(H-h)(H-h+1)} > \epsstat.
    \end{align*}
    The last equality uses the fact that $\cpush p \ge 5H/2$.
    \item Continuing this way, we can see that for all $h \ge 1$, \psdp{} equipped with $\mathsf{CB}_{\epsstat}$ selects $\estpi_h \equiv 0$. We can calculate that:
    \begin{align*}
        Q^{\estpi_{1:H}}(x_\circ, a) \begin{cases}
            = 0 &\text{if}~a=1\\[1em]
            = \frac{\cpush^{H-1}p^{H-2}}{H} \epsstat &\text{if}~a=0\\[1em]
            \le \prn*{ \frac{\cpush^{H-1}p^{H-2}}{H} + \frac{\cpush^{H-1}p^{H-1}}{H-1} } \epsstat \le \frac{2\cpush^{H-1}p^{H-1}}{H-1} \epsstat.  &\text{if}~a = \mathfrak{a_i}~ \text{for any}~i \in [H]
        \end{cases}
    \end{align*}
    For the last case, we use the rough estimate that $\mu_h$ places $1/\cpush$ mass on $x_{h,H}$ and the rest elsewhere.
\end{itemize}
Plugging in the optimal value $V^\star$ we have that the suboptimality of \psdp{} is at least
\begin{align*}
    V^\star - V^{\estpi} \ge \prn*{\frac{\cpush^H p^{H-1}}{H} - \frac{2\cpush^{H-1} p^{H-1}}{H-1}} \epsstat = (\cpush)^{\Omega(H)} \epsstat.
\end{align*}
This concludes the proof of \pref{thm:psdp-lower-bound}. \qed

\section{Agnostic Policy Learning is Impossible for $\mu$-Resets}\label{sec:mu-reset-lb}

In this section, we ask if it is possible to remove representational assumptions such as completeness or realizability altogether. Our main result is \pref{thm:lower-bound-policy-completeness}, which shows that one cannot achieve sample-efficient agnostic policy learning under $\mu$-reset access. 

\begin{theorem}\label{thm:lower-bound-policy-completeness}
For any $H \in \bbN$, there exists a policy class $\Pi$ of size $2^H$, a family of MDPs $\cM$ over a state space of size $2^{O(H)}$, binary action space, horizon $H$, and a reset distribution $\mu$  satisfying $\cconc(\mu; \Pi, M) = 6$ for all $M \in \cM$, so that any proper deterministic algorithm that returns a $1/16$-optimal policy must use at least $2^{\Omega(H)}$ samples from $\mu$-reset access for some MDP in $\cM$.
\end{theorem}

The proof is given in \pref{sec:lower-bound-policy-completeness}. 
\begin{figure}[ht]
    \centering
\includegraphics[scale=0.32, trim={0cm 15cm 30cm 0cm}, clip]{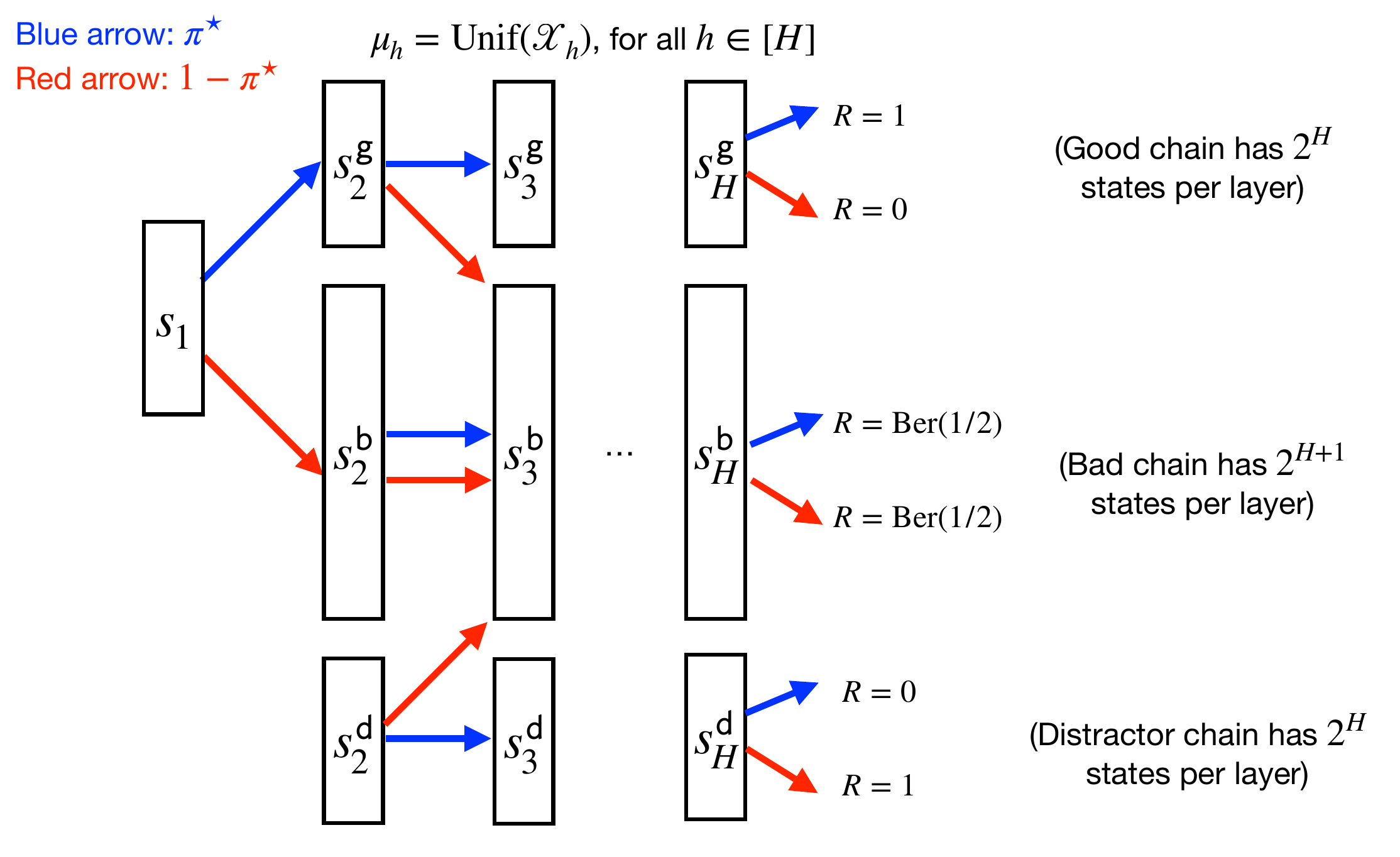}
    \caption{Construction used for \pref{thm:lower-bound-policy-completeness}.} 
    \label{fig:lb2}
\end{figure}

\paragraph{Key Ideas for \pref{thm:lower-bound-policy-completeness}.} An example can be found in \pref{fig:lb2}. Similar to the construction for \pref{thm:lower-bound-coverability}, we use the rich observation combination lock. There is a slight twist, however. We introduce a set of \emph{distractor} latent states $\crl{s_h^\mathsf{d}}_{h \ge 2}$, which are not reachable from the initial distribution $d_1$, and we set $\mu_h$ to be the uniform distribution over all observations in layer $h$. Thus, the exploratory distribution $\mu$ has \emph{overcoverage} over these unreachable states. The distractor states have the same latent transitions as the good states, and the only difference is that the reward at $s_H^\mathsf{d}$ is flipped compared to the reward at $s_H^\mathsf{g}$. This causes rollouts from $\mu_h$ to be noninformative. As for some rough intuition, observe that the distribution of rewards for executing \emph{any} open-loop policy $\pi_{h:H}$ from $\mu_h$ with $h\ge 2$ is $\ber(1/2)$. This is shown by the following casework:
\begin{itemize}
    \item If $\pi_{h:H} = \pi^\star_{h:H}$, then we get a reward of 1 by either sampling $x \sim \emission(s_h^\mathsf{g})$ with probability $1/4$ and getting reward 1 at $s_H^\mathsf{g}$ or sampling $x \sim \emission(s_h^\mathsf{b})$ with probability $1/2$ and getting reward $\ber\prn*{\frac12}$ at $s_h^\mathsf{b}$. Thus the distribution is $\ber\prn*{\frac12}$. 
    \item Similar reasoning holds if $\pi_{h:H} = \pi^\star_{h:H-1} \circ (1-\pi^\star_H)$, but with the reward of 1 coming from sampling the states $x \sim \emission(s_h^\mathsf{d})$.
    \item If $\pi_{h:H}$ is any other policy, then it always reaches $s_H^\mathsf{b}$ and it gets reward $\ber\prn*{\frac12}$.
\end{itemize}
Therefore, observing the reward distribution obtained by executing open-loop policies reveals \emph{no information} about $\optpi$; due to the rich observations, executing non-open loop policies does not really help, and the learner cannot really learn any information about the transition dynamics from the reset $\mu$. Again, the best the learner can do is online RL which requires $2^{\Omega(H)}$ samples. 

\emph{Remark.} If the learner had local simulator access, then it could easily decode states starting from layer $H$, going backwards, since the reward distributions for a particular $(x_H,a_H)$ pair are different depending on the latent state $\optdec(x_H)$. This idea is precisely the intuition that motivates our algorithm for the hybrid resets interaction protocol in \pref{chap:hybrid-resets}.

\subsection{Proof of \pref{thm:lower-bound-policy-completeness}}\label{sec:lower-bound-policy-completeness}


\paragraph{Lower Bound Construction.} We define a family of Block MDPs $\cM = \crl{M_{\optpi, \optdec}}_{\optpi \in \Pi, \optdec \in \Phi}$ which are parameterized by an optimal policy $\optpi \in \Pi$ and a decoding function $\optdec \in \Phi$ (to be described). 
\begin{itemize}
    \item \textbf{Policy Class:} Again, the policy class $\Pi$ is taken to be open loop policies:
\begin{align*}
    \Pi \coloneqq \crl{\pi: \forall x \in \statesp_h, \pi_h(x) \equiv a_h, (a_1, \cdots, a_H) \in \actionsp^H}.
\end{align*}
    \item \textbf{Latent MDP:} The latent state space $\latentsp$ is layered where each $\latentsp_h \coloneqq \crl{ s_h^\sgood, s_h^\sbad, s_h^\sdis}$ is comprised of a good, a bad, and a distractor state. We abbreviate the state as $\crl{\sgood, \sbad, \sdis}$ if the layer $h$ is clear from context. The starting state is always $\sgood$. The action space $\actionsp = \crl{0,1}$. Let $\optpi \in \Pi$ be any policy, which can be represented by a vector in $(\pi^\star_1, \cdots, \pi^\star_H) \in \crl{0,1}^H$. The latent transitions/rewards of an MDP parameterized by $\optpi \in \Pi$ are as follows for every $h \in [H]$: 
    \begin{align*}
        \optlatp(\cdot  \mid  s, a) = \begin{cases}
            \delta_{s_{h+1}^\sgood} & \text{if}~s = s_h^\sgood, a = \pi^\star_h\\
            \delta_{s_{h+1}^\sdis} &\text{if}~s = s_h^\sdis, a = \pi^\star_h\\
            \delta_{s_{h+1}^\sbad}& \text{otherwise}.
        \end{cases} ~~\text{and}~~
        \optlatr(s,a) = \begin{cases}
            1 &\text{if}~s = s_H^\sgood, a = \pi^\star_H\\
            1 &\text{if}~s = s_H^\sdis, a \ne \pi^\star_H\\
            \ber\prn*{\frac12} &\text{if}~s = s_H^\sbad\\
            0 &\text{otherwise.}
        \end{cases}
    \end{align*} 
    \item \textbf{Rich Observations:} The observation state space $\statesp$ is layered where each $\cX_h \coloneqq \crl{x_h^{(1)}, \cdots, x_h^{(m)}}$ with $m = 2^{H+2}$. The decoding function class $\Phi$ is the collection of all decoders which for every $h \ge 2$ assigns $s_h^\sgood, s_h^\sdis$ to disjoint subsets of $\statesp_h$ of size $2^H$ and $s_h^\sbad$ to the rest:
    \begin{align*}
        \Phi \coloneqq \Big\{\optdec:~\statesp \mapsto \latentsp :~&\forall~x_1 \in \statesp_1,~\optdec(x_1) = \sgood, \\
                                                                   &\hspace{-5em}\forall~h \ge 2,~ \abs*{ \crl*{x_h \in \statesp_h: \optdec(x_h) = \sgood }} = 2^H \text{ and } \abs*{ \crl*{x_h \in \statesp_h: \optdec(x_h) = \sdis }} = 2^H \Big\} ,\\
        &\hspace{-8em}\text{so that}~\abs{\Phi} = \prn*{\binom{2^{H+2}}{2^H} \cdot \binom{2^{H+2} - 2^H}{ 2^H} }^{H-1} = 2^{2^{\wt{O}(H)}}.
    \end{align*}
    In the MDP parameterized by $\optdec \in \Phi$, the emission for every $s \in \latentsp$ is $\emission(s) = \unif\prn*{\crl*{x \in \statesp_h: \optdec(x) = s}}$.
    \item \textbf{Exploratory Distribution:} The exploratory distribution $\mu = \crl{\mu_h}_{h \in [H]}$ is set to be $\mu_h = \unif\prn{\statesp_h}$.
\end{itemize}
We establish several facts about any $M_{\optpi, \optdec} \in \cM$ defined by the construction.
\begin{itemize}
    \item The distribution $\mu$ has bounded concentrability: $\cconc(\mu; \Pi, M) \le 4$.
    \item The policy class $\Pi$ does not satisfy realizability, since the optimal policy at layer $H$ requires one to take different actions depending on whether the latent state is $\sgood$ or $\sdis$.
\end{itemize}

\paragraph{Notation for Algorithm.} Similar to the proof of \pref{thm:lower-bound-coverability}, we use $\alg$ to denote a deterministic algorithm that collects $T$ samples, i.e., full-length episodes under $\mu$-reset access. For any $t \in [T]$ we define $\cF_{t-1}$ to be the sigma-field of everything observed in the first $t-1$ episodes. We further define for any $h \in [H]$ the filtration $\cF_{t,h-1}$ to be the sigma-field of everything observed in the first $t-1$ episodes as well as the first $h-1$ steps of the $t$-th sample. For the $\mu$-reset model, $\cF_{t,h-1} = \sigma( \cF_{t-1}, \crl{\prn{X_{t,i}, A_{t,i}, R_{t,i}}}_{h_\bot \le i \le h-1} )$. Here, $h_\bot$ is the starting layer of episode $t$, which is measurable with respect to $\cF_{t-1}$; furthermore, the action $A_{t,h}$ is measurable with respect to $\cF_{t,h-1} \cup \crl{X_{t,h}}$ (since $\alg$ is deterministic).

\paragraph{Sample Complexity Lower Bound.} We will use \pref{thm:interactive-lecam-cvx} to prove our lower bound. First we need to instantiate the parameter space. We will let $\Theta \coloneqq \crl{(\optpi, \optdec):~\optpi \in \Pi, \optdec \in \Phi}$ so that $\cM = \crl{M_\theta}_{\theta \in \Theta} = \crl{M_{\optpi, \optdec}}_{\optpi \in \Pi, \optdec \in \Phi}$. We further denote the subsets 
\begin{align*}
    \Theta_0 &\coloneqq \crl{(\optpi, \optdec):~\optpi \in \Pi~\text{s.t.}~\optpi_H = 0, \optdec \in \Phi} \\
    \Theta_1 &\coloneqq \crl{(\optpi, \optdec):~\optpi \in \Pi~\text{s.t.}~\optpi_H = 1, \optdec \in \Phi}
\end{align*}
The observation space $\cY$ is defined as the set of observations over $T$ rounds as well as returned policy for an algorithm interacting with the MDP, i.e.,
\begin{align*}
    \cY \coloneqq (\statesp \times \actionsp \times [0,1])^{HT} \times \Pi.
\end{align*}
(As a convention, we can assume that each sample collected by $\alg$ in the MDP is of length $H$; if $\alg$ decides to rollout from $\mu_h$ at an intermediate layer $h \ge 2$ then we can simply append ``dummy states'' to the prefix of the trajectory, which does not change the analysis.)

For an observation $y \in \cY$ we define the final returned policy as $y^\pi$. The loss function is given by
\begin{align*}
    L((\optpi, \phi), y) \coloneqq \ind{\optpi \ne y^\pi}. 
\end{align*}
Then we have for any $y \in \cY$, $(\optpi_0, \phi_0) \in \Theta_0$, and $(\optpi_1, \phi_1) \in \Theta_1$ that
\begin{align*}
    L((\optpi_0, \phi_0), y) +  L((\optpi_1, \phi_1), y) \ge 1 \coloneqq 2\Delta,
\end{align*}
since the last bit of $y^\pi$ can be either 0 or 1, thus only matching exactly one of $\optpi_0$ and $\optpi_1$.

Now we are ready to apply \pref{thm:interactive-lecam-cvx}. We get that for any $\alg$, we must have
\begin{align*}
    \sup_{(\optpi, \optdec) \in \Pi \times \Phi} \En_{Y \sim \Pr^{M_{\optpi, \optdec}, \alg}} \brk*{V^\star - V^{\estpi}} &= \sup_{(\optpi, \optdec) \in \Pi \times \Phi} \En_{Y \sim \Pr^{M_{\optpi, \optdec}, \alg}} \brk*{\frac{1}{2} - \frac{1}{2} \ind{\optpi = Y^\pi} } \\
    &= \frac{1}{2} \cdot \sup_{(\optpi, \optdec) \in \Pi \times \Phi} \En_{Y \sim \Pr^{M_{\optpi, \optdec}, \alg}} \brk*{ L((\optpi, \phi), Y) } \\
    &\ge \frac{1}{8} \cdot \max_{\nu_0 \in \Delta(\Theta_0), \nu_1 \in \Delta(\Theta_1)} \prn*{1 - \dtv\prn*{\Pr^{\nu_0, \alg}, \Pr^{\nu_1, \alg}}} \\
    &\ge \frac{1}{8} \cdot \prn*{1 - \dtv\prn*{\Pr^{\unif(\Theta_0), \alg}, \Pr^{\unif(\Theta_1), \alg}}}. 
\end{align*}
It remains to compute an upper bound $\dtv\prn*{\Pr^{\unif(\Theta_0), \alg}, \Pr^{\unif(\Theta_1), \alg}}$ which holds for any $\alg$. This is accomplished by the following lemma.

\begin{lemma}\label{lem:tv-bound-mu-reset}
For any deterministic $\alg$ that adaptively collects $T = 2^{O(H)}$ samples via $\mu$-reset access, we have
\begin{align*}
    \dtv\prn*{\Pr^{\unif(\Theta_0), \alg}, \Pr^{\unif(\Theta_1), \alg}} \le \frac{T^4H}{2^{H-10}}.
\end{align*}
\end{lemma}
Plugging in \pref{lem:tv-bound-mu-reset}, we conclude that for any $\alg$ that collects $2^{cH}$ samples for sufficiently small constant $c > 0$ must be $1/16$-suboptimal in expectation. This concludes the proof of \pref{thm:lower-bound-policy-completeness}.\qed

\subsection{Proof of \pref{lem:tv-bound-mu-reset} (TV Distance Calculation for \pref{thm:lower-bound-policy-completeness})}

Since the proof is similar to that of \pref{lem:tv-bound-generative} we omit some intermediate calculations. In the rest of the proof we denote $\nu_0 \coloneqq \unif(\Theta_0)$ and $\nu_1 \coloneqq \unif(\Theta_1)$. We have
\begin{align*}
     \hspace{2em}&\hspace{-2em} \dtv\prn*{\Pr^{\nu_0, \alg}, \Pr^{\nu_1, \alg}} \\
     &= \underbrace{ \sum_{t=1}^T \sum_{h=1}^{H-1} \En^{\nu_0, \alg} \brk*{\dtv \prn*{\Pr^{\nu_0, \alg} \brk*{ X_{t,h+1} \mid X_{t,h}, A_{t,h}, \cF_{t,h-1} }, \Pr^{\nu_1, \alg} \brk*{ X_{t,h+1} \mid X_{t,h}, A_{t,h}, \cF_{t,h-1} }}} }_{\text{transition TV distance}} \\
     &\qquad + 
     \underbrace{ \sum_{t=1}^T \En^{\nu_0, \alg} \brk*{\dtv \prn*{\Pr^{\nu_0, \alg} \brk*{ R_{t,H} \mid X_{t,H}, A_{t,H}, \cF_{t,H-1} }, \Pr^{\nu_1, \alg} \brk*{ R_{t,H} \mid X_{t,H}, A_{t,H}, \cF_{t,H-1} }}} }_{\text{reward TV distance}}.
\end{align*}

We bound each term separately.

\paragraph{Transition TV Distance.} Using triangle inequality and \pref{lem:tv-distance-from-uniform-resets} we get that for any $t \in [T], h \in [H-1]$:
\begin{align*}
 \En^{\nu_0, \alg} \brk*{\dtv \prn*{ \Pr^{\nu_0, \alg} \brk*{ X_{t,h+1} \mid X_{t,h}, A_{t,h}, \cF_{t, h-1} }, \Pr^{\nu_1, \alg} \brk*{X_{t,h+1} \mid X_{t,h}, A_{t,h}, \cF_{t, h-1} } } } \le \frac{t}{2^{H-3}}.  \numberthis \label{eq:ub-transition-calculation-reset}  
\end{align*}

\paragraph{Reward TV Distance.} Using triangle inequality, the fact that rewards are in $\crl{0,1}$, and \pref{lem:reward-tv-distance-policy-completeness} we get
\begin{align*}
    \hspace{2em}&\hspace{-2em} \En^{\nu_0, \alg} \brk*{\dtv \prn*{\Pr^{\nu_0, \alg} \brk*{ R_{t,H} \mid X_{t,H}, A_{t,H}, \cF_{t,H-1} }, \Pr^{\nu_1, \alg} \brk*{ R_{t,H} \mid X_{t,H}, A_{t,H}, \cF_{t,H-1} }}} \\
    &\le \En^{\nu_0, \alg} \brk*{\abs*{\Pr^{\nu_0, \alg} \brk*{ R_{t,H} = 1\mid X_{t,H}, A_{t,H}, \cF_{t, H-1}} - \frac{1}{2} } } \\
    &\qquad + \En^{\nu_0, \alg} \brk*{\abs*{\Pr^{\nu_1, \alg} \brk*{ R_{t,H} = 1 \mid X_{t,H}, A_{t,H}, \cF_{t,H-1}} - \frac{1}{2} } } \\
    &\le  t \cdot \frac{T^2H}{2^{H-9}}. \numberthis\label{eq:reward-triangle-inequality}
\end{align*}

\paragraph{Final Bound.} Thus, combining Eqs.~\eqref{eq:ub-transition-calculation-reset} and \eqref{eq:reward-triangle-inequality} we can conclude that:
\begin{align*}
    \dtv\prn*{\Pr^{\nu_0, \alg}, \Pr^{\nu_1, \alg}} \le \frac{T^2H}{2^{H-3}} + \frac{T^4H}{2^{H-9}} \le \frac{T^4H}{2^{H-10}}.
\end{align*}
This concludes the proof of \pref{lem:tv-bound-mu-reset}.\qed

\begin{lemma}[Transition TV Distance for the Construction in \pref{thm:lower-bound-policy-completeness}]\label{lem:tv-distance-from-uniform-resets}
    For any $t \in [T], h \in [H]$, we have
    \begin{align*}
        \nrm*{\Pr^{\nu_0, \alg}\brk*{X_{t,h}  \mid  \cF_{t, h-1} } - \unif(\statesp_{h}) }_1 &\le \frac{t}{2^{H-2}}, \\
        \nrm*{\Pr^{\nu_1, \alg}\brk*{X_{t,h}  \mid  \cF_{t, h-1}} - \unif(\statesp_{h})}_1 &\le \frac{t}{2^{H-2}}.
    \end{align*}
    \end{lemma}
    
    \begin{proof}[Proof of \pref{lem:tv-distance-from-uniform-resets}]
    We prove the bound for $\nu_0$, as the proof for $\nu_1$ is identical. If we sample $X_{t,h}$ directly from the $\mu$-reset distribution, then the result immediately follows since the distribution of $X_{t,h} = \unif(\cX_h)$. Otherwise, denote
    \begin{align*}
        \cF_{t, h-1}' = \sigma( \cF_{t, h-1}, \crl{\optdec(X): X \in \cF_{t, h-1}}, \crl{\ind{A = \optpi(X)}: (X,A) \in \cF_{t, h-1}} )
    \end{align*}
    to be the annotated sigma-field which also includes the latent state labels for all of the previous observions as well as whether the action taken followed $\optpi$ or not. Let us denote $\ell = \phi(X_{t,h}) \in \crl{\sgood, \sbad, \sdis}$ to be the latent state label of the next observation. Observe that the label $\ell$ is measurable with respect to $\cF_{t,h-1}'$ since the filtration $\cF_{t-1}'$ includes $\optdec(X_{t,h-1})$ as well as $\ind{A_{t,h-1} = \optpi(X_{t,h-1})}$. Furthermore denote $\cX_\mathsf{obs}$ to denote the total number of observations that we have encountered already in layer $h$ and $\cX_\mathsf{obs}^\ell$ to denote the observations we have encountered whose label is $\ell$.
    
    Under the uniform distribution over decoders, the assignment of the remaining observations is equally likely. Therefore we can write the distribution of $X_{t,h}$ as: 
    \begin{align*}
        \text{if}~\ell = \sgood:&\quad \Pr^{\unif(\Theta_0), \alg} \brk*{X_{t,h} = x  \mid  \cF_{t,h-1}'} = \begin{cases}
            \frac{1}{2^H} &\text{if}~x \in \cX_\mathsf{obs}^\ell \\
            0 &\text{if}~x \in \cX_\mathsf{obs} - \cX_\mathsf{obs}^\ell \\
            \frac{1}{2^H} \cdot \frac{2^H - \abs{\cX_\mathsf{obs}^\ell}}{2^{H+2} - \abs{\cX_\mathsf{obs}}} &\text{if}~x \in \cX_h - \cX_\mathsf{obs}
        \end{cases}\\
        \text{if}~\ell = \sbad:&\quad \Pr^{\unif(\Theta_0), \alg} \brk*{X_{t,h} = x  \mid  \cF_{t,h-1}'} = \begin{cases}
            \frac{1}{2^{H+1}} &\text{if}~x \in \cX_\mathsf{obs}^\ell \\
            0 &\text{if}~x \in \cX_\mathsf{obs} - \cX_\mathsf{obs}^\ell \\
            \frac{1}{2^{H+1}} \cdot \frac{2^{H+1} - \abs{\cX_\mathsf{obs}^\ell}}{2^{H+2} - \abs{\cX_\mathsf{obs}}} &\text{if}~x \in \cX_h - \cX_\mathsf{obs}
        \end{cases}\\
        \text{if}~\ell = \sdis:&\quad \Pr^{\unif(\Theta_0), \alg} \brk*{X_{t,h} = x  \mid  \cF_{t,h-1}'} = \begin{cases}
            \frac{1}{2^H} &\text{if}~x \in \cX_\mathsf{obs}^\ell \\
            0 &\text{if}~x \in \cX_\mathsf{obs} - \cX_\mathsf{obs}^\ell \\
            \frac{1}{2^H} \cdot \frac{2^H - \abs{\cX_\mathsf{obs}^\ell}}{2^{H+2} - \abs{\cX_\mathsf{obs}}} &\text{if}~x \in \cX_h - \cX_\mathsf{obs}
        \end{cases}
    \end{align*}
    We elaborate on the calculation for the last probability in each case. Suppose $\ell = \sgood$. Then for any $x \in \cX_h - \cX_\mathsf{obs}$ which has not been observed yet we assign $\optdec(x) = \ell$ in
    \begin{align*}
        &\binom{2^{H+2} - \abs{\cX_\mathsf{obs}} -1}{ 2^{H} -\abs{\cX_\mathsf{obs}^\ell} - 1 } \text{ ways out of } \binom{2^{H+2} - \abs{\cX_\mathsf{obs}} }{ 2^{H} -\abs{\cX_\mathsf{obs}^\ell } } \text{ assignments.} \\
        &\quad \Longrightarrow \optdec(x) = \sgood \text{ with probability } \frac{2^{H} -\abs{\cX_\mathsf{obs}^\ell}}{2^{H+2} - \abs{\cX_\mathsf{obs}}}.
    \end{align*}
    For each assignment where $\optdec(x) = \sgood$ we will select it with probability $1/2^H$ since the emission is uniform, giving us the final probability as claimed. A similar calculation can be done for the cases where $\ell = \sbad, \sdis$.
    
    Therefore we can calculate the final bound that
    \begin{align*}
        \hspace{2em}&\hspace{-2em}\nrm*{\Pr^{\nu_0, \alg}\brk*{X_{t,h}  \mid  \cF_{t, h-1}'} - \unif(\statesp_{h}) }_1 \\
                    &= \sum_{x \in \cX_h} \abs*{\Pr^{\nu_0, \alg}\brk*{X_{t,h} = x \mid  \cF_{t, h-1}} - \frac{1}{2^{H+2}}} \\
        &\le \begin{cases}
            \frac{\abs{\cX_\mathsf{obs}^\sgood}}{2^H} + \frac{\abs{\cX_\mathsf{obs}^\sbad} + \abs{\cX_\mathsf{obs}^\sdis}}{2^{H+2}} + \abs*{\frac{2^H - \abs{\cX_\mathsf{obs}^\sgood}}{2^{H}} - \frac{2^{H+2} - \abs{\cX_\mathsf{obs}}}{2^{H+2}}} &\text{if}~\ell = \sgood, \\[0.5em]
            \frac{\abs{\cX_\mathsf{obs}^\sbad}}{2^{H+1}} + \frac{\abs{\cX_\mathsf{obs}^\sgood} + \abs{\cX_\mathsf{obs}^\sdis}}{2^{H+2}} + \abs*{\frac{2^{H+1} - \abs{\cX_\mathsf{obs}^\sbad}}{2^{H+1}} - \frac{2^{H+2} - \abs{\cX_\mathsf{obs}}}{2^{H+2}}} &\text{if}~\ell = \sbad, \\[0.5em]
           \frac{\abs{\cX_\mathsf{obs}^\sdis}}{2^H} + \frac{\abs{\cX_\mathsf{obs}^\sbad} + \abs{\cX_\mathsf{obs}^\sgood}}{2^{H+2}} + \abs*{\frac{2^H - \abs{\cX_\mathsf{obs}^\sgood}}{2^{H}} - \frac{2^{H+2} - \abs{\cX_\mathsf{obs}}}{2^{H+2}}} &\text{if}~\ell = \sdis.
        \end{cases}\\
        &\le \frac{4 \cdot \abs{\cX_\mathsf{obs}}}{2^H} \le \frac{4 t}{2^H}.
    \end{align*}
    Since $\Pr^{\nu_0, \alg}\brk*{X_{t,h}  \mid  \cF_{t, h-1}} = \En^{\nu_0, \alg} \Pr^{\nu_0, \alg}\brk{X_{t,h}  \mid  \cF_{t, h-1}'}$, we have by convexity of $\ell_1$ norm and Jensen's inequality,
    \begin{align*}
        \nrm*{\Pr^{\nu_0, \alg}\brk*{X_{t,h}  \mid  \cF_{t, h-1}} - \unif(\statesp_{h}) }_1 \le  \En^{\nu_0, \alg} \brk*{ \nrm*{\Pr^{\nu_0, \alg}\brk*{X_{t,h}  \mid  \cF_{t, h-1}'} - \unif(\statesp_{h}) }_1 } \le \frac{4 t}{2^H},
    \end{align*}
    which concludes the proof of \pref{lem:tv-distance-from-uniform-resets}.
\end{proof}

\begin{lemma}[Reward TV Distance for the Construction in \pref{thm:lower-bound-policy-completeness}]\label{lem:reward-tv-distance-policy-completeness}
For any $t \in [T]$ we have
    \begin{align*}
        \En^{\nu_0, \alg} \brk*{\abs*{\Pr^{\nu_0, \alg} \brk*{ R_{t,H} = 1\mid X_{t,H}, A_{t,H}, \cF_{t, H-1}} - \frac{1}{2} } } &\le t \cdot \frac{T^2H}{2^{H-8}}, \\
        \En^{\nu_0, \alg} \brk*{\abs*{\Pr^{\nu_1, \alg} \brk*{ R_{t,H} = 1\mid X_{t,H}, A_{t,H}, \cF_{t, H-1}} - \frac{1}{2} } } &\le  t \cdot \frac{T^2H}{2^{H-8}}.
    \end{align*}
\end{lemma}

\begin{proof}[Proof of \pref{lem:reward-tv-distance-policy-completeness}] Let us denote $\cfnor \coloneqq \sigma(X_{t,H}, A_{t,H}, \cF_{t, H-1})$. 

We will prove the first inequality of \pref{lem:reward-tv-distance-policy-completeness}; the second inequality is obtained using similar arguments.

\paragraph{Peeling Off Bad Event.} First, let us peel off the event that $\cfnor$ has repeated observations: denoting $\eventnew \coloneqq \crl{X_{t,h} \notin \cF_{t, h-1} ~\forall t \in [T], h \in [H]}$, we have
\begin{align*}
    \hspace{2em}&\hspace{-2em}    \En^{\nu_0, \alg} \brk*{\abs*{\Pr^{\nu_0, \alg} \brk*{ R_{t,H} = 1\mid \cfnor } - \frac{1}{2} } } \\
                &\le \Pr^{\nu_0, \alg}[\eventnew^c] + \En^{\nu_0, \alg} \brk*{\ind{ \eventnew } \abs*{\Pr^{\nu_0, \alg} \brk*{ R_{t,H} = 1\mid \cfnor } - \frac{1}{2} } } \\
    &\le \frac{T^2 H}{2^H} + \En^{\nu_0, \alg} \brk*{ \ind{ \eventnew } \abs*{\Pr^{\nu_0, \alg} \brk*{ R_{t,H} = 1\mid \cfnor} - \frac{1}{2} } }, \numberthis\label{eq:peeled-eventnew}
\end{align*}
where the last inequality follows by an identical argument as \pref{lem:repeated-transitions}. Therefore, it suffices to bound the expectation only for the $\cfnor$ which have no repeated states. 

\paragraph{Inductive Claim.} Now we define the event $\cE_{R, t}$ to be the event that among the first $t$ episodes, $\alg$ never performs an online rollout (meaning it starts from layer 1) which follows $\optpi$, i.e.,
\begin{align*}
    \cE_{R,t} \coloneqq \crl*{ \forall t' \le t: A_{t, 1:H-1} \ne \optpi}.
\end{align*}
A subtle point is that unlike the reward TV distance calculation for \pref{thm:lower-bound-coverability}, the event $\cE_{R,t-1}$ is \emph{not measurable} with respect to $\cF_{t, H-1}$ (since there is still uncertainty as to what $\optpi$ is). This causes some technical complications in the proof. To remedy this, we can consider working with an augmented filtration which appends a special token $\top$ at the end of every online trajectory that $\alg$ takes if the sequence of actions $A_{t,1:H-1}$ matches $\optpi$; \emph{now} $\cE_{R, t-1}$ is measurable with respect to the augmented filtration (namely the event $\cE_{R,t-1}$ holds if the augmented contains no special tokens $\top$). This augmentation does not affect the overall argument, and for the rest of the proof we assume that $\cfnor$ has been augmented in this way. 

Central to our proof is the following claim that $\cE_{R,t} \cup \eventnew$ happens with high probability:
\begin{align*}
    \text{for all}~t \in [T]: \quad \Pr^{\nu_0, \alg} \brk*{ \cE_{R,t}^c \cup \eventnew} \le t \cdot \frac{T^2H}{2^{H-7}}. \numberthis \label{claim:inductive-online}
\end{align*}
Now we will establish \pref{claim:inductive-online} using an inductive argument. The base case of $t=0$ trivially holds. Now suppose that \pref{claim:inductive-online} holds at time $t-1$. Then
\begin{align*}
    \hspace{1em}&\hspace{-1em}    \Pr^{\nu_0, \alg} \brk*{ \cE_{R,t}^c \cup \eventnew} \\
                &\le \Pr^{\nu_0, \alg} \brk*{ \cE_{R,t-1}^c \cup \eventnew} + \En^{\nu_0, \alg} \brk*{ \ind{\cE_{R,t-1}\cup \eventnew} \Pr^{\nu_0, \alg} \brk*{A_{t, 1:H-1} = \optpi \mid \cfnor} } \\
    &\le (t-1) \cdot \frac{T^2H}{2^{H-7}}+ \En^{\nu_0, \alg} \brk*{\ind{\cE_{R,t-1}\cup \eventnew} \Pr^{\nu_0, \alg} \brk*{A_{t, 1:H-1} = \optpi \mid \cfnor} } \\
    &\le (t-1) \cdot \frac{T^2H}{2^{H-7}} + \En^{\nu_0, \alg} \brk*{ \ind{\cE_{R,t-1}\cup \eventnew} \sum_{\pi \in \Pi_{1:H-1}} \ind{A_{t, 1:H-1} = \pi} \Pr^{\nu_0, \alg} \brk*{ \optpi = \pi \mid \cfnor} } \\
    &\le (t-1) \cdot \frac{T^2H}{2^{H-7}} + \frac{1}{2^{H-1}} + \frac{T^2H}{2^{H-6}} \le t \cdot \frac{T^2H}{2^{H-7}},
\end{align*}
Here, the second-to-last inequality uses \pref{lem:posterior-of-optimal} and the fact that $A_{t, H-1}$ can only match a single policy $\pi \in \Pi_{1:H-1}$.

\paragraph{Casework on Reward TV Distance.} Armed with \pref{claim:inductive-online}, we now return to the proof of the reward TV distance calculation. We consider two cases. In the first case, the $t$-th trajectory is generated by an online rollout from $h=1$ with the sequence of actions $A_{1:H}$. In the second case, the $t$-th trajectory is generated by first querying the $\mu$-reset model starting from $h_\bot \ge 2$, then rolling out with the sequence of actions $A_{h_\bot:H}$.

\underline{\emph{Case 1: Online Rollout from Layer 1.}} First, we peel off the probability of $\cE_{R,t-1}$ occuring:
\begin{align*}
    \hspace{2em}&\hspace{-2em} \En^{\nu_0, \alg} \brk*{\ind{\eventnew} \abs*{\Pr^{\nu_0, \alg} \brk*{ R_{t,H} = 1\mid \cfnor} - \frac{1}{2} } } \\
    &\le \Pr^{\nu_0, \alg} \brk*{ \cE_{R,t-1}^c \cup \eventnew} + \En^{\nu_0, \alg} \brk*{\ind{\cE_{R,t-1}\cup \eventnew} \cdot \abs*{\Pr^{\nu_0, \alg} \brk*{ R_{t,H} = 1\mid \cfnor} - \frac{1}{2} } } \\
    &\le (t-1) \cdot \frac{T^2H}{2^{H-7}} + \En^{\nu_0, \alg} \brk*{\ind{\cE_{R,t-1}\cup \eventnew} \cdot \abs*{ \Pr^{\nu_0, \alg} \brk*{ R_{t,H} = 1\mid \cfnor} - \frac{1}{2} } }. 
    \numberthis \label{eq:reward-tv-induction-peel}
\end{align*}


Now we compute
\begin{align*}
    \hspace{2em}&\hspace{-2em} \ind{\cE_{R,t-1}\cup \eventnew} \Pr^{\nu_0, \alg} \brk*{ R_{t,H} = 1 \mid \cfnor } \\
    &= \ind{\cE_{R,t-1}\cup \eventnew} \En^{\nu_0, \alg} \brk*{ \ind{ \phi(X_{t,H}) = \sgood \wedge A_{t,H} = 0 } + \frac12 \ind{ \phi(X_{t,H}) = \sbad } \mid \cfnor  } \\
    &\overset{(i)}{=} \ind{\cE_{R,t-1}\cup \eventnew}  \prn*{ \Pr^{\nu_0, \alg} \brk*{ A_{t, 1:H} = \optpi \circ 0  \mid \cfnor  } + \frac12 \Pr^{\nu_0, \alg} \brk*{  A_{t,1:H-1} \ne \optpi  \mid \cfnor  } }\\
    &\overset{(ii)}{=} \sum_{\pi \in \Pi_{1:H-1}} \prn*{ \ind{A_{t, 1:H} = \pi \circ 0} + \frac{1}{2} \ind{A_{t, 1:H-1} \ne \pi} } \\
    &\qquad\qquad \times \ind{\cE_{R,t-1}\cup \eventnew} \Pr^{\nu_0, \alg}\brk*{ \optpi = \pi \mid \cfnor,  \cE_{R,t-1} } \\
    &\overset{(iii)}{\le} \frac{T^2H}{2^{H-6}}  + \frac{\ind{\cE_{R,t-1}\cup \eventnew}}{2^{H-1}} \sum_{\pi \in \Pi_{1:H-1}} \prn*{ \ind{A_{t, 1:H} = \pi \circ 0} + \frac{1}{2} \ind{A_{t, 1:H-1} \ne \pi} } \\
    &\overset{(iv)}{\le} \frac{T^2H}{2^{H-7}} + \frac{\ind{\cE_{R,t-1}\cup \eventnew}}{2}. \numberthis \label{eq:case1-bound}
\end{align*}
For equality $(i)$ we use the fact that if the $X_{t,H}$ has a good label then we must have taken $\optpi$ for the first $H-1$ layers. For equality $(ii)$ we use the fact that the indicators are measurable with respect to $\cF_{t, H-1}$ and $\crl{\optpi = \pi}$. For $(iii)$ we apply a change-of-measure argument using \pref{lem:posterior-of-optimal}. For $(iv)$ we use the fact that the sequence of actions $A_{t,1:H-1}$ can match at exactly one of the policies in $\Pi_{1:H-1}$. 

Note that the other side of the inequality can be shown analogously. Therefore by plugging in Eq.~\eqref{eq:case1-bound} into \eqref{eq:reward-tv-induction-peel} we get the bound that
\begin{align*}
\En^{\nu_0, \alg} \brk*{\ind{\cE_{R,t-1}\cup \eventnew} \abs*{\Pr^{\nu_0, \alg} \brk*{ R_{t,H} = 1\mid \cfnor} - \frac{1}{2} } } \le t \cdot \frac{T^2H}{2^{H-7}}. \numberthis\label{eq:case1-bound-onreward}
\end{align*}

\underline{\emph{Case 2: $\mu$-Reset Rollout from Layer $h_\bot \ge 2$.}} Let us analyze the second case. Using the construction details,
\begin{align*}
    &\Pr^{\nu_0, \alg} \brk*{ R_{t,H} = 1 \mid \cfnor} \\
    &= \En^{\nu_0, \alg} \brk*{ \ind{ \phi(X_{t,H}) = \sgood \wedge A_{t,H} = 0 } + \frac12 \ind{ \phi(X_{t,H}) = \sbad } + \ind{ \phi(X_{t,H}) = \sdis \wedge A_{t,H} = 1 } \mid \cfnor } \\
    &= \En^{\nu_0, \alg} \brk*{ \ind{ \phi(X_{t,h_\bot}) = \sgood \wedge A_{t, h_\bot:H} = \optpi \circ 0 } \mid \cfnor } \\
    &\qquad + \frac12 \En^{\nu_0, \alg} \brk*{ \ind{ A_{t,h_\bot:H-1} \ne \optpi } + \ind{\phi(X_{t,h_\bot}) = \sbad \wedge A_{t,h_\bot:H-1} = \optpi} \mid \cfnor }\\
    &\qquad + \En^{\nu_0, \alg} \brk*{ \ind{ \phi(X_{t,h_\bot}) = \sdis \wedge A_{t,h_\bot:H} = \optpi\circ 1 } \mid \cfnor } \\
    &= \sum_{\pi \in \Pi_{h_\bot:H-1}} \Pr^{\nu_0, \alg}\brk*{ \optpi = \pi \mid \cfnor } \bigg(  \ind{ A_{t, h_\bot:H} = \pi \circ 0 }
    \Pr^{\nu_0, \alg} \brk*{ \phi(X_{t,h_\bot}) = \sgood  \mid \cfnor, \optpi = \pi }  \\
    &\qquad +  \frac12 \cdot  \ind{ A_{t,h_\bot:H-1} \ne \pi } + \frac12 \cdot \ind{A_{t,h_\bot:H-1} = \pi} \Pr^{\nu_0, \alg} \brk*{  \phi(X_{t,h_\bot}) = \sbad  \mid \cfnor, \optpi = \pi } \\
    &\qquad +  \ind{A_{t,h_\bot:H} = \pi \circ 1} \Pr^{\nu_0, \alg} \brk*{ \phi(X_{t,h_\bot}) = \sdis \mid \cfnor, \optpi = \pi } \bigg).
\end{align*}
We apply \pref{lem:posterior-of-state-label-v2} separately to the terms inside the parentheses for every $\pi$. Then using a casework argument on the value of $A_{t, h_\bot:H}$ and then averaging over the posterior of $\optpi$ gives
\begin{align*}
    \ind{\eventnew} \abs*{ \Pr^{\nu_0, \alg} \brk*{ R_{t,H} = 1 \mid \cfnor} - \frac12 } \le \frac{TH}{2^{H-5}}. \numberthis\label{eq:case2-bound-onreward}
\end{align*}

\paragraph{Putting It Together.} To conclude, the worst-case TV distance is the maximum of the two bounds we have shown in Eqs.~\eqref{eq:case1-bound-onreward} and \eqref{eq:case2-bound-onreward}, so therefore plugging into Eq.~\eqref{eq:peeled-eventnew} we have
\begin{align*}
    \hspace{2em}&\hspace{-2em} \En^{\nu_0, \alg} \brk*{\abs*{\Pr^{\nu_0, \alg} \brk*{ R_{t,H} = 1\mid \cfnor} - \frac{1}{2} } } \le \frac{T^2 H}{2^H} + \max \crl*{ t \cdot \frac{T^2H}{2^{H-7}},  \frac{TH}{2^{H-5}} } \le t \cdot \frac{T^2H}{2^{H-8}}.
\end{align*}
The proof of second inequality is obtained similarly, as one just needs to change the law to be under $\Pr^{\nu_1, \alg}$ in the above argument. This concludes the proof of \pref{lem:reward-tv-distance-policy-completeness}.
\end{proof}

\begin{lemma}[Posterior of $\optpi$]\label{lem:posterior-of-optimal}
Fix any $t \in [T]$. Assume that $\cfnor$ contains no repeated states. Then 
\begin{align*}
    \ind{\cE_{R,t-1}} \cdot \nrm*{\Pr^{\nu_0, \alg} \brk*{\optpi = \cdot \mid \cfnor} - \unif(\Pi_{1:H-1})}_1 \le \frac{T^2H}{2^{H-6}}.
\end{align*}
\end{lemma}

\begin{proof}
In what follows, all of the probabilities $\Pr[\cdot] \coloneqq \Pr^{\nu_0, \alg}[\cdot]$. Let $[x]_{+} \coloneqq \max\crl{x, 0}$. We can compute that
\begin{align*}
    \hspace{2em}&\hspace{-2em}   \ind{\cE_{R,t-1}} \cdot \nrm*{\Pr\brk*{\pi^\star = \cdot \mid \cfnor } - \unif(\Pi_{1:H-1})}_1 \\
                &= \ind{\cE_{R,t-1}} \cdot \sum_{\pi \in \Pi_{1:H-1}} \abs*{ \Pr\brk*{\pi^\star = \pi \mid \cfnor } - \frac{1}{2^{H-1}} }\\
    &= \ind{\cE_{R,t-1}} \cdot 2\sum_{\pi \in \Pi_{1:H-1}} \brk*{ \Pr\brk*{\pi^\star = \pi \mid \cfnor } - \frac{1}{2^{H-1}} }_{+} \\
    &= \ind{\cE_{R,t-1}} \cdot \frac{2}{2^{H-1}} \sum_{\pi \in \Pi_{1:H-1}} \brk*{ \frac{ \Pr\brk*{\cfnor \mid \pi^\star = \pi } }{  \Pr\brk*{  \cfnor } } - 1 }_{+}\\
    &\le 2 \max_{\pi \in \Pi_{1:H-1}} \brk*{ \frac{\ind{\cE_{R,t-1}} \cdot \Pr\brk*{\cfnor \mid \pi^\star = \pi } }{  \Pr\brk*{  \cfnor } } - 1 }_{+}. \numberthis \label{eq:upper-bound-on-posterior-tv}
\end{align*}
We now proceed by explicitly calculating the conditional distribution of $\cfnor$ for every choice of optimal policy $\pi \in \Pi_{1:H-1}$. We will show that regardless of the choice $\pi \in \Pi_{1:H-1}$, the conditional distribution looks roughly like the uniform distribution over observations with a $\ber(1/2)$ reward at the end of every trajectory.

First, we will break up the distribution into trajectories:
\begin{align*}
    \Pr\brk*{\cfnor \mid \pi^\star = \pi } = \prn*{ \prod_{i < t} \Pr \brk*{ \tau_i \mid \pi^\star = \pi , \cF_{i-1} } } \cdot \Pr \brk*{ (X_{t,h_\bot:H}, A_{t,h_\bot:H}) \mid \pi^\star = \pi , \cF_{t-1}}. \numberthis\label{eq:prob-filtration-given-pi}
\end{align*}

\begin{claim}\label{claim:claim1}
Fix any $i \in [t]$. If $\tau_i$ is generated by sampling the $\mu$-reset distribution at some layer $h_\bot \ge 2$ and then rolling out, we have for every $\pi \in \Pi_{1:H-1}$,
\begin{align*}
    \Pr \brk*{ \tau_i \mid \pi^\star = \pi , \cF_{i-1} }  \in \frac{1}{2} \cdot \frac{1}{2^{(H+2) \cdot (H - h_\bot + 1)}} \cdot \brk*{ \prn*{1 - \frac{T}{2^H}}^{H-1}, \prn*{1 + \frac{T}{2^H}}^{H-1} }
\end{align*} 
\end{claim}

We start by showing \pref{claim:claim1}. In our proof, we assume that $h_\bot = 2$ (the proof is easy to adapt to any $h_\bot \ge 2$ with minor modification). Fix any index $i \in [t]$ and let $\tau_i = (X_{2:H}, A_{2:H}, R)$ be the $i$-th trajectory. Also fix any $\pi \in \Pi_{1:H-1}$. Then we can calculate
\begin{align*}
    \hspace{2em}&\hspace{-2em} \Pr\brk*{ (X_{2:H}, A_{2:H}, R) \mid \pi^\star = \pi , \cF_{i-1}} \\
    &= \sum_{\phi_2} \Pr \brk*{ (X_{2:H}, A_{2:H}, R) \mid \pi^\star = \pi , \cF_{i-1}, \phi_2} \Pr\brk*{\phi_2 \mid \pi^\star= \pi , \cF_{i-1}} \\
    &= \sum_{\ell \in \crl{\sgood, \sbad, \sdis}}\sum_{\phi_2: \phi_2(X_2) = \ell} \Pr \brk*{ (X_{2:H}, A_{2:H}, R) \mid \pi^\star = \pi , \cF_{i-1}, \phi_2} \Pr\brk*{\phi_2 \mid \pi^\star= \pi , \cF_{i-1}}.
\end{align*}
We can separately analyze the sum for the different choices of the label of the initial state $X_2$. First, we do the case where $\phi_2(X_2) = \sbad$:
\begin{align*}
    \hspace{2em}&\hspace{-2em} \sum_{\phi_2: \phi_2(X_2) = \sbad} \Pr  \brk*{ (X_{2:H}, A_{2:H}, R) \mid \pi^\star = \pi , \cF_{i-1}, \phi_2} \Pr\brk*{\phi_2 \mid \pi^\star= \pi , \cF_{i-1}} \\
    &= \sum_{\phi_2: \phi_2(X_2) = \sbad} \Pr \brk*{ (X_2, A_2) \mid \pi^\star = \pi , \cF_{i-1}, \phi_2} \\
    &\qquad\qquad  \times \Pr\brk*{ (X_{3:H}, A_{3:H}, R) \mid \pi^\star = \pi , \cF_{i-1}, \phi_2, X_2, A_2} \Pr\brk*{\phi_2 \mid \pi^\star= \pi , \cF_{i-1}} \\
    &\overset{(i)}{=} \frac{1}{2^{H+2}} \sum_{\phi_2: \phi_2(X_2) = \sbad} \Pr\brk*{ (X_{3:H}, A_{3:H}, R) \mid \pi^\star = \pi , \cF_{i-1}, \phi_2, X_2, A_2} \Pr\brk*{\phi_2 \mid \pi^\star= \pi , \cF_{i-1}}\\
    &\overset{(ii)}{=} \frac{1}{2^{H+2}} \sum_{\phi_2: \phi_2(X_2) = \sbad} \Pr\brk*{ (X_{3:H}, A_{3:H}, R) \mid \pi^\star = \pi , \cF_{i-1}, \phi_2(X_2) = \sbad} \Pr\brk*{\phi_2 \mid \pi^\star= \pi , \cF_{i-1}} \\
    &= \frac{ \Pr\brk*{\phi_2(X_2) = \sbad \mid \pi^\star= \pi , \cF_{i-1}} }{2^{H+2}} \Pr\brk*{ (X_{3:H}, A_{3:H}, R) \mid \pi^\star = \pi , \cF_{i-1}, \phi_2(X_2) = \sbad} \\
    &\qquad \vdots \\[0.5em]
    &\overset{(iii)}{=} \frac{ \Pr\brk*{\phi_2(X_2) = \sbad \mid \pi^\star= \pi , \cF_{i-1}} }{2^{H+2}} \times \frac{\Pr \brk*{  \phi_3(X_3) = \sbad \mid \pi^\star = \pi , \cF_{i-1}, \phi_2(X_2) = \sbad}}{2^{H+1}} \\
    &\qquad \dots \times  \frac{\Pr\brk*{\phi_H(X_H) = \sbad \mid \pi^\star= \pi , \cF_{i-1}, \crl*{\phi_h(X_h)= \sbad, 2 \le h \le H-1} } }{2^{H+1}} \\
    &\qquad\qquad \times \Pr\brk*{ R \mid \pi^\star = \pi , \cF_{i-1}, \crl*{\phi_h(X_h) = \sbad, 2 \le h \le H}} \\[0.5em]
    &= \frac{ \Pr\brk*{\phi_2(X_2) = \sbad \mid \pi^\star= \pi , \cF_{i-1}} }{2^{H+2}} \times \frac{\Pr \brk*{  \phi_3(X_3) = \sbad \mid \pi^\star = \pi , \cF_{i-1}, \phi_2(X_2) = \sbad}}{2^{H+1}} \\
    &\qquad \dots \times  \frac{\Pr\brk*{\phi_H(X_H) = \sbad \mid \pi^\star= \pi , \cF_{i-1}, \crl*{\phi_h(X_h)= \sbad, 2 \le h \le H-1} } }{2^{H+1}}\times \frac12.
\end{align*}
The equality $(i)$ follows because the first state is chosen $\unif(\cX_2)$ and the action $A_2$ is the one selected by $\alg$ via the policy $\pi^{(i)}$ which is measurable with respect to $\cF_{i-1}$. The equality $(ii)$ follows because the distribution over the next state is only determined by $\cF_{i-1}$ (which includes some information about the decoder $\phi_3$) and the labeling $\phi_2(X_2) = \sbad$. Equality $(iii)$ follows by applying chain rule over and over, noting that since $\phi_2(X_2) = \sbad$ it must be the case that $\phi_h(X_h) = \sbad$ for all $h > 2$, and therefore the probability of observing any given observation with a bad label is $1/2^{H+1}$.

Now we apply the posterior state label calculation of \pref{lem:posterior-of-state-label} (using the fact that $\cfnor$ contains no repeated states) to each term in the previous display to get that:
\begin{align*}
    \hspace{2em}&\hspace{-2em} \sum_{\phi_2: \phi_2(X_2) = \sbad} \Pr  \brk*{ (X_{2:H}, A_{2:H}, R) \mid \pi^\star = \pi , \cF_{i-1}, \phi_2} \Pr\brk*{\phi_2 \mid \pi^\star= \pi , \cF_{i-1}} \\
    &\in \frac14 \cdot \prn*{ \frac{1}{2^{H+2}} }^{H-1} \cdot \brk*{ \prn*{1 - \frac{T}{2^H}}^{H-1}, \prn*{1 + \frac{T}{2^H}}^{H-1} }. \numberthis\label{eq:case-bad}
\end{align*}
Next, we consider other terms in the sum. To bound the quantity
\begin{align*}
    \sum_{\phi_2: \phi_2(X_2) = \sgood} \Pr  \brk*{ (X_{2:H}, A_{2:H}, R) \mid \pi^\star = \pi ,\cF_{i-1}, \phi_2} \Pr\brk*{\phi_2 \mid \pi^\star= \pi , \cF_{i-1}},
\end{align*}
a bit more care is required. In this case, we start off in the good latent state, and depending on whether the sequence of actions $A_{2:H}$ is equal to $\optpi$ we transit to the bad latent state. Let us denote $\overline{h}\ge 2$ denote the first layer at which $a_{\overline{h}}$ deviates from $\pi_{\overline{h}}$. Then using similar reasoning we have
\begin{align*}
   &\sum_{\phi_2: \phi_2(X_2) = \sgood} \Pr  \brk*{ (X_{2:H}, A_{2:H}, R) \mid \pi^\star = \pi , \cF_{i-1}, \phi_2} \Pr\brk*{\phi_2 \mid \pi^\star= \pi , \cF_{i-1}} \\
    &= \frac{ \Pr\brk*{\phi_2(X_2) = \sgood \mid \pi^\star= \pi , \cF_{i-1}} }{2^{H+2}} \times \cdots  \times \frac{\Pr \brk*{  \phi_{\overline{h}}(X_{\overline{h}}) = \sgood \mid \pi^\star = \pi , \cF_{i-1}, \crl*{\phi_{h}(X_{h}) = \sgood, ~ \forall h < \overline{h}}  } }{2^{H}} \\
    &\quad \times \frac{\Pr \brk*{  \phi_{\overline{h}+1}(X_{\overline{h}+1}) = \sbad \mid \pi^\star = \pi , \cF_{i-1}, \crl*{\phi_{h}(X_{h}) = \sgood, ~ \forall h \le \overline{h}}  } }{2^{H+1}} \\
    &\qquad \vdots \\
    &\quad \times \frac{\Pr\brk*{\phi_H(X_H) = \sbad \mid \pi^\star= \pi , \cF_{i-1},  \crl*{\phi_{h}(X_{h}) = \sgood, ~ \forall h \le \overline{h}}, \crl*{\phi_{h}(X_{h}) = \sbad, ~ \forall \overline{h}< h < H } } }{2^{H+1}} \\
    &\quad \times \Pr\brk*{ R \mid  \pi^\star= \pi , \cF_{i-1},  \crl*{\phi_{h}(X_{h}) = \sgood, ~ \forall h \le \overline{h}}, \crl*{\phi_{h}(X_{h}) = \sbad, ~ \forall \overline{h}< h < H } }.
\end{align*}
Note that the conditional reward distribution given the latent state labels is
\begin{align*}
    \Pr\brk*{ R = 1 \mid \cdots } = \ind{A_{2:H} = \pi \circ 0} + \frac{1}{2} \ind{A_{2:H-1} \ne \pi}.
\end{align*}
Again by applying the posterior calculation of \pref{lem:posterior-of-state-label} we get
\begin{align*}
   & \sum_{\phi_2: \phi_2(X_2) = \sgood} \Pr  \brk*{ (X_{2:H}, A_{2:H}, R) \mid \pi^\star = \pi , \cF_{i-1}, \phi_2} \Pr\brk*{\phi_2 \mid \pi^\star= \pi , \cF_{i-1}} \\
   &\resizebox{0.9\linewidth}{!}{$
       \displaystyle
       \in \begin{cases}
        \prn*{ \ind{A_{2:H} = \pi \circ 0} + \frac{1}{2} \ind{A_{2:H-1} \ne \pi } } \cdot \frac14 \prn*{ \frac{1}{2^{H+2}} }^{H-1} \cdot \brk*{ \prn*{1 - \frac{T}{2^H}}^{H-1}, \prn*{1 + \frac{T}{2^H}}^{H-1} } &\text{if}~ R = 1\\[0.5em]
        \prn*{ \ind{A_{2:H} = \pi \circ 1} + \frac{1}{2} \ind{A_{2:H-1} \ne \pi } } \cdot \frac14 \prn*{ \frac{1}{2^{H+2}} }^{H-1} \cdot \brk*{ \prn*{1 - \frac{T}{2^H}}^{H-1}, \prn*{1 + \frac{T}{2^H}}^{H-1} } &\text{if}~ R= 0
\end{cases}$} \numberthis\label{eq:case-good}
\end{align*}
The last term in the sum, with $\phi_2(X_2) = \sdis$ is similar, so we get that
\begin{align*}
    &\sum_{\phi_2: \phi_2(X_2) = \sdis} \Pr  \brk*{ (X_{2:H}, A_{2:H}, R) \mid \pi^\star = \pi , \cF_{i-1}, \phi_2} \Pr\brk*{\phi_2 \mid \pi^\star= \pi , \cF_{i-1}} \\
                &\resizebox{0.9\linewidth}{!}{$
                    \displaystyle 
                    \in \begin{cases}
        \prn*{ \ind{A_{2:H} = \pi \circ 1} + \frac{1}{2} \ind{A_{2:H-1} \ne \pi } } \cdot \frac14  \prn*{ \frac{1}{2^{H+2}} }^{H-1} \cdot \brk*{ \prn*{1 - \frac{T}{2^H}}^{H-1}, \prn*{1 + \frac{T}{2^H}}^{H-1} } &\text{if}~ R = 1\\[0.5em]
        \prn*{ \ind{A_{2:H} = \pi \circ 0} + \frac{1}{2} \ind{A_{2:H-1} \ne \pi } } \cdot \frac14 \prn*{ \frac{1}{2^{H+2}} }^{H-1} \cdot \brk*{ \prn*{1 - \frac{T}{2^H}}^{H-1}, \prn*{1 + \frac{T}{2^H}}^{H-1} } &\text{if}~ R= 0
\end{cases}$} \numberthis\label{eq:case-dis}
\end{align*}
(Note that the first indicators have been swapped in the previous display compared to Eq.~\eqref{eq:case-good}.)

Summing Eqs.~\eqref{eq:case-bad}, \eqref{eq:case-good}, and \eqref{eq:case-dis}, and applying casework on the different choices of $A_{2:H}$ we get that
\begin{align*}
    \Pr\brk*{ (X_{2:H}, A_{2:H}, R) \mid \pi^\star = \pi , \cF_{i-1}} \in \frac12 \cdot \prn*{ \frac{1}{2^{H+2}} }^{H-1} \cdot \brk*{ \prn*{1 - \frac{T}{2^H}}^{H-1}, \prn*{1 + \frac{T}{2^H}}^{H-1} },
\end{align*}
thus concluding the proof of \pref{claim:claim1}.

\begin{claim}\label{claim:claim2}
    If $\tau_i$ is generated by an online rollout, we have for every $\pi \in \Pi_{1:H-1}$,
    \begin{align*}
        \ind{\cE_{R, t-1}} \Pr \brk*{ \tau_i \mid \pi^\star = \pi, \cF_{i-1} }  \in \ind{\cE_{R, t-1}} \frac{1}{2} \cdot \frac{1}{2^{(H+2) \cdot H}} \cdot \brk*{ \prn*{1 - \frac{T}{2^H}}^{H}, \prn*{1 + \frac{T}{2^H}}^{H} }.
    \end{align*}
\end{claim}

Now we prove \pref{claim:claim2}. Most of the hard work has already been done in the proof of \pref{claim:claim1}.  Note that by construction $\phi_1(X_1) = \sgood$. Using a similar calculation we have
\begin{align*}
    \hspace{2em}&\hspace{-2em} \Pr  \brk*{ (X_{1:H}, A_{1:H}, r) \mid \pi^\star = \pi , \cF_{i-1}} \\
    &\in \begin{cases}
        \prn*{ \ind{A_{1:H} = \pi \circ 0} + \frac{1}{2} \ind{A_{1:H-1} \ne \pi } } \cdot \prn*{ \frac{1}{2^{H+2}} }^{H} \cdot \brk*{ \prn*{1 - \frac{T}{2^H}}^{H}, \prn*{1 + \frac{T}{2^H}}^{H} } &\text{if}~ R = 1\\[0.5em]
        \prn*{ \ind{A_{1:H} = \pi \circ 1} + \frac{1}{2} \ind{A_{1:H-1} \ne \pi } } \cdot \prn*{ \frac{1}{2^{H+2}} }^{H} \cdot \brk*{ \prn*{1 - \frac{T}{2^H}}^{H}, \prn*{1 + \frac{T}{2^H}}^{H} } &\text{if}~ R = 0
    \end{cases}
\end{align*}
However, observe that under the event $\cE_{R, t-1}$ we know that $A_{1:H-1} \ne \optpi$, so the first indicator cannot be $=1$ in either case; so multiplying both sides of the previous display by $\ind{\cE_{R, t-1}}$ gives us the result of \pref{claim:claim2}.

To tidy up, we also state the calculation on the last trajectory, which does not include the prefactor of $\frac12$ because there are no observed rewards at the end:
\begin{claim}\label{claim:claim3}
    \begin{align*}
        \Pr \brk*{ (X_{t, h_\bot:H}, A_{t,h_\bot:H}) \mid \pi^\star = \pi , \cF_{t-1}} \in \frac{1}{2^{(H+2) \cdot (H- h_\bot + 1)}} \brk*{ \prn*{1 - \frac{T}{2^H}}^{H}, \prn*{1 + \frac{T}{2^H}}^{H} }.
    \end{align*}
\end{claim}

Now with \pref{claim:claim1}, \ref{claim:claim2}, and \ref{claim:claim3} in hand, we can finally return to computing a bound on Eq.~\eqref{eq:upper-bound-on-posterior-tv}. Letting $O$ denote the total number of observations in $\cfnor$ (which can be at most $TH$), we have for any $\pi \in \Pi_{1:H-1}$,
\begin{align*}
    \hspace{2em}&\hspace{-2em}\ind{\cE_{R, t-1}} \Pr\brk*{\cfnor \mid \pi^\star = \pi} \\
    &\in \ind{\cE_{R, t-1}} \cdot \prn*{ \frac12 }^{t-1} \cdot \prn*{ \frac{1}{2^{H+2}} }^{O} \cdot \brk*{ \prn*{1 - \frac{T}{2^H}}^{TH}, \prn*{1 + \frac{T}{2^H}}^{TH} } \eqqcolon \ind{\cE_{R, t-1}} \cdot [\underline{B}, \overline{B}].
\end{align*}
Moreover, for any $\cfnor$ we have
\begin{align*}
    \Pr \brk*{ \cfnor} = \frac{1}{2^{H-1}} \sum_{\pi \in \Pi_{1:H-1}} \Pr\brk*{\cfnor \mid \pi^\star = \pi} \ge \frac{2^{H-1} - T}{2^{H-1}} \cdot \underline{B}.
\end{align*}
The last inequality follows because there are at most $T$ different action sequences which have been executed by online trajectories in $\cfnor$, so therefore for all but at most $T$ policies we have $\ind{\cE_{R, t-1}} \Pr\brk*{\cfnor \mid \pi^\star = \pi} = \Pr\brk*{\cfnor \mid \pi^\star = \pi}$. Thus we arrive at the bound
\begin{align*}
    \hspace{2em}&\hspace{-2em} \ind{\cE_{R, t-1}} \nrm*{\Pr\brk*{\pi^\star = \cdot \mid \cfnor} - \unif(\Pi_{1:H-1})}_1 \\
    &\le 2 \max_{\pi \in \Pi_{1:H-1}} \brk*{ \frac{\ind{\cE_{R, t-1}} \Pr\brk*{\cfnor \mid \pi^\star = \pi} }{ \Pr\brk*{ \cfnor } } - 1 }_{+} \le 2  \brk*{ \frac{ \overline{B} }{\prn*{1 - T/{2^{H-1}} } \cdot \underline{B} } - 1 }_{+} \\
    &\le 2 \cdot \prn*{ \prn*{1 + \frac{T}{2^{H-2}}}^{2TH+1} - 1} \le 2 \cdot \frac{2T^2H+T}{2^{H-2}} \exp\prn*{ \frac{2T^2H+T}{2^{H-2}}} \le \frac{T^2H}{2^{H-6}}.
\end{align*}
The second to last inequality uses the fact that $1+y \le e^y$ and $e^y - 1 \le y e^y$, and the last inequality uses the fact that $T = 2^{O(H)}$. This concludes the proof of \pref{lem:posterior-of-optimal}.
\end{proof}

\begin{lemma}[Posterior of New State Label]\label{lem:posterior-of-state-label}
Let $\cF$ be any filtration of $T$ trajectories as well as annotations $\phi(x)$ for a subset of observations $x \in \cF$. Let $\pi \in \Pi_{1:H-1}$ be any policy. Fix any $h \ge 2$, and let $x_\mathrm{new} \in \cX_h - \cF$. Then
\begin{align*}
    \abs*{ \Pr^{\nu_0, \alg} \brk*{  \phi(x_\mathrm{new}) = \sgood    \mid \cF, \optpi = \pi }  - \frac14} &\le \frac{T}{2^{H}},\\
    \abs*{ \Pr^{\nu_0, \alg} \brk*{ \phi(x_\mathrm{new}) = \sdis \mid \cF, \optpi = \pi } - \frac14} &\le \frac{T}{2^{H}},\\
    \abs*{ \Pr^{\nu_0, \alg} \brk*{ \phi(x_\mathrm{new}) = \sbad  \mid \cF, \optpi = \pi } - \frac12} &\le \frac{T}{2^{H}}.
\end{align*}
\end{lemma}

\begin{proof}
Let us denote $\cF'$ to be the completely annotated $\cF$ which includes all labels $\crl{\phi(X): X \in \cF}$. We will show that the conclusion of the lemma applies to every completion $\cF'$, and since 
\begin{align*}
    \Pr^{\nu_0, \alg} \brk*{  \phi(x_\mathrm{new}) = \cdot  \mid \cF, \optpi = \pi } = \En^{\nu_0, \alg} \brk*{ \Pr^{\nu_0, \alg} \brk*{  \phi(x_\mathrm{new}) = \cdot   \mid \cF', \optpi = \pi } \mid \cF, \optpi = \pi},
\end{align*}
this will imply the result by Jensen's inequality and convexity of $\abs{\cdot}$.

We calculate the good label probability:
\begin{align*}
    \Pr^{\nu_0, \alg} \brk*{  \phi(x_\mathrm{new}) = \sgood    \mid \cF', \optpi = \pi } = \frac{2^H - \abs{\crl*{X \in \cF: \phi(X) = \sgood}}}{2^{H+2} - \abs{\cF}}.
\end{align*}
For the lower bound we have
\begin{align*}
    \frac{2^H - \abs{ \crl*{X \in \cF: \phi(X) = \sgood}}}{2^{H+2} - \abs{\cF}} \ge \frac{2^H - T}{2^{H+2} } =  \frac{1}{4}\cdot \prn*{ 1 - \frac{T}{{2^{H}}} }.
\end{align*}
For the upper bound we have
\begin{align*}
    \frac{2^H - \abs{\crl*{x \in \cF: \phi(x) = \sgood}}}{2^{H+2} - \abs{\cF}} \le \frac{2^H}{2^{H+2} - T} =  \frac{1}{4}\cdot \prn*{ 1- \frac{T}{2^{H+2}} }^{-1} \le \frac{1}{4}\cdot  \prn*{ 1 + \frac{T}{{2^{H}}} },
\end{align*}
which holds as long as $T \le 2^{H}$. Combining both upper and lower bounds proves the lemma for the good label. The rest of the calculations are similar, so we omit them. This concludes the proof of \pref{lem:posterior-of-state-label}.
\end{proof}

\begin{lemma}[Posterior of State Label with Rollout]\label{lem:posterior-of-state-label-v2}
Fix any $t \in [T]$. Suppose that episode $t$ is sampled using the $\mu$-reset at layer $h_\bot \ge 2$, and that $\cfnor$ contains no repeated states. Then for any $\pi \in \Pi_{h_\bot:H-1}$,
    \begin{align*}
        \abs*{ \Pr^{\nu_0, \alg} \brk*{  \phi(X_{t, h_\bot}) = \sgood    \mid \cfnor, \optpi = \pi }  - \frac14} &\le \frac{TH}{2^{H-3}},\\
        \abs*{ \Pr^{\nu_0, \alg} \brk*{ \phi(X_{t, h_\bot}) = \sdis \mid \cfnor, \optpi = \pi } - \frac14} &\le \frac{TH}{2^{H-3}},\\
        \abs*{ \Pr^{\nu_0, \alg} \brk*{ \phi(X_{t, h_\bot}) = \sbad  \mid \cfnor, \optpi = \pi } - \frac12} &\le \frac{TH}{2^{H-3}}.
    \end{align*}
\end{lemma}

\begin{proof}
    We will prove the result with $h_\bot = 2$, and it is easy to adapt it to the general case (in fact the setting where $h_\bot > 2$ only results in tighter bounds). Using repeated application of chain rule and \pref{lem:posterior-of-state-label} we get\vspace{-1em}
\begin{center}
 \resizebox{1\linewidth}{!}{
     \(
    \begin{aligned}
    \Pr^{\nu_0, \alg} \brk*{  \phi(X_{t,  2}) = \sgood \wedge (X_{t,  2:H}, A_{t,  2:H}) \mid \cF_{t-1}, \pi^\star = \pi } &\in \frac14 \prn*{ \frac{1}{2^{H+2}} }^{H-1} \cdot \brk*{ \prn*{1 - \frac{T}{2^H}}^{H-1}, \prn*{1 + \frac{T}{2^H}}^{H-1} } \\
    \Pr^{\nu_0, \alg} \brk*{  \phi(X_{t,  2}) = \sbad \wedge (X_{t,  2:H}, A_{t,  2:H}) \mid \cF_{t-1}, \pi^\star = \pi } &\in \frac12 \prn*{ \frac{1}{2^{H+2}} }^{H-1} \cdot \brk*{ \prn*{1 - \frac{T}{2^H}}^{H-1}, \prn*{1 + \frac{T}{2^H}}^{H-1} } \\
    \Pr^{\nu_0, \alg} \brk*{  \phi(X_{t,  2}) = \sdis \wedge (X_{t,  2:H}, A_{t,  2:H}) \mid \cF_{t-1}, \pi^\star = \pi } &\in \frac14 \prn*{ \frac{1}{2^{H+2}} }^{H-1} \cdot \brk*{ \prn*{1 - \frac{T}{2^H}}^{H-1}, \prn*{1 + \frac{T}{2^H}}^{H-1} }.
\end{aligned}
     \)
      
}   
\end{center}

   Let's prove the first inequality in the lemma statement. By Bayes Rule we have
\begin{align*}
    \hspace{2em}&\hspace{-2em}    \Pr^{\nu_0, \alg} \brk*{  \phi(X_{t,  2}) = \sgood \mid \cfnor, \pi^\star = \pi } \\
                &= \frac{ \Pr^{\nu_0, \alg} \brk*{  \phi(X_{t,  2}) = \sgood \wedge (X_{t,  2:H}, A_{t,  2:H}) \mid \cF_{t-1}, \pi^\star = \pi } }{ \Pr^{\nu_0, \alg} \brk*{ (X_{t,  2:H}, A_{t,  2:H}) \mid \cF_{t-1}, \pi^\star = \pi} } \\
    &= \frac{ \Pr^{\nu_0, \alg} \brk*{  \phi(X_{t,  2}) = \sgood \wedge (X_{t,  2:H}, A_{t,  2:H}) \mid \cF_{t-1}, \pi^\star = \pi } }{\sum_{\ell \in \crl{\sgood, \sbad, \sdis}} \Pr^{\nu_0, \alg} \brk*{\phi(X_{t,  2}) = \ell \wedge (X_{t,  2:H}, A_{t,  2:H}) \mid \cF_{t-1}, \pi^\star = \pi} }.
\end{align*}
From here it is easy to compute the upper bound
\begin{align*}
    \Pr^{\nu_0, \alg} \brk*{  \phi(X_{t,  2}) = \sgood \mid \cfnor, \pi^\star = \pi } \le \frac{1}{4} \cdot \prn*{1+\frac{T}{2^{H-1}}}^{2H} \le \frac{1}{4} + \frac{TH}{2^{H-3}}.
\end{align*}
as well as the lower bound
\begin{align*}
    \Pr^{\nu_0, \alg} \brk*{  \phi(X_{t,  2}) = \sgood \mid \cfnor, \pi^\star = \pi } \ge \frac{1}{4} \cdot \prn*{1-\frac{T}{2^H}}^{2H} \ge \frac{1}{4} - \frac{TH}{2^{H-3}}.
\end{align*}
The other two inequalities are similarly shown, and this concludes the proof of \pref{lem:posterior-of-state-label-v2}.
\end{proof}

\section{Open Problem: Benefit of Realizability}\label{sec:mu-reset-open-problems}
   
We state the following open problem, which might be of interest to the reader.

\begin{question}
    Is there a benefit to policy realizability under the $\mu$-reset interaction protocol?
\end{question}

While we show that here, \psdp{} and \cpi{} are not sample-efficient (\pref{thm:psdp-lower-bound}), we do not have an information-theoretic lower bound. Intriguingly, the lower bound for the agnostic setting (\pref{thm:lower-bound-policy-completeness}) critically relies on the agnostic property (since the optimal policy $\pi^\star_H \notin \Pi_H$ in the construction). On the upper bound side, it is unclear how to leverage the realizability of the policy class to design sample-efficient learning algorithms which do not incur the exponential in $H$ dependence that \psdp{} and \cpi{} do.

This question can be viewed as the policy-based analogue of the question raised by \cite{mhammedi2024power} on whether it is possible to achieve sample-efficient learning with standard online access if one assumes only coverability and value function realizability ($Q^\star \in \cF$). 

\section{Bibliographical Remarks}
We will highlight several works which are relevant to the $\mu$-reset setting.

\paragraph{Abstract Policy Classes.} For abstract policy classes, the predominant approaches are Policy Search by Dynamic Programming (\psdp{}) \cite{bagnell2003policy} and Conservative Policy Iteration (\cpi{}) \cite{kakade2002approximately} (see also \cite{scherrer2014approximate, scherrer2014local, brukhim2022boosting, agarwal2023variance}). Roughly speaking, these approaches rely on access to a supervised learning oracle w.r.t.~the given policy class. In particular, \psdp{} is a backbone of many contemporary theoretical works in RL \cite[see e.g.,][]{misra2020kinematic, uchendu2023jump, amortila2024scalable, mhammedi2023representation, mhammedi2024efficient}. Both \psdp{} and \cpi{} operate under the $\mu$-reset setting, assume policy completeness, and achieve similar guarantees.

\paragraph{Smoothly-Parameterized Policy Classes.} In practice, one uses smoothly-parameterized policy classes $\Pi = \crl{\pi_\theta}_{\theta \in \Theta}$ (as given by neural networks with some architecture). Here, many works have studied policy gradient methods such as REINFORCE \cite{williams1992simple}, Policy Gradient \cite{sutton1999policy}, and Natural Policy Gradient \cite{kakade2001natural}. Empirically this has given rise to state-of-the-art algorithms for policy optimization \cite{schulman2017trustregionpolicyoptimization, schulman2017proximal}. A line of theoretical works have focused on tabular settings \cite{shani2020adaptive, agarwal2021theory, xiao2022convergence,zhan2023policy}, which natural satisfy policy completeness. Other work studies policy gradient methods \cite{agarwal2020pc, zanette2021cautiously, liu2024optimistic, sherman2023rate} for the restricted setting of linear MDPs \cite{jin2020provably}, designing exploratory algorithms which do not require $\mu$-reset access (but policy completeness is naturally satisfied). For the log-linear policy class parameterization (which captures neural policies), many works build on the \emph{compatible function approximation} framework \cite{sutton1999policy}, which can be viewed as a closure assumption for the gradient updates, see the papers \cite{agarwal2021theory, alfano2022linear,yuan2022linear}. In addition, we highlight the works \cite{bhandari2024global, huang2024occupancy, sherman2025convergence} which study policy gradient methods but require $\mu$-reset access as well some type of completeness/closure or gradient domination assumptions for global optimality guarantees.

\chapter{Hybrid Resets}\label{chap:hybrid-resets}

Now we study the hybrid resets interaction protocol, which combines both the local simulator as well as $\mu$-resets. We have already shown negative results in this thesis (\pref{thm:lower-bound-coverability} and \ref{thm:lower-bound-policy-completeness}) for these interaction protocols individually. These results rule out sample-efficient learners which can adapt to the intrinsic notion of complexity as measured by coverabilty/concentrability.

The main result in this chapter is a new algorithm called \stochalg{} which leverages hybrid resets to perform sample-efficient learning. We give an overview of our results and approach in \pref{sec:hybrid-overview}. To provide some intuition, we first give a algorithm for a warmup setting in \pref{sec:warmup} before diving into the exposition of the main algorithm in \pref{sec:main-upper-bound}. Open problems are discussed in \pref{sec:hybrid-open-problems}, and all proofs are deferred to \pref{sec:hybrid-deferred-proofs}.

\section{Overview}\label{sec:hybrid-overview}

The previous negative results (\pref{thm:lower-bound-coverability} and \ref{thm:lower-bound-policy-completeness}) motivate us to consider hybrid reset access, where we handle the \emph{exploration challenge} via exploratory resets, and the \emph{error amplification challenge} via local simulator access. For value-based learning, \cite{mhammedi2024power} show that local simulator access can overcome the notorious \emph{double sampling problem}, which leads to error amplification. 
Furthermore, local simulator access circumvents the lower bound construction used to prove~\pref{thm:lower-bound-policy-completeness}. Given this, it is conceivable that local simulators might provide significant power in agnostic policy learning. 

Our main positive result in this chapter formalizes this intuition, where we provide a  new algorithm that leverages hybrid resets for sample-efficient learning in Block MDPs (\pref{sec:background-mdp}). 
 Block MDPs are perhaps the simplest setting with large state spaces for developing RL algorithms, as well as a stepping stone to more challenging settings such as low-rank MDPs or coverable MDPs (the \psdp{}/\cpi{} setting). Since our lower bound constructions are all Block MDPs, a positive result here already indicates the significant power of hybrid resets. 

\begin{theorem}\label{thm:block-mdp-result}
    Let $M$ be a Block MDP of horizon $H$ with $S$ states and $A$ actions. Let $\Pi$ be any policy class. Suppose we are given an exploratory reset distribution $\mu = \crl{\mu_h}_{h=1}^H$ which satisfies pushforward concentrability (\pref{def:exploratory-pushforward-distribution}) with parameter $\cpush$ and can be factorized as $\mu_h = \emission \circ \nu_h$ for some $\nu_h \in \Delta(\latentsp_h)$ for all $h\in[H]$.\footnote{The factorization assumption is made for technical convenience, and can be removed (see \pref{sec:upper-bound-preliminaries}).} With probability at least $1-\delta$, \stochalg{} (\pref{alg:stochastic-bmdp-solver-v2}) returns an $\eps$-optimal policy using 
    \begin{align*}
        \poly\prn*{ \cpush, S, A, H, \frac{1}{\eps}, \log \abs{\Pi}, \log \frac{1}{\delta}} \quad \text{samples from hybrid resets.}
    \end{align*}    
\end{theorem}

To support the presentation of our main result, we first present a simplified algorithm called \detalg{} for an easier setting in \pref{sec:warmup}, then present \stochalg{} in \pref{sec:main-upper-bound}. 

We now give several remarks on \pref{thm:block-mdp-result}.

\begin{itemize} 
    \item As a caveat, we require the exploratory distribution $\mu$ to satisfy \emph{pushforward concentrability}, a strengthened version of concentrability introduced by \cite{xie2021batch}. However, it can be checked that in the lower bound of \pref{thm:lower-bound-policy-completeness}, the constructed resets $\mu$ indeed satisfy bounded pushforward concentrability.

   \item Hybrid resets enables new statistical guarantees which are impossible with just local simulator access (cf.~\pref{thm:lower-bound-coverability}) and $\mu$-resets (cf.~\pref{thm:lower-bound-policy-completeness}).
    \item As discussed in \pref{chap:mu-reset}, \psdp{} provably fails in the absence of policy completeness, and even policy realizability does not help. In contrast, \pref{thm:block-mdp-result} achieves sample-efficient learning in the agnostic setting. Therefore, at least in Block MDPs, policy completeness is not an information theoretic barrier, only an algorithmic barrier. 
    \item Departing from prior work on Block MDPs, we do not require decoder realizability, namely that the learner is given a decoder class $\Phi \subseteq \latentsp^\statesp$ which satisfies $\optdec \in \Phi$. With decoder realizability, sample-efficient learning is possible with standard online RL access. Since an (approximately) realizable policy class of size $\log \abs{\Pi} \le \poly(S,A, \log \abs{\Phi}, 1/\eps)$ can be constructed from a decoder class by a standard covering argument, \pref{thm:block-mdp-result} provides substantially stronger guarantees than previously known (albeit under the stronger hybrid reset access). 
\end{itemize}

\paragraph{Key Technical Insights for the Upper Bound.} The fundamental challenge in agnostic policy learning is to simultaneously estimate the values of all policies $\crl{V^\pi}_{\pi \in \Pi}$ in a statistically efficient manner. As discussed in \pref{sec:pac-rl-basics}, in the absence of any structure, this can require $\Omega(\min \crl{A^H, \abs{\statesp} A, \abs{\Pi} })$ samples. Unfortunately, this sample complexity is too large for most practical scenarios.

To improve upon this result, prior works in agnostic policy learning have identified \emph{additional structure} which facilitates the simultaneous estimation of $\crl{V^\pi}_{\pi \in \Pi}$. For example, \cite{sekhari2021agnostic} utilize autoregressive extrapolation when the MDP is low-rank, and $\mathsf{POPLER}$ (\pref{alg:main}) constructs policy-specific Markov Reward Processes to take advantage of the sunflower property of $\Pi$.

This chapter adds a new technical tool called the \emph{policy emulator} to this burgeoning toolbox (see \pref{def:policy-emulator}). A policy emulator, denoted $\wh{M}$, is a carefully constructed tabular MDP which for a given parameter $\eps > 0$ satisfies 
\begin{align*}
\text{for all \( \pi \in \Pi \):} \quad \abs{ V^\pi - \widehat{V}^\pi } \leq \epsilon. \numberthis\label{eq:informal-emulator}
\end{align*}
Here, $V^\pi$ denotes the value of $\pi$ in the underlying MDP, while $\wh{V}^\pi$ denotes the value of $\pi$ in the policy emulator $\wh{M}$. Once the policy emulator has been constructed, returning an $O(\eps)$-optimal policy is straightforward by simply returning $\argmax_{\pi \in \Pi} \wh{V}^\pi$. In this sense, the policy emulator is a ``minimal object'' for agnostic policy learning. In fact, we show in \pref{sec:beyond-bmdp} that every pushforward-coverable MDP admits a policy emulator of bounded size. The remaining question is: how can we construct this policy emulator using few samples?

Our key contribution is to devise a statistically efficient method for constructing this policy emulator in a bottom-up manner, leveraging the power of hybrid resets. As a warmup, we first explore a simpler scenario in \pref{sec:warmup} where the latent dynamics of the Block MDP are deterministic and the learner has the capability to draw samples from the emission function $\emission(\cdot)$. Here, the emulator can directly be constructed over the latent state space $\latentsp$ in a model-based fashion. We then study the fully general setting in \pref{sec:main-upper-bound}. Here, we construct the emulator directly over $\poly(\cpush, S, A, H, \eps^{-1}, \log \abs{\Pi}, \log \delta^{-1})$ random observations sampled from the reset distributions $\mu_1, \cdots \mu_H$. We will prove that the transitions/rewards of this policy emulator can be accurately estimated so that the guarantee in \eqref{eq:informal-emulator} holds.

\section{\detalg{}: Algorithm and Results for Warmup Setting}\label{sec:warmup}

We first study an easier setting 
and provide a simplified algorithm that illustrates the main approach that we will take in the general setting (in \pref{sec:main-upper-bound}).

\subsection{Warmup Setting: Deterministic Dynamics + Sampling Access to Emissions}
We make the following simplications: 

\begin{assumption}\label{ass:det-transitions} Assume that: \begin{enumerate}[(1)]
    \item $M$ has deterministic latent transitions $\optlatp$ and (possibly) stochastic rewards $\optlatr$. 
    \item The learner is given both local simulator access and sampling access to the emission function $\emission$. 
\end{enumerate}
\end{assumption}
Intuitively, \pref{ass:det-transitions} simplifies the problem considerably. Sampling access to the emission enables us to directly estimate the latent reward function $\optlatr$. Furthermore, we can associate a single observation $x \sim \psi(s)$ with each state allowing us to query for $x' \sim P(\cdot \mid s,a)$. 
However, the fundamental challenge of identifying the \emph{latent transition} $\optdec(x')$ remains, which is the main focus of \detalg{}. A few remarks are in order: 
\begin{itemize}
    \item Without loss of generality, we can restrict ourselves to the open-loop policy class $\Piopen = \crl{\pi: \forall x \in \statesp_h, \pi_h(x) \equiv a_h, (a_1, \cdots, a_H) \in \actionsp^H}$. The reasoning is as follows. The optimal policy $\optpi$ for $M$ is constant over $\supp(\emission(s))$ for every $s \in \latentsp$. Due to deterministic latents, there exists some $\wt{\pi} \in \Piopen$ which experiences the same (latent) trajectory $(s_1^\star, a_H^\star, \cdots, s_H^\star, a_H^\star)$ that $\optpi$ experiences. Such a policy $\wt{\pi}$ achieves the optimal value from the fixed starting latent state $s_1$, even though it may not be the optimal policy $\optpi$ that achieves the optimal value from \emph{every} state. 
    \item We implicitly require knowledge of the latent state space $\latentsp = \latentsp_1 \cup \cdots \cup \latentsp_H$ in order to sample from $\psi$. The main algorithm, \stochalg{}, will only require knowledge of a bound $\abs{\latentsp} \le S$.
    \item Sampling access to the emission is more powerful than $\mu$-reset access, since a reset distribution with $\cpush = S$ can be simulated for any $h \in [H]$ by first $s \sim \unif(\cS_h)$ then sampling $x \sim \emission(s)$.
\end{itemize}

\paragraph{Additional Notation: Monte Carlo Rollouts.} Our algorithms (both \detalg{} and \stochalg{}) interact with the environment primarily by collecting Monte Carlo rollouts from states (or distributions over states). For a partial policy $\pi_{h:h'}$, starting state $x \in \statesp_h$, and sample size $n \in \bbN$, we denote the algorithmic primitive $\mcest(x, \pi_{h:h'}, n)$ that:  
\begin{enumerate}[label=\((\arabic*)\)]   
    \item Collects $n$ rollouts $\crl{(x_h^{(t)}, a_h^{(t)}, r_h^{(t)}, \cdots, x_{h'}^{(t)}, a_{h'}^{(t)}, r_{h'}^{(t)})}_{t\in[n]}$ by running $\pi_{h:h'}$ starting from state $x$, \item Returns the estimate $\frac{1}{n}\sum_{t=1}^n \sum_{h\le k \le h'} r_k^{(t)}$. 
\end{enumerate}
 We overload the notation and use $\mcest(d, \pi_{h:h'}, n)$ for $d \in \Delta(\statesp_h)$ to denote a Monte Carlo estimate which first samples $x_h^{(t)} \sim d$ then rolls out with $\pi_{h:h'}$. 

\subsection{The \detalg{} Algorithm and Analysis Sketch}\label{sec:warmup-algorithm-analysis-sketch}
\begin{algorithm}[ht]
    \caption{\detalg{} (Policy Learning for Hybrid Resets, Deterministic Version)} \label{alg:det-bmdp-solver}
        \begin{algorithmic}[1]
            \Require Sampling access to emission $\emission(\cdot)$, policy class $\Pi = \Piopen$, parameter $\eps > 0$.             \State Initialize $\estlatentmdp = \varnothing$, test policies $\crl{\Pitest_h}_{h\in[H]} = \crl{\varnothing}_{h \in [H]}$, and confidence sets $\cP = \crl{\latentsp}_{(s,a) \in \latentsp\times \actionsp}$. 

            \For {all $(s,a) \in \latentsp \times \actionsp$} \hfill \algcomment{Estimate all rewards.}
            \State Estimate $\estlatr(s, a)$ via Monte Carlo to precision $\eps/H^2$.\label{line:reward-est} 
            \EndFor
            \State Initialize current layer index $\ell \gets H$.
            \While {$\ell \ne 0$}
            \State \textbf{If} $\ell = H$ \textbf{then} go to line \ref*{line:refit}. 
            
            \For {all $(s_\ell, a_\ell) \in \latentsp_\ell \times \actionsp$} \hfill \algcomment{Construct transitions at layer $\ell$.}
            
            \State Set $\cP(s_\ell, a_\ell) \gets \detdecoder(s_\ell, a_\ell, \estlatentmdp, \cP, \Pitest_{\ell+1})$.\hfill \algcomment{Algorithm \ref{alg:decoder}}
            \State Set $\estlatp(s_\ell, a_\ell) \in \cP(s_\ell, a_\ell)$ arbitrarily. \label{line:est-transition}
            
            \hspace{-0.5em}\algcomment{Construct test policies and refit transitions.}
            \EndFor
            \State Set $(\ell_\mathsf{next}, \estlatentmdp, \cP, \crl{\Pitest_h}_{h\in[H]}) \gets \detrefit(\ell, \estlatentmdp, \cP, \crl{\Pitest_h}_{h\in[H]}, \eps)$. \label{line:refit}\hfill \algcomment{Algorithm \ref{alg:refit}} 
            \State Update current layer index $\ell \gets \ell_\mathsf{next}$.
            \EndWhile 
            \State \textbf{Return} $\estpi \gets \argmax_{\pi \in \Pi} \wh{V}^\pi(s_1)$.
        \end{algorithmic}
\end{algorithm}

\begin{algorithm}[ht]
    \caption{\detdecoder{} (Decoder, Deterministic Version)}\label{alg:decoder} 
    \begin{algorithmic}[1]
        \Require Tuple $(s_{h}, a_{h})$, estimated MDP $\estlatentmdp$, confidence sets $\cP$, $\tauref$-valid test policies $\Pitest_{h+1}$.
        \State Sample an observation $x_{h+1} \sim P(\cdot  \mid s_{h}, a_{h})$. 
        \For {$(s,s') \in \cS_{h+1} \times \cS_{h+1}$} 
        \State Estimate $\vestarg{x_{h+1}}{\pi_{s, s'}} \gets \mcest(x_{h+1}, \pi_{s, s'}, \wt{O}(1/\tauref^2))$ to precision $\tauref/2$.  
        \EndFor  
        \State \textbf{Return} $\cP_\mathsf{out} \gets \cP(s_{h}, a_{h}) \cap \crl{s \in \cS_{h+1}: ~\forall s' \ne s,~\abs{\vestarg{x_{h+1}} {\pi_{s, s'}} - \wh{V}^{\pi_{s, s'}}(s)} \le 2\tauref}$.\label{line:decoderd-ret} 
    \end{algorithmic}
\end{algorithm}

\begin{algorithm}[ht] 
\caption{\detrefit{} (Refit, Deterministic Version)}\label{alg:refit}
        \addtolength{\abovedisplayskip}{-5pt}
        \addtolength{\belowdisplayskip}{-5pt}
    \begin{algorithmic}[1]
        \Require Layer $h$, estimated MDP $\estlatentmdp$, confidence sets $\cP$, test policies $\crl{\Pitest_h}_{h\in[H]}$, parameter $\eps > 0$. 
        \State Set tolerance $\tauref \coloneqq 2^5 \cdot \eps/H$.
        \For {\((s , s') \in \cS_h \times \cS_h\) } \hfill \algcomment{Compute candidate test policies at layer $h$}
        \State Let $\pi_{s,s'} \gets \argmax_{\pi \in \Pi} \abs{ \wh{V}^\pi(s) - \wh{V}^\pi(s') }$.\label{line:test-policy} 
        \State Estimate to precision $\eps/H$:\label{line:eval-policy}
        \begin{align*}
            \vestarg{s}{\pi_{s,s'}} &\gets \mcest(\emission(s), \pi_{s,s'}, \wt{O}(H^2/\eps^2)),\\
            \vestarg{s'}{\pi_{s,s'}} &\gets \mcest(\emission(s'), \pi_{s,s'}, \wt{O}(H^2/\eps^2)). 
        \end{align*} 
        \EndFor  
        \State Set \(\violations \gets \crl{(s,\pi) \text{~estimated in \pref{line:eval-policy} s.t.~} \abs{\vestarg{s}{\pi} - \wh{V}^\pi(s)} \ge \tauref -\eps/H}\). 
        \If { \(\violations = \varnothing\) } \hfill \algcomment{No violations found, so return test policies.}  
        \State Set $\Pitest_{h} = \cup_{s,s' \in \latentsp_h} \crl{\pi_{s,s'}}$, and  \textbf{return} $(h-1,\estlatentmdp, \cP, \crl{\Pitest_h}_{h\in[H]} )$. \label{line:great-success} 
        \Else  \hfill \algcomment{Refit transitions to handle violations}  
        \For { $(s, \pi) \in \violations$ } \label{line:violation-statement}
        \State Let $\tau = (\bar{s}_h = s, \cdots \bar{s}_H)$ be the states obtained by executing \(\pi\) from \(s\) in  $\estlatentmdp$. 
        \For {each \(\bar s \in \tau\)}
        \State Estimate $\vestarg{\bar{s}}{\pi} \gets \mcest(\bar{s}, \pi, \wt{O}(H^4/\eps^2))$ to precision $\eps/H^2$.\label{line:mc-additional}
        \EndFor
        \State \textbf{for} each \( \bar s \in \tau \) such that $\abs{ \vestarg{\bar{s}}{\pi} - \estlatr(\bar{s}, \pi) -  \vestarg{\estlatp(\bar{s}, \pi) } {\pi} } \ge 4\eps/H^2$:  
        \State \qquad Update $\cP(\bar{s}, \pi) \gets \cP(\bar{s}, \pi) ~\backslash~\estlatp(\bar{s}, \pi)$.\label{line:bad-state}           
        \EndFor 
        \State Reset $\estlatp(s, a) \in \cP(s, a)$ arbitrarily for all \((s, a)\) updated in \pref{line:bad-state}. \label{line:bad-state-2}
        \State \textbf{Return} $( \ell,\estlatentmdp, \cP,  \crl{\Pitest_h}_{h\in[H]} )$ where \( \ell \) is the max layer for which transitions were updated in \pref{line:bad-state-2}.\label{line:great-success-2}  
        \EndIf
    \end{algorithmic}
\end{algorithm}

Now, we present an algorithm \detalg{} (\pref{alg:det-bmdp-solver}), which achieves the following guarantee.

\begin{theorem}\label{thm:det-bmdp-solver-guarantee}
    Let $\eps, \delta \in (0,1)$ be given and suppose that \pref{ass:det-transitions} holds. Then, with probability at least $1-\delta$, \detalg{} (\pref{alg:det-bmdp-solver}) finds an $\eps$-optimal policy using 
    \begin{align*}
        \wt{O}\prn*{\frac{S^5A^2H^5}{\eps^2} \cdot \log\frac{1}{\delta}} \quad\text{samples.}
    \end{align*}
\end{theorem}
    
The proof of \pref{thm:det-bmdp-solver-guarantee} is found in \pref{sec:proof-warmup}. In the rest of this section, we will explain \detalg{} and illustrate the main ideas.

\detalg{} is an inductive algorithm that works from layer $H$ down to layer $1$. It maintains an estimated latent MDP $\estlatentmdp$, which approximates the ground truth latent transitions and rewards, as well as two other objects: transition confidence sets $\cP$, which assigns a set of plausible next states to each state-action pair, and a set of $S^2$ many test policies $\Pitest$, which it uses to estimate the latent transitions. 
In the pseudocode and analysis, we use $\wh{V}^\pi(\cdot)$ and $\wh{Q}^\pi(\cdot, \cdot)$ to denote the value function and $Q$-function on the estimated $\estlatentmdp$. Furthermore, we let  $\optlatp(s,a)$ (resp.~$\estlatp$) denote the latent state which $(s,a)$ transitions to in $M$ (resp.~$\estlatentmdp$). 


At every layer $h \in [H]$, \detalg{} tries to enforce three invariant properties: 
\begin{enumerate}[(A)]
    \item \emph{Policy Evaluation Accuracy.} For all pairs $(s,a) \in \latentsp_h \times \actionsp$ and $\pi \in \Piopen$: $\abs{Q^\pi(s,a) - \wh{Q}^\pi(s,a)} \le \Gamma_h$, where the error bound $\Gamma_h$ grows linearly with $H-h$.
    \item \emph{Confidence Set Validity.} For all pairs $(s,a)\in \latentsp_h \times \actionsp$, we have $\optlatp(s,a) \in \cP(s,a)$. 
    \item \emph{Test Policy Validity.} The $S^2$ many test policies for layer $h$, i.e. $\Pitest_h := \{\pi_{s,s'}\}_{s,s' \in \cS_h} \subseteq \Piopen$, are defined for pairs of states $s, s' \in \latentsp_h$ and are \emph{valid} (maximally distinguishing and accurate):
    \begin{align}
        \pi_{s,s'} = \argmax_{\pi \in \Pi_{h:H}}~ \abs{\wh{V}^\pi(s) - \wh{V}^\pi(s')}, \quad \text{and} \quad \max_{\bar{s} \in \{s,s'\}}~ |V^{\pi_{s,s'}}(\bar{s}) - \wh{V}^{\pi_{s,s'}}(\bar{s})| \leq \tauref. \label{eq:validity-main}
    \end{align}
    Crucially, the accuracy level $\tauref$ \emph{does not grow} with $H-h$.
\end{enumerate}

\paragraph{Error Decomposition.} To motivate these three properties, we first state a standard error decomposition for $Q$-functions, and then show how \detalg{} controls each of terms separately. In what follows, fix some tuple $(s,a)$. We denote $\optlatr = \optlatr(s,a)$ and $\optlatp = \optlatp(s,a)$, as well as the estimated counterparts $\estlatr, \estlatp$ similarly. The Bellman error for $(s,a)$ can be decomposed as follows:
\begin{align*}
    \abs*{Q^\pi(s,a) - \wh{Q}^\pi(s,a)} \le \underbrace{\abs*{\optlatr - \estlatr}}_{\text{reward error}} + \underbrace{\abs*{\wh{V}^\pi(\optlatp) - \wh{V}^\pi(\estlatp)}}_{\text{transition error}} + \underbrace{\abs*{V^\pi(\optlatp) - \wh{V}^\pi(\optlatp)}}_{\text{policy eval.~error at next layer}}. \numberthis \label{eq:bellman-error-main} 
\end{align*}
Controlling the reward error is easy: we can simply collect i.i.d.~samples using sampling access to $\emission$ to estimate $\estlatr$ up to $\eps$ accuracy (see \pref{line:reward-est}). Furthermore, if (A) holds at layer $h+1$, then we can bound the last term of Eq.~\eqref{eq:bellman-error-main} by $\Gamma_{h+1}$. Controlling the transition error requires more work, since the learner only gets to see observations $x_{\mathsf{new}}\sim P(\cdot\mid{}s,a)$, but not the latent state $\phi(x_{\mathsf{new}})$. 

\paragraph{Decoding via Test Policies.} Our main insight is to estimate the latent state $\optdec(x_\mathsf{new})$ by using rollouts from $x_\mathsf{new}$ to compare value functions with other latent states. Denoting $\vestarg{x_\mathsf{new}}{\pi}$ to be a Monte-Carlo estimate of $V^\pi(x_\mathsf{new})$, if we find some $s' \in \latentsp_{h+1}$ such that 
\begin{align*}
    \vestarg{x_\mathsf{new}}{\pi} \approx \wh{V}^\pi(s'), \quad \text{for all}~\pi \in \Piopen, \numberthis\label{eq:policy-estimation-all}
\end{align*}
then we declare the latent state of $x_\mathsf{new}$ to be $s'$. 
This allows us to bypass the statistical hardness of learning the decoder function $\optdec$ itself,
but, unfortunately, estimating $V^\pi(x_\mathsf{new})$ \emph{for all} $\pi \in \Piopen$ seems to require number of samples proportional to $\spancap(\Piopen) = A^H$ \cite{jia2023agnostic}. In other words, there is nothing better than just executing each policy one-by-one. However, in our algorithm, the test policies $\Pitest$ allow us to circumvent this. In \detdecoder{} (\pref{alg:decoder}), we use $\Pitest$ to run a ``tournament'' with only $S^2$ Monte Carlo rollouts from $x_\mathsf{new}$ to estimate the confidence set $\cP$ of plausible latent states. In \pref{line:decoderd-ret} of \detdecoder{}, the confidence set is updated to be 
\begin{align*}
    \cP(s, a) \gets \cP(s, a) \cap \Big\{s \in \cS_{h+1}: ~\forall s' \ne s,~\abs{ \vestarg{x_\mathsf{new}} {\pi_{s, s'}} - \wh{V}^{\pi_{s, s'}}(s)} \lesssim \tauref \Big\}. \numberthis\label{eq:confidence-set-warmup} 
\end{align*}
We show in \pref{lem:controlling-transition-error} that test policy validity (C) at layer $h+1$ implies that the confidence set \eqref{eq:confidence-set-warmup} is valid (B) for layer $h$ and furthermore, setting the transition to be any $\estlatp \in \cP$ allows us to extrapolate to statement \eqref{eq:policy-estimation-all}, thus giving us a bound on the transition error. As we have shown a bound for all three terms in Eq.~\eqref{eq:bellman-error-main}, we conclude that (A) also holds at layer $h$. 


\paragraph{Refitting Latent Dynamics.} \detrefit{} (\pref{alg:refit}) computes test policies for layer $h$ that satisfy (C) after we have estimated the transitions/rewards. It does so by solving the maximally distinguishing planning problem (\eqref{eq:validity-main}, left) in $\estlatentmdp$ for each $s,s'\in \cS_h$. Since (A) holds at layer $h$, these policies are guaranteed to be accurate; however, test policies are required to satisfy a higher level of accuracy $\tauref \ll \Gamma_{h}$ which \emph{does not increase with the horizon}. To provide intuition on why the higher level of accuracy is required for the test policies, we refer the reader to \pref{fig:error-amplification}.

\begin{figure}[!t]
    \centering
\includegraphics[scale=0.35, trim={0cm 18.5cm 32cm 0.1cm}, clip]{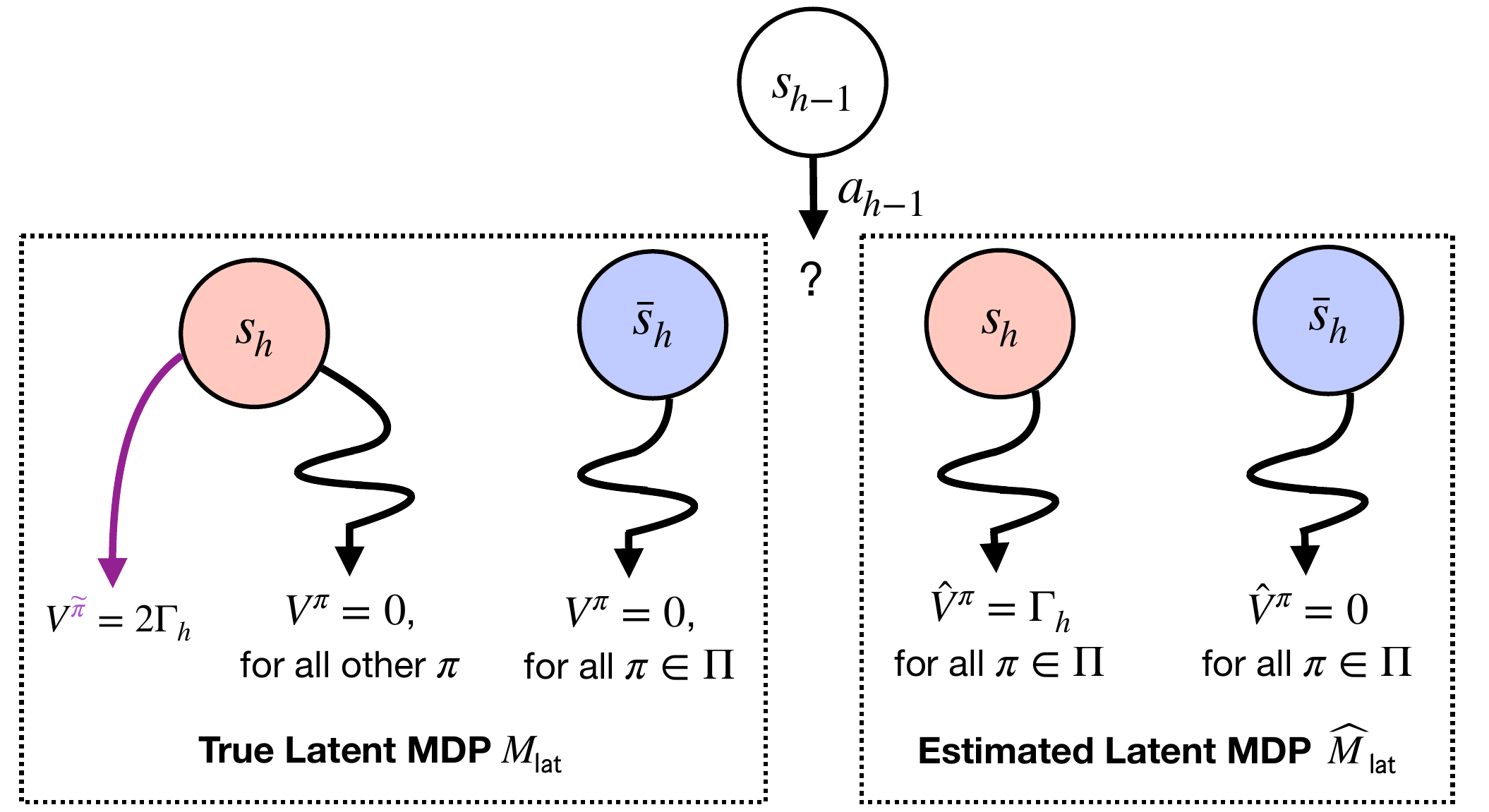}
\caption{Illustration of how certifying accuracy of test policies prevents error amplification. Suppose we want to learn the transition $\optlatp(s_{h-1}, a_{h-1}) = s_h$. In $M_\mathsf{lat}$, all policies get value 0 from both $s_h$ and $\bar{s}_h$, with the exception of a special $\color{Purple}{\wt{\pi}}$ that gets value $2\Gamma_h$ from $s_h$; in $\estlatentmdp$ all policies get value $\Gamma_h$ from $s_h$ and value 0 from $\bar{s}_h$. Thus, $\estlatentmdp$ satisfies $(A)$ but any test policy $\pi_{s_h, \bar{s}_h} \in \Pi$ will not satisfy (C). It is unlikely that $\pi_{s_h, \bar{s}_h} = \color{Purple}{\wt{\pi}}$ is selected, and if we execute any other $\pi$ from the true transition $s_h$, we will observe value $0$, and thus decode the transition to $\estlatp(s_{h-1}, a_{h-1}) = \bar{s}_h$. Therefore, $\abs{Q^{\pi}(s_{h-1}, a_{h-1}) - \wh{Q}^\pi(s_{h-1}, a_{h-1})} = 2\Gamma_{h}$, thus \emph{doubling} the policy evaluation error from layer $h$ to $h-1$. Unchecked, this could cause exponential (in $H$) error amplification. Certifying test policy accuracy prevents this, as \detrefit{} would detect the violation $\abs{V^\pi(s_{h}) -\wh{V}^\pi(s_{h})} = \Gamma_{h} \gg \tauref$ for any $\pi \in \Pi$ and refit $\estlatentmdp$ instead.}
\label{fig:error-amplification}
\end{figure}

Fortunately, since there are only $S^2$ test policies we can use Monte Carlo rollouts to check whether they are $\tauref$-accurate. If they are, we simply decrement to layer $h-1$ and continue (\pref{line:great-success}). If not, the rollouts will find a ``certificate of inaccuracy'': some tuple $(s, \pi)$ for which $\abs{\wh{V}^\pi(s) - V^\pi(s)}$ is large, which we can use to find and delete an erroneous transition in $\estlatp$ from a confidence set. Since this update can occur at some layer $\ell \gg h$, $\estlatentmdp$ may no longer satisfy the inductive hypotheses, so \detrefit{} restarts the outer loop of \detalg{} at the maximum layer $\ell$ for which some transition was updated (\pref{line:great-success-2}). Critically, we show in \pref{lem:refitting} that refitting never deletes the true $P_{\mathsf{lat}}$, so revisiting only happens $SA\cdot(S-1)$ times.

\paragraph{Performance of Estimated Policy.} Eventually, \detalg{} will terminate at layer $h=1$. Thanks to (A), we can evaluate all $\pi \in \Piopen$ on the fully constructed $\estlatentmdp$ and return the policy $\estpi$ which achieves the highest value. The inductive argument we have outlined shows that $\estpi$ is an $\eps$-optimal policy and that \detalg{} uses $\poly(S,A,H, \eps^{-1})$ samples.

\section{\stochalg{}: Algorithm and Main Results}\label{sec:main-upper-bound}
In this section, we extend our result in \pref{thm:det-bmdp-solver-guarantee} to handle the general setting. We give our main algorithm, \stochalg{}, which takes inspiration from \detalg{}. We show how \stochalg{} leverages hybrid resets to solve agnostic policy learning, with sample complexity that scales with the pushforward concentrability $\cpush$ of the reset distribution $\mu$, a measure of the intrinsic difficulty of exploration.

First, we restate our main result of \pref{thm:block-mdp-result} with the precise dependence on the problem parameters. 

\begin{reptheorem}{thm:block-mdp-result} 
    Let $M$ be a Block MDP of horizon $H$ with $S$ states and $A$ actions, and let $\Pi$ be any policy class. Suppose we are given an exploratory reset distribution $\mu = \crl{\mu_h}_{h=1}^H$ which satisfies pushforward concentrability with parameter $\cpush$ and can be factorized as $\mu_h = \emission \circ \nu_h$ for some $\nu_h \in \Delta(\latentsp_h)$ for all $h\in[H]$. With probability at least $1-\delta$, the {\normalfont \stochalg{}} algorithm (\pref{alg:stochastic-bmdp-solver-v2}) returns an $\eps$-optimal policy using
\begin{align*}
    \frac{\cpush^4 S^{24} A^{30} H^{39} }{\eps^{18}} \cdot \mathrm{polylog} \prn*{\cpush, S, A, H, \abs{\Pi}, \eps^{-1}, \delta^{-1}} \quad \text{samples from hybrid resets.}
\end{align*}
\end{reptheorem}

The proof is deferred to \pref{sec:main-upper-bound-proofs}. In the rest of this section, we discuss the main aspects of \stochalg{} and provide intuition for how it addresses new technical challenges once we relax \pref{ass:det-transitions}. 


\subsection{Algorithm Overview}
We now present an overview of \stochalg{}, whose pseudocode can be found in \pref{alg:stochastic-bmdp-solver-v2}. Similar to \detalg{}, it uses two subroutines: \stochdecoder, found in \pref{alg:stochastic-decoder-v2}, and \stochrefit, found in \pref{alg:stochastic-refit-v2}. Overall, \stochalg{} has a similar structure to \detalg{}, but it requires several new ideas to address several  challenges to circumvent needing \pref{ass:det-transitions}:  

\begin{itemize}
    \item Under \pref{ass:det-transitions}, the learner had sampling access to the emission function $\emission$; as a consequence, we could construct an estimate of the latent MDP $\estlatentmdp$ which was defined over the latent state space $\latentsp$. Sampling access to $\emission$ was crucial since it allowed us to disambiguate observations. If the learner only has access to the reset distribution $\mu$, it is nontrivial even to estimate the latent reward function $\optlatr$, since we cannot access the decoder for observations $x \sim \mu$.
    \item In \detalg{}, even though we were supplied a policy class $\Pi$, we could instead use the open-loop policy class $\Pi_\mathsf{open}$ as a proxy, since we were guaranteed that $\max_{\pi \in \Pi_\mathsf{open}} V^\pi \ge \max_{\pi \in \Pi} V^\pi$. If the MDP has stochastic latent transitions, $\Pi_\mathsf{open}$ might not contain any good policy. Thus, we need to directly evaluate the given policies $\pi \in \Pi$ in order to solve the agnostic policy learning problem. 
\end{itemize}

\paragraph{Policy Emulators.} To address these challenges, we take the more straightforward approach: instead of trying to construct latent transitions/rewards, \emph{we directly construct an MDP $\estmdp$ over observations}. The MDP $\estmdp$ has a restricted state space $\estmdpobsspace{} \subseteq \cX$ but inherits the same action space $\cA$ and horizon $H$. Unlike the standard approach taken in tabular RL, we cannot hope to approximate the dynamics of the true MDP $M$ in an information theoretic sense, as the transition $P(\cdot \mid x,a)$ is an $\abs{\statesp}$-dimensional object (requiring $\Omega(\abs{\cX})$ samples to estimate). Taking a step back, all we need is that $\estmdp$ enables accurate policy evaluation, i.e., denoting $\wh{V}^\pi$ to be the value function of $\pi$ on $\estmdp$, we have $\max_{\pi \in \Pi}~\abs{V^\pi - \wh{V}^\pi} \le \eps$. In this sense $\estmdp$ is a ``minimal object'' which allows us to emulate the values of all policies $\pi \in \Pi$. This is formalized in the following definition.\footnote{Similar terminology of an \emph{emulator} is defined in \cite{golowich2024exploring}. Their definition formalizes what it means for estimated transitions to approximate certain Bellman backup operations, and is tailored to linear MDPs.} In the sequel, we denote $\estmdpobsspace{h}$ and $\estmdpobsspace{h:H}$ to be the restriction of the state space of $\estmdp$ to the given layer(s).  

\begin{definition}[Policy Emulator]\label{def:policy-emulator} Let $\Pi$ be a policy class and $M$ be an MDP. Fix any $\nu \in \Delta(\cX)$. We say $\estmdp$ is an $\eps$-accurate \emph{policy emulator} for $\nu$ if there exists $\wh{\nu} \in \Delta(\estmdpobsspace{})$ such that:
\begin{align*}
    \max_{\pi \in \Pi}~\abs*{\En_{x \sim \nu} \brk{V^\pi(x)} - \En_{x \sim \wh{\nu}} \brk{\wh{V}^\pi(x)} } \le \eps.  
\end{align*}
\end{definition}

\pref{def:policy-emulator} naturally extends the concept of \emph{uniform convergence} \cite{shalev2014understanding} to the interactive setting of policy learning. Clearly, if $\estmdp$ is an $\eps$-accurate policy emulator for the starting distribution $d_1$, we can find an $O(\eps)$-optimal policy. One inspiration for \pref{def:policy-emulator} is the Trajectory Tree algorithm (\pref{alg:trajectory-tree}), which can be viewed as a way to use local simulator access to build a policy emulator with $\abs{\estmdpobsspace{}} = \wt{O}(H \spancap(\Pi)/\eps^2)$ states, requiring sample complexity scaling with the worst-case notion of complexity $\spancap(\Pi)$ (see \pref{thm:generative_upper_bound}).

In contrast, \stochalg{} utilizes the reset distribution $\mu$ to construct a policy emulator with state space and sample complexity scaling with the instance-dependent notion of complexity $\cpush$. We do this in an inductive fashion, working back from layer $H$ to layer 1.
\begin{itemize}
    \item At every layer $h$, we sample $\poly(\cpush, S, A, H, \eps^{-1}, \log \abs{\Pi})$ states from $\mu_h$ to form the policy emulator's state space $\estmdpobsspace{h}$. The rewards of every tuple $(x_h,a_h) \in \estmdpobsspace{h} \times \actionsp$ are estimated via the local simulator.
    \item Once the transitions of $\estmdp$ has been constructed from layer $h+1$ onward, we call \stochdecoder{} on every $(x_h, a_h) \in \estmdpobsspace{h} \times \actionsp$. \stochdecoder{} first samples a dataset $\cD$ of transitions from $P(\cdot \mid x_h, a_h)$ (in \pref{line:sample-decode-dataset}) and then performs Monte Carlo rollouts over observations in $\cD$ using test policies $\Pitest_{h+1}$ (in \pref{line:monte-carlo}). In contrast with \detalg{}, since \stochalg{} directly works in observation space, the test policies are defined for pair of observations $x, x' \in \estmdpobsspace{h+1}$, not pairs of latent states. \stochdecoder{} estimates a transition function $\estp(\cdot \mid x_h, a_h) \in \Delta(\estmdpobsspace{h+1})$ as well as a confidence set $\cP(x_h, a_h) \subseteq \Delta(\estmdpobsspace{h+1})$.
    \item After transitions at layer $h$ are constructed, we call \stochrefit{} which tries to compute accurate test policies $\Pitest_h$ for layer $h$. If \stochrefit{} succeeds, then \stochalg{} continues the decoding/refitting loop at layer $h-1$. Otherwise, \stochrefit{} searches in the policy emulator $\estmdp$ for an inaccurate transition $\estp(\cdot| \bar{x}, \bar{a})$ and updates it. The layer index $\ell$ is set to the maximum layer for which an $(\bar{x}, \bar{a})$ is updated, and \stochalg{} restarts at that layer $\ell$. 
\end{itemize}
Eventually, \stochalg{} will reach layer 1, giving a fully-constructed policy emulator $\estmdp$. Returning the best policy in $\estmdp$ is guaranteed to be a near-optimal policy for the true MDP $M$.

\begin{algorithm}[t]
    \caption{\stochalg{} (Policy Learning for Hybrid Resets)}\label{alg:stochastic-bmdp-solver-v2}
        \begin{algorithmic}[1]
            \Require Reset distributions $\mu = \crl{\mu_h}_{h \in [H]}$, policy class $\Pi$, parameters $\eps >0$ and $\delta \in (0,1)$.
            \State Initalize policy emulator $\estmdp = \varnothing$, test policies $\crl{\Pitest_h}_{h\in[H]} = \crl{\varnothing}_{h\in[H]}$, transition confidence sets $\cP = \varnothing$.
            \State Set $\nreset \asymp \tfrac{\cpush SA^2}{\eps^3} \cdot \log \tfrac{SA \abs{\Pi}}{\delta} $.
            
            \For {$h = 1, \cdots, H$} \hfill \algcomment{Initialize policy emulator} \label{line:init-start} 
            \State Sample $\nreset$ observations from $\mu_h$ and add to $\estmdpobsspace{h}$. \label{line:sampling-from-reset}
            \For {every $(\empobs_h,a_h) \in \estmdpobsspace{h} \times \actionsp$}
            \State Estimate $\wh{R}(\empobs_h, a_h) \gets \mcest(x_h, a_h, \wt{O}(H^2/\eps^2))$.\label{line:reward-estimation}
            \State Initialize $\cP(\empobs_h, a_h) = \Delta(\estmdpobsspace{h+1})$.\label{line:init-end}
            \EndFor
            \EndFor
            \State Set current layer index $\ell \gets H$.
            \While {$\ell \ne 0$}
            
            \State \textbf{If} $\ell=H$: \textbf{go to line \ref*{line:refit-v2}.}

            \hspace{-0.5em}\algcomment{Construct transitions at layer $\ell$}

            \For {each $(\empobs_\ell, a_\ell) \in \estmdpobsspace{\ell} \times \cA$} 
            \State Set $\cP(\empobs_\ell, a_\ell) \gets \stochdecoder((\empobs_\ell, a_\ell), \estmdp, \cP, \Pitest_{\ell+1}, \eps, \delta)$ \hfill \algcomment{See \pref{alg:stochastic-decoder-v2}}
            \State Set $\estp(\cdot \mid \empobs_\ell, a_\ell) \in \cP(\empobs_\ell, a_\ell)$ arbitrarily.\label{line:est-transition-v2}
            \EndFor

            \hspace{-0.5em}\algcomment{Construct test policies and refit transitions.}

            \State Set $(\ell_\mathsf{next}, \estmdp, \crl{\Pitest_h}_{h\in[H]}, \cP) \gets \stochrefit(\ell, \estmdp, \cP, \crl{\Pitest_h}_{h\in[H]}, \eps, \delta)$. \label{line:refit-v2} \hfill \algcomment{See \pref{alg:stochastic-refit-v2}}
            \State Update current layer index $\ell \gets \ell_\mathsf{next}$.
            \EndWhile
            \State \textbf{Return} $\estpi \gets  \argmax_{\pi \in \Pi} \En_{x_1 \sim \unif(\estmdpobsspace{1})} \brk{\wh{V}^\pi(x_1)}$.
        \end{algorithmic}
\end{algorithm}

\subsection{\stochdecoder{} Subroutine}\label{sec:stochdecoder}

In this section, we explain \stochdecoder{}, which for a given $(x_h,a_h)$ pair computes a confidence set of transitions $\cP(x_h,a_h)$ over the policy emulator states in the next layer $\estmdpobsspace{h+1}$. The main salient difference with \detdecoder{} is that  we now adopt a more sophisticated confidence set construction to ensure that arbitrary policies $\pi \in \Pi$ can be emulated by $\estmdp$. 

\begin{algorithm}[t]
    \caption{\stochdecoder}\label{alg:stochastic-decoder-v2}
        \begin{algorithmic}[1]
            \addtolength{\abovedisplayskip}{-5pt}
            \addtolength{\belowdisplayskip}{-5pt}
            \Require Tuple~$(x_h, a_h)$, policy emulator $\estmdp$, confidence sets $\cP$, $\taudec$-valid test policies $\Pitest_{h+1}$, parameters $\eps >0$, $\delta \in (0,1)$. 
            \State Set $\ndec \asymp \tfrac{S^2 A^2}{\eps^2} \cdot \log \tfrac{\cpush SAH \abs{\Pi}}{\eps \delta}$, $\nmc \asymp \tfrac{1}{\eps^2} \cdot \log \tfrac{\cpush SAH \abs{\Pi}}{\eps \delta}$.
            \State Sample dataset of $\ndec$ observations $\cD \sim P(\cdot \mid x_h, a_h)$. \label{line:sample-decode-dataset}
            \For {every $x^{(i)} \in \cD$}
            \hfill \algcomment{Individually decode every observation}
            \For {every $(\empobs, \empobs') \in \estmdpobsspace{h+1} \times \estmdpobsspace{h+1}$}:
            \State Estimate $\vestarg{x^{(i)}}{\pitest{\empobs}{\empobs'}} \gets \mcest(x^{(i)}, \pitest{\empobs}{\empobs'}, \nmc)$. \label{line:monte-carlo} 
            \EndFor
            \State Define:
            \begin{align*}
                \cT[x^{(i)}] \gets \crl*{\empobs \in \estmdpobsspace{h+1}: ~\forall \empobs' \ne \empobs,~\abs*{\vestarg{x^{(i)}}{\pitest{\empobs}{\empobs'}} - \wh{V}^{\pitest{\empobs}{\empobs'}}(\empobs) } \le \taudec + 2\eps}. 
            \end{align*}
            \EndFor 
            \State Define $\gobs$ as the decoder graph with \hfill \algcomment{See \pref{def:decoder-graph-obs}}  
            \begin{align*} 
                \cXL \coloneqq \cD, \quad \cXR \coloneqq \estmdpobsspace{h+1}, \quad \text{and decoder function}~\cT.
            \end{align*}
            \State \textbf{Return}: $\cP$ defined using Eq.~\pref{eq:confidence-set-construction-v2}.
        \end{algorithmic}
\end{algorithm}

We first introduce an intermediate object, called the decoder graph.

\begin{definition}[Decoder Graph]\label{def:decoder-graph-obs} Let $\cXL,\cXR \subseteq \cX$, and let $\cT: \cXL \mapsto 2^{\cXR}$ be a decoder function. The \emph{decoder graph}, denoted $\gobs$, is defined as the bipartite graph with vertices $V = \cXL \cup \cXR$ and edges $E = \crl{(x_l, \empobs_r): x_l \in \cXL, \empobs_r \in \cT[x_l]}$. 
\end{definition}
In words, the decoder graph $\gobs$ draws an edge from every observation $x_l$ sampled from the transition to observations $x_r$ sampled from the reset if the value functions for all test policies are similar. Thus, the decoder graph $\gobs$ summarizes the similarity information encoded by individually decoding each observation.

The other ingredient is a notion of \emph{pushforward distribution}, which, when supplied a distribution over observations, collapses a policy $\pi$ to a distribution over actions.

\begin{definition}[Pushforward Distribution/Policy]\label{def:pushforward-policy} Let $\nu \in \Delta(\statesp)$ be a distribution over observations. For any policy $\pi: \statesp \to \Delta(\actionsp)$, define the \emph{pushforward distribution}, denoted $\pi \push \nu \in \Delta(\cA)$, as 
\begin{align*}
    \brk*{\pi \push \nu}(a) \coloneqq \En_{x \sim \nu} \brk*{\ind{\pi(x) = a} } \quad \text{for all}~a \in \cA.
\end{align*}
For any $\pi \in \Pi$, the emission $\emission: \cS \to \Delta(\cX)$ induces a pushforward distribution; we slightly abuse notation and call the function $\pi \push \emission: \cS \to \Delta(\cA)$ the \emph{pushforward policy}. 
\end{definition}

\paragraph{Confidence Set Construction.}
Now we are ready to specify the confidence set construction of \stochdecoder{}. Denote $\crl{\cc_j}_{j \ge 1}$ to be the connected components of $\gobs$. For any $\cc \in \crl{\cc_j}_{j \ge 1}$, denote $\ccL \subseteq \cXL$ and $\ccR \subseteq \cXR$ to be the left/right observation sets respectively. In what follows, we use $p(\cdot \mid \ccR)$ to denote the conditional distribution over $\ccR$, i.e., $p(\empobs \mid \ccR) = p(\empobs)/p(\ccR) \cdot \ind{\empobs \in \ccR}$. Given a decoder graph $\gobs$ and input confidence set $ \cP(x_h, a_h)$, the updated confidence set is defined for $\beta \coloneqq \wt{O}\prn{\prn{\sqrt{SA^2} + S}\eps}$ as 
\begin{align*}
    \cP \coloneqq \Big \{p \in \cP(x_h, a_h): ~ &\sum_{\cc \in \crl{\cc_j}} \abs*{p(\ccR) - \frac{\abs{\ccL}}{\abs{\cXL}} } \le 3\eps, \\
                                                &\quad \max_{\pi \in \Pi}~\sum_{\cc \in \crl{\cc_j}} \frac{\abs{\ccL}}{\abs{\cXL}} \cdot \nrm*{\pi \push \unif(\ccL) - \pi \push p(\cdot \mid \ccR)}_1 \le \beta \Big\}. \numberthis \label{eq:confidence-set-construction-v2} 
\end{align*}
\paragraph{Intuition for \eqref{eq:confidence-set-construction-v2}.} We give some intuition for the construction in \eqref{eq:confidence-set-construction-v2}, and refer the reader to the example in \pref{fig:confidence-set}. The high level goal is to find a set of distributions $\cP(x_h, a_h)$ supported on $\estmdpobsspace{h+1}$ such that if we plug any $\estp \in \cP(x_h, a_h)$ into our policy emulator, the policy evaluation error is bounded, i.e., 
\begin{align*}
    Q^\pi(x_h, a_h) \approx \wh{R}(x_h, a_h) + \En_{x' \sim \estp} \brk{ \wh{V}^\pi(x')}, \quad \text{for all} \quad \pi \in \Pi.
\end{align*}
In particular, we need every $\estp \in \cP$ to witness accurate policy emulation for the distribution $P = P(\cdot\mid x_h, a_h)$, so we require $\estp$ to satisfy a bound on:
\begin{align*}
    \max_{\pi \in \Pi}~ \abs*{ \En_{x' \sim P} \brk{V^\pi(x)} - \En_{x' \sim \estp}\brk{\wh{V}^\pi(x')} }.  \numberthis \label{eq:policy-emulation-transition} 
\end{align*}
Now we discuss how the constraints for $\cP$ control this policy emulation error for every $\wh{P} \in \cP$. Intuitively, the connected components $\crl{\cc_j}_{j \ge 1}$ of $\gobs$ represent a ``soft'' clustering of observations, since all observations in a given connected component $\cc \in \crl{\cc_j}_{j \ge 1}$ have similar $Q$-functions for every test policy. We further prove that this implies that the $Q$-functions are similar within $\cc$ for every $\pi \in \Pi$. Now we discuss the constraints.

\begin{figure}[!t]
    \centering
\includegraphics[scale=0.245, trim={0cm 18cm 1cm 0cm}, clip]{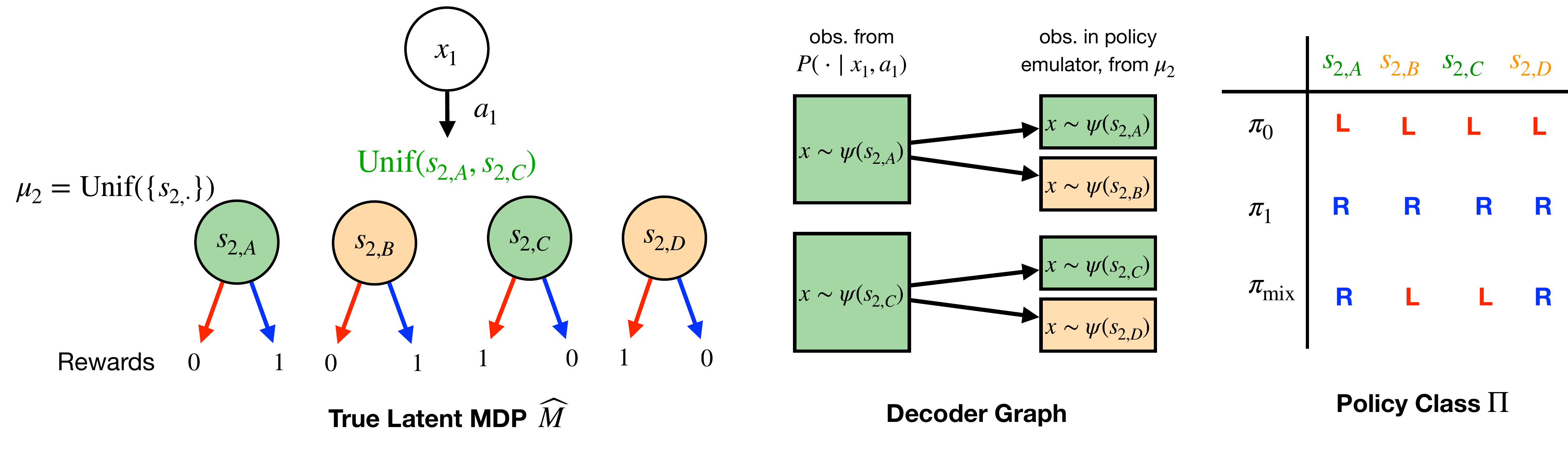} 
\caption{Confidence set construction example with $H=2$. At layer $2$, the MDP has 4 latent states, $s_{2,A}$, $s_{2,B}$, $s_{2,C}$, and $s_{2,D}$. Since $\mu_2$ has uniform mass, we sample representative observations from each latent state in our policy emulator $\estmdp$. Now consider using \stochdecoder{} to learn the transition $P(\cdot \mid x_1, a_1)$. We cannot disambiguate between observations from $s_{2,A}$ and $s_{2,B}$ via test policies (similarly for $s_{2,C}$ and $s_{2,D}$). Thus, the learned decoder graph $\gobs$ has the two connected components as shown. The \textbf{marginal constraint} enforces that every $\estp \in \cP$ must place half the mass on observations from $s_{2,A}$ and $s_{2,B}$ and the other half on observations from $s_{2,C}$ and $s_{2,D}$. This is enough to ensure that the policy evaluation error for $\pi_0$ and $\pi_1$ are controlled (cf.~Eq.~\eqref{eq:policy-emulation-transition}). However, it is not enough to ensure that policy evaluation error for $\pi_\mathrm{mix}$ is controlled, since $\pi_\mathrm{mix}$ is not constant over each connected component. As an example, consider the $\estp$ which puts uniform mass on the observations from $s_{2,B}$ and $s_{2,D}$ (the orange blocks). We have $Q^{\pi_\mathrm{mix}}(x_1, a_1) = 1$ while $\wh{Q}^{\pi_\mathrm{mix}}(x_1, a_1) = 0$. This explains why we need the \textbf{pushforward constraint}, which requires that the pushforward distribution of $\pi_\mathrm{mix}$ is matched on every connected component.} 
\label{fig:confidence-set}
\end{figure}


\begin{itemize}
    \item \textbf{Marginal Constraint}: The first condition expresses a TV distance constraint on the marginals over connected components: that is, the estimated distribution $\estp$ must place a similar amount of mass on each connected component as we observe in the samples from $P$. This ensures that for all $a \in \cA, \pi \in \Pi$:
    \begin{align*}
        \En_{x' \sim P} \brk{Q^\pi(x', a)} \approx \En_{x' \sim \estp}\brk{\wh{Q}^\pi(x', a)}.
    \end{align*}
    \item \textbf{Pushforward Constraint}: However, the marginal constraint is insufficient for accurate policy emulation because in general, policies in the given class $\Pi$ are \emph{not constant} over a given $\cc$. We give an example of this in \pref{fig:confidence-set}. To address this, we need to ensure that over each $\cc$, the pushforward distributions also match. This is precisely captured in an averaged sense by the second condition.
\end{itemize}
We show that the set of $\cP$ which satisfies both constraints yields a bound on the policy emulation error \eqref{eq:policy-emulation-transition}.

\paragraph{Technical Tool: Projected Measures.}
The technical challenge in establishing Eq.~\eqref{eq:policy-emulation-transition} is that the high-dimensional $P$ is supported on $\cX$, while we want to approximate it with $\estp$ supported on the states $\estmdpobsspace{h+1} \subseteq \cX$ of the policy emulator. To address this, we introduce a notion of \emph{projected measures} onto the state space $\estmdpobsspace{h+1}$, denoted $\projm{}: \Delta(\latentsp) \to \Delta(\estmdpobsspace{h+1})$ (see \pref{def:projected-measure} for a formal definition), which approximates $\emission \circ d$ for any distribution over latent states $d$. Using the triangle inequality on Eq.~\eqref{eq:policy-emulation-transition}, we can decompose the policy emulation error using the projected measure as an intermediary quantity: 
\begin{align*}
    \text{\eqref{eq:policy-emulation-transition}} &\le \underbrace{ \abs*{ \En_{x' \sim P} \brk{V^\pi(x')} - \En_{x' \sim \projm{}(\optlatp)} \brk{V^\pi(x')} } }_{\text{projection error}} ~+~ \underbrace{ \abs*{ \En_{x' \sim \projm{}(\optlatp)} \brk{V^\pi(x')}- \En_{x' \sim \projm{}(\optlatp) }\brk{\wh{V}^\pi(x')} } }_{\text{policy eval.~error at next layer}} \\
    &\qquad\qquad + ~ \underbrace{ \abs*{ \En_{x' \sim \projm{}(\optlatp)} \brk{\wh{V}^\pi(x')}- \En_{x' \sim \estp}\brk{\wh{V}^\pi(x')} } }_{\text{transition error}}
\end{align*}
This decomposition generalizes Eq.~\eqref{eq:confidence-set-warmup} to the stochastic BMDP setting. To obtain a bound on the projection error, we observe that pushforward concentrability implies that the observations sampled from $\mu$ are sufficiently representative of observations from the transition $P$, and therefore $\projm{}(\optlatp)$ approximates $P$ well. Similar to the analysis of \detdecoder{}, a bound on the policy evaluation error at the next layer can be shown via induction. Lastly, our analysis shows that the construction \eqref{eq:confidence-set-construction-v2} admits a bound on the transition error.



\begin{algorithm}[t]
    \caption{\stochrefit}\label{alg:stochastic-refit-v2}
        \begin{algorithmic}[1]
            \addtolength{\abovedisplayskip}{-5pt}
            \addtolength{\belowdisplayskip}{-5pt}
            \Require Layer $h$, policy emulator $\estmdp$, confidence sets $\cP$, test policies $\crl{\Pitest_h}_{h\in[H]}$, parameters $\eps > 0$ and $\delta \in (0,1)$.
            \State Set $\tauref \coloneqq 80 \cdot H \eps$, $\nmc \asymp \tfrac{1}{\eps^2} \cdot \log \tfrac{\cpush SAH \abs{\Pi}}{\eps \delta}$
            \For {every $(\empobs, \empobs') \in \estmdpobsspace{h} \times \estmdpobsspace{h}$}: \hfill \algcomment{Construct candidate test policies at layer $h$}
            \State Define $\pitest{\empobs}{\empobs'} \gets \argmax_{\pi \in \cA \circ \Pi_{h+1:H}} \abs{ \wh{V}^\pi(\empobs) - \wh{V}^\pi(\empobs') }$.
            
            \State Estimate:\hfill \algcomment{Verify accuracy of test policies}\label{line:mc1}
            \begin{align*}
                \vestarg{\empobs}{\pitest{\empobs}{\empobs'}} \gets \mcest(\empobs, \pitest{\empobs}{\empobs'}, \nmc), \quad \vestarg{\empobs'}{\pitest{\empobs}{\empobs'}} \gets \mcest(\empobs', \pitest{\empobs}{\empobs'}, \nmc)
            \end{align*}
            \EndFor
            \State Set \(\violations \gets \crl{(x,\pi) \text{~estimated in \pref{line:mc1} such that~} \abs{\vestarg{x}{\pi} - \wh{V}^\pi(x)} \ge \tauref }\). 
            \If { \(\violations = \varnothing\) } \hfill \algcomment{No violations found, so return test policies.}   
            \State Set $\Pitest_h = \cup_{x, x' \in \estmdpobsspace{h}} \crl{\pitest{\empobs}{\empobs'}}$ and \textbf{Return} $(h-1, \estmdp,\cP, \crl{\Pitest_h}_{h\in[H]})$.\label{line:great-success-stochastic}
            \Else \label{line:else-triggered} \hfill \algcomment{Refit transitions to handle violations}
            \For {every $(\empobs, \pi) \in \violations$}\label{line:for-every}
            \State \textbf{for} {each $(\bar{\empobs}, \bar{a}) \in \estmdpobsspace{h:H} \times \actionsp$}: Estimate $\qestarg{\bar{\empobs}, \bar{a}}{\pi} \gets \mcest(\bar{\empobs}, \bar{a} \circ \pi, \nmc)$. \label{line:mc2}
            \State Define for every $(\bar{\empobs}, \bar{a}) \in \estmdpobsspace{h:H} \times \actionsp$:
            \begin{align*}
                \Delta(\bar{\empobs}, \bar{a}) \coloneqq \wh{R}(\bar{\empobs},\bar{a}) + \En_{\empobs' \sim \estp(\cdot  \mid  \bar{\empobs},\bar{a})} \brk*{\qestarg{\empobs', \pi(\empobs')} {\pi} } - \qestarg{\bar{\empobs}, \bar{a}}{\pi}.
            \end{align*}
            \For {every $(\bar{\empobs}, \bar{a})$ such that $\abs{ \Delta(\bar{\empobs}, \bar{a}) } \ge \tauref/(8H)$}:\label{line:bad-obs-stochastic}

            \hspace{2.5em}\algcomment{Define loss vectors, overwriting if already defined.}

            \State Set $\ell_\mathrm{loss}(\bar{\empobs}, \bar{a}) \coloneqq \sign(\Delta(\bar{\empobs}, \bar{a})) \cdot \qestarg{\cdot , \pi(\cdot) }{\pi}\in [0,1]^{\cX_{h(\bar{\empobs})+1}[\estmdp]}$\label{line:loss-vector}
            \EndFor      
            \EndFor

            \hspace{-0.5em}\algcomment{OMD update with negative entropy Bregman Divergence on violations.}

            \State \textbf{for} every $(\bar{\empobs}, \bar{a})$ from  \pref{line:loss-vector}: Update 
            \begin{align*}
                \estp(\cdot  \mid  \bar{\empobs}, \bar{a}) \gets \argmin_{p \in \cP(\bar{\empobs}, \bar{a})}~ \tri*{p, \ell_\mathrm{loss}\prn*{\bar{\empobs}, \bar{a} } } + \frac{1}{\eps} \cdot D_\mathsf{ne}\prn*{p~\Vert~  \estp(\cdot  \mid  \bar{\empobs}, \bar{a})}
            \end{align*}\label{line:loss-vector-obs}
            \State \textbf{Return} $(\ell, \estmdp,\cP, \crl{\Pitest_h}_{h\in[H]})$ where $\ell$ is the maximum layer s.t.~$(\bar{\empobs}, \bar{a}) \in \statesp_\ell \times \cA$ was updated in \pref{line:loss-vector-obs}. \label{line:great-success-stochastic-else}
            \EndIf
        \end{algorithmic}
\end{algorithm}

\subsection{\stochrefit{} Subroutine}

Now we discuss \stochrefit{}. The skeleton is the same as in \detrefit{}: once the transition functions for $\estmdp$ have been estimated for a given layer $h$, \stochrefit{} attempts to compute a set of valid test policies $\Pitest_h$ for pairs of observations (see \pref{def:valid-test-policy}). If it cannot, this implies that at least one transition that we previously estimated in layer $h$ onward must have been incorrectly estimated, and we search for it starting in \pref{line:mc2}. In this case, we revisit the maximum layer where some transition was updated and restart the decoding procedure. 

\paragraph{OMD Regret as a Potential Function.} Our main innovation to control the number of refitting iterations is to design the right potential function. In \detalg{}, since we were working with deterministic transitions, we used the size of $\cP(s,a)$ as the potential function.  Since we are now estimating $\wh{P}(\cdot \mid x,a)$ in a continuous space, this idea does not extend. 

Instead, we use the regret of online mirror descent (OMD) against the competitor vector $\projm{}(\optlatp(\cdot \mid x,a))$ as the potential function. We show that every transition in \pref{line:loss-vector-obs} witnesses constant regret with respect to $\projm{}(P(\cdot \mid x,a))$. In our analysis, we maintain the invariant property that $\cP$ is just big enough so that $\projm{}(P(\cdot \mid x,a)) \in \cP$ throughout the execution of \stochalg{}. Therefore, the standard analysis of OMD (e.g., \pref{thm:omd}) gives us an upper bound on the cumulative regret. Letting $T_\mathrm{refit}$ denote the number of updates on a given $(x,a)$ pair, we have
\begin{align*} 
    \eps \cdot T_\mathrm{refit} \lesssim \text{Regret of OMD} \lesssim \sqrt{\log \abs{\estmdpobsspace{}} \cdot T_\mathrm{refit}}.
\end{align*}
Rearranging, we get a bound on the number of updates $T_\mathrm{refit}$ for any $(x,a)$, and since the total number of states in the policy emulator $\estmdp$ is bounded, we get a bound on the total number of updates made by \stochrefit{}.


\section{Open Problem: Generalizing \pref{thm:block-mdp-result}}\label{sec:hybrid-open-problems}

We state the following open problem, which might be of interest to the reader.

\begin{question}
    Can we extend \pref{thm:block-mdp-result} to more general settings?
\end{question}

The ultimate result here would be to give a sample-efficient learning algorithm for hybrid resets which works when supplied any exploratory reset satisfying bounded concentrability.

Concretely, we believe that the technique of constructing policy emulators can be greatly generalized and simplified. As a starting point, we can show that any (pushforward) coverable MDP admits a policy emulator of bounded size (see \pref{sec:beyond-bmdp} for more details). However, we do not know how to efficiently construct such an emulator in the general setting. One natural class of problems to study is the low-rank MDP, which generalizes the Block MDP and also satisfies low (pushforward) coverability. An algorithm achieving $\poly(d)$ sample complexity would showcase the power of hybrid resets, as prior work \cite{sekhari2021agnostic} shows that $\exp(d)$ sample complexity is necessary and sufficient for agnostic RL in low-rank MDPs with just online access. 

Another direction for improving \pref{thm:block-mdp-result} is replacing the dependence on pushforward concentrability with the (smaller) concentrability. Unfortunately, our guarantee for \stochalg{} breaks down because it uses pushforward concentrability to enable accurate policy emulation of the transitions from every state in the emulator.

\subsection{Existence of Emulators Under Pushforward Coverability}\label{sec:beyond-bmdp}
A natural question to ask is how to generalize \stochalg{} beyond the Block MDP setting. As a starting point for this future research direction,  we can show that every pushforward coverable MDP admits a policy emulator with a bounded state space size. We first define pushforward coverability which posits the existence of a good distribution satisfying pushforward concentrability (c.f.~\pref{def:exploratory-pushforward-distribution}).

\begin{definition}[Pushforward Coverability \cite{mhammedi2024power, amortila2024reinforcement}]\label{def:pushforward-coverability}
The pushforward coverabilty coefficient for an MDP $M$ is
\begin{align*}
    \cpushcov(M) \coloneqq \max_{h \in [H]} \inf_{\mu_h \in \Delta(\statesp_h)} \sup_{(x,a,x') \in \statesp_{h-1} \times \actionsp \times \statesp_{h}} \frac{P(x' \mid x,a)}{\mu_h(x')}.
\end{align*}
When clear from the context we denote the pushforward concentrability coefficient as $\cpushcov$.
\end{definition}

\begin{proposition}[Pushforward Coverable MDPs $\Rightarrow$ Small Policy Emulators]\label{prop:pushforward-coverable}
Let $M$ be an MDP with pushforward coverability coefficient $\cpushcov$ and $\Pi$ be any policy class. Then there exists a policy emulator $\estmdp$ with state space size
\begin{align*}
    \poly \prn*{ \cpushcov, A, H, \eps^{-1}, \log \abs{\Pi}, \log \delta^{-1} }.
\end{align*}
\end{proposition}

A few remarks:
\begin{itemize}
    \item Strictly speaking, the policy emulator we construct in \pref{prop:pushforward-coverable} is not a true MDP, since our construction requires the ``transition'' $\estp(\cdot\mid x_{h-1}, a_{h-1})$ to be an unnormalized measure over the the states in the next layer $\estmdpobsspace{h}$, which may sum to $\cpushcov \ge 1$. Thus, we slightly abuse the notation for expectation:
    \begin{align*}
        \En_{x \sim \estp(\cdot\mid x_{h-1}, a_{h-1})}\brk*{V^\pi(x)} \coloneqq \sum_{ x \in \estmdpobsspace{h}} \estp( \cdot\mid x_{h-1}, a_{h-1}) V^\pi(x).
    \end{align*}
    As discussed, the policy emulator anyways is not guaranteed to be a reasonable approximation of the underlying MDP $M$, just an object which enables uniform policy evaluation, so this issue is minor.
    \item Lemma 3.1 of \cite{amortila2024reinforcement} give a result of similar flavor, which shows that pushforward coverable MDPs are approximately \emph{low-rank}. Their proof, however, seems to be quite different. It relies on the Johnson-Lindenstrauss lemma to construct random embeddings which enable approximation of the Bellman backup operator for any arbitrary value function class $\cF$. 
    \item Unfortunately, we do not know how to leverage hybrid resets to construct such a policy emulator in a statistically efficient manner---the naive way to do so requires sample complexity scaling with $\spancap(\Pi)$ (which could be much larger than $\cpushcov$). We believe this is an interesting direction for future work.
\end{itemize}

\begin{proof}[Proof of \pref{prop:pushforward-coverable}]
We will prove this by explicitly constructing the policy emulator using the same algorithmic template as in \stochalg{}. To construct the state space of the policy emulator, we sample
$\wt{O}(\cpush/\eps^2)$ observations per layer from the distributions $\mu_1, \cdots, \mu_H$, respectively, that witnesses pushforward coverability at every $h \in [H]$. As shown before, the instantaneous rewards $\wh{R}(x,a)$ for every $x \in \estmdpobsspace{} \times \actionsp$ of the emulator can be learned via the local simulator up to $\eps$ accuracy. Now we show that it is possible to define the transition functions $\estp(\cdot \mid x,a)$ for every $x \in \estmdpobsspace{} \times \actionsp$ so that the resulting $\estmdp$ is an $O(\eps)$-accurate policy emulator for $d_1$. We do this inductively:

\begin{claim}
    Let $\Gamma_h > 0$. Suppose that at layer $h \in [H]$:
    \begin{align*}
        \forall~x \in \estmdpobsspace{h},~\forall~\pi \in \Pi:\quad \abs*{V^\pi(x) - \wh{V}^\pi(x)} \le \Gamma_h.
    \end{align*}
    Then for every $(x_{h-1},a_{h-1}) \in \estmdpobsspace{h-1} \times \actionsp$, there exists some $\estp \in \Delta(\estmdpobsspace{h})$ such that
    \begin{align*}
        \forall~\pi \in \Pi:\quad \abs*{Q^\pi(x_{h-1},a_{h-1}) - \wh{R}(x_{h-1}, a_{h-1}) - \En_{x \sim \estp} \brk{\wh{V}^\pi(x)} } \le \Gamma_h + 2 \eps.
    \end{align*}
\end{claim}
Applying this claim backwards from $h=H, \cdots, 1$ and using the fact that $\Gamma_H = \eps$ proves \pref{prop:pushforward-coverable}.

It remains to prove the claim. Let $\estp$ be an unnormalized measure over $\estmdpobsspace{}$ (to be defined later). First, we apply the decomposition
\begin{align*}
    \hspace{2em}&\hspace{-2em} \abs*{Q^\pi(x_{h-1},a_{h-1}) - \wh{R}(x_{h-1}, a_{h-1}) - \En_{x \sim \estp} \brk{\wh{V}^\pi(x)} }\\
    &=~\abs*{R(x_{h-1}, a_{h-1}) - \wh{R}(x_{h-1}, a_{h-1})} + \abs*{ \En_{x \sim P} \brk{V^\pi(x)} - \En_{x \sim \estp} \brk*{V^\pi(x)} } \\
    &\qquad + \abs*{ \En_{x \sim \estp} \brk{V^\pi(x)} - \En_{x \sim \estp} \brk{\wh{V}^\pi(x)} } \\
    &\le~ \Gamma_h + \eps + \abs*{ \En_{x \sim P} \brk{V^\pi(x)} - \En_{x \sim \estp} \brk{V^\pi(x)} }, \numberthis\label{eq:upperbound-pushforward-cov}
\end{align*}
where the last inequality uses the reward estimation accuracy and the assumption in the claim. To control the last term, we apply a change of measure:
\begin{align*}
    \En_{x \sim P} \brk*{V^\pi(x)} = \En_{x \sim \mu_h} \brk*{\frac{P(x \mid x_{h-1}, a_{h-1})}{\mu_h(x)} \cdot V^\pi(x)}.
\end{align*}
Observe that $\estmdpobsspace{h} = \crl{x_h^{(1)}, \cdots, x_h^{(n)} }$ are drawn i.i.d.~from $\mu_h$, and by pushforward coverability, the importance ratio $P(x \mid x_{h-1}, a_{h-1})/\mu_h(x) \le \cpushcov$. Via a standard uniform convergence bound, with probability at least $1-\delta$, for every $\pi \in \Pi$
\begin{align*}
    \bigg| \En_{x \sim \mu_h} \brk*{\frac{P(x \mid x_{h-1}, a_{h-1})}{\mu_h(x)} \cdot V^\pi(x)} - \sum_{i=1}^n \underbrace{ \frac{P(x_h^{(i)} \mid x_{h-1}, a_{h-1})}{n \cdot \mu_h(x_h^{(i)})} }_{\eqqcolon \estp(x_h^{(i)})} \cdot V^\pi(x_h^{(i)}) \bigg| \le \eps. 
\end{align*}
Plugging back our choice of $\estp$ into Eq.~\eqref{eq:upperbound-pushforward-cov} proves the claim.
\end{proof}

\section{Deferred Proofs}\label{sec:hybrid-deferred-proofs}
\subsection{Proof for the Warmup Algorithm \detalg{}}\label{sec:proof-warmup}

In this section, we prove the following sample complexity guarantee for \detalg{}:
\begin{reptheorem}{thm:det-bmdp-solver-guarantee}
    Let $\eps, \delta \in (0,1)$ be given and suppose that \pref{ass:det-transitions} holds. Then with probability at least $1-\delta$, \detalg{} (\pref{alg:det-bmdp-solver}) finds an $\eps$-optimal policy using 
    \begin{align*}
        \wt{O}\prn*{\frac{S^5A^2H^5}{\eps^2} \cdot \log\frac{1}{\delta}} \quad\text{samples.}
    \end{align*}
\end{reptheorem}

\subsubsection{Proof of \pref{thm:det-bmdp-solver-guarantee}}
Our high-level strategy is to apply the inductive argument outlined in \pref{sec:warmup-algorithm-analysis-sketch} to control the growth of the Bellman error for all $(s,a) \in \latentsp_h \times \actionsp$ as we construct $\estlatentmdp$ from layer $H$ backwards. Recall our Bellman error decomposition:
\begin{align*}
    \abs*{Q^\pi(s,a) - \wh{Q}^\pi(s,a)} \le \underbrace{\abs*{\optlatr - \estlatr}}_{\text{reward error}} + \underbrace{\abs*{\wh{V}^\pi(\optlatp) - \wh{V}^\pi(\estlatp)}}_{\text{transition error}} + \underbrace{\abs*{V^\pi(\optlatp) - \wh{V}^\pi(\optlatp)}}_{\text{error at next layer}}. \tag{\ref{eq:bellman-error-main}}
\end{align*}

To control the transition error of Eq.~\eqref{eq:bellman-error-main}, we introduce a notion of test policy validity and give a lemma which shows that if \detdecoder{} is equipped with valid test policies, the transition estimation error can be bounded.

\begin{definition}[Test Policy Validity, Deterministic Version]\label{def:valid-test-policy-dd} Let $\eta > 0$ be a parameter. At layer $h \in [H]$, we say a collection of partial policies $\Pitest_{h} = \crl{\pi_{s,s'} \in \Pi_{h:H}: s, s' \in \latentsp_h}$ is an $\eta$-\emph{valid test policy set} for the estimated latent MDP $\estlatentmdp$ if for every $s, s' \in \latentsp_h$:
\begin{itemize}
    \item (Maximally distinguishing): $\pi_{s,s'} = \argmax_{\pi \in \Pi_{h:H}} \abs{\wh{V}^\pi(s) - \wh{V}^\pi(s')}$.
    \item (Accurate): $\abs{V^{\pi_{s,s'}}(s) - \wh{V}^{\pi_{s,s'}}(s)} \le \eta$ and $\abs{V^{\pi_{s,s'}}(s') - \wh{V}^{\pi_{s,s'}}(s')} \le \eta$.
\end{itemize}
\end{definition}

\begin{lemma}[Decoding]\label{lem:controlling-transition-error}
Fix any layer $h \in [H-1]$. Suppose that \detdecoder{} (\pref{alg:decoder}) is equipped with a $\tauref$-valid test policy $\Pitest_{h+1}$. Fix any tuple $(s_{h}, a_{h})$ and assume that $\optlatp(s_{h}, a_{h}) \in \cP(s_{h}, a_{h})$. With high probability, \detdecoder{} returns an updated $\cP$ such that: 
\begin{enumerate}[(1)]
        \item $\optlatp(s_{h}, a_{h}) \in \cP$;
        \item For every $\bar{s} \in \cP$ we have $\max_{\pi \in \Pi}~\abs{\wh{V}^\pi(\optlatp(s_{h}, a_{h})) - \wh{V}^\pi(\bar{s})} \le 7\tauref/2$. 
\end{enumerate}
\end{lemma}
The proof of \pref{lem:controlling-transition-error} is deferred to \pref{sec:induction-lemmas}. 

In light of \pref{lem:controlling-transition-error}, as long as we have valid test policy sets $\crl{\Pitest_h}_{h\in[H]}$, \pref{lem:controlling-transition-error} provides control on the transition estimation error, and we can iteratively apply Eq.~\eqref{eq:bellman-error-main} to get the final bound on estimation error at layer 1.

\paragraph{Computing Test Policies via \detrefit{}.} Now we will analyze \detrefit{}. By standard concentration arguments, if \pref{line:great-success} is triggered, the test policies must be $\tauref$-accurate; furthermore, they are maximally distinguishing by construction. Unfortunately, since we require the test policies to satisfy a higher level of accuracy $\tauref$, due to estimation errors in $\estlatentmdp$, it may not be possible to find any valid test policies. To address this, we observe that inaccurate test policies act as a ``certificate'' and allow us to search for some transition $\estlatp \ne \optlatp$.

\begin{lemma}[Refitting]\label{lem:refitting}
Let $\eps > 0$ be given. Suppose that at layer $h \in [H]$, \detrefit{} (\pref{alg:refit}) is supplied confidence sets $\cP$ such that for all $(s,a) \in \latentsp_{h:H} \times \actionsp$ we have $\optlatp(s,a) \in \cP(s,a)$. If \detrefit{} terminates at \pref{line:great-success-2}, then with high probability:
    \begin{enumerate}[(1)]
        \item At least one $\estlatp$ was removed from its confidence set $\cP$.
        \item No ground truth transitions $\optlatp$ are removed from their confidence set $\cP$.
    \end{enumerate}
\end{lemma}
The proof of \pref{lem:refitting} is deferred to \pref{sec:induction-lemmas}. 

Our analysis will track the invariant that the confidence sets $\cP$ always contain the ground truth transition $\optlatp$. Therefore, \pref{lem:refitting} allows us to use the size of the confidence sets as a potential function: if \detrefit{} fails to compute valid test policies at some layer $h$, we must delete some incorrect transition $\estlatp$ from its set $\cP$; this process cannot continue indefinitely, since we can delete at most $S(S-1)A$ states.

\paragraph{Proof by Induction.}
With \pref{lem:controlling-transition-error} and \pref{lem:refitting} in hand, we can show the final bound in \pref{thm:det-bmdp-solver-guarantee}. For technical convenience, we will show that the policy returned by \detalg{} is $O(\eps)$ suboptimal; rescaling the parameter $\eps$ does not change the final sample complexity apart from constant factors. Also, we omit the standard arguments (via concentration and union bound) which show that the conclusions of \pref{lem:controlling-transition-error} and \pref{lem:refitting} hold with probability at least $1-\delta$ over the randomness of sampling episodes from the MDP.

Take $\Gamma_h \coloneqq C(H-h+1)/H \cdot \eps$ for some suitably large constant $C > 0$. We will inductively show that these properties hold for all layers $h\in [H]$:
\begin{enumerate}[(A)]
    \item \emph{Policy Evaluation Accuracy.} For all pairs $(s,a) \in \latentsp_h \times \actionsp$ and $\pi \in \Pi_\mathsf{open}$: $\abs{Q^\pi(s,a) - \wh{Q}^\pi(s,a)} \le \Gamma_h$.
    \item \emph{Confidence Set Validity.} For all pairs $(s,a)\in \latentsp_h \times \actionsp$, we have $\optlatp(s,a) \in \cP(s,a)$.
    \item \emph{Test Policy Validity.} $\Pitest_{h}$ are $\tauref$-valid for $\estlatentmdp$ at layer $h$.
\end{enumerate}

To analyze \detalg{}, we will show that these properties always hold throughout at the end of every while loop for all layers $h > \ell_\mathsf{next}$.

\underline{Base Case.} We analyze the first loop with $\ell = H$. Note that (A) holds by concentration of the reward estimates, and (B) trivially holds because there are no transitions to be constructed at layer $H$. Now we investigate what happens when \detrefit{} is called. The computed test policies take the form $\pi_{s,s'} \equiv a$ for some $a \in \cA$; again by concentration of the reward estimates, \pref{line:great-success} of \detrefit{} is triggered. Therefore (A)--(C) hold after refitting, and we jump to $\ell_\mathsf{next} = H-1$.


\underline{Inductive Step.} Suppose the current layer index is $\ell$, and that properties (A)--(C) hold for all $h > \ell$. By \pref{lem:controlling-transition-error}, the updated transition confidence sets returned by \detdecoder{} at layer $\ell$ will satisfy (B). Furthermore, at the end of \pref{line:est-transition},  the error decomposition \eqref{eq:bellman-error-main} implies that for every $(s,a) \in \cS_\ell \times \actionsp$:
    \begin{align*}
        &\max_{\pi \in \Pi}~\abs*{Q^\pi(s, a) - \wh{Q}^\pi(s, a)} \le \Gamma_{\ell+1} + \frac{\eps}{H^2} + \frac{7\tauref}{2} \le \Gamma_{\ell}, \quad \Longrightarrow \quad  \text{Property (A) holds at layer $\ell$.}
    \end{align*}
Now we do casework on the outcome of \detrefit{}.
\begin{itemize}
    \item \textbf{Case 1: Return in \pref{line:great-success}.}  By construction, property (C) is satisfied for layer $\ell$. In this case, since \pref{alg:refit} made no updates to $\estlatentmdp$ or $\cP$, properties (A) and (B) continue to hold at layer $\ell$ onwards.
    \item \textbf{Case 2: Return in \pref{line:great-success-2}.}  By \pref{lem:refitting}, any updates to $\estlatentmdp$ maintain property (B). Let $\ell_\mathsf{next}$ denote the layer at which we jump to. By definition of $\ell_\mathsf{next}$, we made no updates to $\estlatentmdp$ at layers $\ell_\mathsf{next} + 1$ onwards, and therefore the previously computed test policies $\Pitest_{\ell_\mathsf{next}+1:H}$ must still be valid, so therefore properties (A) and (C) continue to hold at layer $\ell_\mathsf{next}$ onwards.
\end{itemize}

Continuing the induction, once $\ell \gets 0$ is reached in \detalg{} (which we know will eventually happen because Case 2 can only occur for $S^2A$ times), the estimated latent MDP $\estlatentmdp$ must satisfy the bound
\begin{align*}
    \max_{\pi \in \Pi}~\abs*{V^\pi(s_1) - \wh{V}^\pi(s_1)} \le \Gamma_1 = O(\eps).
\end{align*}



\paragraph{Sample Complexity Bound.} We now compute the final sample complexity required by \detalg{}:
\begin{itemize}
    \item Estimating rewards in the main algorithm uses $\wt{O}(H^4SA/\eps^2)$ samples.
    \item \detdecoder{} is called at most $SA \times S^2 A$ times, since we (re-)decode every transition $(s,a)$ at most $S^2A$ times. Every call to \detdecoder{} uses $\wt{O}(S^2/\tauref^2) = \wt{O}(S^2H^2/\eps^2)$ samples since we take $\tauref = 2^5 \cdot \eps/H$ in \pref{lem:refitting}. Therefore the total number of samples used by \detdecoder{} is at most $\wt{O}(S^5A^2H^2/\eps^2)$.
    \item \detrefit{} is called at most $S^2 AH$ times, since associated to every layer revisiting is an additional $H$ calls in the main while loop. In every call to \detrefit{}, we  use $\wt{O}(S^2H^2/\eps^2)$ calls to compute and verify the test policy set in \pref{line:eval-policy}. In addition, every time \pref{line:violation-statement} is triggered corresponds to at least one deletion in \pref{line:bad-state-2}, so the number of additional samples used by \pref{line:mc-additional} (across all calls to \detrefit{}) can be bounded by $\wt{O}(S^2 A H^3/\tauref^2) = \wt{O}(S^2 A H^5/\eps^2)$.
\end{itemize}  
Thus the final sample complexity is at most $\wt{O}\prn{S^5A^2H^5/\eps^2}$ samples. \qed

\subsubsection{Proof of Induction Lemmas }\label{sec:induction-lemmas}

\begin{proof}[Proof of \pref{lem:controlling-transition-error}]
First we prove implication (1). Let us denote $s^\star = \optlatp(s_h, a_h)$. If $s^\star \notin \cP$ (the returned set), then there exists some $s'$ for which
\begin{align*}
    \abs*{\vestarg{x_{h+1}}{\pi_{s^\star, s'}} - \wh{V}^{\pi_{s^\star, s'}}(s^\star)} \ge 2\tauref.
\end{align*}
However, by assumption of test policy accuracy we know that
\begin{align*}
    \abs*{V^{\pi_{s^\star, s'}}(s^\star) - \wh{V}^{\pi_{s^\star, s'}}(s^\star)} \le \tauref.
\end{align*}
Since the quantity $\vestarg{x_{h+1}}{\pi_{s^\star, s'}}$ is an unbiased estimate of $V^{\pi_{s^\star, s'}}(s^\star)$ which is estimated to accuracy $\tauref/2$ we have a contradiction, so $s^\star \in \cP$.

Now we prove implication (2). If $\bar{s} \in \cP$, then we must have
\begin{align*}
   \abs*{\vestarg{x_{h+1}}{\pi_{s^\star, \bar{s}}} - \wh{V}^{\pi_{s^\star, \bar{s}}}(\bar{s})} = \abs*{\vestarg{s^\star}{\pi_{s^\star, \bar{s}}} - \wh{V}^{\pi_{s^\star, \bar{s}}}(\bar{s})} \le 2\tauref. 
\end{align*}
 Since we estimated $\vestarg{s^\star}{ \pi_{s^\star, \bar{s}}}$ up to $\tauref/2$ accuracy we know that
 \begin{align*}
     \abs*{V^{\pi_{s^\star, \bar{s}}}(s^\star) - \wh{V}^{\pi_{s^\star, \bar{s}}}(\bar{s})} \le 5\tauref/2, \quad \Longrightarrow \quad  \abs*{\wh{V}^{\pi_{s^\star, \bar{s}}}(s^\star) - \wh{V}^{\pi_{s^\star, \bar{s}}}(\bar{s})} \le 7\tauref/2,
 \end{align*}
 where the implication follows by the accuracy of $\Pitest_{h+1}$.
By the maximal distinguishing property of $\Pitest_{h+1}$, observe that the LHS of the above implication is equal to $\max_{\pi \in \Pi}~\abs{\wh{V}^{\pi}(s^\star) - \wh{V}^{\pi}(\bar{s})}$. This proves the second implication, and concludes the proof of \pref{lem:controlling-transition-error}.
\end{proof}

\begin{proof}[Proof of \pref{lem:refitting}]
We show the first implication. Let $(s_h,\pi)$ be any policy which satisfies $\abs{\vestarg{s_h}{\pi} - \wh{V}^\pi(s_h)} \ge \tauref - \eps/H$. Since we estimated $V^\pi(s_h)$ up to $\eps/H$ error, we have $\abs{V^\pi(s_h) - \wh{V}^\pi(s_h)} \ge \tauref - 2\eps/H$. Let $\bar{s}_h = s_h, \bar{s}_{h+1}, \cdots, \bar{s}_H$ be the sequence of states which are obtained by running $\pi$ on $\estlatentmdp$ starting at $s_h$. 

For sake of contradiction suppose that
\begin{align*}
    \abs*{\vestarg{\bar{s}}{\pi} - \wh{R}(\bar{s}, \pi) - \vestarg{\estlatp(\bar{s}, \pi)}{\pi} } \le \frac{4\eps}{H^2}, \quad \text{for all}~\bar{s} \in \crl{\bar{s}_h,\cdots, \bar{s}_H}.
\end{align*}
Since we estimated every $\vestarg{\cdot}{\pi}$ up to accuracy $\eps/H^2$ we see that
\begin{align*}
    \abs*{V^\pi(\bar{s}) - \wh{R}(\bar{s}, \pi) - V^\pi(\estlatp(\bar{s}, \pi)) } \le \frac{6\eps}{H^2}, \quad \text{for all}~\bar{s} \in \crl{\bar{s}_h,\cdots, \bar{s}_H}.
\end{align*}
By Performance Difference Lemma and applying the previous display recursively,
\begin{align*}
     \abs*{V^\pi(s_h) - \wh{V}^\pi(s_h)} &\le \abs*{V^\pi(\bar{s}_h) - \wh{R}(\bar{s}_h, \pi) - V^\pi(\bar{s}_{h+1})} + \abs*{ V^\pi(\bar{s}_{h+1}) - \wh{V}(\bar{s}_{h+1}) } \\
     &\le \frac{6\eps}{H^2} + \abs*{ V^\pi(\bar{s}_{h+1}) - \wh{V}(\bar{s}_{h+1}) } \le \cdots \le \frac{6\eps}{H}.
\end{align*}
This contradicts the statement that $\abs*{V^\pi(s_h) - \wh{V}^\pi(s_h)} \ge \tauref - 2\eps/H$ by the choice of $\tauref$. So we can conclude that there exists a state $\bar{s} \in \crl{\bar{s}_h,\cdots, \bar{s}_H}$ such that
\begin{align*}
    \abs*{\vestarg{\bar{s}}{\pi} - \wh{R}(\bar{s}, \pi) - \vestarg{\estlatp(\bar{s}, \pi)}{\pi} } \ge \frac{4\eps}{H^2},
\end{align*}
so therefore \pref{line:bad-state-2} is executed at least once, proving the first implication.

To prove the second implication, consider any $(\bar{s}, \pi)$ for which \pref{line:bad-state-2} is executed. We know that
\begin{align*}
    \abs*{ V^\pi(\bar{s}) - \wh{R}(\bar{s}, \pi) - V^\pi(\estlatp(\bar{s}, \pi)) } \ge \frac{2\eps}{H^2}.
\end{align*}
Recall that for all $(s,a)$, the estimation error on the rewards was $\abs{R(s,a) - \wh{R}(s,a)} \le \eps/H^2$. Therefore
\begin{align*}
    \abs*{ V^\pi(\bar{s}) - R(\bar{s}, \pi) - V^\pi(\estlatp(\bar{s}, \pi)) } \ge \frac{\eps}{H^2}, \quad \Longrightarrow \quad  \estlatp(\bar{s},\pi) \ne \optlatp(\bar{s}, \pi).
\end{align*}
Therefore as claimed we always delete $\estlatp(\bar{s},\pi) \ne \optlatp(\bar{s}, \pi)$ in \pref{line:bad-state-2} of \detrefit{}.
\end{proof}

\subsection{Proof of Main Upper Bound}\label{sec:main-upper-bound-proofs}
In this section, we prove \pref{thm:block-mdp-result}.


    
\subsubsection{Preliminaries}\label{sec:upper-bound-preliminaries}
We will define some additional concepts and notation which will be used in the analysis.

\begin{itemize}
    \item For any set $\cX' \subseteq \cX$ we denote the represented states as $\cS[\cX'] \coloneqq \crl{\optdec(x): x \in \cX'}$. For any latent state $s \in \cS$ and subset $\cX' \subseteq \cX$ we let $n_s[\cX'] \coloneqq \abs{\crl{x \in \cX': \optdec(x) = s}}$ count the total number of observations there are emitted from $s$. 
    \item We define the set of $\eps$-pushforward-reachable latent states
    \begin{align*}
        \epsPRS{h} \coloneqq \crl*{s_h: \max_{s_{h-1}, a_{h-1}} \optlatp(s_h \mid  s_{h-1}, a_{h-1}) \ge \frac{\eps}{S} },
    \end{align*}
    and furthermore let $\epsPRS{} \coloneqq \cup_{h=1}^H \epsPRS{h}$.
    \item For any $\cX' \subseteq \cX$, we let $\nrch[\cX'] \coloneqq \abs{\crl{x \in \cX': \optdec(x) = \epsPRS{}}}$ and $\nurch[\cX'] = \abs{\cX'} - \nrch[\cX']$.
\end{itemize}

\paragraph{Estimated Transitions and Projected Measures.} Recall that the ground truth latent transition is denoted $\optlatp: \latentsp \times \actionsp \to \Delta(\latentsp)$. We will use $\samplelatp$ to denote the empirical version of the latent transition which is sampled in \pref{line:sample-decode-dataset} of \pref{alg:stochastic-decoder-v2}:
\begin{align*}
    \samplelatp(\cdot  \mid  \optdec(x_{h}), a_{h}) = \frac{1}{\ndec} \sum_{x \in \cD} \delta_{\optdec(x)}.
\end{align*}
In addition, we introduce a notion of projected measures which will be used to relate the ground truth transition $P = \emission \circ \optlatp$ with the estimated transition $\estp$ of the policy emulator. While our algorithm never directly uses the projected measure, we track it in the analysis.

\begin{definition}[Projected Measure]\label{def:projected-measure}
For a distribution $p \in \Delta(\cS)$, define the \emph{projected measure} onto the observation set $\bar{\cX} \subseteq \statesp$ as
\begin{align*}
    \projm{\bar{\cX}}(p) \coloneqq \sum_{s \in \epsPRS{}} p (s) \cdot \unif (\crl{\empobs \in \bar{\cX}: \optdec(\empobs) = s}).
\end{align*}
Specifically, for any $\empobs \in \bar{\cX}$ we have:
\begin{align*}
    \projm{\bar{\cX}}(p)(\empobs) = p (\optdec(\empobs)) \cdot \frac{\ind{\optdec(\empobs) \in \epsPRS{}}}{n_{\optdec(\empobs)}[\bar{\cX}]}.
\end{align*}
Furthermore, for any subset $\bar{\cX}' \subseteq \bar{\cX}$ we denote $\projm{\bar{\cX}}(p)(\bar{\cX}') = \sum_{\empobs \in \bar{\cX}'} \projm{\bar{\cX}}(p)(\empobs)$.
\end{definition}
Formally, $\projm{\bar{\cX}}$ is not a true probability distribution, as the total measure might not sum up to 1. This would happen if $p(s) > 0$ for $s \in (\epsPRS{})^c$.

\textit{Remark.} In \pref{thm:block-mdp-result}, we assume that the distribution $\mu$ is factorizable. This can be removed with some extra work. One can modify the definition of the projected measure to replace the uniform distribution over observations with some other suitable importance-reweighted distribution; the existence of such distribution with desirable properties that allow concentration of the pushforward policies can be shown using pushforward concentrability (i.e., in \pref{lem:pushforward-concentration}).

\paragraph{Test Policy Validity.}
In our analysis, we will modify \pref{def:valid-test-policy-dd} as below.
\begin{definition}[Valid Test Policy]\label{def:valid-test-policy}
For a layer $h \in [H]$, we say a collection of policies $\Pitest_{h} = \crl*{\pi_{\empobs, \empobs'}}_{\empobs, \empobs' \in \estmdpobsspace{h}}$
    is a $\eta$-\emph{valid test policy set} for policy emulator $\wh{M}$ if the following hold.
    \begin{itemize}
        \item (Maximally distinguishing): $\pi_{\empobs, \empobs'} = \argmax_{\pi \in \cA \circ \Pi_{h+1:H}} \abs*{\wh{V}^\pi(\empobs) - \wh{V}^\pi(\empobs')}$.
        \item (Accurate): For all $\empobs, \empobs' \in \estmdpobsspace{h}$:
        \begin{align*}
            \abs*{V^{\pi_{\empobs, \empobs'}}(\empobs) - \wh{V}^{\pi_{\empobs, \empobs'}}(\empobs)} &\le \eta \quad \text{and} \quad    \abs*{V^{\pi_{\empobs, \empobs'}}(\empobs') - \wh{V}^{\pi_{\empobs, \empobs'}}(\empobs')} \le \eta.
        \end{align*}
    \end{itemize}
\end{definition}

\subsubsection{Supporting Technical Lemmas for Sampling}

In this section, we establish several technical lemmas which show that various conditions that we need in the analysis hold with high probability under samples from $M$.

\paragraph{Properties of Policy Emulator Initialization.} We prove several properties that hold with high probability when the policy emulator is initialized in \pref{line:init-start}-\ref{line:init-end} of \pref{alg:stochastic-bmdp-solver-v2}.

\begin{lemma}[Sampling of Pushforward-Reachable States]\label{lem:sampling-states}
With probability at least $1-\delta$:
\begin{align*}
    \forall~ h \in [H], \forall~ s \in 
\epsPRS{h}: \quad n_{s}[\estmdpobsspace{h}] \ge \frac{\eps}{2 \cpush S}\cdot \nreset.
\end{align*}
\end{lemma} 

\begin{proof}
Fix any $s \in \epsPRS{h}$. For any $i \in [\nreset]$, let $Z^{(i)}$ be the indicator variable of whether observation $x^{(i)}_h \sim \mu_h$ satisfies $\optdec(x^{(i)}_h) = s$. We know that $\En[Z^{(i)}] \ge \eps/(\cpush S)$. By Chernoff bounds we have
\begin{align*}
    \Pr \brk*{ \frac{1}{\nreset} \sum_{i=1}^{\nreset} Z^{(i)} \le \frac{1}{2} \cdot \frac{\eps}{\cpush S}} \le \exp \prn*{-\frac{\nreset \cdot \eps}{8\cpush S}},
\end{align*}
so as long as
\begin{align*}
    \nreset \ge \frac{8\cpush S}{\eps} \log \frac{SH}{\delta},
\end{align*}
by union bound, the conclusion of the lemma holds.
\end{proof}

\begin{lemma}[Pushforward Policy Concentration over $\mu$]\label{lem:pushforward-concentration} Suppose that the conclusion of \pref{lem:sampling-states} holds. Then with probability at least $1-\delta$:
\begin{align*}
    \forall~ h \in [H],& \forall~ (s, a) \in \epsPRS{h} \times \cA:\\
    &\max_{\pi \in \Pi}~\abs*{ \brk*{\pi \push \emission(s)}(a) - \brk*{\pi \push \unif(\crl{\empobs \in \estmdpobsspace{h}: \optdec(\empobs) = s})}(a) } \le \frac{\eps}{A}.
\end{align*}
\end{lemma}

\begin{proof}
Fix any $(s, a) \in \epsPRS{h} \times \cA$. Also fix any policy $\pi \in \Pi$. Denote the set $\cX_{s} = \crl{\empobs \in \estmdpobsspace{h}: \optdec(\empobs) = s}$, and observe that $\cX_{s}$ is drawn i.i.d.~from the emission distribution $\emission(s)$. By Hoeffding bounds we have
\begin{align*}
    \hspace{2em}&\hspace{-2em} \Pr\brk*{ \abs*{\brk*{\pi \push \emission(s)}(a) - \brk*{\pi \push \unif(\crl{\empobs \in \estmdpobsspace{h}: \optdec(\empobs) = s})}(a)} \ge \frac{\eps}{A}} \\
    &\le 2\exp \prn*{ -\frac{2 n_{s}[\estmdpobsspace{h}] \eps^2}{A^2} } \le  2\exp \prn*{ -\frac{\nreset \eps^3}{\cpush S A^2} }, &\text{(\pref{lem:sampling-states})}
\end{align*}
Applying union bound we see that as long as 
\begin{align*}
    \nreset \ge \frac{\cpush SA^2}{\eps^3} \cdot \log \frac{2SAH\abs{\Pi}}{\delta}
\end{align*}
the conclusion of the lemma holds.
\end{proof}

\begin{lemma}[Sampling Rewards]\label{lem:sampling-rewards}
With probability at least $1-\delta$, every reward estimate $\wh{R}(\empobs,a)$ computed in  \pref{line:reward-estimation} of \pref{alg:stochastic-bmdp-solver-v2} satisfies $\abs{\wh{R}(\empobs, a) - R(\empobs,a)} \le \eps/H$.
\end{lemma} 
\begin{proof}
This follows by Hoeffding inequality and union bound over all $\nreset\cdot AH$ pairs $(\empobs, a) \in \estmdpobsspace{} \times \cA$.
\end{proof}

\paragraph{Properties of \stochdecoder{}.} Now we turn to analyzing a single call to \stochdecoder{}.

\begin{lemma}[Sampling Transitions]\label{lem:sampling-transitions} Fix any $(x_{h}, a_{h})$ for which we call $\stochdecoder{}$.  With probability at least $1-\delta$, the dataset $\cD$ sampled in \pref{line:sample-decode-dataset} of \pref{alg:stochastic-decoder-v2} satisfies
\begin{align*}
    \nrm*{\optlatp(\cdot \mid x_{h}, a_{h}) - \samplelatp(\cdot \mid x_{h}, a_{h})}_1 \le \eps.
\end{align*}
\end{lemma}
\begin{proof}
Every time a dataset $\cD$ is sampled, by concentration of discrete distributions we have for any $t > 0$:
\begin{align*}
    \Pr\brk*{\nrm*{\optlatp(\cdot \mid x_h, a_h) - \samplelatp(\cdot \mid x_h, a_h)}_1 \ge \sqrt{S} \cdot \prn*{\frac{1}{\sqrt{\ndec}} + t}} \le \exp(-\ndec t^2).
\end{align*}
Setting the RHS to $\delta$ we have that with probability at least $1-\delta$, 
\begin{align*}
    \nrm*{\optlatp(\cdot \mid x_h, a_h) - \samplelatp(\cdot \mid x_h, a_h)}_1 \le \sqrt{\frac{S \log (1/\delta)}{\ndec}}. 
\end{align*}
Therefore as long as 
\begin{align*}
    \ndec \ge \frac{S}{\eps^2} \cdot \log \frac{1}{\delta},
\end{align*}
the conclusion of the lemma holds.
\end{proof}

\begin{corollary}\label{corr:sampling-transitions} If the conclusion of \pref{lem:sampling-transitions} holds, then the proportion of observations from $(\epsPRS{h})^c$ in $\cD$ is at most $2\eps$.
\end{corollary}

\begin{lemma}[Pushforward Policy Concentration over Transitions]\label{lem:pushforward-concentration-transition}
Fix any $(x_h, a_h)$ for which we call $\stochdecoder{}$. With probability at least $1-\delta$, the dataset $\cD$ sampled in \pref{line:sample-decode-dataset} of \pref{alg:stochastic-decoder-v2} satisfies
\begin{align*}
    \forall~ s \in \cS[\cD],&\forall~a \in \cA, \forall~ \pi \in \Pi:\\
                            &    \Big| \brk*{\pi \push \emission(s)}(a) - \brk*{\pi \push \unif \prn*{\crl*{ \empobs \in \cD: \optdec(\empobs) = s } } }(a) \Big| \le \sqrt{\frac{2\log (2 SA \abs{\Pi}/\delta)}{n_s[\cD]}}.
\end{align*}
\end{lemma}

\begin{proof}
Fix a particular $s \in \latentsp[\cD]$, $a \in \cA$, and $\pi \in \Pi$. The set $\crl{\empobs \in \estmdpobsspace{h+1}: \optdec(\empobs) = s}$ is drawn i.i.d.~from the emission distribution $\emission(s)$. By Hoeffding bounds we have for any $t > 0$:
\begin{align*}
    \Pr\brk*{ \Big| \brk*{\pi \push \emission(s)}(a) - \brk*{\pi \push \unif\prn*{\crl*{ \empobs \in \cD: \optdec(\empobs) = s } } }(a) \Big| \ge t } &\le 2\exp \prn*{ -2 n_{s}[\cD] t^2}. 
\end{align*}
By union bound over all $(s,a)$ and $\pi$, with probability at least $1-\delta$:
\begin{align*}
    \Big| \brk*{\pi \push \emission(s)}(a) - \brk*{\pi \push \unif\prn*{\crl*{ \empobs \in \cD: \optdec(\empobs) = s } } }(a) \Big| &\le \sqrt{\frac{2\log (2 SA \abs{\Pi}/\delta)}{n_s[\cD]}}
\end{align*}
This concludes the proof of the lemma.
\end{proof}

\begin{lemma}[Monte Carlo Estimates for \stochdecoder]\label{lem:monte-carlo}
Fix any $(x_h, a_h)$ for which we call $\stochdecoder{}$. With probability at least $1-\delta$, every Monte Carlo estimate $\vestarg{x}{\pi}$ computed in \pref{line:monte-carlo} of \pref{alg:stochastic-decoder-v2} satisfies
$\abs*{ \vestarg{x}{\pi} - V^{\pi}(x) } \le \eps$.
\end{lemma}

\begin{proof}
By Hoeffding's inequality we know that for a fixed $(x, \pi)$ pair:
\begin{align*}
    \Pr\brk*{\abs*{\vestarg{x}{\pi} - V^{\pi}(x)} \ge \eps} \le 2 \exp(-2\nmc \eps^2).
\end{align*}
In total, we call \pref{line:monte-carlo} at most $\abs{\estmdpobsspace{h+1}}^2 \le \nreset^2$ times. Therefore, by union bound, as long as 
\begin{align*}
    \nmc \ge K \cdot \frac{1}{\eps^2} \cdot \log \frac{\cpush SAH \abs{\Pi}}{\eps \delta }
\end{align*}
where $K>0$ is an absolute constant determined by the value of $\nreset$, then the result holds.
\end{proof}

\paragraph{Properties of \stochrefit{}.} Now we establish the accuracy of estimates in a single call to \stochrefit{}.

\begin{lemma}[Monte Carlo Estimates for \stochrefit]\label{lem:refit-monte-carlo-accuracy}
    With probability at least $1-\delta$, every Monte Carlo estimate computed by \stochrefit{} (\pref{line:mc1} and \ref{line:mc2} of \pref{alg:stochastic-refit-v2}) is accurate up to error $\eps$.
\end{lemma}

\begin{proof}
In \stochrefit{} we compute Monte Carlo estimates for $2\nreset^2 + 2\nreset^3 \cdot AH$ times, since there are $2\nreset^2$ possible certificates $(x, \pi)$ and for each one we perform Monte Carlo estimates over all of the $(\bar{x},\bar{a})$ pairs in our policy emulator $\estmdp$. By Hoeffding bound and union bound we see that as long as 
\begin{align*}
    \nmc \ge K \cdot \frac{1}{\eps^2} \cdot \log \frac{\cpush SAH \abs{\Pi}}{\eps \delta},
\end{align*}
for some absolute constant $K>0$, the conclusion of the lemma holds.
\end{proof}

\paragraph{Additional Notation.}
Henceforth, let us define several events:
\begin{itemize}
    \item $\eventemulator \coloneqq \crl*{\text{the conclusions of \pref{lem:sampling-states} --- \ref{lem:sampling-rewards} hold}}$. We have $\Pr[\eventemulator] \ge 1-3\delta$.
    \item $\eventdec_t \coloneqq \crl*{\text{the conclusions of \pref{lem:sampling-transitions} --- \ref{lem:monte-carlo} hold on the $t$-th call to \stochdecoder{}}}$. We have $\Pr[\eventdec_t ] \ge 1-3\delta$. Furthermore, define the random variable $T_\mathsf{D}$ to be the total number of times that \stochdecoder{} is called.
    \item $\eventref_t \coloneqq \crl*{\text{the conclusion of \pref{lem:refit-monte-carlo-accuracy} holds on the $t$-th call to \stochrefit{}}}$. We have $\Pr[\eventref_t ] \ge 1-\delta$. Furthermore, define the random variable $T_\mathsf{R}$ to be the total number of times that \stochrefit{} is called.
\end{itemize}
In the analysis, we will drop the subscript $t$ when referring to $\eventdec_t$ and $\eventref_t$ if clear from the context.

\subsubsection{Analysis of \stochdecoder}

This section is dedicated to establishing \pref{lem:main-induction}, which is the main inductive lemma.

\begin{lemma}[Induction for \stochdecoder]\label{lem:main-induction} Fix any layer $h \in [H]$ and tuple $(x_h, a_h)$ on which \stochdecoder{} is called. Assume that:
\begin{itemize}
    \item $\eventemulator$ and $\eventdec$ hold.
    \item For all $\empobs \in \estmdpobsspace{h+1}$: $\max_{\pi \in \Pi_{h+1:H}}~ \abs{V^\pi(\empobs) - \wh{V}^\pi(\empobs
    )} \le \Gamma_{h+1}$.
    \item Input confidence set $\cP(x_h, a_h)$ satisfies $\projm{\estmdpobsspace{h+1}}(\optlatp(\cdot  \mid  x_h, a_h)) \in \cP(x_h, a_h)$.
    \item $\Pitest_{h+1}$ are $\taudec$-valid test policies for the policy emulator $\estmdp$.
\end{itemize}

Then \stochdecoder{} returns confidence set $\cP$ via Eq.~\eqref{eq:confidence-set-construction-v2} such that: 
    \begin{enumerate}[(1)]
        \item $\projm{\estmdpobsspace{h+1}}(\optlatp(\cdot  \mid  x_h, a_h)) \in \cP$;
        \item $\max_{\bar{p} \in \cP} \max_{\pi \in \Pi_{h+1:H}} \abs{Q^\pi(x_h, a_h) - \wh{R}(x_h, a_h) - \En_{\empobs \sim \bar{p}} \wh{V}^\pi(\empobs)} \le \Gamma_{h+1} + K \cdot \prn*{\beta + S \taudec}$.
    \end{enumerate} 
Here, $K>0$ is an absolute numerical constant.
\end{lemma}

\subsubsubsection{Structural Properties of the Decoder Graph}\label{sec:decoder-graph}
For the lemmas in this section, we will assume the preconditions of \pref{lem:main-induction} and analyze properties of the decoder graph $\gobs$ constructed in a single call to \stochdecoder{}.

\begin{lemma}[Validity of Decoding Function]\label{lem:strong-validity-obs-decoding} Under the preconditions of \pref{lem:main-induction}, for every $x_l \in \cXL$, we have
\begin{align*}
    \crl{\empobs_r \in \cXR: \optdec(\empobs_r) = \optdec(x_l)} \subseteq \cT[x_l].
\end{align*}
\end{lemma}

\begin{proof}
The proof is a reprise of the argument used in Part (1) of \pref{lem:controlling-transition-error}. We prove this by contradiction. Suppose that there existed some $x_l \in \cXL$ and $\empobs_r \in \cXR$ such that $\optdec(x_l) = \optdec(\empobs_r)$ but $\empobs_r \notin \cT[x_l]$. Then $\empobs_r$ must have lost a test to some other $\empobs_r'$, i.e. there exists some $\empobs_r' \in \estmdpobsspace{h+1}$ such that
\begin{align*}
    \abs*{\vestarg{x_l}{\pi_{\empobs_r, \empobs_r'}} - \wh{V}^{\pi_{\empobs_r, \empobs_r'}}(\empobs_r) } \ge \taudec + 2\eps. \numberthis\label{eq:invalid}
\end{align*}
By accuracy of $\Pitest_{h+1}$ and the fact that $\pi_{\empobs_r, \empobs_r'}$ is open-loop at layer $h+1$, we have
\begin{align*}
    \abs*{ V^{\pi_{\empobs_r, \empobs_r'}}(x_l) - \wh{V}^{\pi_{\empobs_r, \empobs_r'}}(\empobs_r)} = \abs*{ V^{\pi_{\empobs_r, \empobs_r'}}(\empobs_r) - \wh{V}^{\pi_{\empobs_r, \empobs_r'}}(\empobs_r)} \le \taudec. \numberthis \label{eq:valid1}
\end{align*}
Furthermore, by \pref{lem:monte-carlo} we have 
\begin{align*}
    \abs*{ \vestarg{x_l}{\pi_{\empobs_r, \empobs_r'}} - V^{\pi_{\empobs_r, \empobs_r'}}(x_r) } = \abs*{ \vestarg{x_l}{\pi_{\empobs_r, \empobs_r'}} - V^{\pi_{\empobs_r, \empobs_r'}}(x_l) } \le \eps. \numberthis \label{eq:valid2}
\end{align*}
Combining \eqref{eq:valid1} and \eqref{eq:valid2} we get that
\begin{align*}
    \abs*{ \vestarg{x_l}{\pi_{\empobs_r, \empobs_r'}} - \wh{V}^{\pi_{\empobs_r, \empobs_r'}}(\empobs_r) } \le \taudec + \eps, 
\end{align*}
which contradicts \eqref{eq:invalid}. This proves the lemma.
\end{proof}

\begin{lemma}[Biclique Property]\label{lem:biclique}
Under the preconditions of \pref{lem:main-induction}, for any $s \in \cS[\cXL] \cap \cS[\cXR]$ the subgraph of $\gobs$ over vertices $\crl{x \in \cX_\mathsf{L} \cup \cX_\mathsf{R}: \optdec(x) = s}$ is a biclique.
\end{lemma}

\begin{proof}
Fix any $s \in \cS[\cXL] \cap \cS[\cXR]$. By \pref{lem:strong-validity-obs-decoding}, any $x_l \in \cXL$ such that $\optdec(x_l) = s$ has an edge to every observation $\crl{\empobs \in \cXR: \optdec(\empobs) = s}$ in $\gobs$. Therefore, the subgraph over $\crl{x \in \cX_\mathsf{L} \cup \cX_\mathsf{R}: \optdec(x) = s}$ forms a biclique in $\gobs$.
\end{proof}

\begin{lemma}\label{lem:eps-reachable-set-inclusion}
Under the preconditions of \pref{lem:main-induction}, for any connected component $\cc$, $\epsPRS{} \cap \cS[\ccL] \subseteq \epsPRS{} \cap \cS[\ccR]$.
\end{lemma}

\begin{proof}
Fix any $s \in \epsPRS{} \cap \cS[\ccL]$, and let $x_l \in \ccL$ be any arbitrary observation such that $\optdec(x_l) = s$. By \pref{lem:sampling-states}, since $s \in \epsPRS{}$, there exist some $\empobs_r \in \cXR$ such that $\optdec(\empobs_r) = s$; in other words, $s \in \cS[\cXR]$. Moreover by \pref{lem:strong-validity-obs-decoding}, there must be an edge from $x_l$ to $x_r$ in $\gobs$. Therefore $x_r \in \ccR$, so $s \in \cS[\ccR]$. 
\end{proof}

\begin{lemma}\label{lem:bound1}
Let $\empobs, \empobs' \in \cXR$ such that $\optdec(\empobs) = \optdec(\empobs')$. We have $\max_{a \in \cA, \pi \in \Pi}~\abs{ \wh{Q}^\pi(\empobs, a) - \wh{Q}^\pi(\empobs', a) } \le 2\taudec$.
\end{lemma}

\begin{proof}
Denote $\pi_{\empobs, \empobs'} = \argmax_{\pi \in \cA \circ \Pi}~\abs{\wh{V}^\pi(\empobs) - \wh{V}^\pi(\empobs')}$ to be the test policy for the pair $\empobs, \empobs' \in \cXR$. By accuracy of the test policy we know that 
\begin{align*}
    \abs*{V^{\pi_{\empobs, \empobs'}}(\empobs) - \wh{V}^{\pi_{\empobs, \empobs'}}(\empobs)} &\le \taudec \quad \text{and} \quad    \abs*{V^{\pi_{\empobs, \empobs'}}(\empobs') - \wh{V}^{\pi_{\empobs, \empobs'}}(\empobs')} \le \taudec.
\end{align*}
Furthermore since $\empobs, \empobs'$ are observations emitted from the same latent state and $\pi_{\empobs, \empobs'}$ is open loop at layer $h+1$, we have $V^{\pi_{\empobs, \empobs'}}(\empobs) = V^{\pi_{\empobs, \empobs'}}(\empobs')$. Therefore
\begin{align*}
    \max_{a \in \cA, \pi \in \Pi}~\abs*{\wh{Q}^\pi(\empobs,a) - \wh{Q}^\pi(\empobs',a)} &= \abs*{\wh{Q}^{\pi_{\empobs, \empobs'}}(\empobs, \pi_{\empobs, \empobs'}) - \wh{Q}^{\pi_{\empobs, \empobs'}}(\empobs', \pi_{\empobs, \empobs'})} \le 2\taudec.
\end{align*}
This concludes the proof of the lemma.
\end{proof}

\begin{lemma}\label{lem:bound2}
Fix any $x_l \in \cXL$. If $\empobs_r, \empobs_r' \in \cT[x_l]$, then $\max_{a \in \cA, \pi \in \Pi}~\abs{\wh{Q}^{\pi}(\empobs_r, a) - \wh{Q}^{\pi}(\empobs_r', a)} \le 2\taudec + 4 \eps$.
\end{lemma}

\begin{proof}
By definition of $\cT[x_l]$ we have
\begin{align*}
    \abs*{\vestarg{x_l}{\pi_{\empobs_r, \empobs_r'}} - \wh{V}^{\pi_{\empobs_r, \empobs_r'}}(\empobs_r) } \le \taudec+2\eps \quad\text{and}\quad \abs*{\vestarg{x_l}{\pi_{\empobs_r, \empobs_r'}} - \wh{V}^{\pi_{\empobs_r, \empobs_r'}}(\empobs_r') } \le \taudec+2\eps.
\end{align*}
Using the fact that test policies are maximally distinguishing we have
    \begin{align*}
        \max_{a \in \cA, \pi \in \Pi}~\abs*{\wh{Q}^{\pi}(\empobs_r, a) - \wh{Q}^{\pi}(\empobs_r', a)} = \abs*{\wh{V}^{\pi_{\empobs_r, \empobs_r'}}(\empobs_r) - \wh{V}^{\pi_{\empobs_r, \empobs_r'}}(\empobs_r')} \le 2\taudec + 4 \eps.
    \end{align*}
This proves the lemma.
\end{proof}

\begin{lemma}[Bounded Width of $\cc$]\label{lem:width-of-cc}
For any connected component $\cc \in \crl{\cc_j}_{j \ge 1}$ in $\gobs$ we have 
\begin{align*}
    \max_{\empobs, \empobs' \in \ccR} \max_{a \in \cA, \pi \in \Pi}~\abs{ \wh{Q}^\pi(\empobs, a) - \wh{Q}^\pi(\empobs', a) } \le 4S\taudec + 8S\eps.
\end{align*}
\end{lemma}

\begin{proof}
Let us take any $\empobs, \empobs' \in \ccR$. Since $\empobs, \empobs'$ belong to the same connected component, there exists a sequence of observations $\mathsf{seq} = (\empobs_1 = \empobs, \dots, \empobs_n = \empobs') \in (\ccR)^n$ such that for every consecutive pair $\empobs_i, \empobs_{i+1}$ there exists some $x_l \in \ccL$ such that $\empobs_i, \empobs_{i+1} \in \cT[x_l]$. 

Fix any $a \in \cA, \pi \in \Pi$. Now we will bound $\abs{ \wh{Q}^\pi(\empobs, a) - \wh{Q}^\pi(\empobs', a) }$. We construct an auxiliary sequence  $\wt{\mathsf{seq}} = (\wt{\empobs}_1, \dots, \wt{\empobs}_k)$ for some $k \le n$ as follows:
\begin{itemize}
    \item Initialize $\wt{\mathsf{seq}} = \emptyset$.
    \item For $i = 1, \cdots, n$:
    \begin{itemize}
        \item Add $\empobs_i$ to the end of $\wt{\mathsf{seq}}$.
        \item If there exists $\empobs_j$ with $j > i$ such that $\optdec(\empobs_i) = \optdec(\empobs_j)$ then set $i \gets j$.
    \end{itemize}
\end{itemize}
Observe that $\wt{\mathsf{seq}}$ satisfies the following conditions:
\begin{itemize}
    \item $\wt{\empobs}_1 = \empobs$ and $\wt{\empobs}_k = \empobs'$.
    \item  For every $s \in \cS$, at most two observations $\wt{\empobs}, \wt{\empobs}' \in \mathrm{supp}(\emission(s)) \cap \ccR$ are found in $\wt{\mathsf{seq}}$, and these observations must appear sequentially.
    \item For any $i \in [k-1]$, if $\optdec(\wt{\empobs}_i) \ne \optdec(\wt{\empobs}_{i+1})$ then there exists some $x_l \in \ccL$ such that $\wt{\empobs}_i, \wt{\empobs}_{i+1} \in \cT[x_l]$.
\end{itemize}
Now we can apply triangle inequality to $\wt{\mathsf{seq}}$:
\begin{align*}
    \abs*{ \wh{Q}^\pi(\empobs, a) - \wh{Q}^\pi(\empobs', a) } \le \sum_{i=1}^k \abs*{ \wh{Q}^\pi(\wt{\empobs}_i, a) - \wh{Q}^\pi(\wt{\empobs}_{i+1}, a) } \le 4S\taudec + 8S\eps.
\end{align*}
The final bound uses the aforementioned properties of $\wt{\mathsf{seq}}$, as well as  \pref{lem:bound1} and \pref{lem:bound2} to handle the individual terms in the summation. This completes the proof of \pref{lem:width-of-cc}.
\end{proof}

\subsubsubsection{Structural Properties of Projected Measures}\label{sec:projected-measures}
Now we will prove several lemmas regarding the projected measure of the empirical latent distribution
\begin{align*}
    \samplelatp = \frac{1}{\abs{\cXL}} \sum_{x \in \cXL} \delta_{\optdec(x)}.
\end{align*}
which is sampled in \pref{line:sample-decode-dataset} of a single call to \stochdecoder{}.



\begin{lemma}\label{lem:proj-cc-expression}
    Under the preconditions of \pref{lem:main-induction}, for any connected component $\cc$ of $\gobs$:
    \begin{align*}
        \projR(\samplelatp)(\ccR) =  \sum_{s \in \epsPRS{} \cap \cS[\ccL] \cap \cS[\ccR]} \frac{n_{s}[\cXL]}{\abs{\cXL}}  = \sum_{s \in \epsPRS{} \cap \cS[\ccL] \cap \cS[\ccR]} \samplelatp(s).
    \end{align*}
\end{lemma}

\begin{proof}
We compute that
\begin{align*}
    \projR(\samplelatp)(\ccR) &= \sum_{\empobs \in \ccR} \projR(\samplelatp)(\empobs) \\
    &= \sum_{\empobs \in \ccR} \frac{n_{\optdec(\empobs)}[\cXL]}{\abs{\cXL}} \cdot \frac{\ind{\optdec(\empobs) \in \epsPRS{}}}{n_{\optdec(\empobs)}[\cXR]}\\
    &= \sum_{s \in \epsPRS{}} \frac{n_{s}[\cXL]}{\abs{\cXL}} \sum_{\empobs \in \ccR}  \frac{\ind{\optdec(\empobs) = s}}{n_{s}[\cXR]} \\
    &\overset{(i)}{=} \sum_{s \in \epsPRS{} \cap \cS[\ccR]} \frac{n_{s}[\cXL]}{\abs{\cXL}} \sum_{\empobs \in \ccR}  \frac{\ind{\optdec(\empobs) = s}}{n_{s}[\cXR]} \\
    &\overset{(ii)}{=}  \sum_{s \in \epsPRS{} \cap \cS[\cXL] \cap \cS[\ccR]} \frac{n_{s}[\cXL]}{\abs{\cXL}} \sum_{\empobs \in \ccR}  \frac{\ind{\optdec(\empobs) = s}}{n_{s}[\cXR]}.
\end{align*}
For $(i)$, observe that if $s \notin \cS[\ccR]$, then the sum $\sum_{\empobs \in \ccR}  \frac{\ind{\optdec(\empobs) = s}}{n_{s}[\cXR]} = 0$. For $(ii)$, we use the fact that $n_s[\cXL] = 0$ if $s \notin \cS[\cXL]$. From here, we apply the biclique lemma (\pref{lem:biclique}). The biclique lemma implies that if $s \in \cS[\ccR]$, then $\crl{x \in \cXL: \optdec(x) = s} \subseteq \ccL$, and therefore $\cS[\cXL] \cap \cS[\ccR] = \cS[\ccL] \cap \cS[\ccR]$. Furthermore for any $s \in \cS[\ccL] \cap \cS[\ccR]$, all of the observations $\crl{\empobs \in \cXR: \optdec(\empobs) = s} \subseteq \ccR$, so $n_s[\cXR] = n_s[\ccR]$. Thus we can continue the calculation as
\begin{align*}
    \projR(\samplelatp)(\ccR) &=  \sum_{s \in \epsPRS{} \cap \cS[\ccL] \cap \cS[\ccR]} \frac{n_{s}[\cXL]}{\abs{\cXL}} \sum_{\empobs \in \ccR}  \frac{\ind{\optdec(\empobs) = s}}{n_{s}[\cXR]}\\
    &= \sum_{s \in \epsPRS{} \cap \cS[\ccL] \cap \cS[\ccR]} \frac{n_{s}[\cXL]}{\abs{\cXL}} \sum_{\empobs \in \ccR}  \frac{\ind{\optdec(\empobs) = s}}{n_{s}[\ccR]} \\
    &= \sum_{s \in \epsPRS{} \cap \cS[\ccL] \cap \cS[\ccR]} \frac{n_{s}[\cXL]}{\abs{\cXL}} = \sum_{s \in \epsPRS{} \cap \cS[\ccL] \cap \cS[\ccR]} \samplelatp(s).
\end{align*} 
This concludes the proof of \pref{lem:proj-cc-expression}.
\end{proof}

\begin{corollary}\label{corr:estimating-projR} Under the preconditions of \pref{lem:main-induction}, then
$\sum_{\cc \in \crl{\cc_j}} \prn{ \frac{\abs{\ccL}}{\abs{\cXL}} 
 - \projR(\samplelatp)(\ccR) } \in [0,2\eps]$.
\end{corollary}
\begin{proof}
For any $\cc$ we have
\begin{align*}
    \hspace{2em}&\hspace{-2em} \frac{\abs{\ccL}}{\abs{\cXL}} - \projR(\samplelatp)(\ccR) \\
    &= \frac{\abs{\ccL}}{\abs{\cXL}} - \sum_{s \in \epsPRS{} \cap \cS[\ccL] \cap \cS[\ccR]} \frac{n_{s}[\cXL]}{\abs{\cXL}} &\text{(\pref{lem:proj-cc-expression})} \\
    &= \frac{\abs{\ccL}}{\abs{\cXL}} - \sum_{s \in \epsPRS{} \cap \cS[\ccL]} \frac{n_{s}[\cXL]}{\abs{\cXL}} &\text{(\pref{lem:eps-reachable-set-inclusion})} \\
    &= \sum_{s \in \epsPRS{} \cap \cS[\ccL]} \frac{n_s[\ccL]}{\abs{\cXL}} + \sum_{s \in (\epsPRS{})^c \cap \cS[\ccL]} \frac{n_s[\ccL]}{\abs{\cXL}} - \sum_{s \in \epsPRS{} \cap \cS[\ccL]} \frac{n_{s}[\cXL]}{\abs{\cXL}}  \\
    &= \sum_{s \in (\epsPRS{})^c \cap \cS[\ccL]} \frac{n_s[\ccL]}{\abs{\cXL}}.
\end{align*}
The last equality uses the fact that by \pref{lem:eps-reachable-set-inclusion}, $s \in \epsPRS{} \cap \cS[\ccL] \Rightarrow s \in \cS[\ccR]$, so in particular by \pref{lem:biclique} we have $\crl{x\in \cXL: \optdec(x) = s} \subseteq \ccL$, so therefore $n_s[\cXL] = n_s[\ccL]$. 

Summing over all $\cc$ and applying \pref{corr:sampling-transitions} we get that
\begin{align*}
    \sum_{\cc \in \crl{\cc_j}} \frac{\abs{\ccL}}{\abs{\cXL}} - \projR(\samplelatp)(\ccR) = \sum_{\bbC \in \crl{\bbC_j}} \sum_{s \in (\epsPRS{})^c \cap \cS[\ccL]} \frac{n_s[\ccL]}{\abs{\cXL}} = \sum_{s \in (\epsPRS{})^c} \frac{n_s[\cXL]}{\abs{\cXL}} \in [0, 2\eps].
\end{align*}
This proves \pref{corr:estimating-projR}.
\end{proof}

\begin{lemma}\label{lem:relating-proj-to-empirical}
Under the preconditions of \pref{lem:main-induction}, for every $\pi \in \Pi$ and $\cc \in \crl{\cc_j}$: 
\begin{align*}
    \max_{a \in \cA}~\abs*{ \brk*{ \pi\push \projR(\samplelatp)(\cdot \mid \ccR)}(a) - \brk*{\pi \push \unif(\ccL)}(a) } \le \frac{\eps}{A} +  K \cdot \sqrt{\frac{S \log \tfrac{SA \abs{\Pi}}{\delta}}{\nrch[\ccL]}}  + \frac{\nurch[\ccL]}{\nrch[\ccL]},
\end{align*}
where $K>0$ is an absolute constant.
\end{lemma}
\begin{proof}
Fix any $\pi \in \Pi$, $\cc \in \crl{\cc_j}$, and $a \in \cA$. We can calculate that: 
\begin{align*}
    & \brk*{ \pi\push \projR(\samplelatp)(\cdot \mid \ccR)}(a) \\
    &= \sum_{\empobs \in \ccR} \frac{\projR(\samplelatp)(\empobs)}{\projR(\samplelatp)(\ccR)} \indd{\pi(\empobs) = a} \\
    &= \frac{1}{\projR(\samplelatp)(\ccR)} \sum_{\empobs \in \ccR} \frac{n_{\optdec(\empobs)}[\cXL]}{\abs{\cXL}} \cdot \frac{\ind{\optdec(\empobs) \in \epsPRS{}} \indd{\pi(\empobs) = a}}{n_{\optdec(\empobs)}[\cXR]} \\
    &= \frac{1}{\projR(\samplelatp)(\ccR)} \sum_{s \in \epsPRS{}} \prn*{ \frac{n_{s}[\cXL]}{\abs{\cXL}} \cdot \frac{1}{n_{s}[\cXR]} \cdot \sum_{\empobs \in \ccR}  \ind{\optdec(\empobs) =s} \indd{\pi(\empobs) = a} }\\
    &= \frac{1}{\projR(\samplelatp)(\ccR)} \sum_{s \in \epsPRS{} \cap \cS[\ccR] } \prn*{\frac{n_{s}[\cXL]}{\abs{\cXL}} \cdot \frac{1}{n_{s}[\cXR]} \cdot \sum_{\empobs \in \ccR}  \ind{\optdec(\empobs) =s} \indd{\pi(\empobs) = a} } \\
    &= \frac{1}{\projR(\samplelatp)(\ccR)} \sum_{s \in \epsPRS{} \cap \cS[\cXL] \cap \cS[\ccR] } \prn*{\frac{n_{s}[\cXL]}{\abs{\cXL}} \cdot \frac{1}{n_{s}[\cXR]} \cdot \sum_{\empobs \in \ccR}  \ind{\optdec(\empobs) =s} \indd{\pi(\empobs) = a} } \\
    &= \frac{1}{\projR(\samplelatp)(\ccR)} \sum_{s \in \epsPRS{} \cap \cS[\ccL] \cap \cS[\ccR] } \prn*{\frac{n_{s}[\ccL]}{\abs{\cXL}} \cdot \frac{1}{n_{s}[\ccR]} \cdot \sum_{\empobs \in \ccR}  \ind{\optdec(\empobs) =s} \indd{\pi(\empobs) = a} }.
\end{align*}
The last line uses the biclique lemma (\pref{lem:biclique}) in the same fashion as the proof of \pref{lem:proj-cc-expression}. Now we apply the conclusions of \pref{lem:pushforward-concentration} and \pref{lem:pushforward-concentration-transition}, along with the fact that for every $s \in \epsPRS{} \cap \cS[\ccL] \cap \cS[\ccR]$ we have $\crl{\empobs \in \cXR : \optdec(\empobs) = s } \subseteq \ccR$ as well as $\crl{\empobs \in \cXL : \optdec(\empobs) = s } \subseteq \ccL$ (which is again implied by the biclique lemma):
\begin{align*}
     &\hspace{-1.5em} \brk*{ \pi\push \projR(\samplelatp)(\cdot \mid \ccR)}(a) \\
    \le~& \frac{1}{\projR(\samplelatp)(\ccR) } \sum_{s \in \epsPRS{} \cap \cS[\ccL] \cap \cS[\ccR] } \frac{n_{s}[\ccL]}{\abs{\cXL}} \cdot \prn*{ \brk*{\pi \push \emission(s)}(a) + \frac{\eps}{A} } \\
    \le~& \frac{1}{\projR(\samplelatp)(\ccR) } \sum_{s \in \epsPRS{} \cap \cS[\ccL] \cap \cS[\ccR] } \frac{n_{s}[\ccL]}{\abs{\cXL}} \cdot \Bigg( \frac{\eps}{A} + \sqrt{\frac{2 \log \tfrac{2SA \abs{\Pi}}{\delta}}{n_s[\ccL]}} \\
    &\hspace{20em}+ \frac{1}{n_{s}[\ccL]} \cdot \sum_{\empobs \in \ccL}  \ind{\optdec(\empobs) =s} \indd{\pi(\empobs) = a}\Bigg) \\
    =~& \frac{1}{\projR(\samplelatp)(\ccR) \abs{\cXL}} \sum_{ s \in \epsPRS{} \cap \cS[\ccL] \cap \cS[\ccR] } \Bigg( \frac{\eps}{A} \cdot n_{s}[\ccL] + \sqrt{2n_s[\ccL] \log \tfrac{2 SA \abs{\Pi}}{\delta}} \\
    &\hspace{20em}+ \sum_{\empobs \in \ccL}  \ind{\optdec(\empobs) =s} \indd{\pi(\empobs) = a} \Bigg)
\end{align*}
By \pref{lem:proj-cc-expression}, we have $\abs{\cXL} \projR(\samplelatp)(\ccR) = \sum_{s \in \epsPRS{} \cap \cS[\ccL] \cap \cS[\ccR]} n_{s}[\cXL] = \nrch[\ccL] $. Using Cauchy-Schwarz we get that
\begin{flalign*}
& \brk*{ \pi\push \projR(\samplelatp)(\cdot \mid \ccR)}(a) \\
    &\le \frac{\eps}{A} + K \cdot \sqrt{\frac{S \log \tfrac{SA \abs{\Pi}}{\delta}}{\nrch[\ccL]}} + \frac{1}{\nrch[\ccL]} \sum_{s \in \epsPRS{} \cap \cS[\ccL] \cap \cS[\ccR]  } \prn*{\sum_{\empobs \in \ccL}  \ind{\optdec(\empobs) =s} \indd{\pi(\empobs) = a} } 
    \\&= \frac{\eps}{A} + K \cdot \sqrt{\frac{S \log \tfrac{SA \abs{\Pi}}{\delta}}{\nrch[\ccL]}} + \frac{\abs{\ccL}}{\nrch[\ccL]} \cdot \frac{1}{\abs{\ccL}}  \sum_{s \in \epsPRS{} \cap \cS[\ccL] \cap \cS[\ccR]} \prn*{\sum_{\empobs \in \ccL}  \ind{\optdec(\empobs) =s} \indd{\pi(\empobs) = a} } \numberthis \label{eq:partial}
\end{flalign*}
Let us investigate the last term. We have
\begin{align*}
\hspace{4em}&\hspace{-4em}
\frac{1}{\abs{\ccL}}  \sum_{s \in \epsPRS{} \cap \cS[\ccL] \cap \cS[\ccR]} \prn*{\sum_{\empobs \in \ccL}  \ind{\optdec(\empobs) =s} \indd{\pi(\empobs) = a} } \\
    =\quad& \frac{1}{\abs{\ccL}}  \sum_{s \in \epsPRS{} \cap \cS[\ccL]} \prn*{\sum_{\empobs \in \ccL}  \ind{\optdec(\empobs) =s} \indd{\pi(\empobs) = a} } &\text{(\pref{lem:eps-reachable-set-inclusion})}\\
    \le\quad& \frac{1}{\abs{\ccL}}  \sum_{s \in \epsPRS{} \cap \cS[\ccL]} \prn*{\sum_{\empobs \in \ccL}  \ind{\optdec(\empobs) =s} \indd{\pi(\empobs) = a} } \\
    &\hspace{5em} + \frac{1}{\abs{\ccL}}  \sum_{s \in (\epsPRS{})^c \cap \cS[\ccL]} \prn*{\sum_{\empobs \in \ccL}  \ind{\optdec(\empobs) =s} \indd{\pi(\empobs) = a} } \\
    =\quad& \brk*{\pi \push \unif(\ccL)}(a)
\end{align*}
Plugging this back into Eq.~\eqref{eq:partial} we get that
\begin{align*}
\brk*{ \pi\push \projR(\samplelatp)(\cdot \mid \ccR)}(a) &\le \frac{\eps}{A} +  K \cdot \sqrt{\frac{S \log \tfrac{SA \abs{\Pi}}{\delta}}{\nrch[\ccL]}} + \frac{\abs{\ccL}}{\nrch[\ccL]} \cdot \brk*{\pi \push \unif(\ccL)}(a),
\end{align*}
and rearranging and using the fact that $\brk*{\pi \push \unif(\ccL)}(a) \in [0,1]$ we get that
\begin{align*}
    \brk*{ \pi\push \projR(\samplelatp)(\cdot \mid \ccR)}(a) - \brk*{\pi \push \unif(\ccL)}(a) \le \frac{\eps}{A} +   K \cdot \sqrt{\frac{S \log \tfrac{SA \abs{\Pi}}{\delta}}{\nrch[\ccL]}}  + \frac{\nurch[\ccL]}{\nrch[\ccL]}.
\end{align*}
One can repeat the same steps to get the lower bound. Therefore,
\begin{align*}
    \abs*{ \brk*{ \pi\push \projR(\samplelatp)(\cdot \mid \ccR)}(a) - \brk*{\pi \push \unif(\ccL)}(a) } \le \frac{\eps}{A} +   K \cdot \sqrt{\frac{S \log \tfrac{SA \abs{\Pi}}{\delta}}{\nrch[\ccL]}} + \frac{\nurch[\ccL]}{\nrch[\ccL]}.
\end{align*}
This proves \pref{lem:relating-proj-to-empirical}.
\end{proof}

\subsubsubsection{Proof of \pref{lem:main-induction}}

Fix the $(x_h, a_h)$ pair on which we call \stochdecoder.

For notational convenience we will denote $\cXL \coloneqq \cD$ and $\cXR \coloneqq \estmdpobsspace{h+1}$, as well as use $\optlatp = \optlatp(\cdot \mid \optdec(x_{h}), a_{h})$ to denote the ground truth latent transition function. Throughout the proof, we use $K>0$ to denote absolute constants whose values may change line-by-line. 

\paragraph{Part (1).} Since the input confidence set $\cP$ satisfies the third bullet, it suffices to show that that $\projm{\cXR}(\optlatp)$ satisfies both of the constraints in the confidence set construction of Eq.~\eqref{eq:confidence-set-construction-v2}.

For the first constraint, observe that by \pref{corr:estimating-projR},
\begin{align*}
    \sum_{\cc \in \crl{\cc_j}} \abs*{ \frac{\abs{\ccL}}{\abs{\cXL}} 
 - \projR(\samplelatp)(\ccR) } \le 2\eps.
\end{align*}
Therefore it suffices to show that 
\begin{align*}
    \sum_{\cc \in \crl{\cc_j}} \abs*{\projR(\optlatp)(\ccR) - \projR(\samplelatp)(\ccR) } \le \eps.
\end{align*}
We calculate that
\begin{align*}
    \hspace{2em}&\hspace{-2em} \sum_{\cc \in \crl{\cc_j}} \abs*{\projR(\optlatp)(\ccR) - \projR(\samplelatp)(\ccR)} \\
    &\le \sum_{\empobs \in \cXR}  \abs*{\projR(\optlatp)(\empobs) - \projR(\samplelatp)(\empobs)} \\
    &= \sum_{\empobs \in \cXR}  \abs*{\prn*{\optlatp(\optdec(\empobs)) - \samplelatp(\optdec(\empobs))} \cdot \frac{\ind{\optdec(\empobs) \in \epsPRS{}}}{n_{\optdec(\empobs)}[\cXR]}} \\
    &= \sum_{s \in \epsPRS{}} \sum_{\empobs \in \cXR}  \abs*{\prn*{\optlatp(s) - \samplelatp(s)} \cdot \frac{\ind{\optdec(\empobs) = s}}{n_{s}[\cXR]}} \\
    &= \sum_{s \in \epsPRS{}} \abs*{\optlatp(s) - \samplelatp(s) } \le \eps. &\text{(\pref{lem:sampling-transitions})} \numberthis \label{eq:bound-on-empirical-proj}
\end{align*}
Now we prove that $\projR(\optlatp)$ also satisfies the second constraint, i.e.,
\begin{align*}
    \sum_{\cc \in \crl{\cc_j}}  \frac{\abs{\ccL}}{\abs{\cXL}} \cdot \nrm*{\pi \push \unif(\ccL) - \pi \push \projR(\optlatp)(\cdot \mid \ccR) }_1 \le \eps.
\end{align*}

Observe that we can break up the bound as follows:
\begin{align*}
    \hspace{2em}&\hspace{-2em} \sum_{\cc \in \crl{\cc_j}}  \frac{\abs{\ccL}}{\abs{\cXL}} \cdot \nrm*{\pi \push \unif(\ccL) - \pi \push \projR(\optlatp)(\cdot \mid \ccR) }_1 \\
    &\le \underbrace{ \sum_{\cc \in \crl{\cc_j}}  \frac{\abs{\ccL}}{\abs{\cXL}} \cdot \nrm*{\pi \push \unif(\ccL) - \pi \push \projR(\samplelatp)(\cdot \mid \ccR) }_1 }_{\eqqcolon \mathtt{Term}_1}  \\
    &\hspace{5em} + \underbrace{\sum_{\cc \in \crl{\cc_j}}  \frac{\abs{\ccL}}{\abs{\cXL}} \cdot \nrm*{\pi \push \projR(\samplelatp)(\cdot \mid \ccR) - \pi \push \projR(\optlatp)(\cdot \mid \ccR) }_1 }_{\eqqcolon \mathtt{Term}_2}.
\end{align*}

\emph{Bounding $\mathtt{Term}_1$.} To bound $\mathtt{Term}_1$, we compute:
\begin{align*}
   & \sum_{\cc \in \crl{\cc_j}}  \frac{\abs{\ccL}}{\abs{\cXL}} \cdot \nrm*{\pi \push \unif(\ccL) - \pi \push \projR(\samplelatp)(\cdot \mid \ccR) }_1 \\
    &\le \sum_{\cc \in \crl{\cc_j}:  \frac{\abs{\ccL}}{\abs{\cXL}} \ge 4\eps}  \frac{\abs{\ccL}}{\abs{\cXL}} \cdot \nrm*{\pi \push \unif(\ccL) - \pi \push \projR(\samplelatp)(\cdot \mid \ccR) }_1 + \sum_{\cc \in \crl{\cc_j} :  \frac{\abs{\ccL}}{\abs{\cXL}} < 4\eps }  \frac{2\abs{\ccL}}{\abs{\cXL}} \\
    &\overset{(i)}{\le} \sum_{\cc \in \crl{\cc_j}:  \frac{\abs{\ccL}}{\abs{\cXL}} \ge 4\eps}  \frac{\abs{\ccL}}{\abs{\cXL}} \cdot \nrm*{\pi \push \unif(\ccL) - \pi \push \projR(\samplelatp)(\cdot \mid \ccR) }_1 + (8S+4)\eps \\
    &\overset{(ii)}{\le}  \sum_{\cc \in \crl{\cc_j} :  \frac{\abs{\ccL}}{\abs{\cXL}} \ge 4\eps }  \frac{\abs{\ccL}}{\abs{\cXL}} \cdot \prn*{ K \cdot \sqrt{\frac{S A^2 \log \tfrac{SA \abs{\Pi}}{\delta}}{\nrch[\ccL]}} + A \cdot \frac{\nurch[\ccL]}{\nrch[\ccL]} } + (8S+5)\eps\\
    &=  \sum_{\cc \in \crl{\cc_j} :  \frac{\abs{\ccL}}{\abs{\cXL}} \ge 4\eps }  \frac{\nrch[\ccL] + \nurch[\ccL]}{\abs{\cXL}} \cdot \prn*{ K \cdot \sqrt{\frac{S A^2 \log \tfrac{SA \abs{\Pi}}{\delta}}{\nrch[\ccL]}}  + A \cdot \frac{\nurch[\ccL]}{\nrch[\ccL]} } + (8S+5)\eps. \numberthis\label{eq:term1-intermediate}
\end{align*}
The inequality $(i)$ follows by casework on $\cc \in \crl{\cc_j}$:
\begin{itemize}
    \item If $\epsPRS{} \cap \cS[\ccL] \ne \emptyset$ then by the biclique lemma (\pref{lem:biclique}) we have $\crl{x \in \cXL: \optdec(x) \in \epsPRS{} \cap \cS[\ccL]} \subseteq \ccL$. In other words, all of the observations from states in $\epsPRS{} \cap \cS[\ccL]$ are contained in this $\ccL$. Therefore, there can be at most $S$ such components $\cc$, and their contribution to the sum is $8\eps \cdot S$.
    \item If $\epsPRS{} \cap \cS[\ccL] = \emptyset$, then $\ccL$ only contains observations from $(\epsPRS{})^c$, and therefore the total size of such $\ccL$ can be bounded by $2\eps \cdot \abs{\cXL}$ using \pref{corr:sampling-transitions}. Their contribution to the sum is $4\eps$.
\end{itemize}
Furthermore, $(ii)$ uses \pref{lem:relating-proj-to-empirical}.

We now proceed to separately bound the terms in Eq.~\eqref{eq:term1-intermediate}. First, observe that
\begin{align*}
    K \sqrt{SA^2 \log \tfrac{SA \abs{\Pi}}{\delta}}\cdot \sum_{\cc \in \crl{\cc_j}: \frac{\abs{\ccL}}{\abs{\cXL}} \ge 4\eps}  \frac{\sqrt{\nrch[\ccL]} }{\abs{\cXL}}  &\le K \sqrt{\frac{S^2A^2 \log \tfrac{SA \abs{\Pi}}{\delta}}{\nrch[\cXL]}} \\
    &\le K \sqrt{\frac{S^2A^2 \log \tfrac{ SA \abs{\Pi}}{\delta}}{\ndec}} \\
    &\le \eps. \numberthis \label{eq:term1-intermediate2}
\end{align*}
The first inequality follows because by the biclique lemma (\pref{lem:biclique}) we know that the summation must have at most $S$ terms, since each of the $\cc$ contains some $s \in \epsPRS{}$, so we can apply Cauchy-Schwarz for $S$-dimensional vectors. The second inequality is a consequence of \pref{corr:sampling-transitions}, and the last inequality follows by our choice of $\ndec$. 

In addition by \pref{corr:sampling-transitions},
\begin{align*}
    \sum_{\cc \in \crl{\cc_j}: \frac{\abs{\ccL}}{\abs{\cXL}} \ge 4\eps }  \frac{\nrch[\ccL] }{\abs{\cXL}} \frac{\nurch[\ccL]}{\nrch[\ccL]} \le 2\eps. \numberthis \label{eq:term1-intermediate3}
\end{align*}
For the other two terms, observe that by \pref{lem:relating-proj-to-empirical}, when $ \frac{\abs{\ccL}}{\abs{\cXL}} \ge 4\eps$ we must have $\frac{\nurch[\ccL]}{\nrch[\ccL]} \le 1$ so therefore
\begin{align*}
    \hspace{2em}&\hspace{-2em} \sum_{\cc \in \crl{\cc_j} : \frac{\abs{\ccL}}{\abs{\cXL}} \ge 4\eps}  \frac{\nurch[\ccL]}{\abs{\cXL}} 
  \prn*{ K \cdot \sqrt{\frac{SA^2 \log \tfrac{SA \abs{\Pi}}{\delta}}{\nrch[\ccL]}}+ \frac{\nurch[\ccL]}{\nrch[\ccL]} } \\
  &\le \sum_{\cc \in \crl{\cc_j}: \frac{\abs{\ccL}}{\abs{\cXL}} \ge 4\eps} \frac{\nurch[\ccL]}{\abs{\cXL}} 
  \prn*{ K \cdot \sqrt{SA^2 \log \frac{ SA \abs{\Pi}}{\delta} } +1} \\
  &\le K \sqrt{SA^2 \log \frac{ SA \abs{\Pi}}{\delta} } \cdot \eps. \numberthis \label{eq:term1-intermediate4}
\end{align*}
Combining Eqns.~\eqref{eq:term1-intermediate}, \eqref{eq:term1-intermediate2}, \eqref{eq:term1-intermediate3}, and \eqref{eq:term1-intermediate4} we get that
\begin{align*}
    \sum_{\cc \in \crl{\cc_j}}  \frac{\abs{\ccL}}{\abs{\cXL}} \cdot \nrm*{\pi \push \unif(\ccL) - \pi \push \projR(\samplelatp)(\cdot \mid \ccR) }_1 \le K \prn*{  \sqrt{SA^2 \log \frac{ SA \abs{\Pi}}{ \delta} } + S}\eps. \numberthis \label{eq:term1-final}
\end{align*}
\emph{Bounding $\mathtt{Term}_2$.} To bound $\mathtt{Term}_2$, fix any $\cc \in \crl{\cc_j}$. Note that
\begin{align*}
     & \nrm*{\pi \push \projR(\optlatp)(\cdot \mid \ccR) - \pi\push \projR(\samplelatp)(\cdot \mid \ccR)}_1 \\
     &= \sum_{a \in \cA} \abs*{\sum_{\empobs \in \ccR} \prn*{ \frac{\projR(\samplelatp)(\empobs)}{\projR(\samplelatp)(\ccR)}  - \frac{\projR(\optlatp)(\empobs)}{\projR(\optlatp)(\ccR)} }\indd{\pi(\empobs) = a}}\\
     &\le \sum_{\empobs \in \ccR} \abs*{ \frac{\projR(\samplelatp)(\empobs)}{\projR(\samplelatp)(\ccR)}  - \frac{\projR(\optlatp)(\empobs)}{\projR(\optlatp)(\ccR)} } \\
     &=  \sum_{\empobs \in \ccR} \abs*{ \frac{ \samplelatp(\optdec(\empobs))}{\projR(\samplelatp)(\ccR)}  - \frac{\optlatp(\optdec(\empobs))}{\projR(\optlatp)(\ccR)} } \cdot \frac{\ind{\optdec(\empobs) \in \epsPRS{}}}{n_{\optdec(\empobs)}[\cXR]} \\
     &= \sum_{s \in \epsPRS{} \cap \cS[\ccL] \cap \cS[\ccR]} \abs*{ \frac{ \samplelatp(s)}{\projR(\samplelatp)(\ccR)}  - \frac{ \optlatp(s)}{\projR(\optlatp)(\ccR)} } \\
     &=  \frac{ 1 }{\projR(\samplelatp)(\ccR)}\sum_{s \in \epsPRS{} \cap \cS[\ccL] \cap \cS[\ccR]} \abs*{ \samplelatp(s) - \optlatp(s) \cdot \frac{ \projR(\samplelatp)(\ccR)}{\projR(\optlatp)(\ccR)} } \\
     &\le \frac{ \eps }{\projR(\samplelatp)(\ccR)} \\
     &\hspace{3em} + \frac{ 1 }{\projR(\samplelatp)(\ccR)}\sum_{s \in \epsPRS{} \cap \cS[\ccL] \cap \cS[\ccR]}\optlatp(s) \abs*{ 1 - \frac{\projR(\samplelatp)(\ccR)}{\projR(\optlatp)(\ccR)} } &\text{(\pref{lem:sampling-transitions})}\\
     &= \frac{ \eps }{\projR(\samplelatp)(\ccR)} +  \frac{ 1 }{\projR(\samplelatp)(\ccR)} \abs*{\projR(\optlatp)(\ccR)  -  \projR(\samplelatp)(\ccR) } \\
     &\le \frac{ 2\eps }{\projR(\samplelatp)(\ccR)} =  2\eps \frac{\abs{\cXL} }{\nrch[\ccL]}. &\text{(using Eq.~\pref{eq:bound-on-empirical-proj})} 
\end{align*}
Also, we have the trivial bound that
\begin{align*}
    \nrm*{\pi \push \projR(\optlatp)(\cdot \mid \ccR) - \pi\push \projR(\samplelatp)(\cdot \mid \ccR)}_1 \le 2,
\end{align*}
because it is a difference of two probability measures, so we can write the bound
\begin{align*}
    \nrm*{\pi \push \projR(\optlatp)(\cdot \mid \ccR) - \pi\push \projR(\samplelatp)(\cdot \mid \ccR)}_1 \le 2\eps \frac{\abs{\cXL} }{\nrch[\ccL]} \wedge 2. \numberthis\label{eq:term2}
\end{align*}
Using Eq.~\pref{eq:term2} we get that
\begin{align*}
    \hspace{2em}&\hspace{-2em} \sum_{\cc \in \crl{\cc_j}}  \frac{\abs{\ccL}}{\abs{\cXL}} \cdot \nrm*{\pi \push \projR(\samplelatp)(\cdot \mid \ccR) - \pi \push \projR(\optlatp)(\cdot \mid \ccR) }_1 \\
    &\le 2\sum_{\cc \in \crl{\cc_j}}  \frac{\abs{\ccL}}{\abs{\cXL}} \cdot \prn*{ \frac{\eps  \abs{\cXL} }{\nrch[\ccL]} \wedge 1} = 2\sum_{\cc \in \crl{\cc_j}} \prn*{ \frac{\eps \abs{\ccL} }{\nrch[\ccL]} \wedge \frac{\abs{\ccL}}{\abs{\cXL}} } \\
    &\le 2\eps \sum_{\cc \in \crl{\cc_j}: \frac{\abs{\ccL}}{\abs{\cXL}} \ge 4\eps}  \frac{\nrch[\ccL] + \nurch[\ccL]}{\nrch[\ccL]}  + 2 \sum_{\cc \in \crl{\cc_j}: \frac{\abs{\ccL}}{\abs{\cXL}} < 4\eps} \frac{\abs{\ccL}}{\abs{\cXL}} \le (8S + 8)\eps. \numberthis \label{eq:final-bound-term2}
\end{align*}
The last inequality uses the facts that (1) \pref{corr:sampling-transitions} implies that for any $\cc \in \crl{\cc_j}$ such that $\frac{\abs{\ccL}}{\abs{\cXL}} \ge 4\eps$ we have $\frac{\nrch[\ccL] + \nurch[\ccL]}{\nrch[\ccL]} \le 2$ and (2) the same casework we showed above to handle the summation for $\cc \in \crl{\cc_j}$ such that $\frac{\abs{\ccL}}{\abs{\cXL}} < 4\eps$.

Putting together Eqns.~\eqref{eq:term1-final} and \eqref{eq:final-bound-term2}:
\begin{align*}
    \sum_{\cc \in \crl{\cc_j}}  \frac{\abs{\ccL}}{\abs{\cXL}} \cdot \nrm*{\pi \push \unif(\ccL) - \pi \push \projR(\optlatp)(\cdot \mid \ccR) }_1 &\le K \prn*{  \sqrt{SA^2 \log \frac{SA  \abs{\Pi}}{\delta} }+ S}\eps \eqqcolon \beta.
\end{align*}
Thus, we can conclude that $\projR(\optlatp) \in \cP$, thus concluding the proof of Part (1).

\paragraph{Part (2).} Observe that in light of Part (1), the set $\cP$ is nonempty so therefore the maximum is well defined.

We want to show a bound on
\begin{align*}
    \max_{\bar{p} \in \cP} \max_{\pi \in \Pi_{h+1:H}}~\abs*{Q^\pi(x_h, a_h) - \wh{R}(x_h, a_h) - \En_{\empobs \sim \bar{p}} \wh{V}^\pi(\empobs)}.
\end{align*}

Fix any $\bar{p} \in \cP$ and $\pi \in \Pi_{h+1:H}$. We compute
\begin{align*}
    \hspace{2em}&\hspace{-2em} \abs*{Q^\pi(x_h, a_h) - \wh{R}(x_h, a_h) - \En_{\empobs \sim \bar{p}} \wh{V}^\pi(\empobs)} \\
    &\le \frac{\eps}{H} + \abs*{ \En_{s \sim \optlatp} V^\pi(s) - \En_{s \sim \samplelatp} V^\pi(s) } + \abs*{ \En_{s \sim \samplelatp} V^\pi(s) - \En_{\empobs \sim \bar{p}} \wh{V}^\pi(\empobs) } &\text{(\pref{lem:sampling-rewards})}\\
    &\le 2\eps + \abs*{ \En_{s \sim \samplelatp} V^\pi(s) - \En_{\empobs \sim \bar{p}} \wh{V}^\pi(\empobs) } &\text{(\pref{lem:sampling-transitions})} \\
    &\le 2\eps + \underbrace{ \abs*{ \En_{s \sim \samplelatp} V^\pi(s) - \En_{\empobs \sim \projR(\samplelatp)} V^\pi(\empobs) } }_{ \eqqcolon \mathtt{Term}_1} \\
    &\qquad + \underbrace{ \abs*{ \En_{\empobs \sim \projR(\samplelatp)} V^\pi(\empobs) - \En_{\empobs \sim \projR(\samplelatp)} \wh{V}^\pi(\empobs)} }_{\eqqcolon \mathtt{Term}_2} \\
    &\qquad + \underbrace{ \abs*{\En_{\empobs \sim \projR(\samplelatp)} \wh{V}^\pi(\empobs) - \En_{\empobs \sim \bar{p}} \wh{V}^\pi(\empobs) } }_{\eqqcolon \mathtt{Term}_3}.
\end{align*}
\emph{Bounding $\mathtt{Term}_1$.} For the first term, we can calculate that
\begin{align*}
    \mathtt{Term}_1 &= \abs*{ \En_{s \sim \samplelatp} V^\pi(s) - \En_{\empobs \sim \projR(\samplelatp)} V^\pi(\empobs) } = \abs*{ \En_{s \sim \samplelatp} \En_{x \sim \emission(s) } V^\pi(x) - \En_{\empobs \sim \projR(\samplelatp)} V^\pi(\empobs) } \\
    &= \abs*{ \En_{s \sim \samplelatp} \brk*{ \En_{x \sim \emission(s) } V^\pi(x) - \En_{\empobs \sim \mathrm{Unif}(\crl{\empobs \in \cXR: \optdec(\empobs) = s})} V^\pi(\empobs) } } \\
    &\le 2\eps + \abs*{ \En_{s \sim \samplelatp} \brk*{ \ind{s \in \epsPRS{h}} \prn*{ \En_{x \sim \emission(s) } V^\pi(x) - \En_{\empobs \sim \mathrm{Unif}(\crl{\empobs\in \cXR: \optdec(\empobs) = s})} V^\pi(\empobs) } } } \\
    &\le 3\eps. \numberthis\label{eq:t1}
\end{align*}
The first inequality follows by \pref{corr:sampling-transitions}, and the second inequality follows by \pref{lem:pushforward-concentration}.

\emph{Bounding $\mathtt{Term}_2$.} For the second term, we have by assumption that:
\begin{align*}
    \mathtt{Term}_2 = \abs*{ \En_{\empobs \sim \projR(\samplelatp)} \brk*{V^\pi(\empobs) - \wh{V}^\pi(\empobs)}  } \le \Gamma_{h+1}. \numberthis\label{eq:t2}
\end{align*}

\emph{Bounding $\mathtt{Term}_3$.} Now we calculate a bound on $\mathtt{Term}_3$. In what follows for any connected component $\cc$ we let $\empobs_\cc$ denote an arbitrary fixed observation from $\ccR$ (for example, the lowest indexed one). Observe that for any $p \in \Delta(\cXR)$ we have
\begin{align*}
    \En_{\empobs \sim p} \wh{V}^\pi(\empobs) 
    &= \sum_{\cc\in \crl{\cc_j}} \sum_{\empobs \in \ccR} p(\empobs) \cdot \wh{Q}^\pi(\empobs, \pi(\empobs)) \hspace{9em} \text{($\crl{\cc_j}$ form a partition of $\cXR$)}\\
    &\le 4S\taudec + 8S\eps +  \sum_{\cc \in \crl{\cc_j}} \sum_{\empobs \in \ccR} p(\empobs) \cdot \wh{Q}^\pi(\empobs_\cc, \pi(\empobs_\cc)) \hspace{8em}\text{(\pref{lem:width-of-cc})}\\
    &= 4S\taudec + 8S\eps +  \sum_{\cc \in \crl{\cc_j}} p(\ccR) \sum_{\empobs \in \ccR} \frac{p(\empobs)}{p(\ccR)} \wh{Q}^\pi(\empobs_\cc, \pi(\empobs_\cc)).
\end{align*}
Similarly, one can show the lower bound on $\En_{\empobs \sim p} \wh{V}^\pi(\empobs)$. Therefore we apply the bound to get:
\begin{align*}
    & \abs*{\En_{\empobs \sim \projR(\samplelatp)} \wh{V}^\pi(\empobs) - \En_{\empobs \sim \bar{p}} \wh{V}^\pi(\empobs) } \\
    &\le 8S\taudec + 16S\eps \\
    &\quad + \sum_{\cc \in \crl{\cc_j}} \abs*{\projR(\samplelatp)(\ccR) \sum_{\empobs \in \ccR} \frac{\projR(\samplelatp)(\empobs)}{\projR(\samplelatp)(\ccR) } \wh{Q}^\pi(\empobs_\cc, \pi(\empobs)) - \bar{p}(\cc) \sum_{\empobs \in \ccR} \frac{\bar{p}(\empobs)}{\bar{p}(\cc)} \wh{Q}^\pi(\empobs_\cc, \pi(\empobs)) } \\
    &\overset{(i)}{\le} 8S\taudec + 16S\eps + 2\eps \\
    &\quad + \sum_{\cc \in \crl{\cc_j}} \abs*{\frac{\abs{\ccL}}{\abs{\cXL}} \sum_{\empobs \in \cc} \frac{\projR(\samplelatp)(\empobs)}{\projR(\samplelatp)(\ccR) } \wh{Q}^\pi(\empobs_\cc, \pi(\empobs)) - \bar{p}(\cc) \sum_{\empobs \in \cc} \frac{\bar{p}(\empobs)}{\bar{p}(\cc)} \wh{Q}^\pi(\empobs_\cc, \pi(\empobs)) } \\
    &\overset{(ii)}{\le}  8S\taudec + 16S\eps + 5\eps + \sum_{\cc \in \crl{\cc_j}}  \frac{\abs{\ccL}}{\abs{\cXL}} \cdot \abs*{\sum_{\empobs \in \cc} \frac{\projR(\samplelatp)(\empobs)}{\projR(\samplelatp)(\ccR) } \wh{Q}^\pi(\empobs_\cc, \pi(\empobs)) - \frac{\bar{p}(\empobs)}{\bar{p}(\cc)} \wh{Q}^\pi(\empobs_\cc, \pi(\empobs)) },\\
    &\le 8S\taudec + 16S\eps + 5\eps + \sum_{\cc \in \crl{\cc_j}}  \frac{\abs{\ccL}}{\abs{\cXL}} \cdot \nrm*{\pi\push \projR(\samplelatp)(\cdot \mid \ccR) - \pi \push \bar{p}(\cdot \mid \cc) }_1 ,
\end{align*}
where $(i)$ follows by \pref{corr:estimating-projR} and the bound $\frac{\projR(\samplelatp)(\empobs)}{\projR(\samplelatp)(\ccR) } \wh{Q}^\pi(\empobs_\cc, \pi(\empobs)) \in [0,1]$, and $(ii)$ follows by the first constraint on $\bar{p} \in \cP$ and the bound $\frac{p(\empobs)}{p(\cc)} \wh{Q}^\pi(\empobs_\cc, \pi(\empobs)) \in [0,1]$. 

From here, we will use the second constraint on $\bar{p} \in \cP$:
\begin{align*}
    \hspace{2em}&\hspace{-2em}  \abs*{\En_{\empobs \sim \projR(\samplelatp)} \wh{V}^\pi(\empobs) - \En_{\empobs \sim \bar{p}} \wh{V}^\pi(\empobs) } \\
    &\le  8S\taudec + 16S\eps + 5\eps + \sum_{\cc \in \crl{\cc_j}}  \frac{\abs{\ccL}}{\abs{\cXL}} \cdot \nrm*{\pi\push \projR(\samplelatp)(\cdot \mid \ccR) - \pi \push \bar{p}(\cdot \mid \cc) }_1 \\
    &\le  8S\taudec + 16S\eps + 5\eps + \beta + \sum_{\cc \in \crl{\cc_j}}  \frac{\abs{\ccL}}{\abs{\cXL}} \cdot \nrm*{\pi\push \projR(\samplelatp)(\cdot \mid \ccR) -\pi \push \unif(\ccL) }_1 \\
    &\le  8S\taudec + 16S\eps + 5\eps + 2\beta. \numberthis\label{eq:t3}
\end{align*}
The last inequality follows because our proof for Part (1) of the lemma actually showed that $\projR(\samplelatp) \in \cP$. Combining Eqns.~\eqref{eq:t1}, \eqref{eq:t2}, and \eqref{eq:t3} we get the final bound
\begin{align*}
    \abs*{Q^\pi(x_h, a_h) - \wh{R}(x_h, a_h) - \En_{\empobs \sim \bar{p}} \wh{V}^\pi(\empobs)} \le \Gamma_{h+1} + K \cdot \prn*{\beta + S \taudec}.
\end{align*}
This completes the proof of \pref{lem:main-induction}.

\subsubsection{Analysis of \stochrefit}
\begin{lemma}[Certificate Implies Transition Inaccuracy]\label{lem:certificate-obs}
Assume that $\eventemulator$ hold. Let $\wh{M}$ be a policy emulator. Suppose there exists a certificate $(\empobs, \pi) \in \estmdpobsspace{h} \times (\cA \circ \Pi_{h+1:H})$ such that
\begin{align*}
    \abs*{\wh{V}^{\pi}(\empobs) - V^{\pi}(\empobs)} \ge \tauref.
\end{align*}
Then there exists some tuple $(\bar{\empobs}, \bar{a}) \in \estmdpobsspace{} \times \actionsp$ such that
\begin{align*}
    \abs*{\En_{\empobs' \sim \wh{P}(\cdot \mid  \bar{\empobs}, \bar{a})} V^{\pi}(\empobs') - \En_{\empobs' \sim P(\cdot \mid  \bar{\empobs}, \bar{a})} V^{\pi}(\empobs')} \ge \frac{\tauref}{2H}. \numberthis\label{eq:bad-state}
\end{align*}
\end{lemma}

\begin{proof}Suppose that Eq.~\eqref{eq:bad-state} did not hold for any $(\bar{\empobs}, \bar{a})$. Then by the Performance Difference Lemma we have
\begin{align*}
    \hspace{2em}&\hspace{-2em} \abs*{ V^{\pi}(\empobs) - \wh{V}^{\pi}(\empobs) } \\
    &\le \abs*{\wh{R}(x, \pi) - R(x,\pi)} + \abs*{ \En_{x' \sim P(\cdot \mid \empobs, \pi)} V^{\pi}(x') - \En_{\empobs' \sim \wh{P}(\cdot \mid \empobs, \pi)}V^{\pi}(\empobs') } \\
    &\hspace{5em} + \abs*{\En_{\empobs' \sim \wh{P}(\cdot \mid \empobs, \pi)}V^{\pi}(\empobs') - \En_{\empobs' \sim \wh{P}(\cdot \mid \empobs, \pi)} \wh{V}^{\pi}(\empobs') }\\
    &\overset{(i)}{\le} \frac{\tauref}{2H} + \frac{\eps}{H} + \abs*{\En_{\empobs' \sim \wh{P}(\cdot \mid \empobs, \pi)}V^{\pi}(\empobs') - \En_{\empobs' \sim \wh{P}(\cdot \mid \empobs, \pi)} \wh{V}^{\pi}(\empobs') } \\
    &\le \frac{\tauref}{2H} + \frac{\eps}{H} + \max_{\empobs' \in \estmdpobsspace{h+1}}\abs*{  V^{\pi}(\empobs') - \wh{V}^{\pi}(\empobs') } \le \cdots \overset{(ii)}{\le} \frac{\tauref}{2} + \eps,
\end{align*}
where $(i)$ uses \pref{lem:sampling-rewards} and the negation of Eq.~\eqref{eq:bad-state}, and $(ii)$ applies the bound recursively. Since $\tauref > 2\eps$, we have reached a contradiction. This proves \pref{lem:certificate-obs}.
\end{proof}

\begin{lemma}[Refitting Correctness]\label{lem:refitting-correctness} Assume that $\eventemulator, \eventref$ hold. The following are true about \pref{alg:stochastic-refit-v2} in the search for incorrect transitions (\pref{line:else-triggered}-\ref{line:loss-vector-obs} are executed):
\begin{enumerate}
    \item[(1)] For every $(\empobs, \pi)$ from in \pref{line:for-every}, at least one such $(\bar{\empobs}, \bar{a})$ pair is identified by \pref{line:bad-obs-stochastic}.
    \item[(2)] Every $(\bar{\empobs}, \bar{a})$ pair identified by \pref{line:bad-obs-stochastic} satisfies
    \begin{align*}
    \abs*{\En_{\empobs' \sim \wh{P}(\cdot \mid \bar{\empobs}, \bar{a})} V^{\pi}(\empobs') - \En_{\empobs' \sim \projm{\estmdpobsspace{h(\bar{\empobs})+1}}(\optlatp) } V^\pi(\empobs') } \ge \frac{\tauref}{16H}.
    \end{align*}
    \item[(3)] For every $(\bar{\empobs}, \bar{a})$ identified by \pref{line:bad-obs-stochastic}, the corresponding loss vector $\ell_\mathrm{loss}$ from \pref{line:loss-vector-obs} satisfies
    \begin{align*}
        \tri*{\wh{P}(\cdot \mid \bar{\empobs}, \bar{a}) - \projm{\estmdpobsspace{h(\bar{\empobs})+1}}(\optlatp(\cdot \mid \bar{\empobs}, \bar{a})), \ell_\mathrm{loss} } \ge \frac{\tauref}{16H}.
    \end{align*}
\end{enumerate}
\end{lemma}

\begin{proof}
To prove Part (1) we use \pref{lem:certificate-obs}, which shows that there exists at least one such $(\bar{\empobs}, \bar{a})$ such that
\begin{align*}
    \abs*{\En_{\empobs' \sim \wh{P}(\cdot \mid  \bar{\empobs}, \bar{a})} V^{\pi}(\empobs') - \En_{\empobs' \sim P(\cdot \mid  \bar{\empobs}, \bar{a})} V^{\pi}(\empobs')} \ge \frac{\tauref}{2H}. \numberthis\label{eq:lower-bound}
\end{align*}
Therefore we know that for such $(\bar{\empobs}, \bar{a})$:
\begin{align*}
    &\abs*{\En_{\empobs' \sim \wh{P}(\cdot \mid  \bar{\empobs}, \bar{a})} V^{\pi}(\empobs') - \En_{\empobs' \sim P(\cdot \mid  \bar{\empobs}, \bar{a})} V^{\pi}(\empobs')} \\
    &\le \abs*{\En_{\empobs' \sim \wh{P}(\cdot \mid  \bar{\empobs}, \bar{a})} \vestarg{\empobs'}{\pi} + \wh{R}(\bar{\empobs}, \bar{a}) -\qestarg{ \bar{\empobs}, \bar{a} }{\pi}} + 3\eps &\text{(\pref{lem:refit-monte-carlo-accuracy} and \pref{lem:sampling-rewards})}\\
    &= \abs*{ \Delta(\bar{\empobs}, \bar{a}) } + 3\eps \\
    &\Longrightarrow  \quad \abs*{ \Delta(\bar{\empobs}, \bar{a}) } \ge \frac{\tauref}{2H} - 3\eps \ge \frac{\tauref}{8H}, &\text{(Using Eq.~\eqref{eq:lower-bound})}
\end{align*}
so therefore this $(\bar{\empobs}, \bar{a})$ is identified by \pref{line:bad-obs-stochastic} of \stochrefit{}.

Now we prove Part (2). Fix any $(\bar{\empobs}, \bar{a})$ pair identified by \pref{line:bad-obs-stochastic} of \stochrefit{}. Let $h=h(\bar{\empobs})$ denote the layer that a given $\bar{\empobs}$ is found. First we observe that
\begin{align*}
    \hspace{2em}&\hspace{-2em} \En_{\empobs' \sim P(\cdot \mid  \bar{\empobs}, \bar{a})} V^{\pi}(\empobs') - \En_{\empobs' \sim \projm{\estmdpobsspace{h+1}}(\optlatp) } V^\pi(\empobs') = \En_{\empobs' \sim \emission \circ \optlatp } V^\pi(\empobs') - \En_{\empobs' \sim \projm{\estmdpobsspace{h+1}}(\optlatp) } V^\pi(\empobs') \\
    &= \En_{s \sim \optlatp} \brk*{\En_{\empobs' \sim \emission(s) } \brk*{ V^\pi(\empobs') } - \ind{s \in \epsPRS{}} \En_{\empobs' \sim \mathrm{Unif}(\crl{\empobs \in \estmdpobsspace{h+1}: \optdec(\empobs) = s})} \brk*{V^\pi(\empobs')} }\\
    &\le \eps + \En_{s \sim \optlatp} \brk*{\ind{s \in \epsPRS{}} \prn*{ \En_{\empobs' \sim \emission(s) } \brk*{ V^\pi(\empobs') } -  \En_{\empobs' \sim \mathrm{Unif}(\crl{\empobs\in \estmdpobsspace{h+1}: \optdec(\empobs) = s})} \brk*{V^\pi(\empobs')} } } \\
    &\le 2\eps,
\end{align*}
where the last inequality uses \pref{lem:pushforward-concentration}. The other side of the inequality can be similarly shown, so
\begin{align*}
    \abs*{ \En_{\empobs' \sim P(\cdot \mid  \bar{\empobs}, \bar{a})} V^{\pi}(\empobs') - \En_{\empobs' \sim \projm{\estmdpobsspace{h+1}}(\optlatp) } V^\pi(\empobs') }
    \le 2\eps. \numberthis \label{eq:relating-q-to-proj}
\end{align*}
We can compute that
\begin{align*}
    \frac{\tauref}{8H} &\le \abs{ \Delta(\bar{\empobs}, \bar{a}) } \\
    &=  \abs*{\En_{\empobs' \sim \wh{P}(\cdot \mid  \bar{\empobs}, \bar{a})} \vestarg{\empobs'}{\pi} + \wh{R}(\bar{\empobs}, \bar{a}) -\qestarg{ \bar{\empobs}, \bar{a} }{\pi}} \\
    &\le \abs*{\En_{\empobs' \sim \wh{P}(\cdot  \mid  \bar{\empobs},\bar{a})} V^\pi(\empobs') + \wh{R}(\bar{\empobs}, \bar{a}) - Q^\pi(\bar{\empobs}, \bar{a})} + 2\eps &\text{(\pref{lem:refit-monte-carlo-accuracy})} \\
    &\le \abs*{\En_{\empobs' \sim \wh{P}(\cdot  \mid  \bar{\empobs},\bar{a})} V^\pi(\empobs') - \En_{\empobs' \sim P(\cdot \mid  \bar{\empobs}, \bar{a})} V^{\pi}(\empobs') } + 3\eps &\text{(\pref{lem:sampling-rewards})} \\
    &\le \abs*{ \En_{\empobs' \sim \wh{P}(\cdot  \mid  \bar{\empobs},\bar{a})} V^\pi(\empobs') -\En_{\empobs' \sim \projm{\estmdpobsspace{h+1}}(\optlatp) } V^\pi(\empobs') } + 5\eps. &\text{(Eq.~\eqref{eq:relating-q-to-proj})}
\end{align*}
Rearranging we see that
\begin{align*}
    \abs*{ \En_{\empobs' \sim \wh{P}(\cdot  \mid  \bar{\empobs},\bar{a})} V^\pi(\empobs') -\En_{\empobs' \sim \projm{\estmdpobsspace{h+1}}(\optlatp) } V^\pi(\empobs') } \ge \frac{\tauref}{8H} - 5\eps \ge \frac{\tauref}{16H},
\end{align*}
and this proves Part (2).

For Part (3), suppose that $\Delta(\bar{\empobs}, \bar{a}) \ge \tauref/8H$ (the case where $\Delta(\bar{\empobs}, \bar{a}) \le -\tauref/8H$ can be tackled similarly). Then we have $\ell_\mathrm{loss} \coloneqq \qestarg{\cdot, \pi(\cdot)}{\pi} \in [0,1]^{\estmdpobsspace{h+1}}$. We can compute that
\begin{align*}
\frac{\tauref}{8H} &\le \En_{\empobs' \sim \wh{P}(\cdot \mid  \bar{\empobs}, \bar{a})} \qestarg{x', \pi(x')}{\pi} + \wh{R}(\bar{\empobs}, \bar{a}) -\qestarg{ \bar{\empobs}, \bar{a} }{\pi} \\
&= \tri*{ \wh{P}(\cdot  \mid  \bar{\empobs},\bar{a}), \ell_\mathrm{loss}} + \wh{R}(\bar{\empobs}, \bar{a}) -\qestarg{ \bar{\empobs}, \bar{a} }{\pi}  \\
&\le \eps + \tri*{ \wh{P}(\cdot  \mid  \bar{\empobs},\bar{a}), \ell_\mathrm{loss}} + \wh{R}(\bar{\empobs}, \bar{a}) - Q^\pi(\bar{\empobs}, \bar{a}) &\text{(\pref{lem:refit-monte-carlo-accuracy})}\\
&\le 4\eps + \tri*{ \wh{P}(\cdot  \mid  \bar{\empobs},\bar{a}), \ell_\mathrm{loss}} - \En_{\empobs' \sim \projm{\estmdpobsspace{h+1}}(\optlatp) } V^\pi(\empobs') &\text{(\pref{lem:sampling-rewards} and Eq.~\eqref{eq:relating-q-to-proj})}\\
&\le 5\eps + \tri*{ \wh{P}(\cdot  \mid  \bar{\empobs},\bar{a}) -  \projm{\estmdpobsspace{h+1}}(\optlatp(\cdot  \mid  \bar{\empobs},\bar{a})), \ell_\mathrm{loss}} &\text{(\pref{lem:refit-monte-carlo-accuracy})}
\end{align*}
Rearranging we get $ \tri*{ \wh{P}(\cdot  \mid  \bar{\empobs},\bar{a}) -  \projm{\estmdpobsspace{h+1}}(\optlatp(\cdot  \mid  \bar{\empobs},\bar{a})), \ell_\mathrm{loss}}  \ge \tauref/(16H)$, thus proving part (3).
\end{proof}

\begin{lemma}[Bound on Number of Refits]\label{lem:bound-on-number-refits}
Assume that $\eventemulator, \eventref$ hold, and that every time \pref{alg:stochastic-refit-v2} is called, the confidence sets $\cP$ satisfy 
\begin{align*}
    \forall~(x,a) \in \estmdpobsspace{}\times \cA: \quad \projm{\estmdpobsspace{h(x)+1}}(\optlatp(\cdot  \mid  x, a)) \in \cP(x, a).
\end{align*}
Then across all calls to \pref{alg:stochastic-refit-v2}, \pref{line:loss-vector-obs} is executed at most $(\nreset AH/ \eps^2) \cdot \log \nreset$ times.
\end{lemma}  
    
\begin{proof}
Fix $h \in [H]$ and a pair $(\empobs, a) \in \estmdpobsspace{h} \times \cA$. Suppose we execute \pref{line:loss-vector-obs} for $T_\mathrm{refit}$ times on $(\empobs, a)$. Denote the sequence of transition estimates as $\crl{\wh{P}^{(t)}(\cdot \mid x, a)}_{t \in [T_\mathrm{refit}]}$ and the sequence of loss vectors as $\crl{\ell_\mathrm{loss}^{(t)}}_{t \in [T_\mathrm{refit}]}$.
    
    By Part (3) of \pref{lem:refitting-correctness}, for all times $t \in [T_\mathrm{refit}]$,
    \begin{align*}
        \tri*{\wh{P}^{(t)}(\cdot \mid  x, a) - \projm{\estmdpobsspace{h+1}}(\optlatp(\cdot \mid x,a)), \ell_\mathrm{loss}^{(t)} } \ge \frac{\tauref}{16H}. \numberthis \label{eq:refit-lb-obs}
    \end{align*}
    On the other hand, we have a bound on the total regret of OMD with step size $\eps$ \cite[see, e.g., Thm.~5.2 of][]{bubeck2011introduction}:
    \begin{align*}
        \hspace{2em}&\hspace{-2em} \sum_{t=1}^{T_\mathrm{refit}} \tri*{\wh{P}^{(t)}(\cdot \mid x,a) - \projm{\estmdpobsspace{h+1}}(\optlatp(\cdot \mid x,a)), \ell_\mathrm{loss}^{(t)} } \\
        &\le \frac{1}{\eps} D_\mathsf{ne} \prn*{ \projm{\estmdpobsspace{h+1}}(\optlatp(\cdot \mid x,a)) ~\Vert~ \wh{P}^{(1)}(\cdot \mid x,a) } + \frac{\eps}{2} \sum_{t=1}^{T_\mathrm{refit}} \nrm*{\ell_\mathrm{loss}^{(t)}}_\infty\\
        &\le \frac{\log \nreset}{\eps} + \frac{\eps T_\mathrm{refit}}{2}. \numberthis \label{eq:refit-ub-obs}
    \end{align*}
    Therefore, combining Eqs.~\eqref{eq:refit-lb-obs} and \eqref{eq:refit-ub-obs} along with the value $\tauref = 80H\eps$ we have the bound
    \begin{align*}
        T_\mathrm{refit}
        \le \frac{\log \nreset}{\eps^2}.
    \end{align*}
    Using the fact that there are $\nreset AH$ such $(\empobs, a)$ pairs proves the result.
\end{proof}

\subsubsection{Proof of \pref{thm:block-mdp-result}}

In the proof, we will assume that $\eventemulator$ holds, that $\eventdec_t$ holds for all times $t \in [T_\mathsf{D}]$, that $\eventref_t$ holds for all times $t \in [T_\mathsf{R}]$. By union bound, this holds with probability at least $1-(3 T_\mathsf{D} + T_\mathsf{R} + 3) \delta$. 

We will show that under the choice of parameters $\nreset$, $\ndec$, and $\nmc$ in the algorithm pseudocode, \stochalg{} returns a $\wt{O}(SAH^2 \eps)$-optimal policy, and that
\begin{align*}
    T_\mathsf{D}, T_\mathsf{R} \le \poly(\cpush, S,A, H, \eps^{-1}, \log \abs{\Pi}, \log \delta^{-1}).
\end{align*}
Therefore, rescaling $\eps$ and $\delta$ will imply the final result.

\paragraph{Proof by Induction.} Take $\Gamma_h \coloneqq K (H-h+1) \prn{\beta +SH}\eps$ for some suitably large constant $K > 0$. Furthermore set $\taudec = 81 H \eps$. We will show that the following properties holds at every layer $h \in [H]$:
\begin{enumerate}
    \item [(1)] \textit{Transition set includes ground truth:} $\forall~ (x,a) \in \estmdpobsspace{h} \times \cA$, $\projm{\estmdpobsspace{h}}(\optlatp(\cdot  \mid  x,a)) \in \cP(x,a)$.
    \item [(2)] \textit{Accurate value estimates:} $\forall~ (x,a) \in \estmdpobsspace{h} \times \cA, \pi \in \Pi_{h+1:H}$, $\abs{Q^\pi(x,a) - \wh{Q}^\pi(x,a)} \le \Gamma_h$.
    \item [(3)] \textit{Valid test policies:} $\Pitest_h$ are $\taudec$-valid for $\estmdp$ at layer $h$.
\end{enumerate}

To analyze \stochalg{} we will show that at the end of every while loop, these properties always hold for all layers $h > \ell_\mathsf{next}$.

\underline{Base Case.} For the first loop with $\ell = H$, property (1) holds because there are no transitions to be constructed at layer $H$. By \pref{lem:sampling-rewards}, property (2) holds after the initialization of the policy emulator in \pref{line:init-end}. Furthermore, in the first call to \stochrefit{}, the computed test policies are open loop, so again using \pref{lem:sampling-rewards}, we see that \pref{line:great-success-stochastic} is triggered. Therefore, properties (2) and (3) hold at the end of the while loop. The current layer index is set to $\ell \gets H-1$.

\underline{Inductive Step.} Suppose the current layer index is $\ell$, and that properties (1)--(3) hold for all $h > \ell$. By \pref{lem:main-induction}, for every $(x_\ell, a_\ell)$ that we call \stochdecoder{} on the updated confidence sets $\cP$ returned by  satisfy property (1), and the choice $\wh{P} \in \cP$ satisfies property (2). Now we do casework on the outcome of \stochrefit{} called at layer $\ell$.

\begin{itemize}
    \item \textbf{Case 1: Return in \pref{line:great-success-stochastic}.}  By construction, property (3) is satisfied for layer $\ell$. In this case, since \stochrefit{} made no updates to $\estlatentmdp$ or $\cP$, properties (1) and (2) continue to hold at layer $\ell$ onwards.
    \item \textbf{Case 2: Return in \pref{line:great-success-stochastic-else}.} Property (1) holds because \stochrefit{} does not modify $\cP$. Let $\ell_\mathsf{next}$ denote the layer at which we jump to. \stochrefit{} made no updates to $\estlatentmdp$ at layers $\ell_\mathsf{next} + 1$ onwards, and therefore the previously computed test policies $\Pitest_{\ell_\mathsf{next}+1:H}$ must still be valid, so therefore properties (2) and (3) continue to hold at layer $\ell_\mathsf{next}$ onwards.
\end{itemize}

Continuing the induction, once $\ell \gets 0$ is reached in \stochalg{}, the policy emulator $\estmdp$ satisfies the bound
\begin{align*}
    \max_{\pi \in \Pi}~\abs*{V^\pi(s_1) - \En_{x_1 \sim \unif(\estmdpobsspace{1}) }\brk{\wh{V}^\pi(x_1)}} \le \Gamma_1 \le \wt{O}(SAH^2 \eps).\numberthis\label{eq:final-bound}
\end{align*}

\paragraph{Bounding the Number of Calls to \stochdecoder{} and \stochrefit{}.} By \pref{lem:bound-on-number-refits}, the total number of executions of \pref{line:loss-vector-obs} in \pref{alg:stochastic-refit-v2} is at most $(\nreset AH/ \eps^2) \cdot \log \nreset$. In the worst case, every revisit to an already computed layer (i.e., jumping back to $\ell_\mathsf{next} \ge \ell$) requires us to restart \stochdecoder{} at layer $H$ and therefore decode at most $\nreset AH$ additional times, so therefore
\begin{align*}
    T_\mathsf{D} \le \frac{\nreset^2 A^2H^2}{\eps^2} \log \nreset.
\end{align*}
Similarly, every revisit in the worst case requires $H$ additional calls to \stochrefit{} so therefore
\begin{align*}
    T_\mathsf{R} \le \frac{\nreset A H^2}{\eps^2} \log \nreset.
\end{align*}
As claimed, both $T_\mathsf{D}$ and $T_\mathsf{R}$ are $\poly(\cpush, S,A, H, \eps^{-1}, \log \abs{\Pi}, \log \delta^{-1})$.

\paragraph{Final Sample Complexity Bound.} Now we compute the total number of samples.
\begin{itemize}
    \item \pref{alg:stochastic-bmdp-solver-v2} uses $\nreset \cdot AH$ samples to $\mu_h$ to form the state space of the policy emulator, and for each state-action pair we sample $\wt{O}(H^2\eps^{-2})$ times to estimate the reward.
    \item \pref{alg:stochastic-decoder-v2} is called $T_\mathsf{D} \le \wt{O}(\nreset^2 A^2 H^2 \eps^{-2})$  times, and each time we use $\ndec \cdot \nreset^2 \nmc$ rollouts.
    \item \pref{alg:stochastic-refit-v2} is called $T_\mathsf{R} \le \wt{O}(\nreset A H^2 \eps^{-2})$ times, and each time we use $2\nreset^2 \nmc$ to evaluate the test policies. Furthermore, by \pref{lem:bound-on-number-refits}, \pref{line:mc2} is triggered at most $\wt{O}(\nreset A H \eps^{-2})$ times, with every time requiring an additional $\nmc \cdot \nreset AH$ rollouts.
\end{itemize}

Therefore in total we use
\begin{align*}
     \hspace{2em}&\hspace{-2em} \nreset \frac{AH^3}{\eps^2} + \nreset^4 \ndec \nmc \frac{A^2H^2}{\eps^2} + \nreset^3 \nmc \frac{AH^2}{\eps^2} + \nreset^2 \nmc \frac{A^2H^2}{\eps^2} \\
     &= \frac{\cpush^4 S^6 A^{12}H^3 }{\eps^{18}} \cdot \mathrm{polylog} \prn*{\cpush, S, A, H, \abs{\Pi}, \eps^{-1}, \delta^{-1}}\quad \text{samples.}
\end{align*}
Afterwards, we can rescale $\eps \gets \eps/\wt{O}(SAH^2)$ so that the bound Eq.~\eqref{eq:final-bound} is at most $\eps$, and rescale $\delta \gets \delta / (3 T_\mathsf{D} + T_\mathsf{R} + 1)$ so that the guarantee occurs with probability at least $1-\delta$. The final sample complexity is
\begin{align*}
    \frac{\cpush^4 S^{24} A^{30} H^{39} }{\eps^{18}} \cdot \mathrm{polylog} \prn*{\cpush, S, A, H, \abs{\Pi}, \eps^{-1}, \delta^{-1}}\quad \text{samples.}
\end{align*}

This concludes the proof of \pref{thm:block-mdp-result}.\qed


\appendix
\renewcommand{\chaptername}{Appendix}
\chapter{Technical Tools}

\section{Concentration Inequalities}
\begin{lemma}[Hoeffding's Inequality]\label{lem:hoeffding}
Let $Z_1, \cdots, Z_n$ be independent bounded random variables with $Z_i \in \brk{a,b}$ for all $i \in [n]$. Then
\begin{align*}
    \Pr \brk*{\abs*{\frac{1}{n} \sum_{i=1}^n Z_i - \En[Z_i]} \ge t} \le 2\exp\prn*{-\frac{2nt^2}{(b-a)^2}}.
\end{align*}
\end{lemma} 

\begin{lemma}[Multiplicative Chernoff Bound]\label{lem:chernoff}
Let $Z_1, \cdots, Z_n$ be i.i.d.~random variables taking values in $\crl{0,1}$ with expectation $\mu$. Then for any $\delta > 0$,
\begin{align*}
    \Pr \brk*{\frac{1}{n} \sum_{i=1}^n Z_i \ge  \prn*{1+ \delta} \cdot \mu} \le \exp\prn*{-\frac{\delta^2 \mu n }{2+\delta}}.
\end{align*}
Furthermore for any $\delta \in (0,1)$,
\begin{align*}
    \Pr \brk*{\frac{1}{n} \sum_{i=1}^n Z_i \le  \prn*{1- \delta} \cdot \mu} \le \exp\prn*{-\frac{\delta^2 \mu n }{2}}.
\end{align*}
\end{lemma} 

\begin{lemma}[Azuma-Hoeffding]\label{lem:azuma-hoeffding}
    Let $(Z_t)_{t \in[T]}$ be a sequence of real-valued random variables adapted to filtration $(\cF_t)_{t \in [T]}$. Suppose that $\abs{Z_t}\le R$ almost surely. Then with probability at least $1-\delta$,
    \begin{align*}
        \abs*{ \frac{1}{T}\sum_{t=1}^T Z_t - \En_{t-1}[Z_t] } \le R \sqrt{ \frac{8 \log (2/\delta)}{T} }.
    \end{align*}
\end{lemma}

\section{Information Theory}

Our information-theoretic lower bounds are facilitated by recent developments that build a unified framework for \emph{interactive statistical decision making} (ISDM) \cite{chen2024assouad}. We first state an interactive version of Le Cam's convex hull method, which can be derived as a consequence of \cite[Thm.~2][]{chen2024assouad}. For completeness, we include the proof. It closely mirrors the proof of \cite[Prop.~4 of][]{chen2024assouad}, which shows how \cite[Thm.~2 of][]{chen2024assouad} recovers the noninteractive variant of Le Cam's convex hull method. 

\begin{theorem}[Interactive Le Cam's Convex Hull Method]\label{thm:interactive-lecam-cvx}
For parameter space $\Theta$, let $\cM = \crl{M_{\theta} \mid \theta \in \Theta}$ be a family of models indexed by $\Theta$. Let $\cY$ be an observation space. For any fixed $\alg$ and distribution $\nu \in \Delta(\Theta)$, let $\Pr^{\nu, \alg} \in \Delta(\cY)$ be defined as the distribution over observations when (1) a parameter is drawn $\theta \sim \nu$, (2) the algorithm interacts with model $M_\theta$. Let $L:\Theta \times \cY \to \bbR_{+}$ be a loss function. Suppose that $\Theta_0 \subseteq \Theta$ and $\Theta_1 \subseteq \Theta$ are subsets that satisfy the separation condition
\begin{align*}
    L(\theta_0, y) + L(\theta_1, y) \ge 2\Delta, \quad \forall~y \in \cY, \theta_0 \in \Theta_0, \theta_1 \in \Theta_1.
\end{align*}
for some parameter $\Delta > 0$. Then it holds that for any $\alg$,
\begin{align*}
    \sup_{\theta \in \Theta} \En_{Y \sim \Pr^{M_\theta, \alg}} \brk*{L(\theta, Y)} \ge \frac{\Delta}{2} \max_{\nu_0 \in \Delta(\Theta_0), \nu_1 \in \Delta(\Theta_1)} \prn*{1 - \dtv\prn*{\Pr^{\nu_0, \alg}, \Pr^{\nu_1, \alg}}}. \numberthis\label{eq:le-cam-result}
\end{align*}
\end{theorem}

\begin{proof}
We will use \cite[Thm.~2 of][]{chen2024assouad} with total-variational (TV) distance $D_f \coloneqq \dtv$. Define the enlarged model class $\overline{\cM} \coloneqq \crl{M_\nu: \nu \in \Delta(\Theta)}$ as well as the loss function extension $\overline{L}: \overline{\cM} \times \cY \to \bbR_{+}$
\begin{align*}
    \overline{L}(M_\nu, y) \coloneqq \inf_{\theta \in \supp(\nu)} L(M_{\theta}, y).
\end{align*}
By the separation condition we have
\begin{align*}
    \overline{L}(M_{\nu_0}, y) + \overline{L}(M_{\nu_1}, y) \ge 2\Delta, \quad \forall~y \in \cY, \nu_0 \in \Delta(\Theta_0), \nu_1 \in \Delta(\Theta_1).
\end{align*}
We pick the prior $\mu \coloneqq \unif(\crl{M_{\nu_0}, M_{\nu_1}})$ and the reference distribution $\bbQ \coloneqq \En_{M\sim \mu} \brk{\Pr^{M, \alg}}$. Observe that
\begin{align*}
    \rho_{\Delta, \bbQ} \coloneqq \Pr_{M \sim \mu, Y \sim \bbQ} \brk*{\overline{L}(M, Y) < \Delta} \le \frac{1}{2}.
\end{align*}
Furthermore
\begin{align*}
    \En_{M \sim \mu} \brk*{\dtv\prn*{\Pr^{M, \alg} , \bbQ} } &= \frac{1}{2} \dtv\prn*{\Pr^{{\nu_0}, \alg} , \bbQ}  + \frac{1}{2} \dtv\prn*{\Pr^{{\nu_1}, \alg} , \bbQ}  \\
    &\le \frac{1}{2} \dtv\prn*{\Pr^{{\nu_0}, \alg} , \Pr^{{\nu_1}, \alg}  }.
\end{align*}
Therefore for any $\delta \in [0, \tfrac{1}{2} - \tfrac{1}{2} \dtv\prn{\Pr^{{\nu_0}, \alg} , \Pr^{{\nu_1}, \alg}  } )$ we have
\begin{align*}
    \En_{M \sim \mu} \brk*{\dtv\prn*{\Pr^{M, \alg} , \bbQ} } \le \begin{cases} 
        \dtv\prn*{ \ber(1-\delta), \ber(\rho_{\Delta, \bbQ}) } &\text{if}~\rho_{\Delta, \bbQ} \le 1-\delta \\
        0 &\text{otherwise.}
    \end{cases}
\end{align*}
Therefore \cite[Thm.~2][]{chen2024assouad} gives
\begin{align*}
    \En_{\theta \sim \frac{\nu_0 + \nu_1}{2}, Y \sim \Pr^{M_\theta, \alg}} \brk*{L(\theta, Y)} \ge \En_{M \sim \mu, Y \sim \Pr^{M, \alg}} \brk*{\overline{L}(M, Y)} \ge \frac{\Delta}{2} \cdot \prn*{ 1 - \dtv\prn*{\Pr^{{\nu_0}, \alg} , \Pr^{{\nu_1}, \alg}  } }.
\end{align*}
Taking supremum over $\nu_0$ and $\nu_1$ gives the result.
\end{proof}

In light of \pref{thm:interactive-lecam-cvx}, we need to analyze the TV distance between an algorithm $\alg$ interactions with two separate environments given by $\nu_0$ and $\nu_1$. The following chain rule lemma will be useful.

\begin{lemma}[Chain Rule for TV Distance, Exercise I.43 of \cite{polyanskiy2025information}]\label{lem:chain-rule-tv}
Let $\cZ$ be any observation space, let $\bbP^{Z_n}$ and $\bbQ^{Z_n}$ be distributions over $n$-tuples of $\cZ$. Then
\begin{align*}
    \dtv\prn*{ \bbP^{Z_n}, \bbQ^{Z_n} } \le \sum_{i=1}^n \En_{Z_{1:i-1} \sim \bbP^{Z_n}} \brk*{ \dtv \prn*{\bbP \brk*{Z_i \mid Z_{1:i-1}}, \bbQ \brk*{Z_i \mid Z_{1:i-1}}} }.
\end{align*}
\end{lemma}

We also state a form of data-processing inequality for KL divergences.

\begin{lemma}[Lemma 1, \cite{garivier2019explore}]\label{lem:KL-to-kl} Consider a measurable space $(\Omega, \cF)$ equipped with two distributions $\bbP_1$ and $\bbP_2$. For any $\cF$-measurable function $Z: \Omega \mapsto [0,1]$ we have
\begin{align*}
    \KL{\bbP_1}{\bbP_2} \ge \kl{\En_1[Z]}{\En_2[Z]}.
\end{align*}
\end{lemma}

\section{Online Learning}\label{app:online-learning}
We state a version of the regret guarantee for Online Mirror Descent (OMD). The interested reader is referred to the references \cite{bubeck2011introduction,orabona2019modern}. 

In what follows, denote $\Omega \subseteq \bbR^d$ to be an open convex set, and $\overline{\Omega}$ to be its closure.

\begin{definition}[Bregman Divergence] Let $\psi: \overline{\Omega} \to \bbR$ be strictly convex and differentiable on $\Omega$, such that $\lim_{x \to \overline{\Omega}\backslash \Omega} \nrm{\nabla \psi(x)} = \infty$. Then the \emph{Bregman divergence} with respect to $\psi$ is denoted $D_h: \overline{\Omega} \times \Omega \to \bbR$ defined as
    \begin{align*}
        D_\psi(y \Vert x) \coloneqq \psi(y) - \psi(x) - \tri{\nabla \psi(x), y-x}.
    \end{align*}
\end{definition}

The OMD update rule is defined as the following. Fix a non-empty closed convex set $V \subseteq \overline{\Omega}$, initial iterate $x_1 \in V $, and fixed learning rate $\eta > 0$. Suppose at iteration $t \in [T]$ we receive loss functions $\ell_t: V \to \bbR$ with subdifferentials $g_t \in \partial \ell_t(x_t)$. Then the update rule is given by
\begin{align*}
    x_{t+1} \gets \argmin_{x \in V} \tri{g_t, x} + \frac{1}{\eta} D_\psi(x \Vert x_t).
\end{align*}
OMD enjoys the following regret guarantee.
\begin{theorem}\label{thm:omd}
    Suppose $\psi$ is a proper, closed, $\lambda$-strongly convex function with respect to the norm $\nrm{\cdot}$ in $V\cap \Omega$, and denote $D_\psi$ to be its Bregman divergence. Then for any $u \in V$, the following regret bound holds:
    \begin{align*}
        \sum_{t=1}^T \ell_t(x_t) - \sum_{t=1}^T \ell_t(u) \le \frac{D_\psi(u \Vert x_1)}{\eta} + \frac{\eta}{2\lambda}\sum_{t=1}^T \nrm{g_t}_\star^2.
    \end{align*}
\end{theorem}

\section{Multiclass Learning}\label{app:multiclass-learning}

\begin{definition}[Natarajan Dimension \citep{natarajan1989learning}]\label{def:natarajan}
Let $\cX$ be an instance space and $\cY$ be a finite label space. Given a class $\cH \subseteq \cY^\cX$, we define its \emph{Natarajan dimension}, denoted $\natarajan(\cH)$, to be the maximum cardinality of a set $C \subseteq \cX$ that satisfies the following: there exists $h_0, h_1: C \to \cY$ such that (1) for all $x \in C$, $h_0(x) \ne h_1(x)$, and (2) for all $B \subseteq C$, there exists $h \in \cH$ such that for all $x \in B$, $h(x) = h_0(x)$ and for all $x \in C \backslash B$, $h(x) = h_1(x)$.
\end{definition}

A notation we will use throughout is the projection operator. For a hypothesis class $\cH \subseteq \cY^\cX$ and a finite set $X = (x_1, \cdots, x_n) \in \cX^n$, we define the projection of $\cH$ onto $X$ as
\begin{align*}
    \cH\big|_X \coloneqq \crl*{\prn*{h(x_1),\cdots, h(x_n)} : h \in \cH }.
\end{align*}

\begin{lemma}[Sauer's Lemma for Natarajan Classes \citep{haussler1995generalization}]\label{lem:sauer-natarajan}
Given a hypothesis class $\cH \subseteq \cY^\cX$ with $\abs{\cY} = K$ and $\natarajan(\cH) \le d$, we have for every $X = (x_1, \cdots, x_n) \in \cX^n$,
\begin{align*}
    \abs[\Big]{\cH\big|_X} \le \prn*{\frac{ne (K+1)^2}{2d}}^d.
\end{align*}
\end{lemma}

\begin{theorem}[Multiclass Fundamental Theorem \citep{shalev2014understanding}]\label{thm:multiclass-fun}
For any class $\cH \subseteq \cY^\cX$ with $\natarajan(\cH) = d$ and $\abs{\cY} = K$, the minimax sample complexity of $(\eps,\delta)$ agnostic PAC learning $\cH$ can be bounded as
\begin{align*}
    \Omega\prn*{\frac{d + \log(1/\delta)}{\eps^2}} \le n(\Pi; \eps, \delta) \le \cO\prn*{ \frac{d \log K + \log(1/\delta)}{\eps^2}}.
\end{align*}
\end{theorem}

\begin{definition}[Pseudodimension] Let $\cX$ be an instance space. Given a hypothesis class $\cH \subseteq \bbR^\cX$, its pseudodimension, denoted $\pseudo(\cH)$, is defined as $\pseudo(\cH) \coloneqq \mathrm{VC}(\cH^\plus)$, where $\cH^\plus \coloneqq \crl*{(x, \theta) \mapsto \ind{h(x) \le \theta} : h \in \cH}$.
\end{definition}

\begin{definition}[Covering Numbers] Given a hypothesis class $\cH \subseteq \bbR^\cX$, $\alpha > 0$, and $X = (x_1, \cdots, x_n) \in \cX^n$, the covering number $\cN_1(\cH, \alpha, X)$ is the minimum cardinality of a set $C \subset \bbR^n$ such that for any $h \in \cH$ there exists a $c \in C$ such that $\tfrac{1}{n} \sum_{i=1}^n \abs{h(x_i) - c_i} \le \alpha$.
\end{definition}

\begin{lemma}[\cite{jiang2017contextual}, see also \cite{pollard2012convergence, luc1996probabilistic}]\label{lem:uniform-convergence}
Let $\cH \subset [0,1]^\cX$ be a real-valued hypothesis class, and let $X = \prn*{x_1, \cdots, x_n}$ be i.i.d.~samples drawn from some distribution $\cD$ on $\cX$. Then for any $\alpha > 0$
\begin{align*}
    \Pr \brk*{\sup_{h \in \cH} \abs*{\frac{1}{n}\sum_{i=1}^n h(x_i) - \En[h(x)]} > \alpha} \le 8 \En \brk*{\cN_1(\cH, \alpha/8, X)} \cdot \exp \prn*{ -\frac{n \alpha^2}{128} }.
\end{align*}
Furthermore if $\mathrm{Pdim}(\cH) \le d$ then we have the bound
\begin{align*}
\Pr \brk*{\sup_{h \in \cH} \abs*{\frac{1}{n}\sum_{i=1}^n h(x_i) - \En[h(x)]} > \alpha} \le 8e(d+1) \prn*{\frac{16e}{\alpha}}^d \cdot \exp \prn*{ -\frac{n \alpha^2}{128} },
\end{align*}
which is at most $\delta$ as long as $n \ge \tfrac{128}{\alpha^2} \prn*{d \log \tfrac{16e}{\alpha} + \log(8e(d+1)) + \log \frac{1}{\delta}}$.
\end{lemma}

\backmatter
\addcontentsline{toc}{chapter}{Bibliography}
\printbibliography
\end{document}